\pdfoutput=1
\documentclass[11pt]{article}
\usepackage[OT1]{fontenc}
\usepackage{natbib}
\usepackage{smile}
\usepackage[margin=1in]{geometry}
\usepackage{fullpage}
\usepackage[normalem]{ulem}
\def\t{{(t)}}
\def\T{{(T)}}
\newcommand{\noiseop}[1]{{\mathrm{U}_{#1}}}

\newcommand{\LQeq}{ \!\!\!\!\overset{L^2(\QQ)}{=}}

\newcommand{\Sym}{\mathtt{Sym}}
\newcommand{\VSTAT}{\mathrm{VSTAT}}
\newcommand{\SDA}{\mathrm{SDA}}

\newcommand{\hypergeom}{\mathrm{Hypergeometric}}
\newcommand{\binomial}{\mathrm{Binomial}}

\newcommand{\SNR}{\mathrm{SNR}}

\newcommand{\err}{\mathrm{err}}

\title{Can Neural Networks Achieve Optimal Computational-statistical Tradeoff? An Analysis on Single-Index Model}
\author{Siyu Chen$^1$\footnote{Equal contribution} \quad Beining Wu$^{2\,*}$ \quad Miao Lu$^3$ \quad Zhuoran Yang$^1$ \quad Tianhao Wang$^4$
\and
{\small\textit{$^1$Department of Statistics and Data Science, Yale University}}
\and
{\small\textit{$^2$Department of Statistics, University of Chicago}}
\and
{\small\textit{$^3$Department of Management Science and Engineering, Stanford University}}
\and
{\small\textit{$^4$Toyota Technological Institute at Chicago}}
\and
{
    \small\texttt{\{siyu.chen.sc3226, zhuoran.yang\}@yale.edu}
    \quad \texttt{beiningw@uchicago.edu}
}
\and
{
  \small\texttt{miaolu@stanford.edu} \qquad\texttt{tianhao.wang@ttic.edu}
}
}
\date{}

\begin{document}

\maketitle

\begin{abstract}
In this work, we tackle the following question: Can neural networks trained with gradient-based methods achieve the optimal computational-statistical tradeoff in learning Gaussian single-index models?
Prior research has shown that any polynomial-time algorithm under the statistical query (SQ) framework requires $\Omega(d^{s^\star/2}\lor d)$ samples, where $s^\star$ is the generative exponent representing the intrinsic difficulty of learning the underlying model.
However, it remains unknown whether neural networks can achieve this sample complexity.
Inspired by prior techniques such as label transformation and landscape smoothing for learning single-index models, we propose a unified gradient-based algorithm for training a two-layer neural network in polynomial time.
Our method is adaptable to a variety of loss and activation functions, covering a broad class of existing approaches.
We show that our algorithm learns a feature representation that strongly aligns with the unknown signal $\theta^\star$, with sample complexity $\tilde O (d^{s^\star/2} \lor d)$, matching the SQ lower bound up to a polylogarithmic factor for all generative exponents $s^\star\geq 1$.
Furthermore, we extend our approach to the setting where $\theta^\star$ is $k$-sparse for $k = o(\sqrt{d})$ by introducing a novel weight perturbation technique that leverages the sparsity structure.
We derive a corresponding SQ lower bound
of order $\tilde\Omega(k^{s^\star})$, matched by our method up to a polylogarithmic factor.
Our framework, especially the weight perturbation technique, is of independent interest, and suggests potential gradient-based solutions to other problems such as sparse tensor PCA.
\end{abstract}

\section{Introduction}\label{sec:intro}
The success of neural networks is largely attributed to their remarkable ability to learn rich and useful features from data during gradient-based training \citep{girshick2014rich}. This feature-learning capability allows them to outperform traditional methods like kernel-based approaches, which rely on predefined features \citep{allen2019can,ghorbani2019limitations,refinetti2021classifying}.
However, when trained using (stochastic) gradient descent, neural networks can sometimes fall into a ``kernel regime'', where their behavior resembles that of a fixed kernel method, constrained by their random initialization \citep{jacot2018neural,chizat2019lazy}. In this regime, the ability of the network to learn complex representations is severely limited, undermining the primary advantage of deep learning.
Therefore, it is crucial to understand when and how neural networks trained with gradient-based methods can perform effective feature learning to unlock their full potential, particularly in scenarios where a balance between computational efficiency and statistical performance is essential.

In this work, we approach this question in the context of Gaussian single-index models, a canonical class of problems in statistics and learning~\citep{maccullagh1989generalized,ichimura1993semiparametric,hristache2001direct,hardle2004nonparametric}.
The model is defined as follows: for covariates $z \sim \cN(0, I_d)$, the output $y$ depends on the inner product $\langle \theta^\star, z \rangle$ with an unknown signal $\theta^\star \in \RR^d$ through a link distribution $p$, i.e., $y \sim p(\cdot \mid \langle \theta^\star, z \rangle)$. The goal is to recover $\theta^\star$ using i.i.d. samples $(z_1, y_1), \ldots, (z_n, y_n)$ generated by the underlying model.
While $n = \Omega(d)$ samples suffice to recover $\theta^\star$ information-theoretically \citep{bach2017breaking,damian2024computational}, achieving this efficiently is difficult for polynomial-time algorithms, where the required sample size also depends on properties of the link distribution $p$, creating a computational-statistical gap.
For example, when $y$ is a polynomial of $\langle \theta^\star, z\rangle$, it has been shown that two-layer neural networks with square loss need $d^{\Theta(q^\star)}$ samples \citep{arous2021online,bietti2022learning,damian2024smoothing}, where $q^\star$ is the information exponent of the polynomial link function \citep{arous2021online,dudeja2018learning}.
Such sample complexity is indeed inevitable under the correlational statistical query (CSQ) framework, leading to a computational-statistical gap for $q^\star\geq 2$.

However, the CSQ framework does not capture the fundamental limits of all gradient-based algorithms. 
Recent works have shown that by leveraging higher-order terms in the gradient, neural networks can learn polynomials with as few as $\tilde O(d)$ samples \citep{lee2024neural,arnaboldi2024repetita}.
It turns out that the intrinsic learning difficulty is captured by another quantity called the \emph{generative exponent} $s^\star$, which is at most 2 for polynomial link functions, and the corresponding SQ lower bound on the sample complexity is $n=\Omega(d^{s^\star/2})$\footnote{
    This $\Omega(d^{s^\star/2})$ sample complexity lower bound is essentially for the detection problem. \citet{dudeja2021statistical} show that there is an estimation-detection gap for tensor PCA under the SQ framework, though it is unclear whether such gap exists universally.
    Throughout the paper, we always refer to the SQ lower bound as the detection lower bound, since detection in general is assumed to be easier than estimation.
} \citep{damian2024computational}.
Thus, there is no computational-statistical gap up to polylogarithmic factors for learning polynomial single-index models.
However, for general single-index models with $s^\star \geq 3$, no gradient-based algorithm for neural networks has been shown to match the SQ lower bound, leaving it an open problem~\citep{arnaboldi2024repetita, lee2024neural}.
Notably, \citet{lee2024neural} established an upper bound of $n=\tilde O(d^{(s^\star-1)\vee 1})$ for the sample complexity of gradient-based training of a two-layer neural network, which still leaves a gap of $d^{s^\star/2-1}$ for general $s^\star \geq 3$.

Furthermore, learning the Gaussian single-index model can benefit from additional structures in the signal $\theta^\star$, such as sparsity, which can significantly reduce the sample complexity compared to those depending on the ambient dimension $d$ \citep{candes2006robust, donoho2009message,raskutti2012minimax}.
Recent work by \citet{vural2024pruning} examines pruning for learning sparse features, establishing sparsity-aware correlational statistical query (CSQ) lower bounds and showing that pruned neural networks trained by gradient descent can match these bounds under their sparsity-threshold condition.
For sparse single-index models with information exponent $q^\star = 1$, gradient descent on diagonal linear networks nearly achieves the information-theoretic lower bound thanks to its implicit regularization effect \citep{fan2023understanding}.
Nonetheless, how to achieve the optimal sample complexity for general $s^\star\geq 1$ is also unknown under the sparse setting.

\begin{figure}[t]
    \centering
    \begin{tabular}{cc}
        \includegraphics[width=0.47\textwidth]{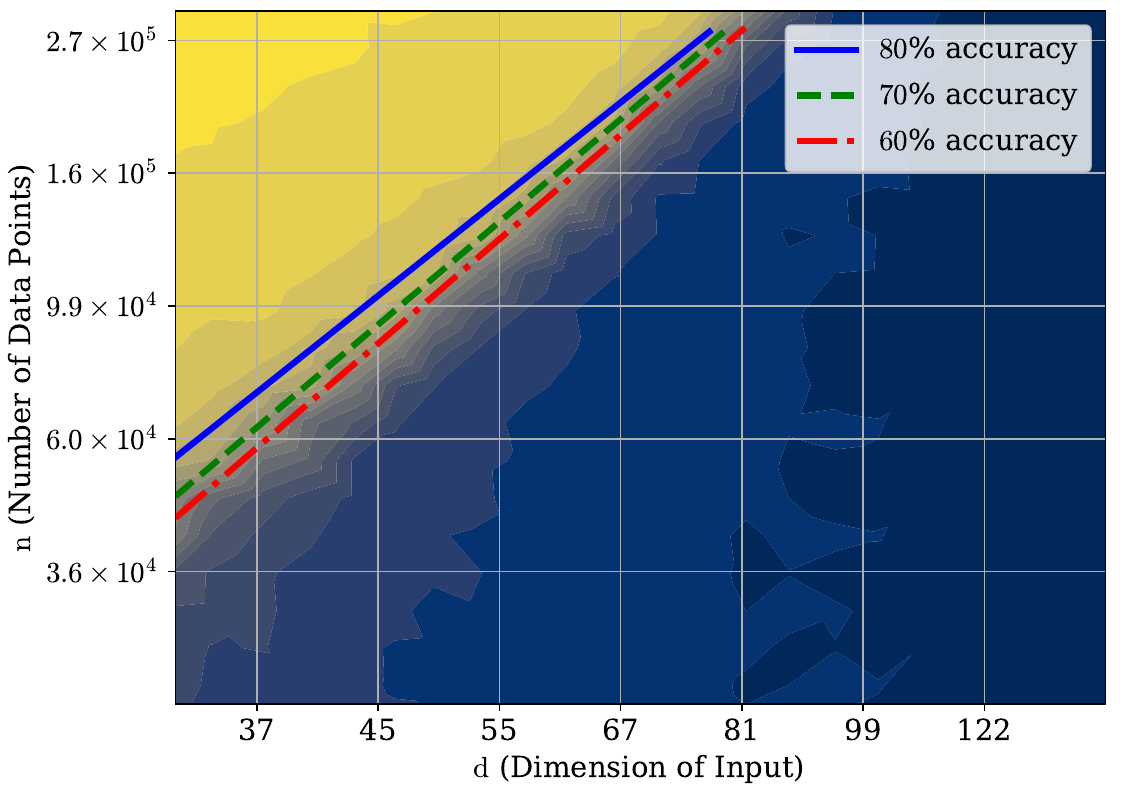} &
        \includegraphics[width=0.47\textwidth]{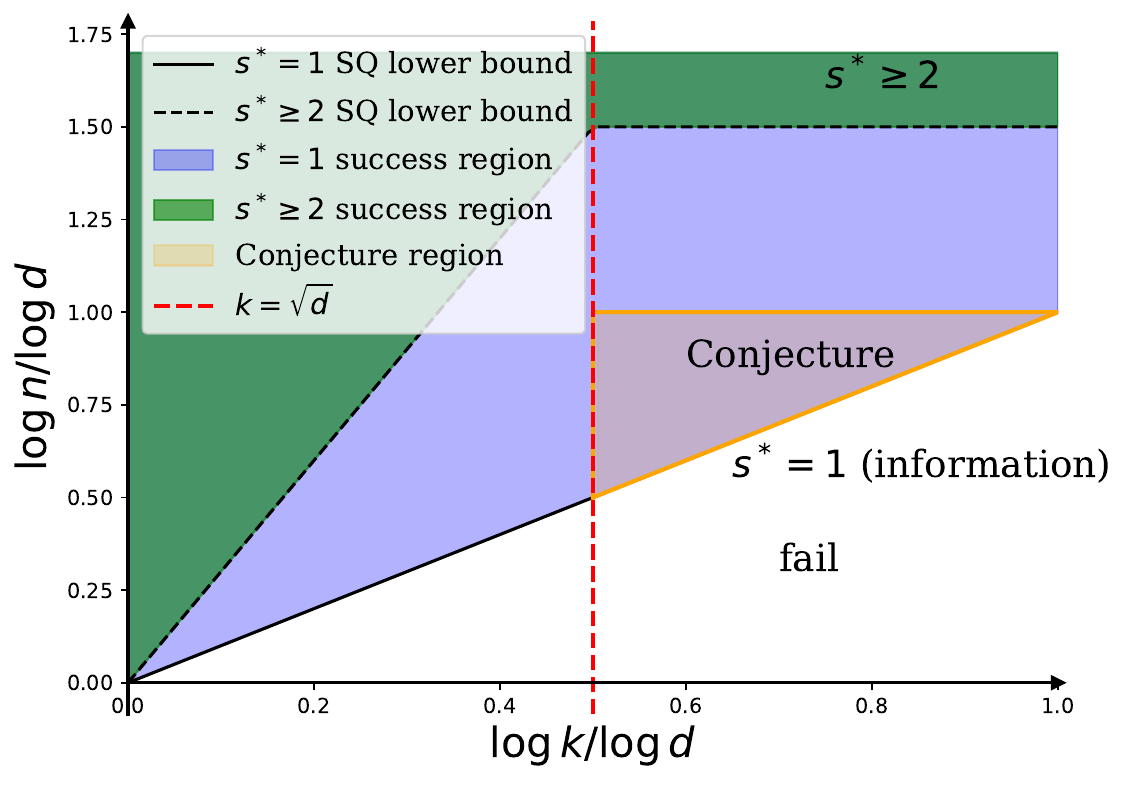}\\
        (a) Scatter plot of $(d, n)$ & (b) $(n, k)$ paradigm\\
    \end{tabular}
    \caption{(a) The contour plots of $(\log d, \log n, \texttt{acc}(d, n))$ for \Cref{alg:meta} under model $y = \langle z,\theta^\star\rangle^2 \exp(- \langle z,\theta^\star\rangle^2)$, which has generative exponent $s^\star=4$ (\Cref{exp:sgm}). 
    Here $\texttt{acc}(d, n)$ is the average of the largest 8 values of the alignment between the neuron weights and the unknown signal $\theta^\star$. 
    The slopes of these contour lines are all close to $2$, indicating a sample complexity $n \approx d^{2}$ for $s^\star = 4$. 
    (b)~The paradigm of sample complexity achieved by our algorithm for different generative exponents $s^\star$ and sparsity levels $k$, illustrating the success of achieving a nearly optimal computational-statistical tradeoff.
    }
    \label{fig:intro}
\end{figure}

\subsection{Contributions}
Towards characterizing the fundamental feature learning capability of neural networks in the Gaussian single-index model, our main result is summarized in the following theorem.
\begin{theorem}[Informal]
For the Gaussian single-index model with an unknown true signal $\theta^\star\in\RR^d$ and an unknown link distribution $p$ with any generative exponent $s^\star\geq 1$, there is a gradient-based algorithm that can train the weights of a two-layer neural network to achieve strong recovery of the true signal with $\tilde O(d^{s^\star/2} \lor d)$ samples and polynomial running time.
Moreover, if the true signal $\theta^\star$ is $k$-sparse for $k=o(\sqrt d)$, only $\tilde O(k^{s^\star})$ samples are required.
In both cases, the sample complexity matches the SQ lower bound up to a polylogarithmic factor.
\end{theorem}

Our contributions are summarized as follows:
\begin{enumerate}
\item We propose a unified recipe of gradient-based algorithms for training a two-layer neural network to learn the Gaussian single-index model.
Our method integrates a general gradient oracle with a weight perturbation technique, carefully designed to exploit the underlying structure of the Gaussian single-index model.
This allows the neural network to perform feature learning of the unknown signal $\theta^\star$ in a computationally efficient manner.
Our framework encompasses many existing approaches as special cases, such as batch reusing \citep{dandi2024benefits,lee2024neural}, label transformation \citep{chen2020learning}, and landscape smoothing \citep{damian2024smoothing}.

\item We show that for an \emph{unknown link distribution} $p$ with \emph{any} generative exponent $s^\star\geq 1$, the weights of the neural network achieve strong recovery of the true signal $\theta^\star$ after training by our algorithm using $\tilde O(d^{s^\star/2} \lor d)$ samples and polynomial running time. 
Our method achieves the SQ lower bound up to a polylogarithmic factor, and is the first gradient-based algorithm for training two-layer neural networks that attains the nearly optimal computational-statistical tradeoff for Gaussian single-index models with any $s^\star\geq 1$.
\Cref{fig:intro} (a) illustrates an example for $s^\star=4$.

\item Furthermore, our method is able to take advantage of additional structural information of the true signal $\theta^\star$.
Specifically, we consider the case where $\theta^\star$ is $k$-sparse for $k= o(\sqrt d)$, and develop a \emph{novel weight perturbation procedure} tailored to the sparsity of $\theta^\star$.
Equipped with this, we show that the weights of the neural network can achieve strong recovery of the sparse signal $\theta^\star$ after training with $\tilde O(k^{s^\star})$ samples in polynomial time for any generative exponent $s^\star\geq 1$.
This sample complexity is also nearly optimal according to the sample complexity lower bound we establish for SQ algorithms, which might be of independent interest.
Also, our method suggests a new approach to achieving the optimal computational-statistical tradeoff for sparse tensor PCA.
\end{enumerate}

In summary, our work provides a unified framework for training neural networks that can achieve the nearly optimal computational-statistical tradeoff for the Gaussian single-index model with any generative exponent $s^\star\geq 1$.
Our method not only tackles the intrinsic difficulty of learning the true signal $\theta^\star$ posed by the link distribution $p$, but also leverages the additional structural information of the true signal $\theta^\star$ that benefits the learning process.
Integrating these results, our method attains a nearly optimal balance between computational efficiency and statistical performance across almost all regimes of sparsity levels and generative exponents $s^\star\geq 1$, as illustrated in \Cref{fig:intro}~(b).

\subsection{Related Works.}
Our work contributes to the recent research on the computational-statistical tradeoff in learning single-index models. The information-theoretic limit for estimating the latent signal is \(n = \Omega(d)\) \citep{bach2017breaking, damian2024computational}, but the sample complexity lower bound varies across computational models, potentially revealing a computational-statistical gap.

The information exponent \(q^\star\) \citep{dudeja2018learning, arous2021online} governs the sample complexity for learning Gaussian single-index models in the CSQ framework \citep{chen2020towards, bietti2022learning, damian2022neural, dandi2023two, abbe2023sgd, ba2024learning}. Notably, \citet{arous2021online} show that online SGD has a sample complexity of \(n = \tilde{O}(d^{q^\star-1})\), which is worse than the CSQ lower bound \(n = \Omega(d^{q^\star/2})\) \citep{abbe2023sgd,damian2022neural}. 
This gap can be closed by a loss landscape smoothing technique \citep{damian2024smoothing}  originally developed for tensor PCA
\citep{anandkumar2017homotopy,biroli2020iron}.

Our work extends beyond the CSQ framework, aligning with more general SQ algorithms \citep{feldman2017statistical, feldman2017general}, where the sample complexity lower bound is \(\Omega(d^{s^\star/2})\), with \(s^\star\) as the generative exponent \citep{damian2024computational}. 
In this context, online SGD with batch reusing suffices for learning polynomial link functions \citep{dandi2024benefits, lee2024neural}, while for \(s^\star \geq 3\), only the partial trace estimator proposed by \citet{damian2024computational} can match the SQ lower bound.

In the sparse setting, including sparse linear models \citep{vaskevicius2019implicit, zhao2022high, gamarnik2017sparse}, sparse PCA \citep{arous2020free}, and planted models \citep{bandeira2022franz}, computational-statistical gaps also exist.
Related to our work, \citet{fan2023understanding} provide a $\tilde O(k)$ sample complexity for learning single index models with $q^\star=1$ using diagonal linear networks, and \citet{neykov2016agnostic} report a $\tilde O(k^2)$ result for phase retrieval where $q^\star = 2$. 
However, as previously noted, the information exponent does not fully characterize the intrinsic computational-statistical tradeoff. 
Our work completes the picture by providing a gradient-based framework that simultaneously handles all sparsity levels and any generative exponent $s^\star\geq 1$.

\section{Problem Setup}\label{sec:setup}
We begin by introducing the notation used in the paper, and then describe the problem setup.
For a probability distribution $\PP$, we denote by $L^2(\PP)$ the space of square-integrable functions with respect to $\PP$, and $\overset{L^2(\PP)}{=}$ means equality in $L^2(\PP)$.
We denote the normalized probabilist's Hermite polynomials by \( \{h_s(\cdot)\}_{s\ge 0} \), where each $ h_s(x) \defeq \frac{(-1)^s}{\sqrt{s!}} \cdot  e^{x^2/2} \cdot \frac{\rd^s}{\rd x^s} e^{-x^2/2}$.
These polynomials form an orthonormal basis for 
\( L^2(\cN(0, 1)) \), i.e., the space of square-integrable functions under the Gaussian measure.

\paragraph{Gaussian single-index model}
We study the following Gaussian single-index model:
The environment first samples an unobservable signal ${\theta^\star} \sim\pi$ from some known prior $\pi\in\sP(\SSS^{d-1})$.
Then i.i.d. data $(z_1, y_1),\ldots,(z_n,y_n)\in\RR^d\times\RR$ are generated according to the following distribution $\PP_{\theta^\star}$ given $\theta^\star$:
\begin{align}
    \PP_{\theta^\star}: \quad z \sim \cN(0, I_d), \quad y \sim p(\cdot\given \langle {\theta^\star}, z \rangle ). 
    \label{eq:single_index_model}
\end{align}
Here $p(\cdot\given \cdot):\RR\mapsto\sP(\RR)$ is referred to as the \emph{link distribution}.
A canonical example is the additive model where $y=\phi(\langle\theta^\star,z\rangle)+\epsilon$ for some deterministic link function $\phi:\RR\to\RR$ and random noise $\epsilon$.
See \citet{damian2024computational} for more complicated examples.

\paragraph{Generative exponent}
The following discussion on the generative exponent is based on the work of \citet{damian2024computational}.
We aim to learn \eqref{eq:single_index_model} where the link distribution $p$ has \emph{generative exponent} $s^\star\geq 1$, a measure of the computational-statistical gap for learning single-index models.
We let $x = \langle {\theta^\star}, z\rangle$. 
Notice that $\PP_{\theta^\star}(y,z)= \PP(y,x) \cdot \cN(z^\perp;0,I_{d-1})$ where we use $\PP$ to denote the joint distribution of $(x,y)$ as this joint distribution is independent of ${\theta^\star}$.
As the marginal distribution of $y$ is also independent of ${\theta^\star}$, we define the \emph{null distribution} $\QQ(y,z) \defeq \cN(z;0,I_d)\otimes\QQ(y)$ and denote $\QQ(y,x)\defeq \cN(x;0,1)\otimes\QQ(y)$ where $\QQ(y) = \int_\RR \PP(y,x) \rd x$.
It can be shown that under a square-integrable condition under $\mathbb{Q}$, the likelihood ratio admits a Hermite expansion with coefficient functions $\{\zeta_s(y)\}_{s\geq 1}$, i.e.,
\begin{align}
    \frac{\PP_{\theta^\star}(y,z)}{\QQ(y,z)} = \frac{\PP(y,x)}{\QQ(y,x)} \overset{L^2(\QQ)}{=} \sum_{s=0}^\infty \zeta_s(y) \cdot h_s(x),\quad 
    \text{where }\zeta_s(y) = \mathbb{E}_{\mathbb{P}}[h_s(x)|y], 
    \label{eq:likelihood-ratio-decomp}
\end{align}
and $\EE_\QQ[\zeta_s(y)^2] \le 1$ for all $s\geq1$.
Note that \eqref{eq:likelihood-ratio-decomp} makes sense only when we are working with the inner product of $\PP/\QQ$ and a square-integrable function under the null distribution $\QQ$.
\begin{definition}[Generative exponent]
    \label{def:generative-exponent}
For the Gaussian single-index model defined in \eqref{eq:single_index_model}, the generative exponent $s^\star$ of the link distribution $p$ is defined as
    $s^\star(p) \defeq \min\{s \ge 1: \EE_{\QQ}[\zeta_s(y)^2] > 0 \}$.
\end{definition}

\begin{example}[Example 2.7, \citet{damian2024computational}]\label{exp:sgm}
Consider the special case of the Gaussian single-index model \eqref{eq:single_index_model} where $y=\phi(\langle\theta^\star,z\rangle)$ for a deterministic link function $\phi:\RR\to\RR$.
When $\phi$ is a polynomial function, it holds that $s^\star(\phi)\leq 2$, and the equality holds if and only if $\phi$ is an even polynomial.
In particular, $s^{\star}(h_s)=1$ for odd $s$ and $s^{\star}(h_s)=2$ for even $s$.
While for the example of $\phi(x) = x^2\exp(-x^2)$, which is not a polynomial, it has generative exponent $s^{\star}(\phi)=4$.
\end{example}

\paragraph{Two-layer neural networks}
We consider using a two-layer neural network with $M$ hidden neurons to learn the single-index model \eqref{eq:single_index_model}.
The weight vector for each neuron $m\in[M]$ is $\theta_m\in\RR^d$, and the weights of the second layer are $a_1,\ldots,a_M\in\RR$.
We collect all the weights and denote $\btheta = (\theta_1, \ldots, \theta_M)\in\RR^{d\times M}$, $\bm{a} = (a_1,\ldots,a_M)^\top\in\RR^{M}$.
Now for any input $z\in\RR^d$, the output of the network is given by
\begin{align}
    f(z; \btheta,\bm{a}) \defeq \sum_{m=1}^M a_m \cdot \sigma\big(\langle z, \theta_m\rangle\big),
\end{align}
where $\sigma:\RR\to\RR$ is the activation.

\section{Overview of techniques}\label{sec:alignment}
In this work, we apply gradient-based methods to learn Gaussian single-index models, with a focus on feature learning in neural networks and the corresponding computational-statistical tradeoff.
To motivate the techniques involved, we begin by discussing an illustrative example that highlights such tradeoffs.
For this overview, we focus on $s^\star>2$ and the uniform prior $\pi = \unif(\SSS^{d-1})$.
It has been shown that a gap exists between the information-theoretic lower bound $\Omega(d)$ and the SQ lower bound $\Omega(d^{s^\star/2})$ under this setting when $s^\star>2$ \citep{bach2017breaking,damian2024computational}.

For illustration, let us consider training a two-layer network with a single neuron under the population square loss.
When the weight of the second layer is small, the reweighted negative gradient $g$ satisfies
\begin{align}
    g = -(2a)^{-1} \nabla_{\theta} \big(f(z;\theta, a) - y\big)^2
    = - \big(a \cdot \sigma(\langle z, \theta\rangle) - y\big) \cdot \sigma'(\langle z, \theta\rangle) \cdot z \approx y \sigma'(\langle z, \theta\rangle) \cdot z.\notag
\end{align}
Taking expectation over $(z,y)\sim\PP_{\theta^\star}$ and using the likelihood ratio decomposition in \eqref{eq:likelihood-ratio-decomp}, we have
\begin{align}
    \mathbb{E}_{\mathbb{P}_{\theta^{\star}}}[g]
    &\approx \underbrace{\EE_\QQ\left[y\right] \cdot \EE_\QQ\left[\sigma'(\langle z, \theta\rangle) \cdot z\right]}_{\dr bias} + \sum_{s\ge s^\star} \underbrace{\EE_\QQ\left[y \zeta_s(y)\right] \cdot \EE_\QQ\left[h_s(\langle \theta^\star, z\rangle) \cdot \sigma'(\langle z, \theta\rangle) \cdot z\right]}_{\dr informative~queries},
    \label{eq:informative queries}
\end{align}
where we use the fact that $y$ and $z$ are independent under the null distribution $\QQ$.
Note that the \emph{bias} term does not contain any information about $\theta^\star$, and it can be easily removed by a debiasing procedure, so we assume for simplicity that $\EE[y]=0$.

\paragraph{Failure of vanilla online minibatch SGD}
We first consider the vanilla online minibatch SGD, which updates the weight vector $\theta$ by $\theta \leftarrow \theta - \eta \sum_{i=1}^n g_i$ for a minibatch of size $n$.
The sample complexity of gradient-based methods is determined by the signal-to-noise ratio (SNR) of the one-sample gradient, which in our case is defined as $\SNR \defeq \EE[\langle g, \theta^\star\rangle]^2/\EE[\norm{g}_2^2]$.
This is the square of the alignment between $g$ and $\theta^\star$, governed primarily by the informative query corresponding to the lowest degree $s^\star$ in \eqref{eq:informative queries} assuming that $\EE_\QQ[y\zeta_{s^\star}(y)] \neq 0$.
It can be shown that the inner product between the lowest-degree informative query in \eqref{eq:informative queries} and the signal $\theta^\star$ satisfies (see \Cref{lem:g-1st-moment})
\begin{align}
    \EE_\QQ\left[h_{s^\star}(\langle \theta^\star, z\rangle) \cdot \sigma'(\langle z, \theta\rangle) \cdot \langle z, \theta^\star\rangle\right] \approx s^{\star} \cdot \hat\sigma_{s^\star} \cdot \langle \theta^\star, \theta\rangle^{s^\star -1} = \hat\sigma_{s^\star} \cdot O(d^{-(s^\star-1)/2}),
    \label{eq:1st-moment-approx}
\end{align}
where $\hat\sigma_{s^\star}$ is the $s^\star$-th coefficient in the Hermite expansion of $\sigma$.
While for $\|g\|_2$, we have
\begin{align}
    \EE_{\PP_{\theta^\star}}\left[\norm{g}_2^2\right]
    &\approx d  \cdot \EE_\QQ\left[ y^2\sigma'(\langle z, \theta\rangle)^2 \right] = \Omega(d),
\end{align}
where the high-order terms in the likelihood ratio decomposition are ignored and we come back to this point later.
Now we can argue why vanilla online minibatch SGD has difficulty achieving the SQ lower bound for generative exponent $s^\star>2$:
Suppose $\EE_\QQ[y \zeta_{s^\star}(y)]$ and $\hat\sigma_{s^\star}$ are both nonzero constants.
Then the one-sample SNR is $O(d^{-s^\star})$.
For a minibatch with $n$ samples, the SNR of the gradient averaged over the minibatch is roughly $n$ times the one-sample SNR\footnote{This argument is not fully rigorous because $\EE_{\PP_{\theta^\star}}[\norm{g}_2^2]$ also includes ``bias'' $\norm{\EE_{\PP_{\theta^\star}}[g]}_2^2$ besides the fluctuations, but it remains valid as long as $\norm{g}_2^2$ is dominated by fluctuations from all $d$ directions at initialization.}, i.e., $nd^{-s^\star}$.
To ensure one update step achieves alignment, i.e., the square root of the $n$-sample SNR, $\sqrt{n d^{-s^\star}}$, exceeding the trivial $d^{-1/2}$ threshold attained by a random vector, it requires at least $d^{s^\star -1}$ samples.
Note that the sample complexity would become even worse if $s^\star<\argmin_{s\geq s^\star} \{s: \EE_\QQ[y \zeta_s(y)]\neq 0\}$.
This contrasts with the sample complexity $O(d^{s^\star/2})$ suggested by the SQ lower bound.

The above failure of vanilla online minibatch SGD exposes three key challenges:
\begin{enumerate}[leftmargin=0.3in,nolistsep]
\item[(\romannumeral1)] (\textbf{Non-polynomial}) How to handle the infinite sum of high-order terms in the likelihood ratio?
\item[(\romannumeral2)] (\textbf{Low SNR}) How to enhance the SNR to achieve the SQ lower bound?
\item[(\romannumeral3)] (\textbf{Zero correlation}) How to ensure that the algorithm still works if $\EE_\QQ[y \zeta_{s^\star}(y)]=0$?
\end{enumerate}
Below we discuss our techniques for addressing these challenges.

\paragraph{Label transformation via general gradient oracle}
The idea to fix the zero correlation problem is to apply a nonlinear transformation $\cT: \RR\rightarrow\RR$ to $y$ such that $\cT(y)$ has nonzero correlation with $\zeta_{s^\star}(y)$.
This label transformation technique has been widely used in the literature \citep{lu2020phase,mondelli2018fundamental,dudeja2018learning,chen2020learning, damian2024computational}.
In particular, \citet{lee2024neural} show that the label transformation can be automatically realized by running two gradient steps on the same batch, a technique termed as \emph{batch-reusing} \citep{dandi2024benefits,arnaboldi2024repetita}.
In this work, we study a more \emph{general class of gradient-based methods} with gradient of form $g = \psi(y, \langle \theta,z\rangle) z$, which is an abstract form of the transformed gradient $\cT(y) \sigma'(\langle z, \theta\rangle) z$.
The desired condition becomes $\EE_\QQ[\hat \psi_{s^\star-1}(y) \zeta_{s^\star}(y)]\neq 0$, where $\hat \psi_{s}(y)$ is the $s$-th Hermite coefficient function of $\psi(y, x)$ in the Hermite basis of $x$.
One particular way to obtain such a gradient is to use a modified loss function, similar to the approach in \citet{joshi2024complexity}, while in our case the specific choice of $\psi$ is also related to the other two challenges addressed as follows.

\paragraph{Exploration by weight perturbation with high-pass activation}
The low-SNR challenge corresponds to the fact that points on the equator of $\SS^{d-1}$ orthogonal to $\theta^\star$ are all saddle points in terms of $|\langle\theta,\theta^\star\rangle|$, and random initialization typically lies near this equator.
To efficiently escape from such saddle points, we perform random weight perturbation, akin to the approach in \citet{jin2017escape} for non-convex optimization.
Specifically, suppose the activation $\sigma$ is high-pass and has the lowest degree $s^\star$, i.e., $\sigma(x) = \sum_{s \ge s^\star} \hat{\sigma}_{s} h_s(x)$, and consider for simplicity the case of odd $s^\star$.
In the extreme case where $\theta$ is perturbed into i.i.d. pure noise $\theta_1,\ldots,\theta_L \sim \unif(\SSS^{d-1})$, we compute the gradient for each $\theta_l$ and aggregate them into
$g=L^{-1}(g_1+\cdots+g_L)$.
Using the properties of the Gaussian noise operator (see \Cref{app:preliminary} for details), the second moment of this aggregated gradient satisfies
\begin{align}
    \EE\left[\norm{g}_2^2\right] \approx \frac{d}{L^2}\sum_{l, l'=1}^L\EE_\QQ[y^2] \cdot \EE_\QQ[\sigma'(\langle z, \theta_l\rangle) \sigma'(\langle z, \theta_{l'}\rangle)] \approx d \sum_{s\ge s^\star} s \cdot \hat\sigma_{s}^2 \cdot \EE_{\theta, \theta'}[\langle \theta, \theta'\rangle^{s - 1}],
    \label{eq:2nd-moment-approx}
\end{align}
where $\theta, \theta'$ are drawn independently from $\unif(\SSS^{d-1})$.
Since $\langle \theta, \theta' \rangle \approx d^{-1/2}$, we have $\EE[\|g\|_2^2]\approx O(d^{-(s^\star-3)/2})$, yielding a higher one-sample SNR as the first moment remains unchanged and pushing the sample complexity towards the SQ lower bound.
Moreover, we also see from the above calculation that the weight perturbation resolves the non-polynomial issue thanks to the near-orthogonality of the perturbed weights.
The above heuristics can be made rigorous for polynomially large $L$, thereby handling non-polynomial link and activation functions.

Our approach also draws inspiration from the landscape smoothing method in \citet{damian2024computational}, but in contrast to their problem setup, we do not require full knowledge of the link distribution in advance.
Instead, it suffices to know the generative exponent $s^\star$ to construct a high-pass activation function as well as the gradient oracle $\psi$.
See \Cref{exp:psi_3} for a detailed discussion on this.

\section{Gradient-based Algorithm for Uniform Prior}\label{sec:nonsparse}

We first present our method and results for the case of $\theta^\star\sim\unif(\SSS^{d-1})$, or equivalently, when there is no structural information on $\theta^\star$.
Motivated by the discussion in \Cref{sec:alignment}, we propose a gradient-based algorithm (\Cref{alg:meta}) that can train a two-layer neural network to learn the unknown signal $\theta^{\star}$ with $\widetilde{O}(d^{s^{\star}/2}\vee d)$ sample complexity, nearly matching the corresponding SQ lower bound.

\begin{algorithm}[t]
\caption{Gradient-based Feature Learning for Uniform Signal Prior}
\label{alg:meta}
\begin{algorithmic}[1]
\STATE \textbf{Input}: Initialization
    $\btheta^{(0)}=(\theta_1^{(0)}, \ldots, \theta_M^{(0)})\in\RR^{d\times M}$,
    where $\theta_m^{(0)}\iidfrom \unif(\SSS^{d-1})$,
    $\bm{a} = a \cdot \vone\in\RR^{M}$,
    number of iterations $T\in \NN$,
    learning rate $\eta>0$,
    batch size $n\in\mathbb{N}$, polarization level $\gamma\in(0,1)$, number of perturbations $L\in\mathbb{N}$.
\FOR{iteration $t=0, 1, \ldots, T-1$}
\STATE Sample a fresh mini-batch of data $\{(z_i^\t, y_i^\t)\}_{i=1}^{n}$.
\STATE Perturb weights $w_{m,l}^{(t)}=(\gamma\theta_m^{(t)}+\xi_{m,l}^{(t)})/\|\gamma\theta_m^{(t)}+\xi_{m,l}^{(t)}\|_2$, $\xi_{m,l}^{(t)}\iidfrom\unif(\SSS^{d-1})$ for all $m,l$. \label{algline:uniform_perturb}
\STATE Compute the gradients $g_{m,l,i}^{(t)}=(\psi(y_i^{(t)}, \langle w_{m,l}^{(t)},z_i^{(t)}\rangle) + \err_{m,l,i}^{(t)})\cdot z_i^{(t)}$ for all $m,l,i$. \label{algline:uniform_gradient}
\STATE Aggregate the gradients: $g_m^{(t)} = (nL)^{-1}\sum_{i=1}^n\sum_{l=1}^L (g_{m,l,i}^{(t)} - \hat\psi_1(y_i^\t) w_{m,l}^\t)$ for all $m$.\label{algline:uniform_aggregate}
\STATE Normalize the update step: $\bar{g}_m^{(t)}=g_m^{(t)}/\|g_m^{(t)}\|_2$ for all $m$. \label{algline:uniform_normalize}
\STATE Update the weights $\theta_m^{(t+1)} = (\theta_m^{(t)} + \eta \bar{g}_m^{(t)}) / \norm{\theta_m^{(t)} + \eta \bar{g}_m^{(t)}}_2$ for all $m$.\label{algline:uniform_update}
\ENDFOR
\STATE \textbf{Output}: Final model weights $\btheta^{(T)}$.
\end{algorithmic}
\end{algorithm}

\subsection{Gradient-based Training Algorithm (Algorithm~\ref{alg:meta})}\label{subsec:uniform_prior_alg}

We initialize each neuron $m$ with $\theta_m^{(0)}\sim \unif(\SSS^{d-1})$, and we set $a_m^{(t)} \equiv a$ for some sufficiently small $a>0$ throughout the training.
In each iteration $t\in[T]$, we sample a new data batch of size $n$.

\paragraph{Weight perturbation}
Before calculating the gradients, we first perturb the weights of each neuron to get $L$ noisy replicas, by injecting uniform noise from the sphere $\SS^{d-1}$ as in Line~\ref{algline:uniform_perturb}.
There is a simple rule for choosing the polarization level $\gamma$. In the previous section, we discussed how $\EE_{\PP_{\theta^\star}}[\norm{g}_2^2]$ in the one-sample SNR depends on the following quantity:
\begin{align}
    \EE_{\xi_{m, l}^\t, \xi_{m, l'}^\t}\langle w_{m,l}^\t, w_{m, l'}^\t\rangle^{s^\star-1} \lesssim \bigl(\gamma^2 \norm{\theta_m^\t}_2^2\bigr)^{s^\star -1} + \EE_{\xi_{m, l}^\t, \xi_{m, l'}^\t}\langle \xi_{m, l}^\t, \xi_{m, l'}^\t \rangle^{s^\star -1} \approx (\gamma^2 \lor d^{-1/2})^{s^\star - 1}.
\end{align}
In this context, $\gamma^2$ represents the bias from the \emph{exploitation} of the learned search direction, and $d^{-1/2}$ accounts for the variance from the \emph{exploration} for the unknown signal.
In fact, $\gamma$ should be set as large as possible to maximize exploitation while still ensuring that the exploration noise dominates.
This balance is necessary to fully gain the SNR enhancement from weight perturbation.
This gives rise to the choice $\gamma = \tilde\Theta(d^{-1/4})$.
Moreover, it suffices to set $L= \tilde\Omega(n/\sqrt d)$ as stated in \Cref{thm:uniform}.

\paragraph{Gradient aggregation and debiasing}
Then for each neuron $m$ and its perturbed weights $w_{m,l}^{(t)}$, we calculate the gradient $g_{m,l,i}^{(t)}$ for every sample (Line~\ref{algline:uniform_gradient}).
The gradient is expressed by decoupling the primary search direction $\psi(y_i, \langle w_{m,l},z_i\rangle)\cdot z_i$
from the error term
$\err_{m, l,i}\cdot z_i$ (omitting the time index $t$).
For two-layer neural networks, we discussed in the previous section an example where $\psi(y_i, \langle w_{m, l}, z_i\rangle) = \cT(y_i) \sigma'(\langle w_{m, l}, z_i\rangle)$ for some transformation $\cT$.
The error term, arising from neuron interactions during backpropagation, can be reduced by keeping the weights of the second layer sufficiently small.
We provide the assumptions and more examples of $\psi$ in \Cref{subsec:uniform_prior_theory}.

Next, we aggregate the gradients for each neuron $m$ by averaging over the $n$ samples and $L$ perturbations to get $g_m^{(t)}$ as shown in Line~\ref{algline:uniform_aggregate}.
Here we additionally subtract a term $\widehat{\psi}_1(y_i^{(t)})w_{m,l}^{(t)}$ to debias the gradient, where $\widehat{\psi}_1(y)$ is the first coefficient function in the Hermite expansion of the oracle function $\psi(y,x)$ with respect to $x$.
Finally, we update $\theta_m^{(t)}$ according to Line~\ref{algline:uniform_normalize} and Line~\ref{algline:uniform_update}.

\subsection{Feature Alignment and Statistical Complexity}\label{subsec:uniform_prior_theory}

\begin{assumption}\label{asp:oracle function}
For the Gaussian single-index model in \eqref{eq:single_index_model} with generative exponent $s^\star\geq1$, the oracle $\psi:\RR\times\RR\to\RR$ satisfies the following conditions:
\begin{enumerate}[
      label=(\alph*),
      ref=\Cref{asp:oracle function}(\alph*),
      leftmargin=0.25in,
      nolistsep
]
      \item (Quadruple-integrable under $\QQ$). \label{asp:quadratic_integrability}
      Both $\psi(y,x)^2x^2$ and $\psi(y,x)^2$ are square-integrable under the null distribution $\QQ$.
      Therefore, $\psi\in L^2(\QQ)$ admits a Hermite expansion
      $
        \psi(y, x) \overset{L^2(\QQ)}{=} \sum_{s=0}^\infty \hat\psi_s(y) \cdot h_s(x),
      $
      where $\hat\psi_s(y)$ is the $s$-th coefficient function and $\sum_{s=0}^\infty \EE_{\QQ}[\hat\psi_s(y)^2] <\infty$.

      \item (High-pass under $\QQ$).\label{asp:high_pass}
      For all $s=1,\ldots, s^\star - 2$, the $s$-th coefficient function is zero, i.e., $\hat\psi_s(y)\equiv 0$.
      In addition, there exists a constant $C>0$ such that $|\EE_\QQ[\zeta_{s^\star}(y) \cdot \hat\psi_{s^\star -1} (y)]|\geq C$.

      \item (Polynomial-like tail under $\PP$ and $\QQ$). \label{asp:polynomial-like}
      There exists a constant $C>0$ such that for all $r\geq 1$, $\max\{\EE_{\PP}\left[|\psi(y,x)|^r\right], \EE_{\QQ}\left[|\psi(y,x)|^r\right]\} \le C \cdot r^{C r}$.
  \end{enumerate}
\end{assumption}

The quadruple-integrability condition (\ref{asp:quadratic_integrability}) ensures that the decomposition of the likelihood ratio in \eqref{eq:likelihood-ratio-decomp} is well defined for calculations involving the second moment, i.e.,
\begin{align}
    \EE_{\PP}\left[\psi(y,x)^2 x^2\right]
    = \EE_\QQ\left[ \psi(y,x)^2 x^2 \cdot \frac{\PP(y, x)}{\QQ(y,x)}\right]
    = \EE_\QQ\left[\psi(y,x)^2 x^2 \cdot  {\sum_{s=0}^\infty} \zeta_s(y)h_s(x)\right].
\end{align}
The high-pass condition (\ref{asp:high_pass}) has been motivated in \Cref{sec:alignment} and guarantees noise reduction for the second moment.
The polynomial-like tail condition (\ref{asp:polynomial-like}) is used for concentration arguments in the proof.
Note that this condition is analogous to the Gaussian hypercontractivity property, where $\EE_{x\sim\cN(0,1)}[|f(x)|^r] \lesssim r^{Dr/2}$ if $f(x)$ is a polynomial of degree at most $D$.
In particular, $\psi$ can be constructed as $\psi(y,x)=\ell'(y)\sigma'(x)$, where $\ell$ is the loss function and $\sigma$ is the activation in the two-layer network.
It suffices to use a loss $\ell$ with bounded derivative and a polynomial activation $\sigma$ for the polynomial-like tail condition (see \Cref{sec:reusing batch} and \ref{sec:label transformation}).

Now we are ready to state our first main result on the sample complexity of \Cref{alg:meta} for uniform prior.
See \Cref{subsec:uniform_prior_sketch} for a proof sketch of the theorem and \Cref{app:proof_nonsparse} for a detailed proof.

\begin{theorem}[Sample complexity for uniform prior]\label{thm:uniform}
Under \Cref{asp:oracle function}, set the initialization of the weights as $\theta_1^{(0)},\ldots,\theta_M^{(0)}\iidfrom \unif(\SS^{d-1})$.
Suppose that the event $\cE = \{|\err_{m,l,i}^{(t)}|\leq d^{-10s^{\star}},\forall (m,l,i,t)\in [M]\times[L]\times[n]\times[T]\}$ holds with probability at least $1-O(d^{-c_0})$ for some constant $c_0>0$ with $(M,L,n,T)$ specified as follows.
Set the learning rate $\eta >2$, the polarization level $\gamma = d^{-1/4}(\log d)^{1/4}$, and the number of neurons $M=C_M\log d$ for a sufficiently large constant $C_M>0$.
Suppose
\begin{align}
    n = \Theta\bigl((d^{s^\star/2}(\log d)^{1+s^\star/2}) \lor d (\log d)^2  \bigr), \quad  L = \Theta\bigl(d^{(s^\star +1)/2} \log d\bigr).
\end{align}
Define $\tau = \eta^{-2}/(1-\eta^{-1})^2$, and let $\Delta=(\log d)^{-1/2}$ if $s^\star\leq 2$ and $\Delta = d^{-1/4}(\log d)^{1/4}$ if $s^\star\geq 3$.
Then with probability at least $1 - O(d^{-c})$ for some constant $c>0$, after running \Cref{alg:meta} for $T=O(\log d + \log(\Delta^{-1})/\log(\tau^{-1}))$ steps, there are at least $\Omega(M)$ neurons having alignment
$
    |\langle \theta_m^\T, \theta^\star\rangle| \ge 1 - O(\sqrt\Delta).
$
\end{theorem}
\Cref{thm:uniform} shows that the sample complexity of \Cref{alg:meta} is $nT = \tilde{\Theta}(d^{s^{\star}/2}\vee d)$, matching the SQ sample complexity lower bound for all $s^{\star}\geq 1$ established by \citet{damian2024computational}.
Compared to the partial trace method in \cite{damian2024computational}, our algorithm does not require special warm-start initialization.
Meanwhile, the computational complexity of \Cref{alg:meta} is $MLnT = \widetilde{\Theta}(d^{s^{\star}+1/2}\vee d^2)$.
So far the gradient oracle $\psi$ is still an abstract object, and next we will instantiate the above general theorem with concrete examples of $\psi$ that yield implementable algorithms.
We consider the special case of polynomial link functions with $s^\star\leq 2$ in \Cref{sec:reusing batch}, and then the general case for any $s^\star\geq 1$ in \Cref{sec:label transformation}.

\begin{remark}[Benefit of overparametrization]
\Cref{alg:meta} trains a two-layer neural network with logarithmic width $M=\Theta(\log d)$, involving $L$ times of perturbation for every neuron in each step.
Indeed, this is equivalent to training a two-layer neural network with width $LM=\widetilde{\Theta}\bigl(d^{(s^\star +1)/2}\bigr)$, where we divide the neurons into $L$ groups, each having $M$ neurons.
In each iteration we perturb the weights and compute the gradients, and then aggregate the gradients within each group of $M$ neurons.
This combination of weight perturbation and \emph{gradient sharing} exploits the \emph{benefit of overparametrization}.
This is also consistent with the interpretation of overparametrization as parallel search in \citet{edelman2024pareto}.
\end{remark}

\subsubsection{Online SGD with Batch Reusing}\label{sec:reusing batch}
The oracle function $\psi$ can be specialized to two consecutive gradient descent steps on the same batch under square loss to handle polynomial link functions corresponding to $s^{\star}\leq 2$ \citep{lee2024neural}.

\begin{example}[Batch-reusing: $\psi$ for polynomial link functions]\label{exp:psi_2}
Suppose that the link distribution is a polynomial of degree $q$.
We consider $\psi$ induced by batch-reusing on a single neuron, i.e., $\psi(y,x) = y\sigma'(x) + y\sigma'(x+y\sigma'(x))$ (see Section 4.2 of \cite{lee2024neural} for derivation) and choose $\sigma'(x) = \sum_{i=0}^{C_q}c_ih_i(x)$ where $C_q\in\mathbb{N}$ is a constant depending only on $q$ and each $c_i\sim \mathrm{Unif}([0,1])$.
\end{example}

\begin{corollary}[Batch-reusing for polynomial link function]\label{cor:psi_2}
    Suppose that the link distribution is given by a polynomial link function.
    Under the same setups in \Cref{thm:uniform} with the oracle $\psi$ given by \Cref{exp:psi_2}, the sample complexity of \Cref{alg:meta} is $\widetilde{\Theta}(d)$, recovering the result of \cite{lee2024neural}.
\end{corollary}

The proof is deferred to \Cref{app:example_psi_2}.
However, batch-reusing may not be optimal for $s^{\star}\geq 3$ due to violation of the high-pass condition, necessitating a more general approach to construct $\psi$.
\subsubsection{Label Transformation via Modified Loss}
\label{sec:label transformation}

We discuss another approach to construct $\psi$ by modifying the loss function, a universal method for arbitrary generative exponent $s^\star\geq 1$.
Additional details and proofs are postponed to \Cref{app:example_psi_3}.

\begin{example}[$\psi$ based on modified loss]\label{exp:psi_3}
    Let $\psi(y,x) = \ell'(y) \sigma'(x)$ with $\ell(y)$ being a certain loss function and $\sigma(x)$ being some activation function.
    Such a form corresponds to the gradient of the loss $\ell(y-f(z;\btheta))$ (assuming that $\ba$ is fixed and has small entries), since
    \begin{align}
        a_m^{-1} \nabla_{\theta_m}\ell\big(y-f(z;\btheta)\big) &= -\ell'\big(y-f(z;\btheta)\big)\cdot \sigma'(\langle \theta_m, z\rangle) \cdot z
        = - \underbrace{\ell'(y)\cdot \sigma'(\langle \theta_m, z\rangle)}_{:=\psi(y, \langle \theta_m, z\rangle)} \cdot z  + \err_m\cdot z,
    \end{align}
    where $\err_m \defeq  [\ell'(y) - \ell'(y-f(z;\btheta, \ba))]\cdot \sigma'(\langle \theta_m, z\rangle) = O(f(z;\btheta, \ba))\cdot \sigma'(\langle \theta_m, z\rangle)$ denotes small error for sufficiently small $\ba$.
    In \Cref{app:example_psi_3}, we provide a specific choice of the activation function $\sigma(x)$ (order-$s^{\star}$ Hermite polynomial) and the loss function $\ell(y)$ (a carefully designed random loss function), satisfying all the conditions in \Cref{asp:oracle function}.
\end{example}

\begin{corollary}[Modified loss for general $s^\star$]\label{cor:psi_3}
The oracle $\psi$ given by \Cref{exp:psi_3} satisfies all the assumptions of \Cref{thm:uniform}, thus the results of \Cref{thm:uniform} hold for \Cref{alg:meta} using this $\psi$.
\end{corollary}

\subsection{Proof Sketch of the Main Theorem for Uniform Prior}\label{subsec:uniform_prior_sketch}

For simplicity, denote $\rho_m^{(t)} := \langle\theta_m^{(t)},\theta^\star\rangle$, the alignment between the weights of neuron $m$ and the signal $\theta^\star$ at time $t$.
Recall from Line~\ref{algline:uniform_update} in \Cref{alg:meta} that the update for neuron $m$ at time step $t$ is
\begin{align}\label{eq:update_alignment}
    \theta_m^{(t+1)} = \frac{\theta_m^{(t)} + \eta \bar g_m^{(t)}}{\|\theta_m^{(t)} + \eta \bar g_m^{(t)}\|_2}.
\end{align}
This implies that the alignment of the next iteration, $\rho_m^{(t+1)}$, is a convex combination of the previous alignment $\rho_m^{(t)}$ and the alignment of the update step $\langle\bar g_m^{(t)},\theta^\star\rangle=\langle g_m^{(t)},\theta^\star\rangle / \|g_m^{(t)}\|_2$.
Therefore, to show that the alignment improves after one iteration, we need to first analyze the scale of $\langle g_m^{(t)},\theta^\star\rangle$ and $\|g_m^{(t)}\|_2$; then we will be able to characterize the improvement of $\rho_m^{(t)}$ across iterations.

\paragraph{Alignment of the update step $\langle \bar g_m^{(t)},\theta^\star\rangle$}
To this end, we calculate the first and second moments of $g_m^{(t)}$ over the randomness of the data $\{(z_i^{(t)},y_i^{(t)})\}_{i=1}^n$, and combining these leads to the concentration of $\langle g_m^{(t)},\theta^\star\rangle/\|g_m^{(t)}\|_2$.
More specifically, for the first moment of $g_m^{(t)}$, we have
\begin{align}
    \langle\EE_{\PP_{\theta^\star}} [g_m^{(t)}],\theta^\star\rangle &\approx \rho_m^{(t)} \gamma \cdot (|\rho_m^{(t)}|\gamma + d^{-1/2})^{s^\star - 2},
\end{align}
while the magnitude of $\EE_{\PP_{\theta^\star}}[g_m^{(t)}]$ in any other direction orthogonal to $\theta^\star$ is of strictly higher order.
For the second moment of $g_m^{(t)}$, setting $\gamma=\tilde \Theta(d^{-1/4})$, it can be shown that for any direction $v\in\SS^{d-1}$, $\EE_{\PP_{\theta^\star}}[\langle g_m^{(t)},v\rangle^2]=\tilde O(d^{-(s^\star-1)/2})$.
Now choosing $n=\tilde\Omega(d^{s^\star/2})$, it follows from a Bernstein-type concentration inequality that the fluctuation of $\langle g_m^{(t)},v\rangle$ is of the same order $\tilde O(d^{-s^\star/2+1/4})$ for any direction $v\in\SS^{d-1}$.
Therefore, with high probability,
\begin{align}\label{eq:g_property}
\begin{aligned}
    \langle g_m^{(t)},\theta^\star\rangle &\geq \rho_m^{(t)} \gamma \cdot (|\rho_m^{(t)}|\gamma + d^{-1/2})^{s^\star - 2} - \tilde O(d^{-s^\star/2+1/4}),\\
     \|g_m^{(t)}\|_2 &\leq \rho_m^{(t)} \gamma\cdot (|\rho_m^{(t)}|\gamma + d^{-1/2})^{s^\star - 2} + \tilde O(d^{-s^\star/2+3/4}).
\end{aligned}
\end{align}
The displayed bounds are stated for the actual implemented gradient.
Here the oracle-error term $\err_{m,l,i}^{(t)}$, incorporated in the gradient computation in Line~\ref{algline:uniform_gradient}, is the residual between the ideal oracle search direction $\psi(y_i^{(t)},\langle w_{m,l}^{(t)},z_i^{(t)}\rangle)z_i^{(t)}$ and the corresponding gradient contribution produced by the concrete training implementation.
The detailed proof first analyzes the gradient obtained by dropping this residual, and then uses the oracle-error event proved in \Cref{app:preliminary} to show that adding it changes each directional projection and the norm by at most $d^{-9s^\star}$, which is negligible compared with the signal and fluctuation scales above.

\paragraph{Phase 1: from $d^{-1/2}$ to weak alignment}
Since $M=\Theta(\log d)$ and each random initialization has $|\rho_m^{(0)}|\ge d^{-1/2}/2$ with constant probability, a Chernoff bound implies that this holds for $\Omega(M)$ many neurons with high probability.
Therefore, it suffices to consider a neuron with $|\rho_m^{(t)}|=\Omega(d^{-1/2})$.
When $\Omega(d^{-1/2})\leq \rho_m^{(t)}\leq O(1)$, by choosing $\gamma = \tilde\Theta(d^{-1/4})$, we can ensure that
the first term in the lower bound for $\langle g_m^{(t)},\theta^\star\rangle$ in \eqref{eq:g_property} dominates the $\tilde O(d^{-s^\star/2+1/4})$ fluctuation.
Based on this, we can leverage \eqref{eq:g_property} to further show that $\langle \bar g_m^{(t)},\theta^\star\rangle = \langle g_m^{(t)},\theta^\star\rangle/\|g_m^{(t)}\|_2 \geq (1+c) \rho_m^{(t)}$ for a constant $c>0$.
Consequently, it follows from \eqref{eq:update_alignment} that
\begin{align}
    \rho_m^{(t+1)} \geq (1+c) \rho_m^{(t)}\quad \text{ for some constant $c>0$.}
\end{align}
Therefore, it takes $O(\log d)$ many steps for $\rho_m^{(t)}$ to increase from $d^{-1/2}$ to $O(1)$.
During this period, the dynamics will go through two phases separated by a critical alignment level $\rho^\star$ such that
\begin{align}
  \rho_m^{(t)} \gamma\cdot (|\rho_m^{(t)}|\gamma + d^{-1/2})^{s^\star - 2} \approx \tilde O(d^{-s^\star/2+3/4}),
\end{align}
which gives that $\rho^\star=\tilde\Theta(d^{-1/4})$.
After the alignment $\rho_m^{(t)}$ reaches $\rho^\star$, there is a short period where $\rho_m^{(t)}$ grows rapidly as $\rho_m^{(t+1)}/\rho^\star\geq (\rho_m^{(t)}/\rho^\star)^{s^\star-1}$, until the alignment reaches $d^{-1/4+1/4(s^\star-1)}$.
The length of this period is very short compared to the other periods on the road to weak alignment.

\paragraph{Phase 2: from weak to strong alignment}
Finally, after $\rho_m^{(t)}$ grows to a constant scale, we need to track the value of $1-\rho_m^{(t)}$.
Again using \eqref{eq:g_property}, we can show that $1-(\rho_m^{(t+1)})^2 \leq \tau(1-(\rho_m^{(t)})^2)+o(1)$ for some constant $\tau<1$.
Hence, it takes another $O(\log d)$ steps to eventually achieve strong alignment.

\section{Exploiting the Structure: Algorithm for Sparse Prior}\label{sec:sparse}
We have seen that the sample complexity scales polynomially with the ambient dimension $d$ when the prior on $\theta^\star$ is uninformative, and our method achieves nearly optimal computational-statistical tradeoff under the SQ framework.
It is natural to ask whether our method can benefit from extra structural information on $\theta^\star$, one classic example being sparsity.
Towards this end, we consider an extension of the framework in the previous section to the setting where $\theta^\star$ is a $k$-sparse vector.
We first introduce the algorithm for extreme sparsity $k=o(\sqrt d)$ and then discuss the general sparse case.

\paragraph{Gaussian single-index model with sparse signal}
Given sparsity level $k=o(\sqrt{d})$, we consider the Gaussian single-index model in \eqref{eq:single_index_model} with $\theta^\star$ drawn from a $k$-sparse prior:
\begin{align}\label{eq:sparse_prior}
    \pi_k: \quad \theta^\star\giv \phi^\star \sim \unif\big(\SS^{k-1}(\phi^\star)\big), \quad \phi^\star \sim \unif(\cS_{k}),
\end{align}
where $\cS_k \defeq \{\phi \subset [d]: |\phi| = k\}$ is the collection of all $k$-sparse support sets, and $\SSS^{k-1}(\phi) \defeq \{x\in \RR^d: \sum_{i\in \phi} x_i^2 =1, x_j =0, \forall j\notin \phi\}$ is the associated $k$-dimensional unit sphere for any $\phi\in \cS_k$.

We show that \Cref{alg:meta} for the uniform prior can be easily adapted to leverage the sparsity of $\theta^\star$, striking a nearly optimal computational-statistical balance with $\tilde O(k^{s^\star})$ sample complexity.

\subsection{Algorithm Design: How to Leverage Sparsity?}\label{sec:sparse_design}
Note that \Cref{alg:meta} can also learn the $k$-sparse Gaussian single-index model, albeit with $\tilde O(d^{s^\star/2} \lor d)$ samples, which is apparently suboptimal in light of the classic example of sparse linear regression.
Here the key challenge is \emph{support identification} of $\phi^\star$,
and the issue of \Cref{alg:meta} lies in the weight perturbation using noise $\xi\sim\unif(\SS^{d-1})$, thus unaware of the sparsity of $\theta^\star$.
Below we discuss how to calibrate the weight perturbation with the sparse prior.

\paragraph{Perturbation by replicating the prior is not enough}
An intuitive first-cut attempt is to use perturbation noise from the same distribution as the sparse prior $\pi_k$ in \eqref{eq:sparse_prior}, which turns out to be suboptimal as well.
To illustrate this, we assume for simplicity a balanced $\theta^\star$ where every nonzero entry of $\theta^\star$ is equal to $k^{-1/2}$, and consider i.i.d. $\xi_1,\ldots,\xi_L\sim\pi_k$.
For each $j\in\phi^\star$, consider the $j$-th entry of the lowest-degree informative query (analogous to \eqref{eq:1st-moment-approx}), whose first moment satisfies
\begin{align}
    \EE_{\PP_{\theta^\star}}[g_j] \approx \EE_\QQ[y \zeta_s(y)] \cdot s^\star \hat \sigma_{s^\star} \cdot \frac 1 L \sum_{l=1}^L \langle \theta^\star, \theta_l\rangle^{s^\star - 1} \theta_j
    \approx \frac{\EE_{\theta\sim\pi_k}[\langle \theta^\star, \theta\rangle^{s^\star - 1}]}{\sqrt{k}}
    \simeq \frac{k^2}{d} \cdot \frac{k^{-(s^\star - 1)}}{\sqrt k},
\end{align}
where the last step follows from direct calculation for $\theta\sim\pi_k$.
Similarly, the second moment satisfies
\begin{align}
    \EE_{\PP_{\theta^\star}}[g_j^2] \approx \sum_{s\ge s^\star} s \cdot (\hat\sigma_{s})^2 \cdot \EE_{\theta, \theta'\sim \pi_k}[\langle \theta, \theta'\rangle^{s - 1}] \simeq \frac{k^2}{d} \cdot k^{-(s^\star - 1)},
\end{align}
where $\theta$ and $\theta'$ are drawn independently from the prior $\pi_k$.
This calculation implies that the fluctuation of each entry of the aggregated gradient is of order $\sqrt{k^2/d} \cdot k^{-(s^\star - 1)/2} \cdot n^{-1/2}$ for batch size $n$.
To successfully identify the true support $\phi^\star$, the signal must be larger than the fluctuation for entries in $\phi^\star$, resulting in a sample complexity of $n=\tilde{O}(k^{s^\star} \cdot d / k^2)$.
In comparison to the SQ lower bound in \Cref{thm:sq-lower-bound-sparse}, this is suboptimal by a factor $d/k^2$.
However, for $s^\star = 1$, we observe that $\EE_{\PP_{\theta^\star}}[g_j] = \Omega(k^{-1})$ for $j \in \phi^\star$ and $\EE_{\PP_{\theta^\star}}[g_j^2] = O(1)$, indicating that the support $\phi^\star$ can still be identified using $\tilde O(k)$ samples.
The form of the perturbation does not matter here since both $\langle \theta^\star, \theta\rangle$ and $\langle \theta, \theta'\rangle$ are degree zero in terms of $k^{-1}$.
Therefore, we conjecture that our algorithm succeeds for $s^\star = 1$ even without weight perturbation as outlined in \Cref{conj:s^star = 1}.

\paragraph{Perturbation by groups that cover the prior}
The suboptimality of the previous strategy originates from the fact that $\phi^\star$ is sampled from a uniform distribution over $\binom{d}{k}$ different $k$-sparse support sets, making it unlikely for two independent sets to overlap (only with $k^2/d$ probability).
Then how to perturb the weights in a way that guarantees a significant overlap with $\phi^\star$? The solution is to construct \emph{a polynomial-size cover} for the prior $\pi_k$.
Specifically, we divide $\cS_k$ into $d$ subsets, where the $j$-th subset is defined as $\cS_{k,j} := \{\phi \in \cS_k \mid j \in \phi\}$, which contains all $k$-sparse support sets that include the $j$-th coordinate.
Now suppose $\theta^\star\in\cS_{k,j}$, then for any perturbed weight $\theta_l$ with support from the same subset $\cS_{k,j}$, its support overlaps with $\phi^\star$ almost surely, thereby eliminating the $d/k^2$ factor.

In particular, considering a two-layer neural network with width $d$, the above strategy can be carried out by perturbing each neuron $m$ using $\theta_{m, 1},\ldots,\theta_{m,L}$ whose support sets are sampled from the same group $\cS_{k,m}$.
As a result, at least $k$ neurons will have one or more overlapping coordinates with the true signal $\theta^\star$.
For these neurons, the signal in the aggregated gradient would be strong enough for simple thresholding methods to correctly identify the true support $\phi^\star$ with $\tilde O(k^\star)$ samples.
We further refine this by first projecting the aggregated gradient for each neuron onto its top-$k$ support, and then selecting the \emph{strongest} projected gradient to update the weights.

Combining these yields \Cref{alg:sparse-warm-up} for the sparse case, where we define the support projection matrix $P_\phi\coloneqq\sum_{i\in \phi} e_i e_i^\top$ and the top-$k$ operator
$\mathsf{Top}_k (v) \defeq \argmax_{\phi\subset \cS_k } \norm{P_\phi(v)}_1$, which extracts the $k$-sparse support $\phi$ corresponding to the largest (in absolute value) $k$ entries of $v$.
We set the polarization level $\gamma = k^{-1/2}$, following the same balance between exploitation and exploration as in the uniform case, since the exploration noise is now of order $k^{-1}$.

Our sparsity-aware perturbation is related in spirit to the pruning approach of \citet{vural2024pruning}, who show that pruning neural networks to the sparsity level of the relevant features can improve the sample complexity for sparse single- and multi-index models and match sparsity-aware CSQ lower bounds in the sufficiently sparse regime.
Their guarantees for sparse single-index models are governed by the information exponent of polynomial link functions, whereas our target rate is governed by the generative exponent $s^\star$ of a general link distribution.
The mechanisms are also different: pruning starts from a dense model and adaptively selects sparse relevant features during training, whereas our method builds the sparsity structure directly into the exploration step.
Different neurons probe a cover of sparse supports, the aggregated gradients are projected onto their top-$k$ coordinates, and the strongest projected direction is shared across neurons.
This removes the $d/k^2$ overlap loss of prior-matched random perturbations and yields the $\tilde O(k^{s^\star})$ rate matching the SQ lower bound.

\subsection{Sample Complexity Analysis for Sparse Prior}\label{sec:sparse_alg}

\begin{algorithm}[t]
        \caption{Gradient-based Feature Learning for Sparse Signal Prior}
        \label{alg:sparse-warm-up}
        \begin{algorithmic}[1]
            \STATE \textbf{Input}:
            Initialization
            $\btheta^{(0)}=(\theta_1^{(0)},\ldots,\theta_M^{(0)})\in\RR^{d\times M}$, where $\theta^{(0)}_m = e_m$,
            $\bm{a} = a \cdot \vone\in\RR^{M}$ with number of neurons $M=d$,
            number of iterations $T\in\mathbb{N}$,
            batch size $n\in\mathbb{N}$,
            polarization level $\gamma\in(0,1)$,
            number of perturbations $L\in\mathbb{N}$.
            \FOR{iteration $t=0,1,\ldots,T-1$}
            \STATE Sample a fresh mini-batch $\{(z_i^\t, y_i^\t )\}_{i=1}^{n}$.
                \STATE Perturb as Line~\ref{algline:uniform_perturb} in \Cref{alg:meta} with $\xi_{m,l}^\t\overset{\iid}{\sim}\unif(\SS^{k-1}(\phi_{m,l}))$ and $\phi_{m,l}\overset{\iid}{\sim}\unif(\cS_{k,m})$.
                \STATE Compute and aggregate the gradients to get $g_m^{(t)}$, same as Lines~\ref{algline:uniform_gradient} and \ref{algline:uniform_aggregate} in \Cref{alg:meta}.
            \STATE Find the top-$k$ support of $g_m^\t$ and project: $\phi_m^\t = \mathsf{Top}_k(g_m^\t)$, $\tilde g_m^\t = P_{\phi_m^\t}(g_m^\t)$ for all $m$.
            \STATE Locate the neuron with the largest $\norm{\tilde g_m^\t}_2$: $\wh m=\argmax_m \norm{\tilde g_m^\t}_2$.
            \STATE Update weights by gradient sharing: $\theta^{(t+1)}_m = \tilde g_{\wh m}^\t / \norm{\tilde g_{\wh m}^\t}_2$ for all $m$. \label{alg:sparse_update}
            \ENDFOR
            \STATE \textbf{Output}: Model weights $\btheta^{(T)}$.
        \end{algorithmic}
\end{algorithm}

\begin{theorem}[Sample complexity for sparse prior]\label{thm:sparse}
Under \Cref{asp:oracle function},
consider the sparse prior in \eqref{eq:sparse_prior} with sparsity level $k$ satisfying $\omega(d^{\iota}) < k < o(\sqrt d)$ for a small $\iota>0$.
Suppose that the event $\cE = \{|\err_{m,l,i}^{(t)}|\leq d^{-10s^{\star}},\forall (m,l,i,t)\in [M]\times[L]\times[n]\times[T]\}$ holds with probability at least $1-O(d^{-c_0})$ for some constant $c_0>0$ with $(M,L,n,T)$ specified as follows.
Let $\gamma=k^{-1/2}$ and
\begin{align}
   n=\Omega \big((k \log^3 k)^{s^\star}\cdot \log d\big),\quad L= \Omega\big(k^{(s^\star+3)/2}\cdot \log(k)^{s^\star-1}\big).
\end{align}
Then with probability at least $1-O(k^{-c})$ for some $c>0$, after running \Cref{alg:sparse-warm-up} with $T=2$ iterations, all $M$ output neurons share a weight with alignment
\begin{align}
|\dotp{\theta^{(T)}_m }{\theta^{\star}}| \ge 1-O(\Delta),\quad \forall m\in[M],
\end{align}
where $\Delta = k^{-1} \vee \big(k^{-(s^\star-1)/2}\cdot \log(k)^{-3/2} \big) = o(1)$.
\end{theorem}

\Cref{thm:sparse} shows that the sample complexity of \Cref{alg:sparse-warm-up} is $n=\tilde O(k^{s^\star})$, matching the SQ lower bound established in \Cref{thm:sq-lower-bound-sparse} which will be presented below.
Here for simplicity, we essentially use an infinitely large learning rate when updating the weights in Line~\ref{alg:sparse_update} of \Cref{alg:sparse-warm-up}, so it takes only two iterations to achieve strong alignment.
This is equivalent to running the same algorithm with a finite learning rate but with a larger number of iterations, which is omitted for brevity.

\begin{remark}[Implication for sparse tensor PCA]
The connection between the Gaussian single-index model and tensor PCA has been discussed in \citet{damian2024smoothing}, by showing that estimating $\theta^\star$ corresponds to a tensor PCA problem defined over the empirical Hermite tensors.
Our weight perturbation technique can be potentially applied to iteratively solve sparse tensor PCA problems.
\end{remark}

Next we present the conjecture on the success of our algorithm for $s^\star=1$ whose intuition is already discussed in \Cref{sec:sparse_design}.
\begin{conjecture}\label{conj:s^star = 1}
For $d^\iota < k < o(d)$ with $s^\star=1$, \Cref{alg:sparse-warm-up} succeeds with sample complexity $n=\tilde O(k)$.
Furthermore, the same guarantee applies even without perturbing the weights.
\end{conjecture}

Finally, we present the following SQ lower bound for the sparse prior, complementing \Cref{thm:sparse}.
\begin{theorem}[SQ lower bound]\label{thm:sq-lower-bound-sparse}
Consider the Gaussian single-index model in \eqref{eq:single_index_model} with generative exponent $s^\star\geq 1$.
Suppose $\theta^\star$ is $k$-sparse for $\omega((\log d)^2) \le k \le d/2$.
Take $c>2$ as a constant.
For any (stochastic) algorithm using $\exp(\Omega((\log d)^c))$ calls to the $\VSTAT(\PP_{\theta^\star}, n)$ oracle with sample size $n$, in order to achieve nontrivial alignment $|\langle \hat\theta, \theta^\star\rangle|>\rho$ with probability at least $2/3$, it requires
\begin{align}
     n \gtrsim \frac{k^{s^\star}}{(\log d)^{cs^\star}}, &\where \rho = \tilde\omega(k^{-1}) \quad \mbox{ if } (\log d)^2 < k < \sqrt{d (\log d)^c }, \label{eq:sparse_lb}\\
    n \gtrsim \frac{d^{s^\star/2}}{(\log d)^{c s^\star/2}}, &\where \rho = \tilde\omega(d^{-1/2}) \quad \mbox{ if } \sqrt{d(\log d)^c } \le k \le d/2.\label{eq:nonsparse_lb}
\end{align}
\end{theorem}
In fact, the effective SQ lower bound should be no smaller than the information-theoretic lower bound $\Omega(k\log(d/k))$ \citep{neykov2016agnostic}.
For $k = o(\sqrt d)$, running \Cref{alg:sparse-warm-up} will succeed with $\tilde O(k^{s^\star})$ samples, matching the lower bound in \eqref{eq:sparse_lb} for every $s^\star \ge 1$.
For $k = \Omega(\sqrt{d})$, running \Cref{alg:meta} will succeed with $\tilde O(d^{s^\star/2})$ samples, matching the lower bound in \eqref{eq:nonsparse_lb} for $s^\star \ge 2$.
In addition, for $k = \Omega(\sqrt{d})$ with $s^\star = 1$, we conjecture $n=\tilde O(k)$ samples to be sufficient, where the information-theoretic lower bound is $\Omega(k \log(d/k))$.
This gives rise to the paradigm in \Cref{fig:intro} (b).

\section{Conclusions and Future Directions}
\label{sec:conclusion}

We studied whether neural networks trained by gradient-based methods can attain the optimal computational-statistical tradeoff for learning Gaussian single-index models.
For an unknown link distribution with generative exponent $s^\star$, we gave a unified gradient-based procedure for training a two-layer network whose weights strongly align with the planted direction using $\widetilde O(d^{s^\star/2}\vee d)$ samples.
This matches the SQ lower bound up to polylogarithmic factors for every $s^\star\geq 1$.
The algorithm combines two ingredients: an oracle view of the gradient, which captures label transformation and modified losses, and a weight perturbation scheme, which uses random exploration to suppress the low-SNR directions that obstruct vanilla online SGD.
In this sense, the result identifies a route by which feature learning can go beyond the kernel or CSQ limitations while remaining within a gradient-based training framework.

We also showed that the same perspective can exploit structure in the hidden direction.
When $\theta^\star$ is $k$-sparse with $k=o(\sqrt d)$, perturbing weights according to a polynomial-size cover of the sparse prior removes the overlap loss that would arise from naive sparse random perturbations.
This yields a gradient-based algorithm with sample complexity $\widetilde O(k^{s^\star})$, together with a matching SQ lower bound up to logarithmic factors in the relevant sparse regime.
Thus, the right notion of exploration is not necessarily isotropic: it should be matched to the geometry of the prior.
This observation may be useful beyond single-index models, for example in designing iterative methods for sparse tensor PCA and related planted problems.

Several questions remain open.
First, our dense-prior algorithm uses fresh weight perturbations at each step to realize the smoothing effect needed in the proof.
We expect that, at least in the nonsparse setting, sufficiently overparameterized random initialization may already contain enough exploratory directions, so that explicit resampling of perturbation noise is not essential.
Making this precise would bring the theory closer to standard neural-network training, where the randomness is often injected only at initialization.

Second, the present analysis is naturally online: each iteration uses a fresh mini-batch, and the proof tracks the signal-to-noise ratio of the aggregated gradient.
It would be valuable to understand whether the same optimal rates can be obtained by genuinely single-pass SGD, by fixed mini-batch protocols with more data reuse, or by full-batch gradient descent.
These variants differ in how gradient noise, batch reuse, and weight perturbation interact, and a sharp comparison may clarify which part of the algorithm is doing the statistical work; see also the recent work of \citet{wei2026improved}, which shows that Langevin dynamics with iterate averaging can emulate smoothing and conjectures that minibatch SGD may achieve the same rate without additional noise.

Finally, the Gaussian design assumption is used throughout the analysis, most visibly through the Hermite expansion and the Gaussian noise operator.
Extending the picture beyond isotropic Gaussian inputs is an important direction, especially in light of recent progress on feature learning and single-index models under more general covariate distributions \citep{mousavi2023gradient,ba2024learning}.
Such an extension would require replacing the exact Hermite orthogonality and rotational invariance used here with tools that capture the relevant low-dimensional structure of non-Gaussian designs.

\subsubsection*{Acknowledgments}
We thank Jason D. Lee and Denny Wu for helpful discussions.

\bibliography{reference}

@article{edelman2024pareto,
  title={Pareto frontiers in deep feature learning: Data, compute, width, and luck},
  author={Edelman, Benjamin and Goel, Surbhi and Kakade, Sham and Malach, Eran and Zhang, Cyril},
  journal={Advances in Neural Information Processing Systems},
  volume={36},
  year={2024}
}

@article{mousavi2023gradient,
  title={Gradient-based feature learning under structured data},
  author={Mousavi-Hosseini, Alireza and Wu, Denny and Suzuki, Taiji and Erdogdu, Murat A},
  journal={Advances in Neural Information Processing Systems},
  volume={36},
  pages={71449--71485},
  year={2023}
}

@article{vural2024pruning,
  title={Pruning is Optimal for Learning Sparse Features in High-Dimensions},
  author={Vural, Nuri Mert and Erdogdu, Murat A},
  journal={arXiv preprint arXiv:2406.08658},
  year={2024}
}

@inproceedings{refinetti2021classifying,
  title={Classifying high-dimensional gaussian mixtures: Where kernel methods fail and neural networks succeed},
  author={Refinetti, Maria and Goldt, Sebastian and Krzakala, Florent and Zdeborov{\'a}, Lenka},
  booktitle={International Conference on Machine Learning},
  pages={8936--8947},
  year={2021},
  organization={PMLR}
}

@inproceedings{mondelli2018fundamental,
  title={Fundamental limits of weak recovery with applications to phase retrieval},
  author={Mondelli, Marco and Montanari, Andrea},
  booktitle={Conference On Learning Theory},
  pages={1445--1450},
  year={2018},
  organization={PMLR}
}

@article{lu2020phase,
  title={Phase transitions of spectral initialization for high-dimensional non-convex estimation},
  author={Lu, Yue M and Li, Gen},
  journal={Information and Inference: A Journal of the IMA},
  volume={9},
  number={3},
  pages={507--541},
  year={2020},
  publisher={Oxford University Press}
}

@inproceedings{jin2017escape,
  title={How to escape saddle points efficiently},
  author={Jin, Chi and Ge, Rong and Netrapalli, Praneeth and Kakade, Sham M and Jordan, Michael I},
  booktitle={International conference on machine learning},
  pages={1724--1732},
  year={2017},
  organization={PMLR}
}

@article{vaskevicius2019implicit,
  title={Implicit regularization for optimal sparse recovery},
  author={Vaskevicius, Tomas and Kanade, Varun and Rebeschini, Patrick},
  journal={Advances in Neural Information Processing Systems},
  volume={32},
  year={2019}
}

@article{zhao2022high,
  title={High-dimensional linear regression via implicit regularization},
  author={Zhao, Peng and Yang, Yun and He, Qiao-Chu},
  journal={Biometrika},
  volume={109},
  number={4},
  pages={1033--1046},
  year={2022},
  publisher={Oxford University Press}
}

@article{candes2006robust,
  title={Robust uncertainty principles: Exact signal reconstruction from highly incomplete frequency information},
  author={Cand{\`e}s, Emmanuel J and Romberg, Justin and Tao, Terence},
  journal={IEEE Transactions on information theory},
  volume={52},
  number={2},
  pages={489--509},
  year={2006},
  publisher={IEEE}
}

@article{donoho2009message,
  title={Message-passing algorithms for compressed sensing},
  author={Donoho, David L and Maleki, Arian and Montanari, Andrea},
  journal={Proceedings of the National Academy of Sciences},
  volume={106},
  number={45},
  pages={18914--18919},
  year={2009},
  publisher={National Acad Sciences}
}

@article{raskutti2012minimax,
  title={Minimax-optimal rates for sparse additive models over kernel classes via convex programming.},
  author={Raskutti, Garvesh and J Wainwright, Martin and Yu, Bin},
  journal={Journal of machine learning research},
  volume={13},
  number={2},
  year={2012}
}

@inproceedings{arnaboldi2024repetita,
  title={Repetita Iuvant: Data Repetition Allows SGD to Learn High-Dimensional Multi-Index Functions},
  author={Arnaboldi, Luca and Dandi, Yatin and Krzakala, Florent and Pesce, Luca and Stephan, Ludovic},
  booktitle={High-dimensional Learning Dynamics 2024: The Emergence of Structure and Reasoning},
  year={2024}
}

@article{biroli2020iron,
  title={How to iron out rough landscapes and get optimal performances: averaged gradient descent and its application to tensor PCA},
  author={Biroli, Giulio and Cammarota, Chiara and Ricci-Tersenghi, Federico},
  journal={Journal of Physics A: Mathematical and Theoretical},
  volume={53},
  number={17},
  pages={174003},
  year={2020},
  publisher={IOP Publishing}
}

@inproceedings{anandkumar2017homotopy,
  title={Homotopy analysis for tensor PCA},
  author={Anandkumar, Anima and Deng, Yuan and Ge, Rong and Mobahi, Hossein},
  booktitle={Conference on Learning Theory},
  pages={79--104},
  year={2017},
  organization={PMLR}
}

@article{chizat2019lazy,
  title={On lazy training in differentiable programming},
  author={Chizat, Lenaic and Oyallon, Edouard and Bach, Francis},
  journal={Advances in neural information processing systems},
  volume={32},
  year={2019}
}

@article{jacot2018neural,
  title={Neural tangent kernel: Convergence and generalization in neural networks},
  author={Jacot, Arthur and Gabriel, Franck and Hongler, Cl{\'e}ment},
  journal={Advances in neural information processing systems},
  volume={31},
  year={2018}
}

@article{ghorbani2019limitations,
  title={Limitations of lazy training of two-layers neural network},
  author={Ghorbani, Behrooz and Mei, Song and Misiakiewicz, Theodor and Montanari, Andrea},
  journal={Advances in Neural Information Processing Systems},
  volume={32},
  year={2019}
}

@article{allen2019can,
  title={What can resnet learn efficiently, going beyond kernels?},
  author={Allen-Zhu, Zeyuan and Li, Yuanzhi},
  journal={Advances in Neural Information Processing Systems},
  volume={32},
  year={2019}
}

@inproceedings{girshick2014rich,
  title={Rich feature hierarchies for accurate object detection and semantic segmentation},
  author={Girshick, Ross and Donahue, Jeff and Darrell, Trevor and Malik, Jitendra},
  booktitle={Proceedings of the IEEE conference on computer vision and pattern recognition},
  pages={580--587},
  year={2014}
}

@inproceedings{dudeja2018learning,
  title={Learning single-index models in gaussian space},
  author={Dudeja, Rishabh and Hsu, Daniel},
  booktitle={Conference On Learning Theory},
  pages={1887--1930},
  year={2018},
  organization={PMLR}
}

@book{hardle2004nonparametric,
  title={Nonparametric and semiparametric models},
  author={H{\"a}rdle, Wolfgang and M{\"u}ller, Marlene and Sperlich, Stefan and Werwatz, Axel and others},
  volume={1},
  year={2004},
  publisher={Springer}
}

@article{hristache2001direct,
  title={Direct estimation of the index coefficient in a single-index model},
  author={Hristache, Marian and Juditsky, Anatoli and Spokoiny, Vladimir},
  journal={Annals of Statistics},
  pages={595--623},
  year={2001},
  publisher={JSTOR}
}

@article{ichimura1993semiparametric,
  title={Semiparametric least squares (SLS) and weighted SLS estimation of single-index models},
  author={Ichimura, Hidehiko},
  journal={Journal of econometrics},
  volume={58},
  number={1-2},
  pages={71--120},
  year={1993},
  publisher={Elsevier}
}

@article{maccullagh1989generalized,
  title={Generalized linear models},
  author={MacCullagh, P and Nelder, JA},
  journal={Monographs on statistics and applied probability (37)},
  year={1989},
}

@article{joshi2024complexity,
  title={On the Complexity of Learning Sparse Functions with Statistical and Gradient Queries},
  author={Joshi, Nirmit and Misiakiewicz, Theodor and Srebro, Nathan},
  journal={arXiv preprint arXiv:2407.05622},
  year={2024}
}

@article{lee2024neural,
  title={Neural network learns low-dimensional polynomials with SGD near the information-theoretic limit},
  author={Lee, Jason D and Oko, Kazusato and Suzuki, Taiji and Wu, Denny},
  journal={arXiv preprint arXiv:2406.01581},
  year={2024}
}

@inproceedings{dandi2024benefits,
  title={The Benefits of Reusing Batches for Gradient Descent in Two-Layer Networks: Breaking the Curse of Information and Leap Exponents},
  author={Dandi, Yatin and Troiani, Emanuele and Arnaboldi, Luca and Pesce, Luca and Zdeborova, Lenka and Krzakala, Florent},
  booktitle={Forty-first International Conference on Machine Learning},
  year={2024}
}

@article{damian2024smoothing,
  title={Smoothing the landscape boosts the signal for sgd: Optimal sample complexity for learning single index models},
  author={Damian, Alex and Nichani, Eshaan and Ge, Rong and Lee, Jason D},
  journal={Advances in Neural Information Processing Systems},
  volume={36},
  year={2023}
}

@article{damian2024computational,
  title={The computational complexity of learning gaussian single-index models},
  author={Damian, Alex and Pillaud-Vivien, Loucas and Lee, Jason D and Bruna, Joan},
  journal={arXiv preprint arXiv:2403.05529},
  year={2024}
}

@book{o2014analysis,
  title={Analysis of boolean functions},
  author={O'Donnell, Ryan},
  year={2014},
  publisher={Cambridge University Press}
}

@article{folland2001integrate,
  title={How to integrate a polynomial over a sphere},
  author={Folland, Gerald B},
  journal={The American Mathematical Monthly},
  volume={108},
  number={5},
  pages={446--448},
  year={2001},
  publisher={Taylor \& Francis}
}

@article{feldman2017statistical,
  title={Statistical algorithms and a lower bound for detecting planted cliques},
  author={Feldman, Vitaly and Grigorescu, Elena and Reyzin, Lev and Vempala, Santosh S and Xiao, Ying},
  journal={Journal of the ACM (JACM)},
  volume={64},
  number={2},
  pages={1--37},
  year={2017},
  publisher={ACM New York, NY, USA}
}

@article{klenke2010stochastic,
  title={Stochastic ordering of classical discrete distributions},
  author={Klenke, Achim and Mattner, Lutz},
  journal={Advances in Applied probability},
  volume={42},
  number={2},
  pages={392--410},
  year={2010},
  publisher={Cambridge University Press}
}

@book{szekli2012stochastic,
  title={Stochastic ordering and dependence in applied probability},
  author={Szekli, Ryszard},
  volume={97},
  year={2012},
  publisher={Springer Science \& Business Media}
}

@article{arous2021online,
  title={Online stochastic gradient descent on non-convex losses from high-dimensional inference},
  author={Arous, Gerard Ben and Gheissari, Reza and Jagannath, Aukosh},
  journal={Journal of Machine Learning Research},
  volume={22},
  number={106},
  pages={1--51},
  year={2021}
}

@article{bietti2022learning,
  title={Learning single-index models with shallow neural networks},
  author={Bietti, Alberto and Bruna, Joan and Sanford, Clayton and Song, Min Jae},
  journal={Advances in Neural Information Processing Systems},
  volume={35},
  pages={9768--9783},
  year={2022}
}

@inproceedings{damian2022neural,
  title={Neural networks can learn representations with gradient descent},
  author={Damian, Alexandru and Lee, Jason and Soltanolkotabi, Mahdi},
  booktitle={Conference on Learning Theory},
  pages={5413--5452},
  year={2022},
  organization={PMLR}
}

@inproceedings{dandi2023two,
  title={How Two-Layer Neural Networks Learn, One (Giant) Step at a Time},
  author={Dandi, Yatin and Krzakala, Florent and Loureiro, Bruno and Pesce, Luca and Stephan, Ludovic},
  booktitle={NeurIPS 2023 Workshop on Mathematics of Modern Machine Learning},
  year={2023}
}

@inproceedings{abbe2023sgd,
  title={Sgd learning on neural networks: leap complexity and saddle-to-saddle dynamics},
  author={Abbe, Emmanuel and Adsera, Enric Boix and Misiakiewicz, Theodor},
  booktitle={The Thirty Sixth Annual Conference on Learning Theory},
  pages={2552--2623},
  year={2023},
  organization={PMLR}
}

@article{chen2020towards,
  title={Towards understanding hierarchical learning: Benefits of neural representations},
  author={Chen, Minshuo and Bai, Yu and Lee, Jason D and Zhao, Tuo and Wang, Huan and Xiong, Caiming and Socher, Richard},
  journal={Advances in Neural Information Processing Systems},
  volume={33},
  pages={22134--22145},
  year={2020}
}

@inproceedings{chen2020learning,
  title={Learning polynomials in few relevant dimensions},
  author={Chen, Sitan and Meka, Raghu},
  booktitle={Conference on Learning Theory},
  pages={1161--1227},
  year={2020},
  organization={PMLR}
}

@article{ba2024learning,
  title={Learning in the presence of low-dimensional structure: a spiked random matrix perspective},
  author={Ba, Jimmy and Erdogdu, Murat A and Suzuki, Taiji and Wang, Zhichao and Wu, Denny},
  journal={Advances in Neural Information Processing Systems},
  volume={36},
  year={2023}
}

@article{bach2017breaking,
  title={Breaking the curse of dimensionality with convex neural networks},
  author={Bach, Francis},
  journal={Journal of Machine Learning Research},
  volume={18},
  number={19},
  pages={1--53},
  year={2017}
}

@article{neykov2016agnostic,
  title={Agnostic estimation for misspecified phase retrieval models},
  author={Neykov, Matey and Wang, Zhaoran and Liu, Han},
  journal={Advances in Neural Information Processing Systems},
  volume={29},
  year={2016}
}

@article{fan2023understanding,
  title={Understanding implicit regularization in over-parameterized single index model},
  author={Fan, Jianqing and Yang, Zhuoran and Yu, Mengxin},
  journal={Journal of the American Statistical Association},
  volume={118},
  number={544},
  pages={2315--2328},
  year={2023},
  publisher={Taylor \& Francis}
}

@article{dudeja2021statistical,
  title={Statistical query lower bounds for tensor pca},
  author={Dudeja, Rishabh and Hsu, Daniel},
  journal={Journal of Machine Learning Research},
  volume={22},
  number={83},
  pages={1--51},
  year={2021}
}

@inproceedings{arous2020free,
  title={Free energy wells and overlap gap property in sparse PCA},
  author={Arous, G{\'e}rard Ben and Wein, Alexander S and Zadik, Ilias},
  booktitle={Conference on Learning Theory},
  pages={479--482},
  year={2020},
  organization={PMLR}
}

@article{gamarnik2017sparse,
  title={Sparse high-dimensional linear regression. algorithmic barriers and a local search algorithm},
  author={Gamarnik, David and Zadik, Ilias},
  journal={arXiv preprint arXiv:1711.04952},
  year={2017}
}

@article{bandeira2022franz,
  title={The Franz-Parisi criterion and computational trade-offs in high dimensional statistics},
  author={Bandeira, Afonso S and El Alaoui, Ahmed and Hopkins, Samuel and Schramm, Tselil and Wein, Alexander S and Zadik, Ilias},
  journal={Advances in Neural Information Processing Systems},
  volume={35},
  pages={33831--33844},
  year={2022}
}

@inproceedings{feldman2017general,
  title={A general characterization of the statistical query complexity},
  author={Feldman, Vitaly},
  booktitle={Conference on learning theory},
  pages={785--830},
  year={2017},
  organization={PMLR}
}

@article{wei2026improved,
  title={Improved high-dimensional estimation with Langevin dynamics and stochastic weight averaging},
  author={Wei, Stanley and Damian, Alex and Lee, Jason D.},
  journal={arXiv preprint arXiv:2603.06028},
  year={2026}
}
\bibliographystyle{ims}

\newpage
\appendix
\tableofcontents

\section{Numerical Experiments}\label{sec:experiments}

We conduct extensive simulation experiments to validate the sample complexity result established in \cref{thm:uniform}.
Specifically, for a fixed Gaussian single-index model, we run \cref{alg:meta} over a variety of problem instances with diverse scales and report the average accuracy in terms of the alignment.
We lay out the details of the experimental setup as follows.

{\noindent \bf $\bullet$ Gaussian single-index model.} We focus on the Gaussian single-index model introduced in \eqref{eq:single_index_model} with a deterministic link function and Gaussian additive noise.
We set $y = \phi(\langle z,\theta^\star\rangle)+\epsilon$, where $\phi(x) = x^2 \cdot \exp(-x^2)$ and $\epsilon \sim \cN(0, \sigma^2)$ with $\sigma^2 = 0.5$. As shown in \cref{exp:sgm}, the generative exponent of the function $\phi$ is $s^\star(\phi) = 4$.
In addition, the signal parameter $\theta^\star$ is uniformly sampled from the unit sphere in $\RR^d$.

{\bf \noindent $\bullet$ Neural network architecture.} We adopt the two-layer neural network introduced in \cref{sec:setup} with $M$ set to $15$ and $a_m = 1$ for all $m \in [M]$ in all experiments.
Since $s^\star(\phi) = 4$, we set the activation function as $\sigma(x) = h_4(x)$, i.e., the fourth-order Hermite polynomial.

{\bf \noindent $\bullet$ Training using \cref{alg:meta}.}
In \cref{alg:meta}, we set $\psi(y,x) = y \cdot \sigma'(x)$, as stated in \cref{exp:psi_3}. Such a $\psi$ is justified by considering the following alignment loss:
\begin{align}\label{eq:alignment_loss}
L(\btheta) = 1 - {y \cdot f(z; \btheta, \ba )} = 1 - \sum_{m=1}^M a \cdot \sigma(\la z, \theta_m \ra) \cdot y,
\end{align}
where each entry of $\ba $ is equal to $a $.
By \eqref{eq:alignment_loss}, we have
\begin{align}
    \label{eq:grad_alignment_loss}
   -a^{-1} \cdot \nabla_{\theta_m} L(\btheta) = y \cdot \sigma'(\la z, \theta_m \ra) \cdot z = \psi(y, \la z, \theta_m \ra) \cdot z.
\end{align}
Thus, we can alternatively interpret the gradients in \cref{alg:meta}
as negative gradients with respect to the alignment loss $L(\btheta)$.
Thus, the choice of $a$ does not matter in this case, and we set $a = 1$ for simplicity.
Furthermore,
other details of \cref{alg:meta} are specified as follows:
\begin{itemize}
    \item The parameters $\{ \theta_m \}_{m\in [M]}$ are initialized as i.i.d. random vectors in $\RR^d$ uniformly sampled from the unit sphere.
    \item We fix $M = 15$, $a = 1$, $\eta = 3$, $T = 24$, and $L = 500$ throughout all experiments with different values of $n$ and $d$.
    \item We enumerate $n$ and $d$ over a grid with $d \in [32, 499]$ and $n \in [5\times 10^3, 3\times 10^6]$. Since $\log d \in (3, 7)$, our choice of $T$ satisfies the requirement in Theorem \ref{thm:uniform}.
\end{itemize}

{\noindent \bf  $\bullet$ Choices of $(d, n)$.}
We select 40 different values of $d$ and 30 different values of $n$ within the ranges $d \in [32, 499]$ and $n \in [5 \times 10^3, 3 \times 10^6]$, respectively. These values form an evenly spaced grid in terms of $\log n$ and $\log d$. See Figure \ref{fig:grid} for an illustration.

\begin{figure}[ht]
\centering
\begin{tabular}{@{}cc@{}}
    \includegraphics[width=0.48\textwidth]{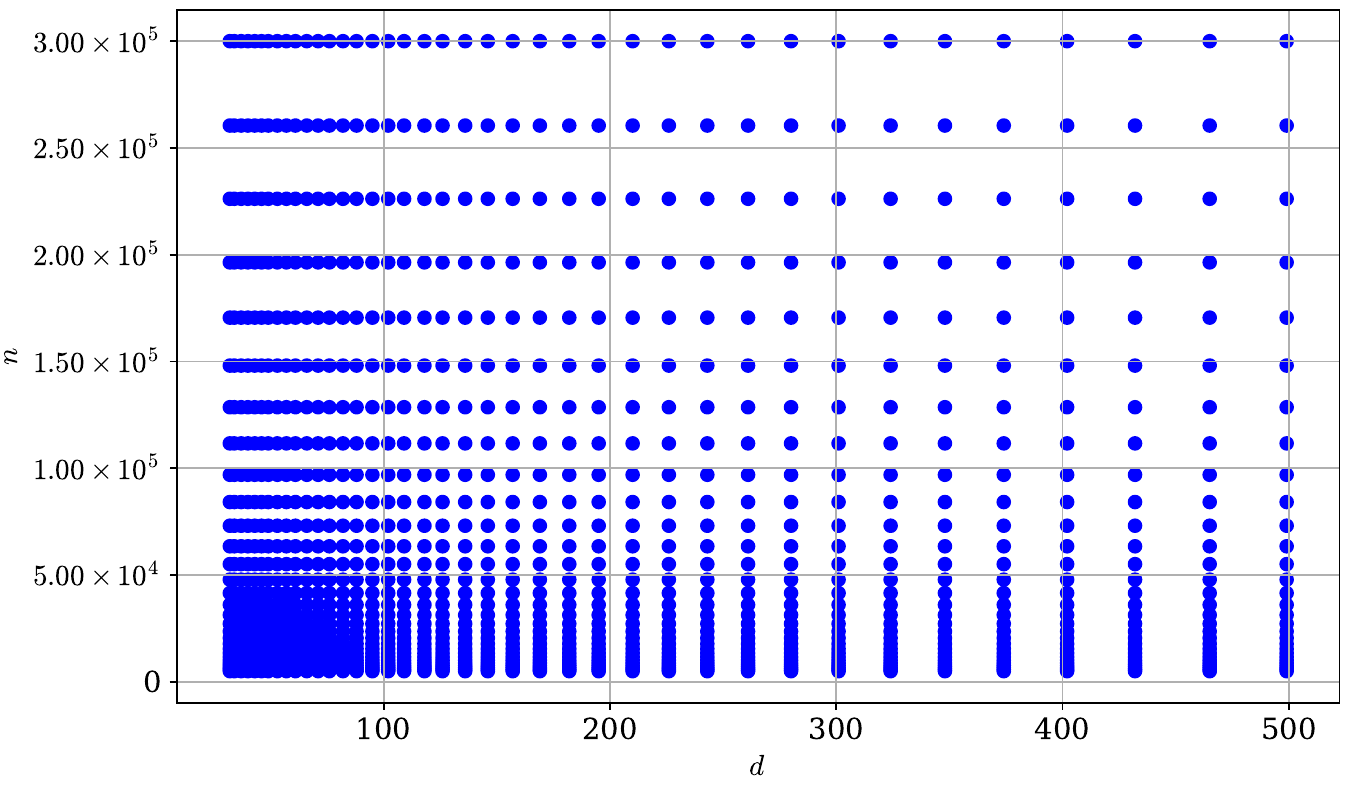} &
    \includegraphics[width=0.48\textwidth]{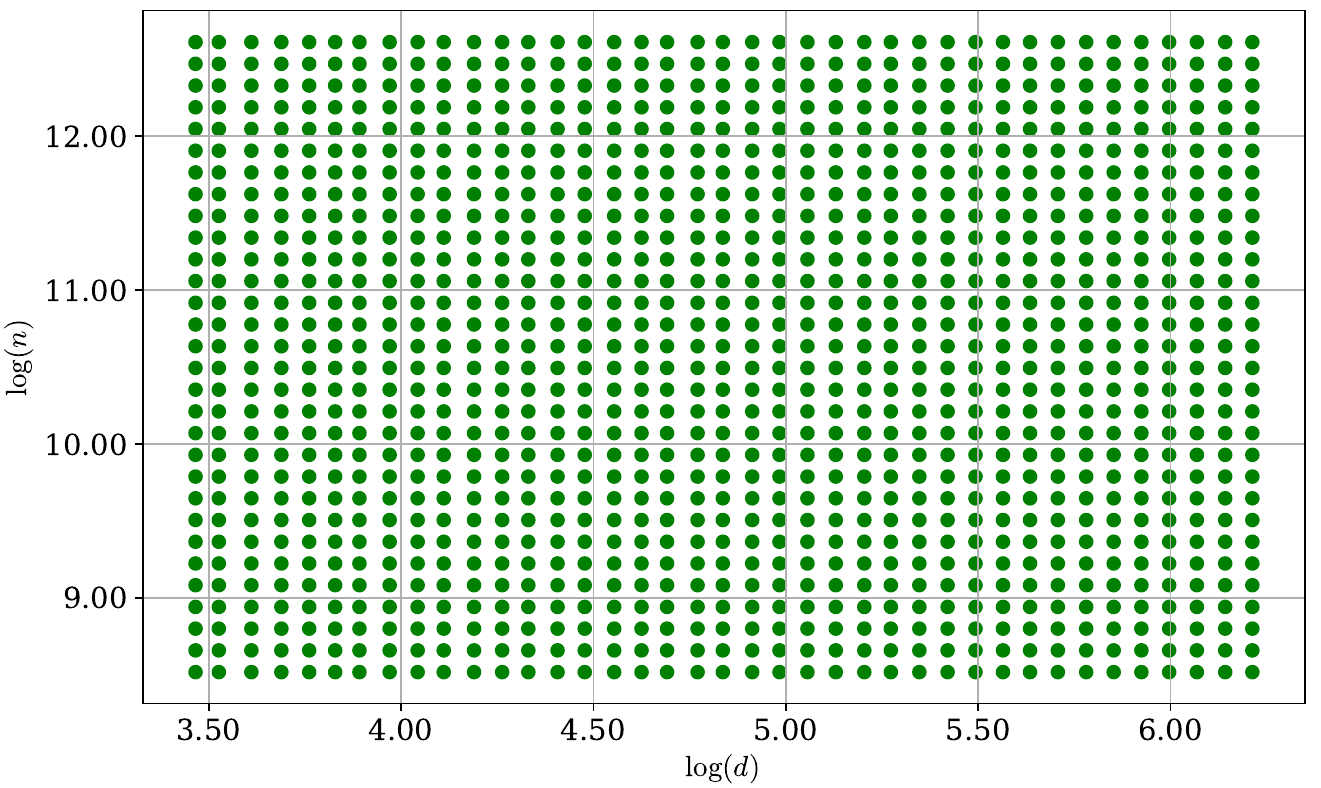} \\
    (a) Scatter plot of $(d, n)$ & (b) Scatter plot of $(\log d, \log n)$ \\
\end{tabular}
\caption{Scatter plots of $(d, n)$ and $(\log d, \log n)$. In (a) we plot $n$ against $d$ and in (b) we plot $\log n$ against $\log d$. As shown in (b), we choose $n$ and $d$ such that they form an evenly spaced grid after taking logarithms.}
\label{fig:grid}
\end{figure}

{\noindent \bf $\bullet$ Evaluation.} We report the accuracy of \cref{alg:meta} based on $25$ repeated experiments for every choice of $(n, d)$. We report two types of accuracy metrics:
\begin{itemize}
    \item [(i)] Average accuracy: We report $M^{-1} \sum_{m=1}^M | \la \theta_m , \theta^\star\ra | $ in each experiment and then average over the $25$ experiments.
    \item [(ii)] Top-8 accuracy: Given $\{\theta_m \}_{m \in [M]}$ returned by the algorithm, we sort the alignment values $ \{ | \la \theta_m , \theta^\star\ra |\}_{m \in [M]}$.
    Then we report the average of the largest $8$ numbers. The rationale is that if the top-8 accuracy is close to one, at least half of the neurons correctly find $\theta^\star$.
\end{itemize}

{\noindent \bf Contour plots.} After calculating these two accuracy metrics for every $(d, n)$ pair,
we generate the contour plots based on $(\log d, \log n, \texttt{acc}(d, n))$, where $\texttt{acc}(d, n)$ is one of the two accuracy metrics introduced above.
We report these two contour plots in Figure \ref{fig:contour_1} and Figure \ref{fig:contour_2}, where in Figure \ref{fig:contour_1} we zoom in to a smaller range of $d$ for better visualization.
In these plots, points with the same color indicate $(\log d, \log n)$ with the same level of accuracy.

{\bf \noindent Validation of $\tilde \Theta(d^{s^\star / 2} )$ sample complexity.} As shown in these figures, the average accuracy and the top-8 accuracy clearly exhibit {\bf a linear relationship}.
That is, for a fixed accuracy level $\delta$, pairs $(d, n)$ with $\texttt{acc}(d, n) \approx \delta$ are well fit by a line of the form $\log n = c_1 \cdot \log d + c_2$.
To determine $c_1$ and $c_2$, we further fit linear models to the points $(\log d, \log n)$ with the same accuracy level $\delta$, where $\delta \in \{ 0.6, 0.7, 0.8\}$.
For both the average accuracy and the top-8 accuracy, the coefficient $c_1$ in the linear models is close to $2$.
We report the linear models corresponding to different accuracy levels in Table \ref{tab:linear_model}.
Since $s^\star = 4$, this finding indicates that $n \propto d^2$.
Moreover, since we compute the accuracy for all $(d, n)$ on the grid, the fact that $c_1 \approx 2$ indicates that the $\tilde \Theta(d^{s^\star/2})$ sample complexity is sharp.

\begin{figure}[ht]
\centering
\begin{tabular}{@{}cc@{}}
    \includegraphics[width=0.52\textwidth]{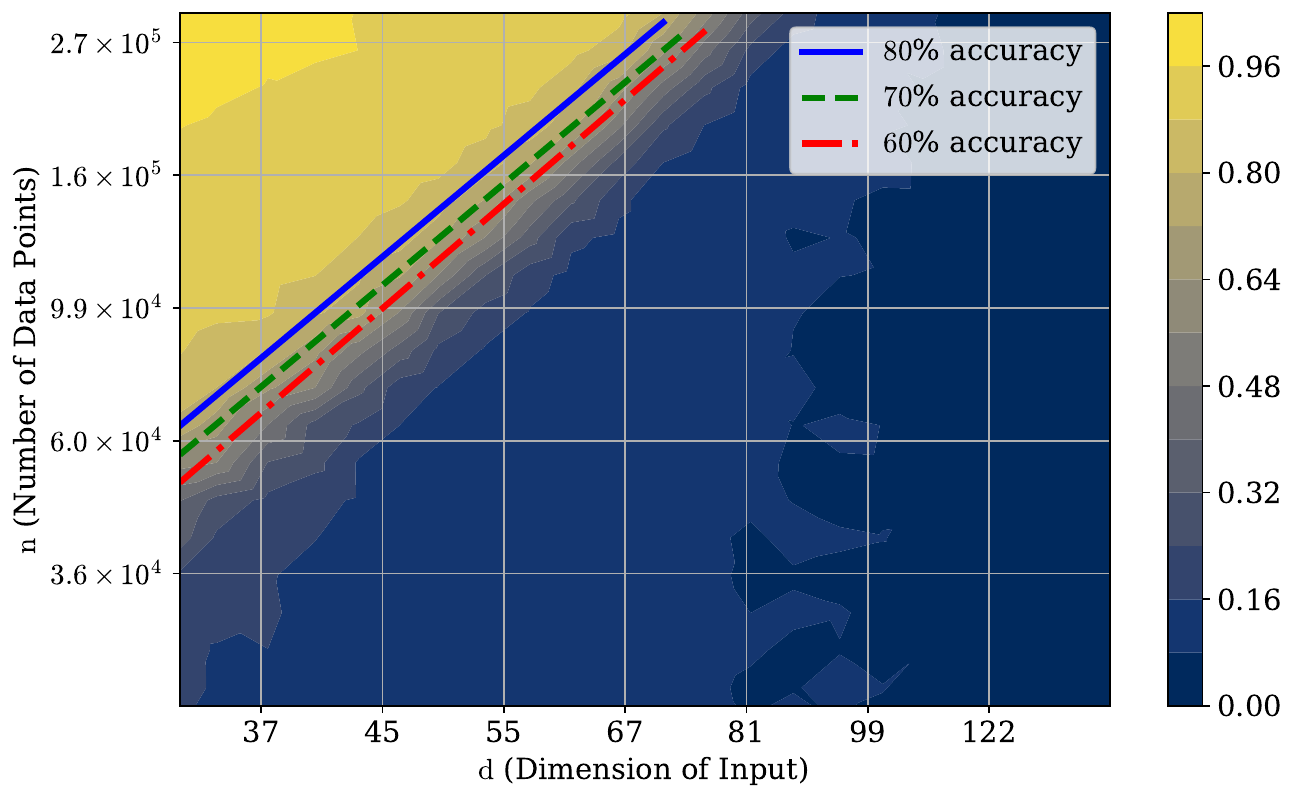} &
    \includegraphics[width=0.45\textwidth]{figures/topK_error_final_small.pdf} \\
    (a) Average accuracy & (b) Top-8 accuracy \\
\end{tabular}
\caption{The contour plots of $(\log d, \log n, \texttt{acc}(d, n))$, where $\texttt{acc}(d, n)$ is either the average accuracy or the top-8 accuracy.
Here we zoom in to a smaller subset of $d$'s for better visualization.
We also plot the lines containing $(\log d, \log n)$ with the same accuracy level among $\{0.6, 0.7, 0.8\}$.
The slopes of these lines are all close to $2$. This indicates that $n \approx d^{2}$ samples are sufficient and necessary for accurate estimation.}
\label{fig:contour_1}
\end{figure}

\begin{figure}[ht]
\centering
\begin{tabular}{@{}cc@{}}
    \includegraphics[width=0.52\textwidth]{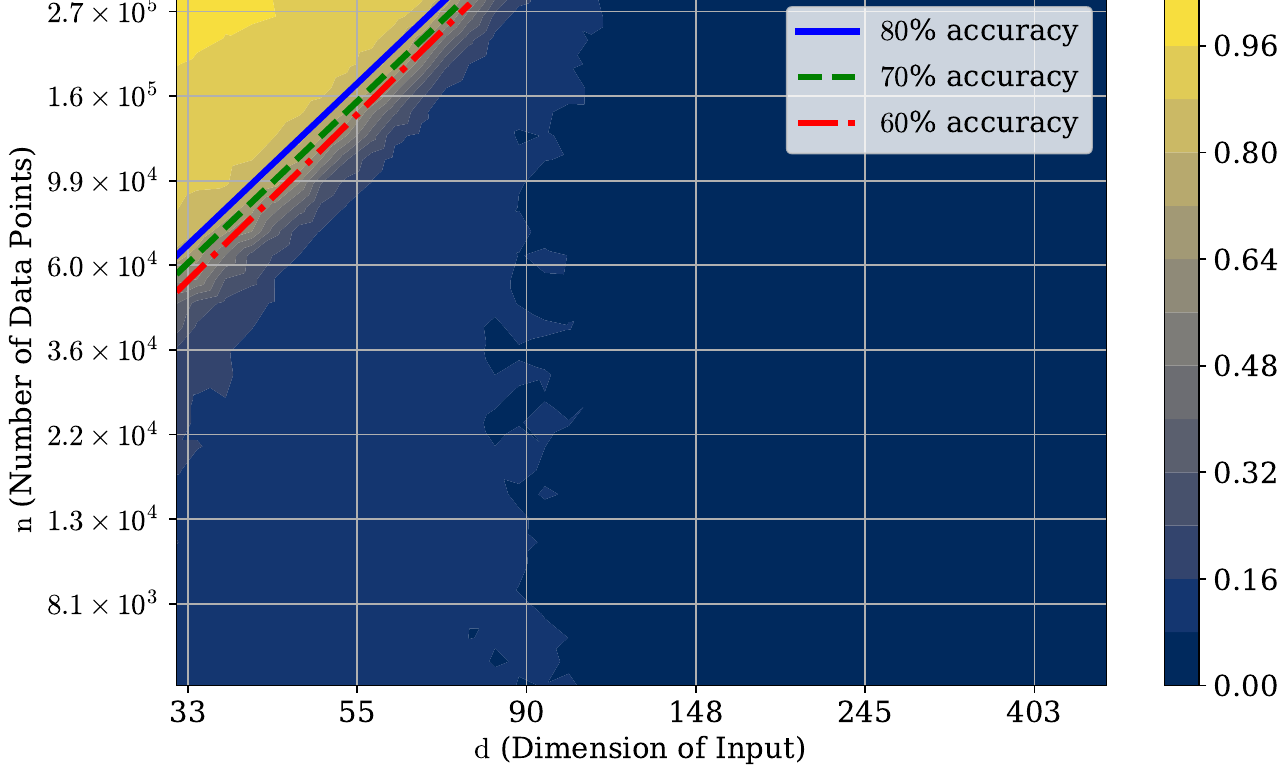} &
    \includegraphics[width=0.45\textwidth]{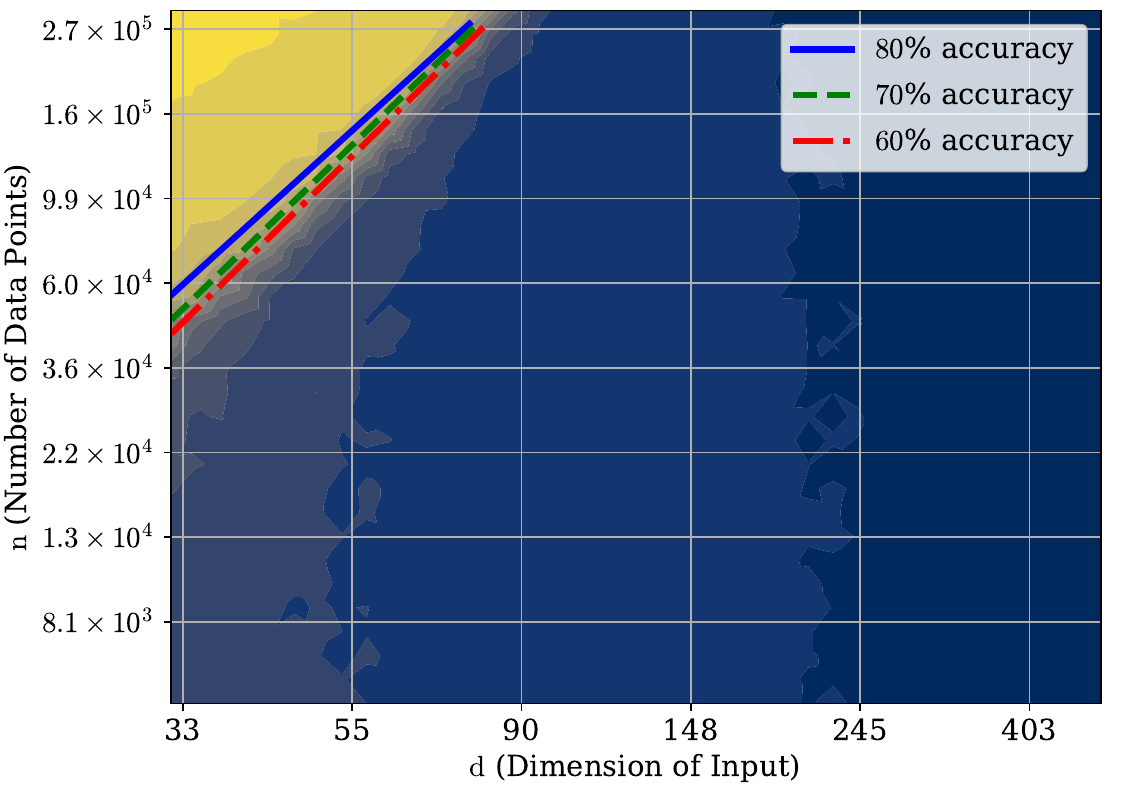} \\
    (a) Average accuracy & (b) Top-8 accuracy \\
\end{tabular}
\caption{The contour plots of $(\log d, \log n, \texttt{acc}(d, n))$, where $\texttt{acc}(d, n)$ is either the average accuracy or the top-8 accuracy.
We also plot the lines containing $(\log d, \log n)$ with the same accuracy level among $\{0.6, 0.7, 0.8\}$. The slopes of these lines are all close to $2$. This indicates that $n \approx d^{2}$ samples are sufficient and necessary for accurate estimation.}
\label{fig:contour_2}
\end{figure}

\begin{table}[ht]
\centering
\caption{Fitted linear equations of the form $\log n = c_1 \cdot \log d + c_2$ for $n, d$ with the desired accuracy level. Notably, the slopes of these equations are all close to $s^\star / 2 = 2$, which shows that $n \propto d^{s^\star /2}$.}
\begin{tabular}{|c|c|c|}
\hline
Accuracy level & Average accuracy & Top-8 accuracy \\ \hline
0.8 & $\log(n) = 1.9058 \cdot \log(d) + 4.4516$ & $\log(n) = 1.8201 \cdot
  \log(d) + 4.6218$ \\ \hline
0.7 & $\log(n) = 1.9103 \cdot \log(d) + 4.3273$ & $\log(n) = 1.9343 \cdot \log(d) + 4.0790$ \\ \hline
0.6 & $\log(n) = 1.9640 \cdot \log(d) + 4.0361$ & $\log(n) = 1.9653 \cdot \log(d) + 3.8901$ \\ \hline
\end{tabular}
\label{tab:linear_model}
\end{table}

\section{Notation and Preliminaries}
\label{app:preliminary}

\paragraph{Notations}
We use $\NN$ to denote the set of positive integers and $\NN_0$ to denote the set of nonnegative integers.
For vector $z\in\RR^d$, we denote by $\RR_{n}[z]$ the set of polynomials of degree at most $n$ in $z$ with real coefficients.
For $s\in\NN$, we denote by $\varPi_s$ the symmetric group of all permutations of $[s]$.
We denote by $\cN_d(\cdot)$ and $\cN(\cdot)$ the standard normal distribution in $\RR^d$ and $\RR$, respectively.

For two tensors $S \in(\RR^d)^{\otimes s}$ and $T \in (\RR^d)^{\otimes t}$ where $s\ge t$,
\begin{align}
    (S[T])_{j_1, \ldots, j_{s-t}} \defeq \sum_{i_1, \ldots, i_t=1}^d S_{j_1, \ldots, j_{s-t}, i_1, \ldots, i_t} T_{i_1, \ldots, i_t}.
\end{align}
Here, $S[T]$ produces a tensor of order $s-t$ and dimension $d$. 
We also define the symmetrization operation for a tensor $T\in(\RR^d)^{\otimes t}$ as
\begin{align}\label{eq:def_symmetrization}
    \Sym(T)_{i_1, \ldots, i_t} \defeq \frac{1}{t!} \sum_{\pi\in \varPi_t} T_{i_{\pi(1)}, \ldots, i_{\pi(t)}}.
\end{align}

The following are some notations for the relationship between two quantities or matrices:
\begin{description}
    \item[$a \simeq b$:] There exists a constant $C=O(1)$ such that $a \le C b$ and $b \le C a$. Note that $a$ and $b$ should have the same sign. $a = \Theta(b)$ also has the same meaning.
    \item[$a \approxeq b$:] $a \le \polylog(d) \cdot b$ and $b \le \polylog(d) \cdot a$, and the same for $a = \tilde \Theta(b)$.
    \item[$a \lesssim b$:] There exists a constant $C=O(1)$ such that $a \le C b$, and the same for $a=O(b)$. The use of $a \gtrsim b$ and $a=\Omega(b)$ is similar.
    \item[$a \lessapprox b$:] $a \le \polylog(d) \cdot b$, and the same for $a = \tilde O(b)$.
    The use of $a \gtrapprox b$ and $a = \tilde \Omega(b)$ is similar.
    \item[$a\ll b$:] $a \le (\polylog(d))^{-1} \cdot b$. The use of $a \gg b$ is similar.
\end{description}
In addition, we denote by $a = b\pm \varepsilon$, $a \simeq b \pm \varepsilon$, $a \approxeq b \pm \varepsilon$ that $b - \varepsilon \le a \le b + \varepsilon$, $a - \varepsilon \lesssim b \lesssim a + \varepsilon$, $a - \varepsilon \lessapprox b \lessapprox a + \varepsilon$, respectively.

For square matrices $A$ and $B$, $A \preceq B$ means that $B - A$ is positive semidefinite, and $A \precsim B$ means that there exists a constant $C = O(1)$ such that $C\cdot B - A$ is positive semidefinite.

\subsection{Background on Hermite Polynomials}
The probabilists' Hermite polynomials satisfy the following recurrence relations
\begin{align}
    h_s(x)' = \sqrt{s} \cdot h_{s-1}(x),\quad 
    x \cdot h_s(x) = \sqrt{s+1} \cdot h_{s+1}(x) + \sqrt{s} \cdot h_{s-1}(x), 
    \label{eq:hermite-recurrence}
\end{align}
where we adopt the convention that $h_{-1}(x) \equiv 0$.

For any function $f\in L^2(\cN(0,1))$, its Hermite expansion is given by 
\begin{align}
    f(x) = \sum_{s=0}^\infty \hat f_s \cdot h_s(x),
\end{align}
where we denote by $\hat f_s$ the $s$-th coefficient of the Hermite expansion of $f$.

\paragraph{Gaussian noise operator}
For $\rho\in[-1,1]$, define the Gaussian noise operator as 
\begin{align}
    \noiseop{\rho} f(x) = \EE_{x'\sim\cN(0,1)}\big[
        f(\rho x + \sqrt{1-\rho^2} \cdot x')\big].
\end{align}
Proposition 11.37 of \citet{o2014analysis} shows that the Hermite expansion of $\noiseop{\rho} f$ is given by
\begin{align}
    \noiseop{\rho} f(x) 
    \overset{L^2(\cN(0,1))}{=}
    \sum_{s=0}^\infty \rho^s \cdot \hat f_s \cdot h_s(x).  \label{eq:hermite-noise-spectral}
\end{align}
A direct implication of this identity is 
\begin{align}
    \EE_{x\sim \cN(0, 1)}[\noiseop{\rho}f(x) g(x)] = \EE_{x\sim \cN(0, 1)}[f(x) \noiseop{\rho}g(x)] = \sum_{s=0}^\infty \rho^s \hat f_s \hat g_s. \label{eq:hermite-noise-swap}
\end{align}
As a result, for any fixed $w, \theta\in\SSS^{d-1}$, it holds that 
\begin{align}
    \EE_{z\sim\cN(0,I_d)}\left[f(\langle w, z\rangle) h_s(\langle \theta, z\rangle)\right] = \EE_{x\sim \cN(0,1)}\left[\noiseop{\rho} f \cdot h_s(x)\right] = \langle w,\theta\rangle^s \cdot \hat f_s.
    \label{eq:hermite-noise-operator}
\end{align}

\paragraph{Hermite tensor}
Corresponding to the Hermite polynomials defined for scalar variables, we define the Hermite tensors over $z\in\RR^d$:
\begin{align}
    \bh_{s}(z) \defeq \frac{(-1)^s}{\sqrt{s!}} \cdot  e^{\norm{z}_2^2/2} \cdot \nabla^{s} e^{-\norm{z}_2^2/2} \in (\RR^d)^{\otimes s}, \text{ for }s\geq 0.
    \label{eq:hermite-tensor}
\end{align}
The scalar-valued Hermite polynomials and tensor-valued Hermite tensors are related as follows:
\begin{align}
    h_s(\langle \theta, z\rangle) = \bh_s(z)[\theta^{\otimes s}], \quad \forall \theta\in\SSS^{d-1}.
    \label{eq:hermite-tensor-directional}
\end{align}
Now let $f:\RR^d\to\RR$ be a $s$-times differentiable function such that for all $k\leq s$, every component of $\nabla^k f$ belongs to $L^2(\cN(0,I_d))$.
Then it follows from integration by parts that
\begin{align}
    \EE_{z\sim\cN(0,I_d)}\left[f(z) \bh_s(z)\right] = \frac{1}{\sqrt{s!}} \cdot \EE_{z\sim\cN(0,I_d)}\left[\nabla^s f(z)\right].
    \label{eq:hermite-tensor-Stein's lemma}
\end{align}
This is a version of Stein's lemma for tensor-valued functions.

\section{Supplementary Proofs for the Main Text}
\label{app:main-paper-supplementary}

\subsection{Proofs for \texorpdfstring{\Cref{sec:alignment}}{Section 3}}

In this section, we first argue why $\EE_{x\sim\cN(0, 1)}[y \zeta_{s^\star}(y)] = 0$ is the major difficulty for vanilla (stochastic) gradient descent to achieve the information-theoretic lower bound $O(d)$ (the same for the SQ lower bound) when the information exponent $q^\star$ is larger than $2$.
It has been shown by \citet{damian2024computational} that the generative exponent $s^\star$ for polynomial models is either $1$ or $2$.
Consider the information exponent $q^\star > 2$.
We have the following lemma saying that the correlation $\EE_{x\sim\cN(0, 1)}[y \zeta_{s}(y)] = 0$ for any $s < q^\star$.
\begin{lemma}
    \label{lem:IE-coef}
    Recall that $\zeta_s$ is the coefficient function for degree $s$ in the decomposition of the likelihood ratio $\PP(x, y)/\QQ(x, y)$ in \eqref{eq:likelihood-ratio-decomp}.
    For any $q^\star\geq 2$, consider the Gaussian single-index model given by $y= \beta_0 + \sum_{p\ge p^\star} \beta_p h_p(x)$ with $x\sim\cN(0, 1)$. 
    Then for any $1\le s < p^\star$, $\EE_\QQ[y \zeta_s(y)] = 0$.
\end{lemma}
\begin{proof} 
The proof can be done by noting that $\zeta_s(y) = \EE_{\QQ}[\PP(x, y)/\QQ(x, y) \cdot h_s(x)\given y]$, and 
\begin{align}
    \EE_\QQ[y \zeta_{s}(y)] 
    &= \EE_\QQ\Bigl[y \cdot \EE_{\QQ}\Bigl[\frac{\PP(x, y)}{\QQ(x, y)} \cdot h_{s}(x)\Biggiven y\Bigr]\Bigr] = \EE_\PP\left[y \cdot h_{s}(x)\right] \\
    &= \beta_0\cdot \EE_{x\sim\cN(0, 1)}[h_s(x)] + \sum_{i\ge p^\star} \beta_i \cdot\EE_{x\sim \cN(0, 1)}[h_i(x) h_{s}(x)] = 0, 
\end{align}
where the second equality follows from the independence between $x$ and $y$ under $\QQ$.
\end{proof}
Therefore, the first nonzero term in the informative queries of \eqref{eq:informative queries} is of order at least $p^\star$. 
This gives rise to sample complexity $d^{p^\star-1}$ for vanilla online SGD \citep{arous2021online} and $d^{p^\star/2}$ for SGD after smoothing the landscape \citep{damian2024smoothing}. 
This sample complexity $d^{p^\star/2}$ matches the correlated statistical query (CSQ) lower bound with gradient of the form $y\phi(z)$ \citep{abbe2023sgd, damian2022neural}.

\subsection{Proofs for Examples of Oracle Function}\label{app:example_psi}

Here we complete the discussion of the specific examples of $\psi$ in \Cref{exp:psi_2} and \Cref{exp:psi_3}.

\subsubsection{Batch-reusing for polynomial link function}\label{app:example_psi_2}

We consider a polynomial link function $y = p(x) = \sum_{q^\star\leq q'\leq q} \beta_{q'}h_{q'}(x)$ for general $q^\star\in\NN$ and $\beta_{q'}\in\RR$, where $q^\star$ is the information exponent of the link function, and we also denote it by $q^\star(p)$ in the sequel. 
For batch-reusing, we take $\psi(y,x) =  y\sigma'(x) + y\sigma'(x+y\sigma'(x))$,
where the activation function $\sigma(x)$ satisfies
\begin{align}
    \sigma(x) = \sum_{j=0}^{C_q}\alpha_j\cdot h_j(x),\quad \sigma'(x) = \sum_{j=1}^{C_q} \sqrt j \cdot \alpha_j\cdot h_{j-1}(x).\label{eq:psi_2_sigma}
\end{align}
Here the degree $C_q\in\mathbb{N}_+$ only depends on the degree $q$ of the link function and is specified later, and each coefficient $\alpha_j\sim \mathrm{Unif}([0,1])$. 
The second equality in \eqref{eq:psi_2_sigma} follows from the property of Hermite polynomials in \eqref{eq:hermite-recurrence}.
The error term $\err_{m,l,i}^{(t)}$ now comes from the difference between $\psi(y_i^{(t)}, \langle w_{m,l}^{(t)},z_i^{(t)}\rangle)\cdot z_i^{(t)}$ and the exact form of the update step obtained from two consecutive gradient descent steps on the same data under the square loss.
More specifically, let us consider a single neuron whose weight is $w_{m,l}$ and a single data point $(z_i, y_i)$. 
Here we omit the time index $t$ for convenience.
Then two gradient descent steps on $(z_i,y_i)$ give
\begin{align}
    -g^{\mathrm{Re}}_{m,l}(z_i, y_i) &= \big(y_i - f(z_i;\{w_{m,l}\}_{m\in[M]})\big)\cdot \sigma'(\langle w_{m,l},z_i\rangle)\cdot z_i\\
    &\qquad + \big(y_i - f(z_i;\{w_{m,l}^+\}_{m\in[M]})\big)\cdot \sigma'(\langle w_{m,l}^+,z_i\rangle)\cdot z_i,\label{eq:psi_2_0}
\end{align}
where $w_{m,l}^+ = w_{m,l} + \eta^{\mathrm{Re}}_i\cdot (y_i - f(z_i;\{w_{m,l}\}_{m\in[M]}))\cdot \sigma'(\langle w_{m,l},z_i\rangle)\cdot z_i$.
Here $\eta^{\mathrm{Re}}_i$ is the learning rate for batch-reusing, different from the learning rate $\eta$ in our algorithm.
More specifically, to fit the gradient form \eqref{eq:psi_2_0} into our general framework with oracle function $\psi(y,x)$, we take the batch-reusing learning rate $\eta^{\mathrm{Re}}_i = 1/\|z_i\|_2^2$. 
Then the error term is given by
\begin{align}
    \err_{m,l,i} = -g^{\mathrm{Re}}_{m,l}(z_i, y_i) - \psi(y_i, \langle w_{m,l},z_i\rangle) = \err_{m,l,i, 1} + \err_{m,l,i, 2} + \err_{m,l,i, 3},\label{eq:psi_2_0+}
\end{align}
where $\err_{m,l,i, 1}$, $\err_{m,l,i, 2}$, and $\err_{m,l,i, 3}$ are given by 
\begin{align}
    \err_{m,l,i, 1} &= -f(z_i;\{w_{m,l}\}_{m\in[M]}) \cdot \sigma'(\langle w_{m,l},z_i\rangle),\\
    \err_{m,l,i, 2} &= -f(z_i;\{w_{m,l}^+\}_{m\in[M]}) \cdot \sigma'(\langle w_{m,l}^+,z_i\rangle),\\
    \err_{m,l,i, 3} & = y_i\sigma'\big(\langle w_{m,l},z_i\rangle + y_i\sigma'(\langle w_{m,l},z_i\rangle) + \err_{m,l,i, 1}\big)\\
    &\qquad -y_i\sigma'\big(\langle w_{m,l},z_i\rangle + y_i\sigma'(\langle w_{m,l},z_i\rangle)\big).\label{eq:psi_2_0++}
\end{align}

\begin{proof}[Proof of Corollary~\ref{cor:psi_2}]

To prove Corollary~\ref{cor:psi_2}, it suffices to show that (i) \Cref{asp:oracle function} holds, and (ii) the event $\cE$ holds with the desired high probability. 
In the following, we first verify \Cref{asp:oracle function}, and then check the event $\cE$.

\textbf{Verifying \Cref{asp:oracle function}.}
Note that the fact that $y=p(x)$ is a polynomial immediately implies that both the square-integrability condition (\ref{asp:quadratic_integrability}) and the polynomial-like tail condition (\ref{asp:polynomial-like}) are satisfied.
It remains to check the high-pass condition (\ref{asp:high_pass}). 
Since now $s^{\star}\leq 2$, we only need to check the condition that $|\EE_\QQ[\zeta_{s^\star}(y) \cdot \hat\psi_{s^\star -1} (y)]|>0$.

\textbf{Case 1: $s^{\star}=2$.} In this case, we have that
\begin{align}
    \widehat{\psi}_1(y) &= \mathbb{E}_{x\sim \cN}\Big[x\cdot\Big(y\sigma'(x) + y\sigma'\big(x+y\sigma'(x)\big)\Big)\Big] \\
    &=y\cdot\mathbb{E}_{x\sim \cN}\bigg[\sum_{j=1}^{C_q} \sqrt j\cdot \alpha_j\cdot x\cdot h_{j-1}(x) + \sum_{j=1}^{C_q} \sqrt j \cdot \alpha_j\cdot x \cdot h_{j-1}\big(x+y\sigma'(x)\big)\bigg]\label{eq:psi_2_1}
\end{align}
For the first summation in \eqref{eq:psi_2_1}, only the $j=2$ summand is nonzero, so we obtain
\begin{align}
    y\cdot\mathbb{E}_{x\sim \cN}\bigg[\sum_{j=1}^{C_q} \sqrt j\cdot \alpha_j\cdot x\cdot h_{j-1}(x)\bigg] = \sqrt 2 \alpha_2\cdot y.\label{eq:psi_2_1+}
\end{align}
For the second summation in \eqref{eq:psi_2_1}, we have the following expansion:
\allowdisplaybreaks
\begin{align}
    &y\cdot\mathbb{E}_{x\sim \cN}\bigg[\sum_{j=1}^{C_q}j\cdot \alpha_j\cdot x\cdot h_{j-1}\big(x+y\sigma'(x)\big)\bigg] \notag\\
    &\qquad = y\cdot \mathbb{E}_{x\sim \cN}\bigg[\sum_{j=1}^{C_q}j\cdot \alpha_j\cdot x\cdot \sum_{k=0}^{j-1} r_{j-1,k} \cdot h_{j-k-1}(x)\cdot\big(y\sigma'(x)\big)^k\bigg]\notag\\
    &\qquad = \sum_{k=0}^{C_q-1}\underbrace{\left\{\sum_{j=k+1}^{C_q}j\cdot \alpha_j\cdot r_{j-1,k}\cdot\mathbb{E}_{x\sim \cN}\left[x\cdot h_{j-k-1}(x)\cdot \big(\sigma'(x)\big)^k\right] \right\}}_{\displaystyle{:=\varsigma_k(\alpha)}}  \cdot  y^{k+1}. \label{eq:psi_2_1++}
\end{align}
where $\varsigma_0(\alpha),\ldots,\varsigma_{C_q-1}(\alpha)$ are just polynomials of $\alpha = (\alpha_1,\cdots,\alpha_{C_q})$ (recall the definition of $\sigma'(x)$ in \eqref{eq:psi_2_sigma}) and each $r_{j-1,k}$ is a positive number. 
Combining \eqref{eq:psi_2_1+} and \eqref{eq:psi_2_1++}, we get the following decomposition of $\hat\psi_1(y)$:
\begin{align}
    \hat\psi_1(y) = \sqrt{2}\alpha_2\cdot y + \sum_{k=0}^{C_q-1}\varsigma_k(\alpha)\cdot y^{k+1}.
\end{align}
Further using $y=p(x)$ and the definition of $\zeta_2(y)$, we get
\begin{align}
    \EE_\QQ[\zeta_{2}(y) \cdot \hat\psi_{1} (y)] 
    & = \mathbb{E}_{\mathbb{P}}[h_2(x)\cdot \widehat{\psi}_1(y)] \\
    & = \sqrt 2 \alpha_2\cdot \mathbb{E}_{x\sim \cN}\left[h_2(x)\cdot p(x)\right]  + \sum_{k=0}^{C_q-1}\varsigma_k(\alpha)\cdot  \mathbb{E}_{x\sim \cN}\left[h_2(x)\cdot p(x)^{k+1}\right].
\end{align}
According to Proposition 5 of \cite{lee2024neural}, we can set $C_q\in\mathbb{N}^+$ (only depending on $q$) such that there exists a smallest $I\leq C_q$ such that the information exponent $q^{\star}(p^I)\leq 2$. 
We notice that in this case $s^{\star}(p) = 2$, where we abuse notation and let $s^\star(p)$ be the generative exponent of the polynomial $p$.
In fact, $s^\star(p) = 1$ means $\EE_{\PP}[\cT(y) \cdot h_1(x)] \equiv 0$ for all label transformations $\cT$.
Hence, the only possibility is that $q^{\star}(p^I)=  2$ since $p^I$ is just a special case of a label transformation and we cannot get any first-order term from $p^I$.
Therefore, we further simplify the target quantity as
\begin{align}
    \EE_\QQ[\zeta_{s^\star}(y) \cdot \hat\psi_{s^\star -1} (y)] &= \sqrt 2 \alpha_2 \cdot \underbrace{\mathbb{E}_{x\sim \cN}\left[h_2(x)\cdot p(x)\right]}_{\displaystyle{\defeq b_1}} 
    + \sum_{k=I}^{C_q}\varsigma_{k-1}(\alpha) \cdot \underbrace{\mathbb{E}_{x\sim \cN}\left[h_2(x)\cdot p(x)^{k}\right]}_{\ds \defeq b_k}\\
    &= b_1\cdot \sqrt{2} \alpha_2 + \sum_{k=I}^{C_q} b_k\cdot \varsigma_{k-1}(\alpha),\label{eq:psi_2_target}
\end{align}
where $b_I\neq 0$ according to the definition of $I$.
Now we take a closer look at the polynomials $\{\varsigma_k(\alpha_1,\cdots,\alpha_{C_q})\}_{k=I-1}^{C_q-1}$.
We claim that: (i) they are different polynomials and are linearly independent, and (ii) in particular, they are all nonzero.
To see these, recall that 
\begin{align}
    \varsigma_k(\alpha_1,\cdots,\alpha_{C_q}) = \sum_{j=k+1}^{C_q}\sqrt j \cdot \alpha_j\cdot r_{j-1,k}\cdot\mathbb{E}_{x\sim \cN}\left[x\cdot h_{j-k-1}(x)\cdot \big(\sigma'(x)\big)^k\right],
\end{align}
where $\sigma'(x) = \sum_{j=1}^{C_q}\sqrt j \cdot \alpha_j\cdot h_{j-1}(x)$. 
We can calculate that for $k=0$, $\varsigma_0(\alpha_1,\cdots,\alpha_{C_q}) = \sqrt 2 \alpha_2$, which is a nonzero polynomial.
For $k=1$, we have that
\begin{align}
    \varsigma_1(\alpha_1,\cdots,\alpha_{C_q}) &= \sum_{j=2}^{C_q}\sqrt j \cdot \alpha_j\cdot r_{j-1,1} \cdot\mathbb{E}_{x\sim \cN}\left[x\cdot h_{j-2}(x)\cdot \sigma'(x)\right] \\
    & = 2 r_{1,1} \cdot \alpha_2^2  + \sum_{j=3}^{C_q}\sqrt j \cdot \alpha_j\cdot r_{j-1,1} \cdot\mathbb{E}_{x\sim \cN}\left[x\cdot h_{j-2}(x)\cdot \sigma'(x)\right]
\end{align}
which is nonzero (since in the summation from $j=3$ to $C_q$ there would be no term in the form of $\alpha_2^2$) and is linearly independent of $\varsigma_0$ because each term in the summation here has degree exactly $2$.
Now consider for $k\geq 2$, 
\begin{align}
    \varsigma_k(\alpha_1,\cdots,\alpha_{C_q}) &= \sqrt{2(k+1)} k \cdot r_{k,k} \cdot  \alpha_1^{k-1} \alpha_2 \alpha_{k+1} \\
    &\qquad + \sum_{j=k+2}^{C_q} \sqrt j \alpha_j r_{j-1,k}\cdot\mathbb{E}_{x\sim \cN}\left[x\cdot h_{j-k-1}(x)\cdot \big(\sigma'(x)\big)^k\right].
\end{align}
Again, this polynomial is nonzero (since in the summation from $j=k+2$ to $C_q$ there would be no term in the form of $\alpha_1^{k-1} \alpha_2 \alpha_{k+1}$) and is linearly independent of $\varsigma_0,\cdots,\varsigma_{k-1}$ because the highest degree of these polynomials is no larger than $k$, and the degree of each term in $\varsigma_{k}$ is exactly $k+1$.
Thus we have proved the two claims by induction.
Now recall that we are aiming to prove that the RHS of \eqref{eq:psi_2_target} is nonzero.
By our two claims just proved, the RHS of \eqref{eq:psi_2_target} is a linear combination of $C_q - I +2$ linearly independent and nonzero polynomials where at least one of the combination coefficients is nonzero (which is $b_I$).
Thus we obtain that the RHS of \eqref{eq:psi_2_target} is a nonzero polynomial of $(\alpha_1,\cdots,\alpha_{C_q})$ and its zeros form a measure-zero set.
This proves that with probability $1$ over the randomness of $(\alpha_1,\cdots,\alpha_{C_q})$, the high-pass condition holds.

\textbf{Case 2: $s^{\star}=1$.} For this case of $s^{\star}=1$, the proof is almost the same as that for $s^{\star}=2$, where we additionally utilize the fact that a polynomial link function with generative exponent $s^{\star}=1$ cannot be an even polynomial (\Cref{exp:sgm}) and thus there always exists some $I\leq C_q\in\mathbb{N}_+$ such that the information exponent $q^{\star}(p^I)=1$ (see Proposition 5 of \cite{lee2024neural}). With this fact, repeating the above argument gives the desired high-pass property.

\textbf{Verifying the event $\cE$.}
Now we verify that the desired event 
\begin{align}
    \cE=\big\{|\err_{m,l,i}^{(t)}|\leq d^{-10s^{\star}},\forall (m,l,i,t)\in [M]\times[L]\times[n]\times[T]\big\}
\end{align}
holds with probability at least $1-O(d^{-c_0})$ for some constant $c_0>0$ that we specify later.
With \eqref{eq:psi_2_0+} and \eqref{eq:psi_2_0++}, it suffices to look at the error terms $\err_{m,l,i, 1}^{(t)}$, $\err_{m,l,i, 2}^{(t)}$, and $\err_{m,l,i, 3}^{(t)}$ separately.
For $\err_{m,l,i, 1}^{(t)}$, 
\begin{align}
    \left|\err_{m,l,i, 1}^{(t)}\right| \leq \sum_{m'=1}^M |a_{m'}|\cdot \left|\sigma(\langle w_{m',l}^{(t)},z_i\rangle)\right|\cdot\left|\sigma'(\langle w_{m,l}^{(t)},z_i\rangle)\right|\label{eq:psi_2_2}
\end{align}
Note that $\{\langle w_{m,l}^{(t)},z_i\rangle\}_{m\in[M]}$ are standard Gaussians since $\{w_{m,l}^{(t)}\}_{m\in[M]}\subset\SSS^{d-1}$. 
Therefore, with probability at least $1-d^{-c_0}$ for some constant $c_0>0$, we have that $|\langle w_{m,l}^{(t)},z_i\rangle| = \widetilde{O}(1)$.
Meanwhile, since $\sigma$ and $\sigma'$ are both polynomials with constant order and bounded coefficients, and since $M = O(d)$, by taking $a_m = d^{-12s^{\star}}$, we conclude from \eqref{eq:psi_2_2} that
\begin{align}
    \big|\err_{m,l,i, 1}^{(t)}\big|\leq \widetilde{O}(d^{-10s^{\star}}).\label{eq:psi_2_2+}
\end{align}
For the second error term $\err_{m,l,i, 2}^{(t)}$, similarly we have that 
\begin{align}
     \left|\err_{m,l,i, 2}^{(t)}\right| \leq \sum_{m'\in[M]}|a_{m'}|\cdot \left|\sigma(\langle (w^{(t)}_{m',l})^+,z_i\rangle)\right|\cdot\left|\sigma'(\langle (w^{(t)}_{m,l})^+,z_i\rangle)\right|.\label{eq:psi_2_3}
\end{align}
Note that the one-step updated weights satisfy
\begin{align}
    \left|\langle (w^{(t)}_{m,l})^+,z_i\rangle\right| = \left|\langle w_{m,l}^{(t)},z_i\rangle + y_i\cdot \sigma'(\langle w_{m,l}^{(t)},z_i\rangle) + \err_{m,l,i, 1}^{(t)}\right| = \widetilde{O}(1),\label{eq:psi_2_4}
\end{align}
with probability at least $1-d^{-c_0}$ since $y_i=p(\langle \theta^{\star},z_i\rangle)$ and $p$ is also a polynomial of constant degree and coefficients. 
Therefore, with the choices of $a_m$, by \eqref{eq:psi_2_3}, we conclude that
\begin{align}
     \left|\err_{m,l,i, 2}^{(t)}\right| \leq \widetilde{O}(d^{-10s^{\star}}).\label{eq:psi_2_4_+}
\end{align}
Finally, regarding $\err_{m,l,i, 3}^{(t)}$, note that with the same argument as \eqref{eq:psi_2_4}, we know that with probability at least $1-d^{-c_0}$, both $\langle w_{m,l}^{(t)},z_i\rangle + y_i\cdot \sigma'(\langle w_{m,l}^{(t)},z_i\rangle) + \err_{m,l,i, 1}^{(t)}$ and $\langle w_{m,l}^{(t)},z_i\rangle + y_i\cdot \sigma'(\langle w_{m,l}^{(t)},z_i\rangle)$ are $\widetilde{O}(1)$. 
Since $\sigma'$ is a polynomial, it is $\widetilde{O}(1)$-Lipschitz continuous for inputs that are $\widetilde{O}(1)$.
Therefore, combining this with \eqref{eq:psi_2_2+}, we obtain that
\begin{align}
    \big|\err_{m,l,i, 3}^{(t)}\big| \leq |y_i|\cdot \widetilde{O}\big(|\err_{m,l,i, 1}^{(t)}|\big) = \widetilde{O}(d^{-10s^{\star}}).\label{eq:psi_2_4_++}
\end{align}
Finally, combining \eqref{eq:psi_2_2+}, \eqref{eq:psi_2_4_+}, and \eqref{eq:psi_2_4_++}, we obtain that for a given $(m,l,i,t)$, with probability at least $1-O(d^{-c_0})$, it holds that
\begin{align}
    \big|\err_{m,l,i}^{(t)}\big| \le d^{-10s^{\star}}
\end{align}
for all sufficiently large $d$.
Under the parameter choices in \Cref{thm:uniform}, $MLnT=\widetilde{\Theta}(d^{(s^\star+1)/2+(s^\star/2\vee 1)})$.
Finally, taking $c_0$ larger than $(s^\star+1)/2+(s^\star/2\vee 1)$ and applying a union bound, we obtain that
\begin{align}
    \mathrm{Pr}(\cE)\geq 1-MLnT\cdot \widetilde{O}(d^{-c_0}) \geq 1-\widetilde{\Theta}\left(d^{(s^\star+1)/2+(s^\star/2\vee 1)}\right)\cdot \widetilde{O}(d^{-c_0})  \geq 1 - \widetilde{O}(d^{-c_0'})  
\end{align}
for some other constant $c_0'>0$. 
Here we have applied our choice of $(M,L,n,T)$ in \Cref{thm:uniform}. 
Thus we verify the property of the event $\cE$, proving Corollary~\ref{cor:psi_2}.
\end{proof}

\subsubsection{Modified loss for general \texorpdfstring{$s^{\star}\geq 1$}{s geq 1}}\label{app:example_psi_3}
Here we give a specific choice of the activation function $\sigma$ and the loss function $\ell$. 
We mainly focus on the situation where $\mathbb{Q}_y$ has a continuous cumulative distribution function $F_{\mathbb{Q}_y}$ with bounded density $f_{\mathbb{Q}_y}$. 
For the situation where $\mathbb{Q}_{y}$ is a discrete distribution (e.g., classification tasks), we discuss it at the end of this section.
For the activation function $\sigma$, we let $\sigma(x):= (1/\sqrt{s^{\star}})\cdot h_{s^{\star}}(x)$. 
Since then,
\begin{align}
    \hat\psi_s(y) = \mathbb{E}_{\mathbb{Q}}[\psi(y,x)\cdot h_s(x)\given y] = \mathbb{E}_{\mathbb{Q}}[\sigma'(x)\cdot h_s(x)]\cdot \ell'(y) = 0,\quad \forall s<s^{\star}-1. \label{eq:psi_3_0}
\end{align}
Regarding the choice of the loss function $\ell$, we remark that if one chooses a fixed loss function, there always exist instances such that the second assumption in the high-pass condition fails.
To address this issue, we propose to construct a random loss function $\ell$. 
To rule out pathological examples of the underlying distribution $\mathbb{P}$, we make the following assumption on the coefficient function $\zeta_{s^{\star}}$. 
\begin{assumption}\label{asp:fourier}
    We assume that the expansion of $\widetilde{\zeta}_{s^{\star}}:=\zeta_{s^\star}\circ F_{\mathbb{Q}_y}^{-1}:[0,1]\mapsto\mathbb{R}$ on the Fourier basis $\{\varphi_i(x)\}_{i\geq 0}$ of $[0,1]$ has a non-zero coefficient of order at most $D = O(1)$.
\end{assumption}
We then choose the loss function $\ell$ as the following,
    \begin{align}
        \ell'(y) = \sum_{i=0}^D \alpha_i\cdot \varphi_i\circ F_{\mathbb{Q}_y}(y),\quad \alpha_i\sim \mathrm{Unif}([0,1]),\quad\forall 0\leq i\leq D.
    \end{align}
Notice that $F_{\QQ_y}$ can be estimated from data using a one-dimensional density estimator.
Thus here we directly assume the accessibility of the function $F_{\QQ_y}$.
This further gives that
    \begin{align}
        \EE_\QQ[\zeta_{s^\star}(y) \cdot \hat\psi_{s^\star -1} (y)] = \mathbb{E}_{\mathbb{Q}}[\zeta_{s^{\star}}(y)\cdot\ell'(y)] = \sum_{i=0}^D \alpha_i\cdot \mathbb{E}_{\mathrm{Unif}([0,1])}\left[\widetilde{\zeta}_{s^{\star}}(\widetilde{y})\cdot\varphi_i(\widetilde{y})\right],\label{eq:psi_3_1}
    \end{align}
which is a non-zero polynomial of the coefficients $\{\alpha_i\}_{i\leq D}$ due to \Cref{asp:fourier}.

\begin{proof}[Proof of Corollary~\ref{cor:psi_3}]
    To prove \Cref{cor:psi_3}, it suffices to show that (i) \Cref{asp:oracle function} holds, and (ii) the event $\cE$ holds with the desired high probability. 
    In the following, we first verify \Cref{asp:oracle function}, and then check the event $\cE$.

    \textbf{Verifying \Cref{asp:oracle function}.} 
    First, since $\sigma'$ is a polynomial and $\ell'$ is bounded (the Fourier basis is bounded and $D=O(1)$), we know that both \ref{asp:quadratic_integrability} and \ref{asp:polynomial-like} are satisfied. 
    Then, by the discussion before the proof, we know from \eqref{eq:psi_3_0} that the first condition in the high-pass assumption (\ref{asp:high_pass}) is satisfied.
    Furthermore, according to \eqref{eq:psi_3_1} and \Cref{asp:fourier}, $\EE_\QQ[\zeta_{s^\star}(y) \cdot \hat\psi_{s^\star -1} (y)]$ is a nonzero polynomial of the coefficients $\{\alpha_i\}_{i\leq D}$ and thus its zeros form a measure-zero set.
    This means that with probability $1$ over the randomness of $(\alpha_0,\cdots,\alpha_D)$, the second condition in the high-pass assumption is also satisfied. This verifies \Cref{asp:oracle function}.

    \textbf{Verifying the event $\cE$.} 
    Recall our definition in \Cref{exp:psi_3}, the error term is defined as 
    \begin{align}
        \err_{m,l,i}^{(t)} = \left(\ell'(y_i) - \ell'\left(y_i - f(z_i;\{w_{m,l}^{(t)}\}_{m\in[M]}\})\right)\right)\cdot \sigma'(\langle w_{m,l}^{(t)},z_i\rangle).
    \end{align}
    First, since $w_{m,l}^{(t)}\in\SSS^{d-1}$, $\langle w_{m,l}^{(t)},z_i\rangle$ is a standard Gaussian and therefore $|\langle w_{m,l}^{(t)},z_i\rangle| = \widetilde{O}(1)$ with probability at least $1-d^{-c_0}$ for some constant $c_0>0$. 
    Since $\sigma'$ is a polynomial of constant degree, we then obtain that $|\sigma'(\langle w_{m,l}^{(t)},z_i\rangle)| = \widetilde{O}(1)$ with probability at least $1-d^{-c_0}$. 
    Second, consider that the second derivative of the loss function $\ell''(y)$ is given by 
    \begin{align}
        \ell''(y) = \sum_{i=0}^D\alpha_i\cdot \varphi_i'\left(F_{\QQ_y}(y)\right)\cdot f_{\QQ_y}(y),
    \end{align}
    which satisfies $|\ell''(y)| = O(1)$ since the derivative of the Fourier basis is still bounded and the density of $\QQ_y$ is assumed to be bounded.
    Therefore, we have 
    \begin{align}
        \ell'(y_i) - \ell'\left(y_i - f(z_i;\{w_{m,l}^{(t)}\}_{m\in[M]}\})\right) &=O\left(\left|f(z_i;\{w_{m,l}^{(t)}\}_{m\in[M]}\})\right|\right) \\
        &= O\left(\sum_{m=1}^M |a_{m}|\cdot \left|\sigma(\langle w_{m,l}^{(t)},z_i\rangle)\right|\right).
    \end{align}
    Since $\sigma$ is a polynomial of constant degree and $\langle w_{m,l}^{(t)},z_i\rangle$ are all standard Gaussians, we have that $|\sigma(\langle w_{m,l}^{(t)},z_i\rangle)| = \widetilde{O}(1)$ with probability at least $1-O(d^{-c_0})$.
    Now given that $M = O(d)$ and taking $a_m = d^{-12s^{\star}}$, we have that
    \begin{align}
        \ell'(y_i) - \ell'\left(y_i - f(z_i;\{w_{m,l}^{(t)}\}_{m\in[M]}\})\right) &= O(d^{-10 s^{\star}}),
    \end{align}
    with probability at least $1-O(d^{-c_0})$.
    Therefore, for any given $(m,l,i,t)$, with probability at least $1-O(d^{-c_0})$, it holds that
    \begin{align}
        \big|\err_{m,l,i}^{(t)}\big| \le d^{-10s^{\star}}
    \end{align}
    for all sufficiently large $d$.
    Finally, as in the proof of \Cref{cor:psi_2}, we take $c_0$ larger than $(s^\star+1)/2+(s^\star/2\vee 1)$ and apply a union bound to obtain that
    \begin{align}
        \mathrm{Pr}(\cE)\geq 1-MLnT\cdot \widetilde{O}(d^{-c_0}) \geq 1-\widetilde{\Theta}\left(d^{(s^\star+1)/2+(s^\star/2\vee 1)}\right)\cdot \widetilde{O}(d^{-c_0})  \geq 1 - \widetilde{O}(d^{-c_0'})  
    \end{align}
    for some other constant $c_0'>0$. 
    Thus we verify the property of the event $\cE$, proving Corollary~\ref{cor:psi_3}.
\end{proof}

\begin{remark}[Discrete labels]
    For the case of a discrete label $y$ supported on a finite set $\cY$ (e.g., classification tasks), the construction of $\psi$ is more direct.
    In this case, we can still consider an oracle function in the form of $\psi(y,x) = \sigma'(x)\cdot \varphi(y)$ for some function $\varphi(y)$.
    The activation function $\sigma(x) = (1/\sqrt{s^{\star}})\cdot h_{s^{\star}}(x)$ and the function $\varphi(y)$ are given by
    \begin{align}
        \varphi(y)\sim \mathrm{Unif}([0,1]),\quad \forall y\in\cY.\label{eq:psi_3_2}
    \end{align}
    On the one hand, we can directly conclude as in \eqref{eq:psi_3_0} that $\widehat{\psi}_s(y)=0$ for all $s<s^{\star}-1$.
    On the other hand, we have that 
    \begin{align}
         \EE_\QQ[\zeta_{s^\star}(y) \cdot \hat\psi_{s^\star -1} (y)] = \mathbb{E}_{\mathbb{Q}}[\zeta_{s^{\star}}(y)\cdot\varphi(y)] = \sum_{y\in\cY}\QQ_y(y)\cdot \zeta_{s^{\star}}(y)\cdot\varphi(y).
    \end{align}
    By the definition of generative exponent (\Cref{def:generative-exponent}), $\mathbb{E}_{\QQ_y}[\zeta_{s^{\star}}(y)^2]>0$ and thus at least one of $\{\QQ_y(y)\cdot \zeta_{s^{\star}}(y)\}_{y\in\cY}$ is non-zero. 
    Thus under \eqref{eq:psi_3_2}, $\EE_\QQ[\zeta_{s^\star}(y) \cdot \hat\psi_{s^\star -1} (y)]$ is non-zero with probability $1$ over the randomness in $\varphi$.
    Thus we have verified \ref{asp:high_pass}.
    \ref{asp:quadratic_integrability} and \ref{asp:polynomial-like} can be verified in the same way as for the continuous case, and thus \cref{asp:oracle function} is checked.
    Finally, we remark that in the discrete case we do not attempt to derive the oracle function from a loss derivative and thus simply set the error terms $\err$ as zero.
    Thus all the conditions in Theorem~\ref{thm:uniform} hold and Corollary~\ref{cor:psi_3} is proved.
\end{remark}

\section{Proof of the Main Theorem for the Uniform Prior}\label{app:proof_nonsparse}
Now we present the proof for \Cref{thm:uniform}.
For convenience, we introduce the following shorthand.

\begin{definition}[Shorthand for the alignment]\label{def:alignment}
For $\theta\in\SS^{d-1}$, we denote $\rho = \langle\theta,\theta^\star\rangle$ as the alignment between $\theta$ and $\theta^\star$, which also inherits the subscript and superscript of $\theta$ as well, i.e., $\rho_m^{(t)} = \langle\theta_m^{(t)},\theta^\star\rangle$.
\end{definition}

Recall from \Cref{alg:meta} that in each step, given the normalized gradient step $\barg_m^{(t)}=g_m^{(t)} / \|g_m^{(t)}\|_2$, the updated weight parameter is given by
\begin{align}
    \theta_m^{(t+1)} = \frac{\theta_m^{(t)} + \eta \barg_m^{(t)}}{\big\|\theta_m^{(t)} + \eta \barg_m^{(t)}\big\|_2}.
\end{align}
It is clear that the alignment $\rho_m^{(t+1)}=\langle\theta_m^{(t+1)},\theta^\star\rangle$ is determined jointly by the alignment $\rho_m^{(t)}$ of the previous iterate and the alignment $\langle\barg_m^{(t)},\theta^\star\rangle$ of the normalized gradient in the current step.
Therefore, we will first analyze the properties of the normalized gradient $\barg_m^{(t)}$ in \Cref{sec:nonsparse-gradient-step-property}, and then use these properties to analyze the dynamics of $\rho_m^{(t)}$.

\subsection{Properties of the Gradient Step} \label{sec:nonsparse-gradient-step-property}

Recall from \Cref{alg:meta} that each gradient $g_{m,l,i}^{(t)}$ contains an error part $\err_{m,l,i}^{(t)} z_i^{(t)}$.
This error term can be made negligible by setting the weights of the second layer to be sufficiently small, and the moments of $\barg_m^{(t)}$ are determined primarily by $\{\psi(y_i^{(t)},\langle w_{m,l}^{(t)},z_i^{(t)}\rangle) z_i^{(t)}\}$.
Therefore, we first analyze the properties of the normalized gradient $\barg_m^{(t)}$ without the error part, while the error part will be taken into account in \Cref{sec:nonsparse-proof-main-theorem}.
To this end, we introduce the following definition for the gradient without the error term.

\begin{definition}\label{def:elem_setting}
Under \Cref{asp:oracle function}, consider any fixed $\theta\in\SS^{d-1}$ and data $(z_1,y_1),\ldots,(z_n,y_n) \iidfrom \PP_{\theta^\star}$.
For polarization level $\gamma\in(0,1/2)$, define $w_l=(\gamma\theta+\xi_l)/\|\gamma\theta+\xi_l\|_2$ for $l=1,\ldots,L$, where $\xi_1,\ldots\xi_L \iidfrom\unif(\SS^{d-1})$ are independent of $\{(z_i,y_i)\}_{i=1}^n$, and define the corresponding gradient as
\begin{align}
    \hatg = \frac{1}{nL}\sum_{i=1}^n\sum_{l=1}^L \big(\psi(y_i,\langle w_l,z_i\rangle)\cdot  z_i - \hat\psi_1(y_i) \cdot w_l\big).
\end{align}
Moreover, let $v_1=\theta^\star$ and $v_2=(\theta-\rho\theta^\star)/\sqrt{1-\rho^2}$, and further expand them to the following orthonormal basis for $\RR^d$:
\begin{align}
    v_1 =  \theta^\star, v_2= \frac{\theta - \rho  \theta^\star  }{\sqrt{1-\rho^2}}, v_3,\ldots,v_d.  \label{eq:nonsparse-basis}
\end{align}
Corresponding to the above orthonormal basis, define the following nice events for $\epsilon,\tilde\epsilon>0$:
\begin{align}
    \cE(\epsilon)&= \Big\{
        \max_{1 \le i \le d } |\langle \xi_l, v_i  \rangle| < \epsilon, \quad \forall l\in[L]\Big\}; \label{eq:nice-event-cE}
        \\
    \tilde \cE(\tilde \epsilon) &= \left\{
        |\langle \xi_l, \xi_{l'}\rangle| < \tilde\epsilon, \quad \forall l, l'\in[L] \st l\neq l'\right\}.  \label{eq:nice-event-cE-tilde}
\end{align}
\end{definition}

\paragraph{Perturbed weights: near uniformity and orthogonality}
When $d$ is large, the perturbation noise vectors $\xi_1,\ldots,\xi_L$ are nearly orthogonal to each other and to any fixed direction with high probability.
Also, we will use a small polarization level $\gamma$ (roughly $d^{-1/4}$), which implies that the perturbed weights $w_1,\ldots,w_L$ are nearly orthogonal as well, which is formalized in the lemma below.
\begin{lemma}[Perturbed weights on the nice event]
    \label{lem:polarized-weight-nice-event}
Under the setting introduced in \Cref{def:elem_setting}, on the nice event $\cE(\epsilon)$, it holds for all $l\in[L]$ that
\begin{gather}
    |\langle w_{l}, \theta^\star\rangle| \le 2 (\gamma |\langle \theta, \theta^\star\rangle|  + \epsilon), \quad |\langle w_{l}, v_{2}\rangle| \le 2 \Big(\gamma\sqrt{1-\langle \theta, \theta^\star\rangle^2} + \epsilon\Big), \\
    |\langle w_{l}, v_{i}\rangle| \le 2 \cdot \epsilon, \quad 3 \le i \le d.
\end{gather}
Additionally, on the nice event $\cE(\epsilon)\cap\widetilde\cE(\tilde\epsilon)$, $\dotp{w_{l}}{w_{l'}} \le 4 (\gamma^2 + 2\gamma \epsilon+ \tilde \epsilon)$ for any $l \neq l'$.
\end{lemma}

\begin{proof}[Proof of \Cref{lem:polarized-weight-nice-event}]
    See \Cref{proof:polarized-weight-nice-event}.
\end{proof}

\paragraph{Characterizing error-free gradient}
Next, we characterize $\dotp{\hatg}{\theta^\star} / \norm{\hatg}_2$ in two steps:
\begin{enumerate}
\item In \Cref{prop:nonsparse-first-moment}, we analyze the expected magnitude of $\dotp{\hatg}{\theta^\star}$ and $P_{\theta^\star}^\perp \hatg$ under $\PP_{\theta^\star}$.
\item  In \Cref{prop:nonsparse-sample-concentration}, we establish the concentration of $\hatg$ along each direction.
\end{enumerate}
Both propositions are established under the setting of \Cref{def:elem_setting}, and the proof of these propositions is deferred to \Cref{sec:nonsparse-proof-key-results}.
A key ingredient in the proof is the following lemma for the decomposition of the first moment of the gradient (see \Cref{proof:first-moment-decomp} for a proof).
\begin{replemma}{lem:first-moment-decomp}[Decomposition of the gradient]
Under the setting of \Cref{def:elem_setting}, it holds that
\begin{align}
     \EE_{\PP_{\theta^\star}}[\hatg] &=
     \sum_{s\ge s^\star} \EE_{\QQ}[\zeta_s(y)\cdot \hat\psi_{s-1}(y)] \cdot \frac{\sqrt{s}}{L} \sum_{l=1}^L  \langle  w_l, \theta^\star\rangle^{s-1} \cdot \theta^\star \\
     &\qquad +\sum_{s\ge s^\star} \EE_{\QQ}[\zeta_s(y)\cdot \hat\psi_{s+1}(y)] \cdot \frac{\sqrt{s+1}}{L} \sum_{l=1}^L  \langle  w_l, \theta^\star\rangle^{s}\cdot w_l  .
\end{align}
\end{replemma}
We see from the above expression that $\EE_{\PP_{\theta^\star}}[\hatg]$ is a linear combination of $\theta^\star$ and $w_l$.
Since $\langle w_l,\theta^\star\rangle$ is small on the nice events by \Cref{lem:polarized-weight-nice-event}, the dominant term in the signal direction corresponds to the coefficient $\langle w_l,\theta^\star\rangle^{s^\star-1}$.
This is indeed a consequence of the high-pass assumption in \ref{asp:high_pass} and the recursive structure of the Hermite polynomials.
The following proposition provides a quantitative characterization for this.

\begin{proposition}[Alignment of expected gradient] \label{prop:nonsparse-first-moment}
Under the setting of \Cref{def:elem_setting}, let $T=\Theta(\log d)$, and suppose that $1-\Pr(\cE(\epsilon)) =O( d^{-s^\star/2 })$ for $\gamma = o(1)$ and $\epsilon = o(1)$.
Set
\begin{align}\label{eq:alignment_expected_gradient_condition_L}
    L \gtrsim \Big(\big((\epsilon\vee \gamma)^{s^\star-1}\cdot d^{s^\star/2}\big) \vee d\Big)\cdot\log d.
\end{align}
Then for any constant $c>0$, there exists a $\{\xi_l\}_{l\in [L]}$-measurable event $\cE_{1}$ with $\Pr(\cE_1) \ge 1- d^{-c} (MT)^{-1}$, such that on the event $\cE_{1}\cap \cE (\epsilon)$, the following holds:
\begin{itemize}
\item The dominant term in the alignment of $\EE_{\PP_{\theta^\star}}[\hatg]$ satisfies
\begin{align}
    \sqrt{s^\star}\cdot \EE_{\QQ} [ \zeta_{s^\star}(y) \cdot \hatpsi_{s^\star -1}(y)]\cdot  \EE_{w_l}[\dotp{w_l}{\theta^\star}^{s^\star-1}] \simeq  \begin{cases}
    \gamma \rho  \cdot(\gamma |\rho| + d^{-1/2 })^{s^\star - 2} & \text{if } s^\star \text{ is even};\\
    (\gamma |\rho| + d^{-1/2 })^{s^\star - 1} & \text{if } s^\star \text{ is odd}.
    \end{cases}
\end{align}

\item The residual error can be bounded as
\begin{align}
    \big|\dotp{\EE_{\PP_{\theta^\star}}[\hatg]}{\theta^\star} -\sqrt{s^\star}\cdot \EE_{\QQ} [ \zeta_{s^\star}(y) \cdot \hatpsi_{s^\star -1}(y)]\cdot  \EE_{w_l}[\dotp{w_l}{\theta^\star}^{s^\star-1}]\big| &\lesssim  (\gamma|\rho| +  d^{-1/2})^{s^\star}.
\end{align}

\item Also, $|\norm{\EE_{\PP_{\theta^\star}}[\hatg]}_2  -   |\dotp{\EE_{\PP_{\theta^\star} }[\hatg]}{\theta^\star}|| \lesssim  (\gamma|\rho| +  d^{-1/2})^{s^\star}$.
\end{itemize}
\end{proposition}

\begin{proof}[Proof of \Cref{prop:nonsparse-first-moment}]
    See \Cref{proof:nonsparse-first-moment}.
\end{proof}

Complementary to the above result on the expected gradient, the next proposition provides a concentration result for $\hatg$.

\begin{proposition}[Concentration of the gradient] \label{prop:nonsparse-sample-concentration}
Under the settings in \Cref{def:elem_setting}, suppose $\epsilon^2\le \tilde \epsilon\ll 1$ and $2\gamma \epsilon \le     \tilde\epsilon$.
Also, suppose that sample size $n$ satisfies
\begin{align}
    n = \Omega\Big(\big( (\gamma^2 + \tilde\epsilon) ^{s^\star - 1 } +L^{-1}\big)^{-1}\cdot\log(d)^{2C_p+2} \Big),
\end{align}
where $C_p$ is the constant defined in \ref{asp:polynomial-like}.
Then there exists a $\{(z_i,y_i)\}_{i\in [n]}$-measurable event $\cE_{2}$ with $\Pr({\cE_2}) \ge 1-  d^{-c} \cdot (MT)^{-1}$ such that the following holds on $\cE_{2}\cap \cE (\epsilon)\cap \tilde \cE(\tilde \epsilon)$:
\begin{itemize}
\item First, $\langle\hat g,\theta^\star\rangle$ concentrates around its expectation as
\begin{align}
    \big|\dotp{\hatg}{\theta^\star} - \dotp{\EE_{\PP_{\theta^\star}}[\hatg]}{\theta^\star}\big| \lesssim  \sqrt{\frac{((\gamma^2 + \tilde\epsilon)^{s^\star - 1} +L^{-1})\cdot \log (d)}{n}}.
\end{align}
\item Second, $\hat g$ also concentrates around its expectation as
\begin{align}
    \big\|\hatg - \EE_{\PP_{\theta^\star}}[\hatg ]\big\|_2
    \lesssim  \sqrt{\frac{((\gamma^2 + \tilde\epsilon)^{s^\star - 1} +L^{-1})\cdot d  \log (d)}{n}}.
\end{align}
\end{itemize}
\end{proposition}
\begin{proof}[Proof of \Cref{prop:nonsparse-sample-concentration}]
See \Cref{proof:nonsparse-sample-concentration}.
\end{proof}

\subsection{Proof of the Main Theorem for Uniform Prior} \label{sec:nonsparse-proof-main-theorem}
Now we are ready to present the proof of \Cref{thm:uniform}.
\begin{proof}[Proof of \Cref{thm:uniform}]
To start with, we clarify the event we will work with by verifying that our configuration is compatible with the conditions in \Cref{prop:nonsparse-first-moment} and \Cref{prop:nonsparse-sample-concentration}. We denote
\begin{align}
\hatg_m^\t = \frac{1}{nL} \sum_{i=1}^n \sum_{l=1}^L \big(\psi(y_i^\t, \langle w_{m,l}^\t, z_i^\t\rangle)\cdot z_i^\t  - \widehat{\psi}_1(y_i^\t) \cdot w_{m,l}^\t\big).
\end{align}
Let $\theta_m^\t$ be the $m$-th neuron at the $t$-th step, and denote $\rho_m^\t =\dotp{\theta_m^\t}{\theta^\star}$. We choose the corresponding orthonormal basis as
\begin{align}
    v_{m,1}^\t = \theta^\star, v_{m,2}^\t = \frac{\theta_m^\t - \rho_m^\t \theta^\star}{\sqrt{1-{\rho_m^\t}^2}}, v_{m,3}^\t,\ldots,v_{m,d}^\t .
\end{align}
Also, we define the associated good events as
\begin{align}
    \cE_m^\t(\epsilon) &= \big\{ |\langle \xi_{m,l}^\t, \theta^\star\rangle| < \epsilon, \quad \max_{2 \le i \le d } |\langle \xi_{m,l}^\t, v_{m,i}^\t  \rangle| < \epsilon, \quad \forall l\in[L]\big\}; \\
    \tilde \cE_m^\t(\tilde \epsilon) &= \big\{ |\langle \xi_{m,l}^\t, \xi_{m,l'}^\t\rangle| < \tilde\epsilon, \quad \forall l, l'\in[L] \st l\neq l'\big\}.
\end{align}
With these notations, we see that $\hatg_m^\t$, the mini-batch data $\{(y_i^\t, z_i^\t)\}_{i\in[n]}$, and the good events $\cE_m^\t(\epsilon)$ and $\tilde\cE_m^\t(\tepsilon)$ together form an instance in \Cref{def:elem_setting}.

Recall that $\gamma = (d^{-1}\cdot \log d)^{1/4}=o(1)$ is fixed throughout the algorithm.
We set $\epsilon = d^{-1/2}\log d$. \Cref{lem:sphere_coordinate_tail} implies that for any $t$ and $m$, it holds that
    \begin{align}
        \Pr \Big( {\cE_m^\t(\epsilon)}^c\Big) \le Ld\cdot \big( \exp(-d/16) + d^{-\log d/4}\big)
    \end{align}
For the number of perturbations, we choose $L = \Omega\big( d^{(s^\star +1)/2} \log d\big)$, which is at most polynomial in $d$.
Therefore $\Pr\Big( {\cE_m^\t(\epsilon)}^c\Big)$ decays faster than any constant-degree polynomial in $d$.
For sufficiently large $d$, it holds that $\Pr\Big( {\cE_m^\t(\epsilon)}^c \Big) = O(d^{-s^\star/2})$.
On the other hand, the choice of $L$ ensures that
\begin{align}
L \gtrsim \big(d^{(s^\star+1)/4} \cdot \log(d)^{s^\star -1})\vee d\cdot \log d \simeq \big( (\epsilon \vee \gamma)^{s^\star -1}\cdot d^{s^\star/2}\vee d \big)\cdot \log d.
\end{align}
So far, we see that all the conditions in \Cref{prop:nonsparse-first-moment} are satisfied and we denote the associated event therein as $\cE_{m,1}^\t$.

Next, we verify the conditions in \Cref{prop:nonsparse-sample-concentration}. We choose $\tilde\epsilon =  \sqrt{4\big(c+\log_d (MTL^2)\big)\cdot d^{-1}\log d}$, then it holds by \Cref{lem:sphere_coordinate_tail} that
    \begin{align}
         \Pr\Big( \tilde \cE_m^\t(\tilde \epsilon)^c \Big) \le  L^2\cdot \big( \exp(-d/16) + d^{-(c+\log_d(MTL^2))}\big)\lesssim d^{-c} /MT.
    \end{align}
Additionally, we see that both $\epsilon^2 \le \tepsilon \ll 1$ and $2\gamma \epsilon\le  \tepsilon$ are satisfied for sufficiently large $d$.
It is easily verified that our choice of $L = \Omega\big( d^{(s^\star +1)/2 }\vee (d\log d)\big)$ clearly meets the condition that
    \begin{align}
        L \gtrsim d^{(s^\star-1)/2}\gtrsim   (\gamma^2 + \tepsilon)^{-(s^\star-1)}.
    \end{align}
Therefore, plugging in $\gamma,\tepsilon$ and the lower bound on $L$, we have that
\begin{align}
    \frac{\log(d)^{2C_p+2}  }{ \big((\gamma^2 + \tilde\epsilon)^{s^\star - 1} +L^{-1}\big)} \lesssim  \log(d)^{2C_p+2}  \cdot d^{(s^\star -1)/2}
\end{align}
Since we choose  $n\gtrsim (d\log d)^{s^\star/2}\vee d\log d$, it follows from the last inequality that
\begin{align}
     n\gtrsim  \frac{\log(d)^{2C_p+2}  }{ \big((\gamma^2 + \tilde\epsilon)^{s^\star - 1} +L^{-1}\big)}.
\end{align}
Hence, all the conditions in \Cref{prop:nonsparse-sample-concentration} are satisfied and we denote the associated event therein as $\cE_{m,2}^\t$.

Noting that the gradient in \Cref{def:elem_setting} does not include the error term $\err_{m,l,i}^\t$, we need an upper bound on $\norm{z_i^\t}_2$ for the purpose of controlling $\norm{\err_{m,l,i}^\t \cdot z_i^\t}_2$. We define
    \begin{align}
        \cE_{m,3}^\t = \left\{ \max_{i\in[n]} \norm{z_i^\t}_2 \le \sqrt{d} \right\}.
    \end{align}
By standard Bernstein's inequality, we have that $\Pr\Big({\cE_{m,3}^\t}^c\Big)\le Ld \cdot \exp\{-d/8\} = O( \exp\{-C' d\})$ for some $C'>0$.
Thanks to the event $\cE_{m,3}^\t$, we always have for any $v\in\SS^{d-1}$ that
    \begin{align}
        \big| \dotp{g_m^\t}{v} - \dotp{\hatg_m^\t}{v} \big| \vee \big|\norm{g_m^\t}_2 - \norm{\hatg_m^\t}_2\big| & \le  \norm{g_m^\t -\hatg_m^\t}_2 \\
        &\le d^{1/2} \cdot \max_{l,i} |\err_{m,l,i}^\t|  \\& \le d^{-9s^\star}. \label{eq:err-bound}
    \end{align}
To put things together, we work on the following event:
    \begin{align}
        \cE =\bigcap_{m=1}^M \bigcap_{t=1}^T \big(\cE_{m}^\t (\epsilon)\cap \tilde\cE_m^\t(\tepsilon)\cap \cE_{m,1}^\t \cap \cE_{m,2}^\t\cap \cE_{m,3}^\t\big) ,
    \end{align}
which satisfies $\Pr(\cE) \ge 1- O(d^{-c})$ for some $c>0$ by the union bound argument.

In the sequel, we restrict our attention to neurons that have at least $d^{-1/2}/2$ alignment, i.e., the indices $m$ such that $|\dotp{\theta_m^{(0)}}{\theta^\star}|\ge d^{-1/2}/2$.
From now on, we drop the neuron index $m$ and the iteration index $\t$ in the following analysis for simplicity. The updated weight parameter is denoted as $\theta'$, and the alignment after the update is denoted as $\rho' = \dotp{\theta'}{\theta^\star}$.
    By rotational invariance of the random initialization, there exists an absolute constant $p_0>0$ such that for all sufficiently large $d$,
    \begin{align}
        \Pr\big(|\dotp{\theta_m^{(0)}}{\theta^\star}|\ge d^{-1/2}/2\big)\ge p_0.
    \end{align}
    Since $M=C_M\log d$ with $C_M$ sufficiently large, a Chernoff bound implies that with probability at least $1-O(d^{-c})$ for some constant $c>0$, the number of neurons with initial alignment $|\rho| \ge d^{-1/2} /2$ is at least $\Omega(M)$.
    For convenience, define
    \begin{align}
        \kappa \defeq \frac{n}{(d\log d)^{s^\star/2} \lor d \log d} \cdot (\log d)^{-1} = \Omega(1).
    \end{align}
    Recall that we defined $\hatg$ as the error-free gradient.
    Under the preceding configuration, \Cref{prop:nonsparse-sample-concentration} and Eq.~\eqref{eq:err-bound} together imply that the error of $\langle g, \theta^\star\rangle$ can be bounded by
    \begin{align}
        \big| \dotp{g}{\theta^\star} - \EE_{\PP_{\theta^\star}}[\dotp{\hatg}{\theta^\star}]\big|
        &\lesssim \sqrt{\frac{((\gamma^2 + \tilde\epsilon)^{s^\star - 1} + L^{-1}) \cdot \log (d) }{n}}  + \big| \dotp{\hatg}{\theta^\star} - \dotp{g}{\theta^\star } \big| \\
        &\lesssim \sqrt{\frac{(d^{-1}\log d)^{(s^\star - 1)/2} \cdot \log (d) }{n}} + d^{-9s^\star}\\
        &\simeq\begin{cases}
            d^{-(2s^\star - 1)/4} \cdot (\log d)^{-1/4} \cdot \kappa^{-1/2} & \text{if $s^\star \ge 2$},\\
            d^{-1/2} \cdot (\log d)^{-1/2}\cdot \kappa^{-1/2} & \text{if $s^\star = 1$}.
        \end{cases}
        \label{eq:alignment-update-1}
    \end{align}
    On the other hand, we have by \Cref{prop:nonsparse-first-moment} that
    \begin{align}
        \big|\EE_{\PP_{\theta^\star}}[\dotp{\hatg}{\theta^\star}]\big|& \simeq (|\rho|\gamma)^{\ind\{s^\star \text{ is even}\}} \cdot(|\rho|\gamma +d^{-1/2})^{s^\star - 1 - \ind\{s^\star \text{ is even}\}}
        \\&\ge |\rho|\gamma (|\rho|\gamma + d^{-1/2})^{s^\star - 2}
        \\ &\ge d^{-(2s^\star - 1)/4} \cdot (\log d)^{1/4},
    \end{align}
    where the last line holds since we focus on the neurons with $|\rho| \gtrsim d^{-1/2}$ and $\gamma = d^{-1/4}(\log d)^{1/4}$.
    Therefore, when $\kappa$ is sufficiently large, the fluctuations are strictly bounded by half of the signal strength.
    Thus, we have
    \begin{align}
        |\langle g, \theta^\star\rangle|
        &\ge \frac 1 2 \cdot |\langle \EE_{\PP_{\theta^\star}}[\hatg], \theta^\star\rangle|.
    \end{align}
    For the norm of $g$, we have by the triangle inequality that
    \begin{align}
        \norm{g}_2 &\le  \big|\norm{g}_2 - \norm{\hatg}_2 \big| + \big|\norm{\hatg}_2 - \norm{\EE_{\PP_{\theta^\star}}[\hatg]}_2\big|\\&\qquad + \big|\norm{\EE_{\PP_{\theta^\star}}[\hatg]}_2 - |\EE_{\PP_{\theta^\star}}[\dotp{\hatg}{\theta^\star}]|\big| + \big|\EE_{\PP_{\theta^\star}}[\dotp{\hatg}{\theta^\star}]\big|\\
        &\le d^{- 9s^\star} + \norm{ \hatg- \EE_{\PP_{\theta^\star}}[\hatg]}_2 + \big|\norm{\EE_{\PP_{\theta^\star}}[ \hatg]}_2  - |\dotp{\EE_{\PP_{\theta^\star}}[\hatg]}{\theta^\star}| \big| + | \dotp{\EE_{\PP_{\theta^\star}} [\hatg]}{\theta^\star}|
        \\
        &\lesssim  \big|\EE_{\PP_{\theta^\star}}[\dotp{\hatg}{\theta^\star}]\big|+(\gamma|\rho| +d^{-1/2})^{s^\star} + \sqrt{\frac{(d^{-1}\log d)^{(s^\star -1)/2}\cdot d \log d}{n}},  \label{eq:alignment-update-1.4}
    \end{align}
    where in the second inequality, we use Eq.~\eqref{eq:err-bound} and the fact that $\big|\norm{\hatg} -  \bignorm{\EE_{\PP_{\theta^{\star}}} [\hatg]}\big| \le  \bignorm{ \hatg - \EE_{\PP_{\theta ^\star}}[\hatg]}$. The last inequality is deduced by combining \Cref{prop:nonsparse-first-moment} and the fact that $d^{-9s^\star} \ll  (|\rho|\gamma +d^{-1/2})^{s^\star}$. For the leading term, it holds by \Cref{prop:nonsparse-first-moment} that
    \begin{align}
        |\langle \EE_{\PP_{\theta^\star}}[\hatg], \theta^\star\rangle| \simeq \begin{cases}
            (|\rho|\gamma + d^{-1/2})^{s^\star - 1} & \text{if $s^\star$ is odd},\\
            |\rho| \gamma (|\rho|\gamma + d^{-1/2})^{s^\star - 2} & \text{if $s^\star$ is even}. \label{eq:alignment-update-1.5}
        \end{cases}
    \end{align}
    Recall that the alignment admits the following iterative update rule:
    \begin{align}
        |\langle\theta',\theta^\star\rangle| = \left|\left\langle \frac{\theta + \eta g / \norm{g}_2}{\norm{\theta + \eta g / \norm{g}_2}_2}, \theta^\star\right\rangle \right| \geq \frac{|\langle g / \norm{g}_2, \theta^\star\rangle| - \eta^{-1} |\langle\theta,\theta^\star\rangle|}{1 + \eta^{-1}}. \label{eq:alignment-update-1.6}
    \end{align}
    In the following, let $\Delta$ be the quantity defined in \Cref{thm:uniform}, and define
    \[
        \rho^\star = d^{-1/4}(\log d)^{1/4}.
    \]
    For the fixed neuron $m$, we first consider the case $s^\star\ge 2$ and introduce two stopping times. The first one marks the end of the initial weak-alignment stage:
    \[
        T_{1,m}=\min\{t\ge 0: |\rho_m^\t|\ge \rho^\star\}.
    \]
    The second one marks the first time the alignment reaches a constant level:
    \[
        T_{2,m}=\min\{t\ge T_{1,m}: |\rho_m^\t|\ge c_{\mathrm w}\},
    \]
    where $c_{\mathrm w}>0$ is a sufficiently small absolute constant to be chosen below.

    \paragraph{Stage \RNum{1}: reaching $\rho^\star$}
    If $|\rho_m^{(0)}|\ge \rho^\star$, then $T_{1,m}=0$ and there is nothing to prove in this stage.
    Otherwise, $|\rho_m^{(0)}|<\rho^\star$.
    Suppose first that $s^\star\ge 3$ is odd and $|\rho|\le \rho^\star$. Combining Eq.~\eqref{eq:alignment-update-1}, Eq.~\eqref{eq:alignment-update-1.4}, and Eq.~\eqref{eq:alignment-update-1.5}, we have
    \begin{align}
        |\langle g, \theta^\star\rangle|
        &\gtrsim (|\rho| \cdot (d\log d)^{-1/4} + d^{-1/2})^{s^\star-1}, \\
        \norm{g}_2 &\lesssim (|\rho|\cdot (d\log d)^{-1/4}+ d^{-1/2})^{s^\star-1} + d^{-(2s^\star - 3)/4} \cdot (\log d)^{-1/4} \cdot \kappa^{-1/2}.
    \end{align}
    Therefore,
    \begin{align}
        \frac{|\langle g, \theta^\star\rangle|}{\norm{g}_2}
        &\gtrsim
        \frac{(|\rho|/\rho^\star + 1)^{s^\star-1}}{(|\rho|/\rho^\star + 1)^{s^\star-1} +  \kappa^{-1/2}/\rho^\star}.
        \label{eq:alignment-update-2}
    \end{align}
    Since $|\rho|\le\rho^\star$ and $\kappa$ is a sufficiently large constant, Eq.~\eqref{eq:alignment-update-1.6} implies that $|\rho'|\ge \rho^\star$. Hence, for odd $s^\star\ge3$, we have $T_{1,m}\le 1$.

    Now suppose that $s^\star$ is even. Before $T_{1,m}$, the same estimates yield
    \begin{align}
        \frac{|\langle g, \theta^\star\rangle|}{\norm{g}_2}
        &\gtrsim
        \frac{|\rho|\cdot (|\rho|/\rho^\star + 1)^{s^\star-2}}{|\rho|\cdot (|\rho|/\rho^\star + 1)^{s^\star-2} +  \kappa^{-1/2}}.
        \label{eq:alignment-update-3}
    \end{align}
    Thus, by taking $\kappa$ to be a sufficiently large constant, Eq.~\eqref{eq:alignment-update-1.6} gives $|\rho'|\ge 2|\rho|$ whenever $|\rho|<\rho^\star$. If $T_{1,m}=0$, the claim is trivial. Otherwise,
    \begin{align}
        \rho^\star \ge |\rho_m^{(T_{1,m}-1)}| \ge  2^{T_{1,m}-1} \cdot d^{-1/2}/2,
    \end{align}
    and therefore $T_{1,m}\le \frac{\log(d\log d)}{4\log 2}+2=O(\log d)$. Combining the odd and even cases, $T_{1,m}=O(\log d)$ for every $s^\star\ge2$.

    \paragraph{Stage \RNum{2}: reaching constant alignment}
    We next show that $T_{2,m}-T_{1,m}=O(\log d)$. For $t\in[T_{1,m},T_{2,m}-1]$, define $r_m^\t=|\rho_m^\t|/\rho^\star$. Then $1\le r_m^\t<c_{\mathrm w}/\rho^\star$.
    In this regime, $\gamma|\rho_m^\t|=r_m^\t d^{-1/2}(\log d)^{1/2}$ dominates the smoothing term up to logarithmic factors. Combining Eq.~\eqref{eq:alignment-update-1}, Eq.~\eqref{eq:alignment-update-1.4}, and Eq.~\eqref{eq:alignment-update-1.5} again, and using the weaker lower bound obtained by dropping these logarithmic factors, there exists an absolute constant $C_1>0$ such that
    \begin{align}
        \frac{|\dotp{g_m^\t}{\theta^\star}|}{\norm{g_m^\t}_2}
        &\ge C_1\cdot \frac{(r_m^\t)^{s^\star-1}}{(r_m^\t)^{s^\star-1}+\kappa^{-1/2}/\rho^\star} \\
        &\ge C_1\cdot \left(1\wedge \sqrt\kappa\,\rho^\star (r_m^\t)^{s^\star-1}\right).
    \end{align}
    Applying Eq.~\eqref{eq:alignment-update-1.6} and dividing by $\rho^\star$ gives
    \begin{align}
        r_m^{(t+1)}
        &\ge C_2\cdot \left((\rho^\star)^{-1}\wedge \sqrt\kappa (r_m^\t)^{s^\star-1}\right)-\frac{1}{3}r_m^\t
    \end{align}
    for another absolute constant $C_2>0$, where we used $\eta>2$. Choose $\kappa$ sufficiently large and then choose $c_{\mathrm w}>0$ sufficiently small so that for all $1\le r<c_{\mathrm w}/\rho^\star$,
    \[
        C_2\cdot \left((\rho^\star)^{-1}\wedge \sqrt\kappa r^{s^\star-1}\right)-\frac{1}{3}r \ge 2r.
    \]
    It follows that $r_m^{(t+1)}\ge 2r_m^\t$ for all $t\in[T_{1,m},T_{2,m}-1]$. Therefore, if $T_{2,m}>T_{1,m}$,
    \begin{align}
        2^{T_{2,m}-T_{1,m}-1}\le r_m^{(T_{2,m}-1)} < c_{\mathrm w}/\rho^\star,
    \end{align}
    which implies $T_{2,m}-T_{1,m}\le \lceil \log_2(c_{\mathrm w}/\rho^\star)\rceil+1=O(\log d)$.

    \paragraph{Strong alignment for $s^\star\ge2$}
    For all $t\ge T_{2,m}$ such that $|\rho_m^\t|\ge c_{\mathrm w}$, the estimates above imply
    \begin{align}
        |\langle g_m^\t,\theta^\star\rangle|
        &\ge \gamma^{s^\star-1}\cdot\left(B|\rho_m^\t|^{s^\star-1}-O(\Delta)\right),\\
        \norm{g_m^\t}_2
        &\le \gamma^{s^\star-1}\cdot\left(B|\rho_m^\t|^{s^\star-1}+O(\Delta)\right),
    \end{align}
    where $B=\Omega(1)$ may depend on $\rho_m^\t$ and the perturbations but is bounded away from zero on the constant-alignment region. Thus, for sufficiently large $d$,
    \begin{align}
        A_m^\t\defeq \frac{|\dotp{g_m^\t}{\theta^\star}|}{\norm{g_m^\t}_2}\ge 1-C_3\Delta.
    \end{align}
    This also preserves the constant-alignment region: by Eq.~\eqref{eq:alignment-update-1.6}, $|\rho_m^{(t+1)}|\ge c_{\mathrm w}$ after decreasing $c_{\mathrm w}$ if necessary.

    Let $G_m^\t=g_m^\t/\norm{g_m^\t}_2$. Using $A_m^\t\ge 1-C_3\Delta$, the triangle inequality, and the Pythagorean theorem, we obtain
    \begin{align}
        |\rho_m^{(t+1)}|^2
        &\ge
        \frac{(1-C_3\Delta-\eta^{-1}|\rho_m^\t|)^2}
        {(1-C_3\Delta-\eta^{-1}|\rho_m^\t|)^2+
        (\sqrt{1-(1-C_3\Delta)^2}+\eta^{-1}\sqrt{1-(\rho_m^\t)^2})^2}\\
        &\ge
        \frac{(1-\eta^{-1})^2}
        {(1-\eta^{-1})^2+\eta^{-2}(1-(\rho_m^\t)^2)}
        -O(\sqrt\Delta).
    \end{align}
    With $\tau=\eta^{-2}/(1-\eta^{-1})^2<1$, this implies
    \begin{align}
        1-|\rho_m^{(t+1)}|^2
        \le \tau\big(1-|\rho_m^\t|^2\big)+O(\sqrt\Delta).
        \label{eq:strong-alignment-recursion}
    \end{align}
    Iterating \eqref{eq:strong-alignment-recursion}, for every $\ell\ge0$,
    \begin{align}
        1-|\rho_m^{(T_{2,m}+\ell)}|^2
        \le \tau^\ell + O\big((1-\tau)^{-1}\sqrt\Delta\big).
    \end{align}
    Hence, after $\ell=O(\log(\Delta^{-1})/\log(\tau^{-1}))$ additional steps, it holds that $|\rho_m^{(T_{2,m}+\ell)}|\ge 1-O(\sqrt\Delta)$.

    \paragraph{Case $s^\star=1$}
    When $s^\star=1$, no weak-alignment phase is needed. The preceding estimates give, uniformly over the alignment level,
    \begin{align}
        |\langle \EE_{\PP_{\theta^\star}}[\hatg], \theta^\star\rangle| = B = \Omega(1),
    \end{align}
    and hence, using $n=\kappa d(\log d)^2$,
    \begin{align}
        \frac{|\langle g_m^\t,\theta^\star\rangle|}{\norm{g_m^\t}_2}
        \ge 1-O((\log d)^{-1/2}\kappa^{-1/2})
        \ge 1-O(\Delta),
    \end{align}
    where now $\Delta=(\log d)^{-1/2}$. Therefore, the same strong-alignment recursion \eqref{eq:strong-alignment-recursion} applies from $t=0$, and after $O(\log(\Delta^{-1})/\log(\tau^{-1}))$ steps we again obtain $|\rho_m^\t|\ge 1-O(\sqrt\Delta)$.

    Combining the stages above, every neuron with $|\rho_m^{(0)}|\ge d^{-1/2}/2$ reaches alignment at least $1-O(\sqrt\Delta)$ within $O(\log d+\log(\Delta^{-1})/\log(\tau^{-1}))$ steps. Since there are at least $\Omega(M)$ such neurons, the proof is complete.
\end{proof}

\subsection{Proof of Key Results} \label{sec:nonsparse-proof-key-results}

\subsubsection{Proof of \texorpdfstring{\Cref{lem:polarized-weight-nice-event}}{the polarized-weight lemma}}\label{proof:polarized-weight-nice-event}
\begin{proof}[Proof of \Cref{lem:polarized-weight-nice-event}]
For each $l\in[L]$, since both $\theta$ and $\xi_l$ are unit vectors and $\gamma\in(0,1/2)$,
\begin{align}
    \norm{\gamma \theta+\xi_l}_2\ge \norm{\xi_l}_2 - \gamma \norm{\theta}_2
    = 1-\gamma \ge 1/2. \label{eq:norm-perturbation-lower-bound}
\end{align}
Hence, we can further control $\langle w_l,\theta^\star\rangle$ as follows:
\begin{align}
    |\langle w_{l}, \theta^\star\rangle| = \frac{|\gamma \langle \theta, \theta^\star\rangle + \langle \xi_{l}, \theta^\star\rangle|}{\norm{\gamma \theta + \xi_{l}}_2}
    \leq 2 (\gamma |\rho|  + |\dotp{\xi_{l}}{\theta^\star}|).
\end{align}
Then on the event $\cE(\epsilon)$ where $\max_{l\in[L]}| \dotp{\xi_l}{\theta^\star}|\le \epsilon$, we conclude for all $l\in[L]$ that
\begin{align}
    |\dotp{w_l }{\theta^\star}| &\le  2(\gamma|\rho|+\epsilon).
\end{align}

Next, for the second direction $v_{2} = (\theta - \rho \theta^\star) / \sqrt{1-\rho^2}$, we have $\langle\theta,v_2\rangle=\sqrt{1-\rho^2}$.
Consequently, it follows from \eqref{eq:norm-perturbation-lower-bound} that
\begin{align}
    |\langle w_{l}, v_{2}\rangle| &= \frac{|\gamma \langle \theta, v_{2}\rangle + \langle \xi_{l}, v_{2}\rangle|}{\norm{\gamma \theta + \xi_{l}}_2}
    \le 2\big( |\gamma \langle \theta, v_{2}\rangle| + |\langle \xi_{l}, v_{2}\rangle|\big)
    \le 2  ( \gamma \sqrt{1-\rho^2} + \epsilon),
\end{align}
where the second inequality holds on the event $\cE(\epsilon)$.
For the remaining directions $v_3,\ldots,v_d$, we again apply \eqref{eq:norm-perturbation-lower-bound} to obtain that on the event $\cE(\epsilon)$,
\begin{align}
    |\langle w_{l}, v_{i}\rangle| &= \frac{|\gamma \langle \theta, v_{i}\rangle + \langle \xi_{l}, v_{i}\rangle|}{\norm{\gamma \theta + \xi_{l}}_2}
    \le 2 \big(|\gamma \langle \theta, v_{i}\rangle| + |\langle \xi_{l}, v_{i}\rangle| \big)
    \le  2\epsilon,
\end{align}
where for the second inequality we use the fact that $\langle \theta, v_{i}\rangle = 0$ for $i\ge 2$.

For the second part, we first note that $\theta =\sqrt{1-\rho^2}\cdot v_2 + \rho v_1$.
Then on the event $\cE(\epsilon)$,
\begin{align}
    |\dotp{\theta}{\xi_{l}}|
    \le \sqrt{1-\rho^2}\cdot|\dotp{v_2}{\xi_l}| + |\rho| \cdot |\dotp{v_1}{\xi_l}|
    \le 2\epsilon. \label{eq:dotp-theta-xi}
\end{align}
where the second inequality holds because $\rho\in[-1,1]$ and $|\dotp{v_1}{\xi_l}|\vee  |\dotp{v_2}{\xi_l}| \le \epsilon$ on the nice event $\cE(\epsilon)$.
Besides, on the nice event $\tilde \cE(\tilde \epsilon)$, it holds for all $l\neq l'$ that $|\dotp{\xi_l}{\xi_{l'}}| \le \tepsilon$.
Therefore, on the event $\cE(\epsilon)\cap \tilde \cE(\tilde \epsilon)$, combining Eq.~\eqref{eq:norm-perturbation-lower-bound} and~\eqref{eq:dotp-theta-xi} yields
\begin{align}
    |\dotp{w_{l}}{w_{l'}}| &\le \frac{|\dotp{\xi_l}{\xi_{l'}}| + \gamma |\dotp{\xi_l}{\theta}| + \gamma |\dotp{\xi_{l'}}{\theta}| + \gamma^2\cdot \norm{\theta}_2^2}{\norm{\gamma \theta+\xi_l}_2\cdot\norm{\gamma \theta+\xi_{l'}}_2}
    \le 4 (\gamma^2 + \tilde{\epsilon} + 4\gamma \epsilon).
\end{align}
This finishes the proof.
\end{proof}

\subsubsection{Proof of \texorpdfstring{\Cref{prop:nonsparse-first-moment}}{the first-moment proposition}}\label{proof:nonsparse-first-moment}
\begin{proof}[Proof of \Cref{prop:nonsparse-first-moment}]
Invoking \Cref{lem:first-moment-decomp}  with the fact that $\|\theta^\star\|_2=1$, we can decompose $\dotp{\EE_{\PP_{\theta^\star}}[\hatg]}{\theta^\star}$ as
\begin{align}
    \langle \EE_{\PP_{\theta^\star}}[\hatg] , \theta^\star\rangle
    &=
    \sum_{s=s^\star}^\infty \frac{\sqrt{s+1}}{L} \sum_{l=1}^L \EE_{\QQ}[
        \zeta_s(y)\cdot \hat\psi_{s+1}(y)]
    \cdot \langle w_l, \theta^\star\rangle^{s+1}\nonumber\\
    &\qquad + \sum_{s=s^\star}^\infty \frac{\sqrt s}{L} \sum_{l=1}^L \EE_{\QQ}[
        \zeta_s(y)\cdot \hat\psi_{s-1}(y)]
    \cdot \langle w_l, \theta^\star\rangle^{s-1}\\
    &= \EE_{\QQ}[
        \zeta_{s^\star}(y)\cdot \hat\psi_{s^\star-1}(y)] \cdot \frac{\sqrt{s^\star}}{L} \sum_{l=1}^L \langle w_l, \theta^\star\rangle^{s^\star-1}
    + R,
    \label{eq:signal-2}
\end{align}
where the remainder term $R$ is defined as
\begin{align}
    R := \sum_{s=s^\star}^\infty \frac{\sqrt{s+1}}{L} \sum_{l=1}^L \Big(\EE_{\QQ}[\zeta_s(y)\cdot \hat\psi_{s+1}(y)] \langle w_l, \theta^\star\rangle + \EE_{\QQ}[\zeta_{s+1}(y) \cdot \hat\psi_s(y)]\Big)\cdot \langle w_l, \theta^\star\rangle^{s}.
\end{align}
Below, we analyze the scale of each term in \eqref{eq:signal-2}, and show that the remainder term $R$ is dominated by the first term in \eqref{eq:signal-2} with high probability over the randomness of $\xi_1,\ldots,\xi_L$.

\paragraph{Analysis for the remainder term $R$ in \eqref{eq:signal-2}}
To bound $|R|$, we apply the triangle inequality with the fact that $|\langle w_l, \theta^\star\rangle| \le 1$ and get
\begin{align}\label{eq:R_bound_0}
    |R| &\leq \sum_{s=s^\star}^\infty \frac{\sqrt{s+1}}{L} \sum_{l=1}^L \EE_\QQ\Big[|\zeta_s(y)\cdot\hat\psi_{s+1}(y)| + |\zeta_{s+1}(y)\cdot\hat\psi_s(y)|\Big] \cdot |\langle w_l, \theta^\star\rangle|^{s}.
\end{align}
Recall from \eqref{eq:likelihood-ratio-decomp} that $\zeta_s(y)=\EE_\PP[h_s(x)\given y]$, so by Jensen's inequality, $\EE_\QQ[\zeta_s(y)^2]\leq \EE_\PP[h_s(x)^2]=1$.
Also, by \Cref{asp:quadratic_integrability}, we have $\EE_\QQ[ \hatpsi_s(y)^2] = O(1)$.
Now applying Cauchy-Schwarz inequality, we get that
\begin{align}
    \EE_{\QQ}[| \zeta_{s+1}(y) \cdot \hat \psi_s(y)|] &\le
    \EE_\QQ [\zeta_{s+1}(y)^2]^{1/2} \cdot \EE_\QQ[\hat \psi_s(y)^2]^{1/2} = O(1),  \label{eq:coef-bound}
\end{align}
and similarly for $\EE_{\QQ}[|\zeta_{s}(y)\cdot\hat \psi_{s+1}(y)|]$.
Therefore,
\begin{align}
    |R| \lesssim \sum_{s=s^\star}^\infty \frac{\sqrt{s+1}}{L} \sum_{l=1}^L |\langle w_l,\theta^\star\rangle|^s,
\end{align}
and it remains to bound the right hand side of the above inequality.
Recall the definition of the nice event $\cE(\epsilon)$ in \eqref{eq:nice-event-cE}, where $\{v_1 = \theta^\star, v_2 =  (\theta - \rho \theta^\star)/\sqrt{1-\rho^2}, v_3,\dots v_d\}$ is an orthonormal basis.
Since $\gamma=o(1)$ and $\epsilon=o(1)$, it follows from \Cref{lem:polarized-weight-nice-event} that for sufficiently large $d$, $|\langle w_l , \theta^\star\rangle| < 1/2$ on $\cE(\epsilon)$ for all $l\in[L]$.
Consequently, it holds on $\cE(\epsilon)$ that
\begin{align}
    |R| &\lesssim \sum_{s=s^\star}^\infty {\sqrt{s+1}} \cdot \left(\frac 1 2\right)^{s-s^\star} \cdot \frac 1 L \sum_{l=1}^L |\langle w_l, \theta^\star\rangle|^{s^\star}
    \lesssim \frac 1 L \sum_{l=1}^L |\langle w_l, \theta^\star\rangle|^{s^\star}.
    \label{eq:signal-4}
\end{align}
Now it remains to bound the right hand side in \eqref{eq:signal-4}, and the strategy is to establish concentration over $\xi_1,\ldots,\xi_L$.
To proceed, we define
\begin{align}\label{eq:def_tilde_w}
        \tilde w_l=  \begin{cases}
            w_l & \text{if } \sup_{1\le i\le d}|\langle \xi_l, v_i \rangle| < \epsilon;\\
            0 & \text{otherwise}.
        \end{cases}
\end{align}
Note that $\tilde w_l = w_l$ for any $l\in [L]$ on $\cE(\epsilon)$, and $\tilde w_1,\ldots,\tilde w_L$ are independent random vectors.
Then by \Cref{lem:polarized-weight-nice-event}, it holds on $\cE(\epsilon)$ that
\begin{align}\label{eq:tilde_w_first_moment}
    |\dotp{\tw_{l}}{\theta^\star}| \le  2(\gamma+\epsilon)\quad \text{for all }l\in[L].
\end{align}
Next, we consider the second moment of $|\dotp{\tw_{l}}{\theta^\star}|^{s^\star}$:
\begin{align}
    \EE_{\tilde w_l}[\langle \tilde w_l, \theta^\star\rangle^{2s^\star}]
    & = \EE_{w_l}\big[\langle w_l, \theta^\star\rangle^{2s^\star} \cdot \ind\big\{\sup_i |\dotp{w_l}{v_i}|\le \epsilon \big\}\big]
    \le \EE_{w_l}[\langle w_l, \theta^\star\rangle^{2s^\star}]
\end{align}
Here we consider a different orthonormal basis given by $v_1'=\theta, v_2'=(\theta^\star-\rho\theta)/\sqrt{1-\rho^2}, v_3'=v_3,\ldots,v_d'=v_d$.
Under this basis, $\theta$ is equivalent to $e_1$ under the canonical basis, and $\theta^\star$ is equivalent to $\rho e_1 + \sqrt{1-\rho^2} e_2$ under the canonical basis.
Then since the distribution of $\xi_l$ is orthogonally invariant, we can apply \Cref{lem:g-1st-w-expectation} to obtain
\begin{align}\label{eq:tilde_w_second_moment}
    \EE_{\tilde w_l}[\langle \tilde w_l, \theta^\star\rangle^{2s^\star}]
    \le \EE_{w_l}[\langle w_l, \theta^\star\rangle^{2s^\star}]
    \lesssim (|\rho|\gamma + d^{-1/2})^{2s^\star}.
\end{align}
Combining the first moment bound in \eqref{eq:tilde_w_first_moment} and the second moment bound in \eqref{eq:tilde_w_second_moment}, we can apply Bernstein's inequality (\Cref{lem:bernstein}) to $|\langle\tilde w_1,\theta^\star\rangle|^{s^\star},\ldots,|\langle\tilde w_L,\theta^\star\rangle|^{s^\star}$ to obtain that there exists an event $\cE_{1,1}$ with $\Pr(\cE_{1,1}) \ge 1 - d^{-c}/(MT)$ such that on the event $\cE_{1,1}\cap\cE(\epsilon)$,
\begin{align}
    \frac 1 L \sum_{l=1}^L |\langle w_l, \theta^\star\rangle|^{s^\star}
    &= \frac 1 L \sum_{l=1}^L |\langle \tilde w_l, \theta^\star\rangle|^{s^\star} \\
    &\lesssim \Big(1 + \sqrt{ L^{-1}\log (d^c MT )}\Big)\cdot  (|\rho|\gamma + d^{-1/2})^{s^\star} + \frac{(\epsilon\vee \gamma)^{s^\star} \cdot \log (d^c MT )}{L}\\
    &\lesssim \Big(1 + \sqrt{ L^{-1}\log d}\Big)\cdot  (|\rho|\gamma + d^{-1/2})^{s^\star} + \frac{(\epsilon\vee \gamma)^{s^\star} \cdot \log d}{L}.
    \label{eq:signal-5}
\end{align}
Here we use the fact that $M\cdot T$ is at most polynomial in $d$.
Moreover, by the condition \eqref{eq:alignment_expected_gradient_condition_L} on $L$,
\begin{align}
\begin{aligned}
    L^{-1}\cdot (\epsilon\vee \gamma)^{s^\star}\cdot \log d  &\lesssim  d^{-s^\star/2} \le  (|\rho|\gamma + d^{-1/2})^{s^\star}; \\
    L^{-1} \cdot \log d  &\lesssim  1. \label{eq:signal-5.5}
\end{aligned}
\end{align}
Now, combining \eqref{eq:signal-4}, \eqref{eq:signal-5}, and \eqref{eq:signal-5.5}, we conclude that on the event $\cE_{1,1}\cap \cE(\epsilon)$,
\begin{align}\label{eq:R_bound}
    |R| \lesssim  (|\rho|\gamma + d^{-1/2})^{s^\star}.
\end{align}
\paragraph{Analysis for the dominant term in \eqref{eq:signal-2}}
We still use $\tilde w_1,\ldots,\tilde w_L$ defined in \eqref{eq:def_tilde_w}.
Since $|\dotp{w_l}{\theta^\star}|\le 1 $, we can approximate the expectation $\EE_{\tw_l}[\dotp{\tw_l}{\theta^\star}^{s^\star-1}]$ as follows:
\begin{align}
    \EE_{\tilde w_l}[\langle \tilde w_l, \theta^\star\rangle^{s^\star - 1}]
    &= \EE_{w_l}\Big[\langle w_l, \theta^\star\rangle^{s^\star - 1} \cdot \ind\Big\{\sup_{i\in[d]} |\dotp{w_l}{v_i}|\le \epsilon \Big\}\Big] \\
    &= \EE_{w_l}[\langle w_l, \theta^\star\rangle^{s^\star - 1}] \pm \Pr\Big (\sup_{i\in[d]} |\dotp{w_l}{v_i}|> \epsilon \Big)\\
    &= \EE_{w_l}[\langle w_l, \theta^\star\rangle^{s^\star - 1}] \pm \big(1-\Pr({\cE(\epsilon)})\big),
\end{align}
where we use the fact that the event $\{\sup_i |\dotp{w_l}{v_i}|> \epsilon\}\subset {\cE(\epsilon)}^c$.
Similar to the argument for \eqref{eq:tilde_w_second_moment}, it follows from \Cref{lem:g-1st-w-expectation} that
\begin{align}\label{eq:dominant_term_order}
    \EE_{w_l}[\langle w_l, \theta^\star\rangle^{s^\star - 1}]&\simeq \begin{cases}
            \ds
            (|\rho|\gamma + d^{-1/2})^{s^\star-1} & \text{if $s^\star$ is odd};
            \\
            \ds \rho \gamma (|\rho|\gamma + d^{-1/2})^{s^\star-2} & \text{if $s^\star$ is even}.
    \end{cases}
\end{align}
This gives rise to the first part of the main result.

Similarly, the second moment of $\langle\tilde w_l,\theta^\star\rangle^{s^\star-1}$ satisfies
\begin{align}
    \EE_{\tilde w_l}[\langle \tilde w_l, \theta^\star\rangle^{2(s^\star - 1)}]
    & = \EE_{w_l}\Big[\langle w_l, \theta^\star\rangle^{2(s^\star - 1)} \cdot \ind\Big\{\sup_i |\dotp{w_l}{v_i}|\le \epsilon \Big\}\Big] \\
    &\le \EE_{w_l}[\langle w_l, \theta^\star\rangle^{2(s^\star - 1)}] \\
    &\simeq (|\rho|\gamma + d^{-1/2})^{2(s^\star-1)}.
\end{align}
Moreover, we have $|\dotp{\tw_l}{\theta^\star}|\le 2(\gamma \vee \epsilon)$ on $\cE(\epsilon)$.
Combining the above, applying Bernstein's inequality (\Cref{lem:bernstein}) to $\langle\tilde w_1,\theta^\star\rangle^{s^\star-1},\ldots,\langle\tilde w_L,\theta^\star\rangle^{s^\star-1}$ yields that there exists an event $\cE_{1,2}$ with $\Pr(\cE_{1,2}) \ge 1 - d^{-c}/(MT)$ such that on the event $\cE_{1,2}\cap \cE(\epsilon)$,
\begin{align}\label{eq:dominant_term_bound}
    \frac 1 L \sum_{l=1}^L \langle w_l, \theta^\star\rangle^{s^\star-1}
    = \frac 1 L \sum_{l=1}^L \langle \tilde w_l, \theta^\star\rangle^{s^\star-1}
    = \EE_{w_l}[\langle w_l, \theta^\star\rangle^{s^\star - 1}] + E ,
\end{align}
where the error term $E$ can be bounded as
\begin{align}\label{eq:error_term_bound}
    |E|&\lesssim (|\rho|\gamma + d^{-1/2})^{s^\star-1} \cdot \sqrt\frac{\log d}{L} + \frac{(\epsilon\vee \gamma)^{s^\star -1} \log d}{L} + 1 - \Pr\big ({\cE(\epsilon)}\big)
    \lesssim (|\rho|\gamma + d^{-1/2})^{s^\star}
\end{align}
where the second inequality follows from the assumption that $1-\Pr(\cE(\epsilon))=O(d^{-s^\star/2})$ and the condition \eqref{eq:alignment_expected_gradient_condition_L} on $L$ using a similar argument as in \eqref{eq:signal-5.5}.
Now, combining \eqref{eq:R_bound}, \eqref{eq:dominant_term_bound}, and \eqref{eq:error_term_bound} yields the second part of the main result.

\paragraph{Norm of $\EE_{\PP_{\theta^\star}}[\hatg]$}
Recall that $P_{\theta^\star}^\perp = I_d - \theta^\star \theta^{\star\top}$ is the projection matrix onto the orthogonal complement of $\theta^\star$.
Then applying \Cref{lem:first-moment-decomp}, we have
\begin{align}
    P_{\theta^\star}^\perp \EE_{\PP_{\theta^\star}}[\hatg]
    &= \sum_{s=s^\star}^\infty \frac{\sqrt{s+1}}{L} \sum_{l=1}^L \EE_{\QQ}[
    \zeta_s(y)\cdot \hat\psi_{s+1}(y)] \cdot \dotp{w_l}{\theta^\star}^{s}\cdot P_{\theta^\star}^\perp w_l.
\end{align}
Since $\|P_{\theta^\star}^\perp w_l\|_2 \le 1$, recognizing that here we can apply the same argument as that for \eqref{eq:R_bound_0}, we conclude that on the event $\cE_{1,1}\cap \cE(\epsilon)$,
\begin{align}
    \|P_{\theta^\star}^\perp \EE_{\PP_{\theta^\star}}[\hatg]\|_2
    \lesssim (|\rho|\gamma + d^{-1/2})^{s^\star}.
\end{align}
Then the third part of the main result follows:
\begin{align}
    \norm{\EE_{\PP_{\theta^\star}}[\hatg]}_2 - |\langle \EE_{\PP_{\theta^\star}}[\hatg], \theta^\star\rangle|
    &= \sqrt{\norm{P_{\theta^\star}^\perp \EE_{\PP_{\theta^\star}}[\hatg]}_2^2 + |\dotp{\EE_{\PP_{\theta^\star}}[\hatg]}{\theta^\star}|^2 }  - |\dotp{\EE_{\PP_{\theta^\star}}[\hatg]}{\theta^\star}|\\
    &\le \norm{P_{\theta^\star}^\perp \EE_{\PP_{\theta^\star}}[\hatg]}_2
    \lesssim  (|\rho|\gamma + d^{-1/2})^{s^\star} .
\end{align}

Finally, setting $\cE_1= \cE_{1,1}\cap \cE_{1,2}$ finishes the proof.
\end{proof}

\subsubsection{Proof of \texorpdfstring{\Cref{prop:nonsparse-sample-concentration}}{the sample-concentration proposition}}\label{proof:nonsparse-sample-concentration}
\begin{proof}[Proof of \Cref{prop:nonsparse-sample-concentration}]
Recall from \Cref{def:elem_setting} that we consider the orthonormal basis $\{v_1=\theta^\star, v_2=(\theta-\rho\theta^\star)/\sqrt{1-\rho^2},v_3,\ldots,v_d\}$.
Our goal is to show that $\hatg$ has small fluctuation along each of these directions, and we fix any $v\in \{v_1, v_2, \ldots, v_{d}\}$ in the sequel.

As each sample $(z_i, y_i)$ is independently drawn from $\PP_{\theta^\star}$, $\hat g$ is the average of $n$ independent copies of $\hat g_1$ defined as
\begin{align}
    \hatg_1 = \frac{1}{L}\sum_{l=1}^L \big(\psi(y_1, \langle w_l, z_1\rangle) \cdot z_1 - \wh\psi_1(y_1) \cdot w_l \big).
\end{align}
Thus we will analyze the moments of $\langle\hat g_1,v\rangle$, and then apply concentration inequalities to obtain the desired result.

\paragraph{Verifying the polynomial-like tail condition}
First, we need to upper bound the $L^r(\PP_{\theta^\star})$-norm of $ \dotp{\hatg_1}{v}$. To this end, we define $G_v: (\RR^d \times \RR)\times\RR^d \to \RR$ as
\begin{align}
    G_v (z,y,w) := |\psi(y, \dotp{w}{z})\cdot \dotp{z}{v}| + |\wh\psi_1(y)\cdot \dotp{w}{v}|.
\end{align}
Also, we define the empirical measure $\diff\mu(w) = L^{-1}\sum_{l}\delta(w_l)$, where $\delta_a(\cdot)$ is the point-mass assigned to $a\in\RR$.  It holds by the integral Minkowski's inequality that
\begin{align}
    \EE_{\PP_{\theta^\star}}[ |\dotp{\hatg_1}{v}|^r]^{1/r} &\le   \Big(\int \diff\PP_{\theta^\star}(y,z)\Big(\int \diff\mu (w)\cdot G_v(z,y,w)\Big)^r\Big)^{1/r} \\
    &\le  \int \diff\mu(w) \Big(\int \diff\PP_{\theta^\star}(y,z)\cdot G_v(z,y,w)^r\Big)^{1/r} \\
    & =  \frac{1}{L}\sum_{l} \EE_{ \PP_{\theta^\star}}\big[\big(|\psi(y,\dotp{w_l}{z})\cdot \dotp{z}{v}|+  |\hatpsi_1(y)\cdot \dotp{w_l}{v}|\big)^r \big]^{1/r} \\
    &\lesssim   \frac{1}{L} \sum_l  \EE_{\PP_{\theta^\star}}\big[ |\psi(y,\dotp{w_l}{z})\cdot \dotp{z}{v}|^r\big]^{1/r}  + r^{C_p}\cdot\frac{1}{L} \sum_l   |\dotp{w_l}{v}|, \label{eq:second-upper}
\end{align}
where the last inequality holds since $\EE[|U+V|^{r}]^{1/r}\le \EE[|U|^r]^{1/r} +\EE[|V|^r]^{1/r}$ and $\EE_{\PP_{\theta^\star}}[| \hatpsi_1(y)|^r]^{1/r}\lesssim r^{C_p}$ by the polynomial-like tail condition in \Cref{asp:oracle function}.

For the second term in \eqref{eq:second-upper}, we have on $\cE(\epsilon)$ that $|\dotp{w_l}{v}|\le 2\gamma +2\epsilon \le 1$. Applying the Cauchy-Schwarz inequality, we have for the summand of the first term in \eqref{eq:second-upper} that
\begin{align}
    \EE_{\PP_{\theta^\star}}[|\psi(y,\dotp{w_l}{z})\cdot  \dotp{z}{v}|^r]^{1/r} &\le  \EE_{\PP_{\theta^\star}}[|\psi(y,\dotp{w_l}{z})|^{2r}]^{1/2r}\cdot  \EE_{\PP_{\theta^\star}}[|\dotp{z}{v}|^{2r}]^{1/2r}.
\end{align}
Note that $\EE_{\PP_{\theta^\star}}[|\dotp{z}{v}|^{2r}]^{1/2r} \le (2r-1)!!^{1/(2r)}\lesssim  r^{1/2}$, so it remains to bound $\EE_{\PP_{\theta^\star}}[|\psi(y,\dotp{w_l}{z})|^{2r}]^{1/2r}$. We can decompose $w_l$ into components that are correlated with and independent of $y$, and it holds that
\begin{align}
    \dotp{w_l}{z} &= \dotp{w_l - \dotp{w_l}{\theta^\star} \theta^\star}{z} + \dotp{w_l}{\theta^\star} \dotp{\theta^\star}{z} \\& =   \sqrt{1-\dotp{w_l}{\theta^\star}^2} \cdot x' + \dotp{w_l}{\theta^\star} x
\end{align}
where $x = \dotp{\theta^\star}{z}$ is independent of $x' = (1- \dotp{w_l}{\theta^\star}^2)^{-1/2}\cdot \dotp{w_l- \dotp{w_l}{\theta^\star} \theta^\star}{z}$.
Both $x,x'\sim \cN(0,1)$.
Therefore, we consider the Gaussian noise operator $\noiseop{\rho} \psi(y, x) = \EE_{x'\sim \cN(0, 1)}[\psi(y, \rho x + \sqrt{1-\rho^2} x')]$.
Conditioning on $(x,y)$ and averaging over $x'$, we can bound the $2r$-moment of $\psi(y, \dotp{w_l}{z})$ under $\PP_{\theta^\star}$ as
\begin{align}
    \EE_{\PP_{\theta^\star}}[|\psi(y,\dotp{w_l}{z})|^{2r}] &= \EE_{\PP}[\noiseop{\langle \theta^\star, w_l\rangle} \big(|\psi(y, x)|^{2r} \big)] \\
    &=  \EE_{\QQ}\left[\noiseop{\langle \theta^\star, w_l\rangle} \big(|\psi(y, x)|^{2r}\big) \cdot \frac{\PP(x, y)}{\QQ(x, y)}\right] \\
    &= \EE_{\QQ}\left[|\psi(y, x)|^{2r} \cdot \noiseop{\langle \theta^\star, w_l\rangle} \left(\frac{\PP(x, y)}{\QQ(x, y)}\right)\right] \\
    &\le \left(\EE_{\QQ}[|\psi(y, x)|^{4r}] \cdot \EE_{\QQ}\left[\left(\noiseop{\langle \theta^\star, w_l\rangle} \Bigl(\frac{\PP(x, y)}{\QQ(x, y)}\Bigr)\right)^2\right]\right)^{1/2},  \label{eq:gaussian-noise-op}
\end{align}
where the second line follows from the property of the Gaussian noise operator in \eqref{eq:hermite-noise-swap}.
By assumption of the tail bound in \Cref{asp:oracle function}, we have that
$\EE_\QQ[|\psi(y, x)|^{4r}] \le C_p  (4r)^{4C_p  r}$. For the second term, we have by spectral decomposition of Gaussian noise operator in \eqref{eq:hermite-noise-spectral} and Parseval's identity that
\begin{align}
    \EE_{\QQ}\left[\left(\noiseop{\langle \theta^\star, w_l\rangle} \Bigl(\frac{\PP(x, y)}{\QQ(x, y)}\Bigr)\right)^2\right] &= 1 + \sum_{s\ge s^\star} \langle \theta^\star, w_l\rangle^{2s} \cdot \EE_{\QQ} \left[\zeta_s(y)^2\right]\\
    &\le 1+ \sum_{s\ge s^\star} |\dotp{\theta^\star}{w_l}|^{2s} \le 3,\label{eq:gaussian-noise-op-2}
\end{align}
where we use the property that on the nice event $\cE(\epsilon)$ we have $|\langle w_l, \theta^\star\rangle| \le 2(\gamma |\rho| + \epsilon) \le 2(\gamma + \epsilon) < 1/2$ and also $\EE_\QQ[\zeta_s(y)^2] \le 1$. In conclusion, we get for the first term in \eqref{eq:second-upper} that
\begin{align}
\frac{1}{L}\sum_l \EE_{\PP_{\theta^\star}}[|\psi(y,\dotp{w_l}{z})\cdot  \dotp{z}{v}|^r]^{1/r} \lesssim r^{C_p+1/2}.
\end{align}
Combining everything, we have
\begin{align}
    \EE_{\PP_{\theta^\star}}[|\langle \hatg_1, v\rangle|^r]^{1/r} &\lesssim  r^{C_p + 1/2}, \quad \forall v\in \{\theta^\star, v_2, \ldots, v_{d}\}.
\end{align}

\paragraph{Calculating the second moment}  To bound the variance, we have for each $v$ among the designated directions that
    \begin{align}
        \Var_{\PP_{\theta^\star}}[\langle \hatg_1, v\rangle] &=
        \EE_{\PP_{\theta^\star}}[\langle \hatg_1, v\rangle^2] - \EE_{\PP_{\theta^\star}}[\langle\hatg_1, v\rangle]^2 \\&\le \EE_{\PP_{\theta^\star}}[\langle \hatg_1, v\rangle^2].
    \end{align}
    From this we see that it suffices to bound the second moment of $\langle \hatg_1, v\rangle$, which is given by
    \begin{align}
        \EE_{\PP_{\theta^\star}}[\langle \hatg_1, v\rangle^2]
        &\lesssim \frac{1}{L^2} \sum_{l, l'=1}^L \EE_{\PP_{\theta^\star}}\left[
            \psi(y, \langle w_l, z\rangle)  \psi(y, \langle w_{l'}, z\rangle) \langle z, v\rangle^2
        \right]  + \frac{1}{L^2} \sum_{l, l'=1}^L \EE_{\PP_{\theta^\star}}\left[
            \hat\psi_1(y)^2   \langle w_{l}, v\rangle \langle w_{l'}, v\rangle
        \right]\\
        &= \frac{1}{L^2} \sum_{l\neq l'} \EE_{\QQ}\bigg[
            \psi(y, \langle w_l, z\rangle) \psi(y, \langle w_{l'}, z\rangle) \langle z, v\rangle^2 \cdot
            \bigg(
                1 + \sum_{s\ge s^\star} \zeta_s(y) h_s(\langle \theta^\star, z\rangle)
            \bigg)
        \bigg] \\
        &\qquad + \frac{1}{L^2} \sum_{l=1}^L \EE_{\PP_{\theta^\star}}\left[
            \psi(y, \langle w_l, z\rangle)^2\cdot  \langle z, v\rangle^2
        \right] + \frac{1}{L^2} \sum_{l\neq l'} \EE_\QQ[\hat\psi_1(y)^2] \langle w_l, v\rangle \langle w_{l'}, v\rangle \\
        &\qquad + \frac{1}{L^2} \sum_{l=1}^L \EE_\QQ[\hat\psi_1(y)^2] \langle w_l, v\rangle^2.
    \end{align}
    As $\psi(y, z) z$ is quadruple-integrable by \Cref{asp:oracle function}, the above integral is well-defined.
    We split the summation into two parts: $l = l'$ and $l\neq l'$.
    For $l = l'$,  we have for the gradient term that
    \begin{align}
\frac{1}{L^2} \sum_{l=1}^L  \EE_{\PP_{\theta^\star}} [ \psi(y,\dotp{w_l}{z})^2\cdot\dotp{z}{v}^2]  &\le \frac{1}{L^2} \sum_{l=1}^L \EE_{\PP_{ \theta^\star}} [ \psi(y,\dotp{w_l}{z})^4]^{1/2} \cdot \EE_{\PP_{ \theta^\star}} [ \dotp{z}{v}^4]^{1/2} \\
&\le \frac{2}{L^2} \sum_{l=1}^L \EE_{\PP_{ \theta^\star}} [ \psi(y,\dotp{w_l}{z})^4]^{1/2},
    \end{align}
    where the first line holds by Cauchy-Schwarz inequality and the second line holds since $ \dotp{z}{v}$ is a standard normal random variable. Invoking Eq.~\eqref{eq:gaussian-noise-op} and Eq.~\eqref{eq:gaussian-noise-op-2}, we see that each summand in the expression above is bounded by a constant. Hence
    \begin{align}
    \frac{1}{L^2} \sum_{l=1}^L  \EE_{\PP_{\theta^\star}} [ \psi(y,\dotp{w_l}{z})^2\cdot\dotp{z}{v}^2]  \lesssim  L^{-1} .
    \end{align}
    For the debiasing term, it holds that $\EE_\QQ[\hatpsi_1(y)^2] \le 1$ and $|\dotp{w_l}{v}|\le 1$, and similarly it holds that
    \begin{align}
        \frac{1}{L^2} \sum_{l=1}^L \EE_\QQ[\hat\psi_1(y)^2] \langle w_l, v\rangle^2 \lesssim L^{-1}.
    \end{align}

    For $l\neq l'$, \Cref{lem:polarized-weight-nice-event} implies that we have on the nice event $\cE(\epsilon)\cap\tilde\cE(\tilde\epsilon)$ that
    \begin{align}
        |\langle w_l, w_{l'}\rangle| &\le 4(\gamma^2 + 2\gamma\epsilon+ \tilde\epsilon) \\
        &\le 4(2\gamma^2 + \epsilon^2 + \tilde\epsilon)  \\
        &\le 8(\gamma^2 + \tilde\epsilon) \defeq \epsilon_2,
    \end{align}
    where we have used the assumption that $\epsilon^2 \le \tilde\epsilon$.
    Defining
    \begin{align}
        \epsilon_1 \defeq \max\{ |\langle w_l, \theta^\star\rangle|, |\langle  w_{l'}, \theta^\star\rangle|\},\qquad \epsilon_0 \defeq \max\{ |\langle w_l, v\rangle|, |\langle w_{l'}, v\rangle|\},
    \end{align}
    we have by \Cref{lem:g-2nd-moment} that, for any $v\in \{\theta^\star, v_2, \ldots, v_{d}\}$, it holds on the same event that
    \begin{align}
        & \EE_{\QQ}\bigg[
            \psi(y, \langle w_l, z\rangle) \psi(y, \langle w_{l'}, z\rangle) \langle z, v\rangle^2 \cdot
            \bigg(
                1 + \sum_{s\ge s^\star} \zeta_s(y) h_s(\langle \theta^\star, z\rangle)
            \bigg)
        \bigg] \\
        &\qquad \lesssim\,
        \epsilon_2^{s^\star - 1} \cdot \left( 1 +  \frac{\epsilon_1^2}{\epsilon_2} + \left(\frac{\epsilon_1^2}{\epsilon_2}\right)^{s^\star - 1} \cdot \epsilon + \ind(v\perp \theta^\star)\cdot \left(\frac{\epsilon_1^2}{\epsilon_2}\right)^{s^\star - 2} \cdot \frac{\epsilon_0^2}{\epsilon_2} \cdot (\epsilon_1^2 + \epsilon_1 \cdot \ind(s^\star \ge 4))\right).
        \label{eq:g-2nd-moment-l neq l'}
    \end{align}
    If the nice event $\cE (\epsilon)$ also holds, on which the following holds for all $l\in[L]$:
    \begin{align}
        |\langle \xi_{l}, \theta^\star\rangle| < \epsilon, \quad \max_{2\le i \le d} |\langle \xi_{l}, v_{i}\rangle| < \epsilon,
    \end{align}
    then by \Cref{lem:polarized-weight-nice-event}, for any $i\ge 3$ and $ l\le L$, it holds that
    \begin{align}
        |\langle w_l, \theta^\star\rangle| \lesssim \gamma |\rho| + \epsilon, \quad |\langle w_l, v_2\rangle| \lesssim \sqrt{1-\rho^2} \gamma + \epsilon, \quad |\langle w_l, v_i\rangle| \lesssim \epsilon,
    \end{align}
    Consequently, we can set $\epsilon_1 \simeq \gamma |\rho| + \epsilon = o(1)$ and
    \begin{align}
        \epsilon_0 \simeq \begin{cases}
            \gamma |\rho| + \epsilon, & \text{if $v = \theta^\star$},\\
            \gamma \sqrt{1-\rho^2} + \epsilon, & \text{if $v = v_2$},\\
            \epsilon, & \text{otherwise}.
        \end{cases}
    \end{align}
    Therefore, we can simplify the ratio $\epsilon_1^2/\epsilon_2$ and $\epsilon_0^2/\epsilon_2$ as
    \begin{align}
        \frac{\epsilon_1^2}{\epsilon_2} \simeq \frac{(\gamma |\rho| + \epsilon)^2}{4(\gamma^2 + \tilde\epsilon)} \simeq \frac{\gamma^2 |\rho|^2 + \epsilon^2}{\gamma^2 + \tilde\epsilon}, \quad
        \frac{\epsilon_0^2}{\epsilon_2} \simeq \begin{cases}
            \frac{(\gamma |\rho| + \epsilon)^2}{8(\gamma^2 + \tilde\epsilon)} \simeq \frac{\gamma^2 |\rho|^2 + \epsilon^2}{\gamma^2 + \tilde\epsilon}, & \text{if $v = \theta^\star$},\\
            \frac{(\gamma \sqrt{1-\rho^2} + \epsilon)^2}{8(\gamma^2 + \tilde\epsilon)} \simeq \frac{\gamma^2 (1-\rho^2) + \epsilon^2}{\gamma^2 + \tilde\epsilon}, & \text{if $v = v_2$},\\
            \frac{\epsilon^2}{4(\gamma^2 + \tilde\epsilon)} \simeq \frac{\epsilon^2}{\gamma^2 + \tilde\epsilon}, & \text{otherwise}.
        \end{cases}
    \end{align}
    Since $\epsilon^2\le \tepsilon$, we can conclude that $\epsilon_1^2/\epsilon_2 \lesssim 1$ and $\epsilon_0^2/\epsilon_2 \lesssim 1$. Hence, the right-hand side of \eqref{eq:g-2nd-moment-l neq l'} is bounded by $\epsilon_2^{s^\star - 1} \simeq (\gamma^2 + \tilde\epsilon)^{s^\star - 1}$ for all $v\in \{\theta^\star, v_2, \ldots, v_{d}\}$.
    Similarly, let us consider the term $L^{-2} \cdot \sum_{l\neq l'} \EE_\QQ[\hat\psi_1(y)^2] \langle w_l, v\rangle \langle w_{l'}, v\rangle$. On the nice event $\cE(\epsilon)$, we have
    \begin{align}
        \frac{1}{L^2} \cdot \sum_{l\neq l'} \EE_\QQ[\hat\psi_1(y)^2] \langle w_l, v\rangle \langle w_{l'}, v\rangle &\lesssim \epsilon_0^2  \cdot \ind\{s^\star \le 2\} \\
        &\lesssim \epsilon_2 \ind( s^\star \le 2 )\\
        &\lesssim (\gamma^2 + \tilde\epsilon)^{s^\star - 1}\cdot \ind\{s^\star \le 2\}.
    \end{align}
    The first inequality holds because $\hat\psi_1(y)=0$ whenever $s^\star \ge 2$ because of \ref{asp:high_pass}, and the second inequality holds due to the condition that $\epsilon^2 \le \tilde\epsilon$.

    Combining the results for $l = l'$ and $l\neq l'$, it holds for any $v\in \{\theta^\star, v_2, \ldots, v_{d}\}$ on the event $\cE(\epsilon)\cap \tilde\cE(\tepsilon)$ that
    \begin{align}
        \Var_{\PP_{\theta^\star}}[\langle \hatg_1, v\rangle] {\lesssim} (\gamma^2 + \tilde\epsilon)^{s^\star - 1} + \frac{1}{L}.
    \end{align}

\paragraph{Concentration} So far we have shown that $\psi(y, \dotp{w_l}{z})$ satisfies the condition of \Cref{lem:poly_tail_bound}.
Additionally, we now have an upper bound on the variance of $\langle \hatg_1, v\rangle$ for any $v\in \{\theta^\star, v_2, \ldots, v_{d}\}$ in hand.
Thus, by \Cref{lem:poly_tail_bound} there exists a $\{(z_i,y_i)\}_{i\in[n]}$-measurable event $\cE_{2,1}$ with $\Pr(\cE_{2,1})\ge 1 - d^{-c}/(MT)$, and it holds on $\cE_{2,1}$ that
\begin{align}
    &\left|\langle \hatg, v\rangle - \EE_{\PP_{\theta^\star}}[\langle \hatg, v\rangle] \right| \\
    &\qquad {\lesssim} \sqrt{\frac{\EE_{\PP_{\theta^\star}}[\langle \hatg_1, v\rangle^2] \cdot \log (d^{c} M T) }{n}} + \frac{\log(d^{c}MT) \cdot \log(d^{c}MTn)^{C_p+1/2}}{n} \\
    &\qquad
    {\lesssim} \sqrt{\frac{((\gamma^2 + \tilde\epsilon)^{s^\star - 1} + L^{-1}) \cdot \log (d) }{n}} + \frac{\log(d)^{C_p+3/2}}{n}, \quad \forall v\in \{\theta^\star, v_2, \ldots, v_{d}\}.
    \label{eq:g-concentration-1}
\end{align}
where we use the fact that $T$, $M$, and $n$ all depend polynomially on $d$. Moreover, since we assume that
\begin{align}
    n = \Omega\Big(\big((\gamma^2 + \tilde\epsilon)^{s^\star - 1} + L^{-1}\big)^{-1} \cdot \log (d)^{2C_p+2}\Big),
\end{align}
we have that the first term in \eqref{eq:g-concentration-1} dominates, which further implies that
\begin{align}
    \left|\langle \hatg, \theta^\star\rangle - \EE_{\PP_{\theta^\star}}[\langle \hatg, \theta^\star\rangle] \right| \lesssim \sqrt{\frac{((\gamma^2 + \tilde\epsilon)^{s^\star - 1} + L^{-1}) \cdot \log (d) }{n}}.
\end{align}
Meanwhile, for the $\ell_2$-norm of $\hatg$, we have by Jensen's inequality that for any $r\ge 1$,
\begin{align}
    \EE_{\PP_{\theta^\star}}[\norm{\hatg_1}_2^r]^{1/r}
    & = \bigg(\EE_{\PP_{\theta^\star}}\bigg[\sum_{v\in\{\theta^\star, v_2,\ldots, v_{d}\}} \langle \hatg_1, v\rangle^2\bigg]^{r/2}\bigg)^{1/r} \\
    &\le \sqrt d \cdot \bigg(\frac 1 d \cdot \sum_{v\in\{\theta^\star, v_2,\ldots, v_{d}\}} \EE_{\PP_{\theta^\star}}\left[ |\langle \hatg_1, v\rangle|^r\right]\bigg)^{1/r} \lesssim \sqrt d \cdot (r)^{C_p + 1/2}.
\end{align}
This polynomial tail bound enables us to apply \Cref{lem:poly_tail_bound} for the $\ell_2$-norm of $\hatg$, which implies that there exists some event $\cE_{2,2}$ with $\Pr(\cE_{2,2})\ge 1 - d^{-c}/(MT)$, and it holds on $\cE_{2,2}$ that
\begin{align}
    \big\| \hatg  - \EE_{\PP_{\theta^\star}}[\hatg] \big\|_2 \lesssim \sqrt{\frac{  \big((\gamma^2 + \tilde\epsilon)^{s^\star - 1} + L^{-1}\big) \cdot d\cdot  \log (d) }{n}}.
\end{align}
Setting $\cE_2= \cE_{2,1}\cap \cE_{2,2}$ gives the desired event. This concludes the proof of \Cref{prop:nonsparse-sample-concentration}.
\end{proof}

\newcommand{\aast}{{\star\star}}
\newcommand{\tst}{{\theta^{\star}}}
\newcommand{\sst}{{s^{\star}}}
\newcommand{\cg}{{\check{g}}}
\section{Proof of the Main Theorem for the Sparse Prior}\label{sec:sparse-proof}
\subsection{Proof Outline and Preliminaries}

In this section, we provide a detailed proof for \Cref{thm:sparse}. We begin with some good events that we will work with.

\paragraph{Signal concentration}
We begin with a good event on which the signal spreads almost evenly within its support. Define a series of events:
\begin{align}
    \cE_{0,r}&\coloneqq \{ \norm{\theta^\star}_r^r\le C_r \cdot k^{1-r/2} \}  ;\\
    \cE_{0,\infty}&\coloneqq \{ \norm{\theta^\star}_{\infty}\le C_\infty \cdot k^{-1/2}\log(k)^{1/2}\};\\
    \cE_{0,\sharp} &\coloneqq \Big\{\sum_{j \in [d]} \ind\big\{|\theta_j^\star| \ge \frac{1}{\sqrt{2k }}\big\} \ge \frac{k}{4}\Big\}.
\end{align}
The following lemma guarantees that all the events above hold with high probability.
    \begin{lemma}[Good signal]\label{lem:good-signal}
        Suppose that $k$ is sufficiently large such that $k/\log(k) \ge  32c(r\vee1)$, $k/\log^{r+2} k \ge \sqrt{2c+2}$, then it holds that
        \begin{align}
             \Pr(\cE_{0,r}) \wedge \Pr(\cE_{0,\infty})  &\ge 1- O\big(k^{-c_{0,1}}\big);\\
                \Pr(\cE_{0,\sharp })&\ge 1- O\big(\exp\{-c_{0,2} k\}\big),
        \end{align}
        for some constants $c_{0,1},c_{0,2}>0$.
    \end{lemma}
\begin{proof}[Proof of \Cref{lem:good-signal}]
 See \Cref{proof:good-signal}.
\end{proof}
For fixed $s^\star$, we collect all the indices $r$ such that the corresponding nice event $\cE_{0,r}$ will be involved in the coming analysis. Define $S(s^\star) = \big\{s^\star-1,s^\star -\ind\{s^\star \text{ odd}\},2s^\star,4s^\star\big\}$. And we will stick to the following high probability event
\begin{align}
     \cE_{0}\coloneqq\cE_{0,\infty} \cap \cE_{0,\sharp}\cap (\cap_{r\in S(s^\star)} \cE_{0,r}). \label{eq:signal-good-event}
\end{align}
With Lemma \ref{lem:good-signal}, we have that $\Pr(\cE_0)\ge 1-O(k^{-c_0})$ for some constant $c_0>0$.

\paragraph{Preparation for characterizing one-step gradient} Following the same manner as the proof for the non-sparse case, we first characterize the alignment of the gradient step (without adversarial error term $\err_{m,l,i}^\t$). We begin with the definition of a minimal setup, that collects all the essential elements to form the one-step gradient. The following definition is the sparse analogue of \Cref{def:elem_setting}. Apart from the method of generating the noise, in the sparse case, we will analyze the gradient in a coordinate-wise manner to adapt to the sparse structure.

\begin{definition}\label{def:elem-setting-sparse} Fix $k$-sparse vectors $\theta,\theta^\star\in\SS^{d-1}$ with $\phi = \mathsf{supp}(\theta)$ and $\phi^\star =  \mathsf{supp}(\theta^\star)$. Let $\rho = \dotp{ \theta}{\theta^\star }$. Suppose that a single batch of data  $\{(z_i,y_i)\}_{i\in [n]}$ is i.i.d. generated from $\PP_{\theta^\star}$. We fix the index $m$ as the current neuron. We first sample $\phi_{m,1}, \phi_{ m,2},\ldots \phi_{m,L}\iidfrom \unif(\cS_{k,m})$, i.e., uniform distribution over all $k$-sparse supports with $m$-th index always included. Given these random supports, we sample independent noises $\xi_{m,l}\sim \unif(\SS^{k-1}(\phi_{m,l}))$ for $l \in[L]$. Now, for each $l\in[L]$ and $\gamma = o(1)$, we define $w_{m,l} = (\gamma\theta + \xi_{m,l})/\norm{\gamma\theta + \xi_{m,l}}_2$. Then our target is
\begin{align}
    \barg_{m} &= \frac{1}{nL} \sum_{i=1}^n \sum_{l=1}^L \big(\psi(y_i, \dotp{w_{m,l}}{z_i})\cdot z_i - \wh\psi_1(y_i)\cdot w_{m,l} \big).  \label{eq:sparse-gradient-0}
\end{align}
The associated good event is defined as
\begin{align}
\cE_m(\epsilon) &= \{\sup_{l,j} |\dotp{\xi_{m,l}}{e_j}|\le \epsilon\}; \label{eq:sparse-nice-event-cE}\\
\tilde{\cE}_m &= \Big\{ \max\big\{\sup_{l\neq l' }|\phi_{m,l}\cap \phi_{m,l'}|, \sup_{l} |\phi_{m,l} \cap \phi^\star|, \sup_{l} |\phi_{m,l}\cap \phi| \big\}\le \log k \Big\} \label{eq:sparse-nice-event-tildeE}.
\end{align}
\end{definition}

\paragraph{Almost orthogonality} Recall that our perturbation noise $\xi_{m,l}$ is sampled from $\unif(\SS^{k-1}(\phi_{m,l}))$, which is approximately isotropic. One can presume that each $\xi_{m,l}$ is evenly distributed among different coordinates. Additionally, we expect that $(\xi_{m,l}, \phi_{m,l})$ and $(\xi_{m,l'}, \phi_{m,l'})$ should have a negligible overlap. These two qualitative properties, which can help simplify the analysis, are captured by Eq.~\eqref{eq:sparse-nice-event-cE} and Eq.~\eqref{eq:sparse-nice-event-tildeE} in \Cref{def:elem-setting-sparse}. The following lemma characterizes the property of the perturbed weights $w_{m,l}$ on the nice event $\cE_m(\epsilon)\cap \tilde{\cE}_m$.

\begin{lemma}[Polarized weight on a nice event, sparse case]\label{lem:polarized-weight-nice-event-sparse} Consider the setting in \Cref{def:elem-setting-sparse} with $\gamma <1/2$. Suppose that the nice event $\cE_m(\epsilon) \cap \tilde\cE_m$ holds and $\norm{\theta^\star}_{\infty} \le 1/\log k$, then we have that
    \begin{align}
        \sup_{l,j} |\dotp{w_{m,l}}{e_j}| &\le 2  (\gamma|\theta_j|+\epsilon);\\
        \sup_{l} |\dotp{w_{m,l}}{\theta^\star}| &\le 2(\gamma|\rho| + \epsilon).
    \end{align}
    Additionally, we have that
    \begin{align}
        \sup_{ l\neq l'} |\dotp{w_{m,l}}{w_{m,l'}}| \le  4(\gamma^2 + \epsilon^2 \log k ).
    \end{align}
\end{lemma}
\begin{proof}[Proof of \Cref{lem:polarized-weight-nice-event-sparse}]
    See \Cref{proof:polarized-weight-nice-event-sparse}.
\end{proof}

This lemma, serving as the counterpart of \Cref{lem:polarized-weight-nice-event}, controls the behavior of the perturbed weights $w_{m,l}$ in their coordinates and alignment with $\theta^\star$. Additionally, these weights are approximately orthogonal to each other, which allows for a good characterization of the second moment of the gradient.

One may notice that the definition of $\tilde\cE$ differs from the non-sparse case, where we explicitly bound the correlation between different $\xi_{m,l}$. This is the benefit of the sparse structure, as two randomly sampled $k$-sparse supports are naturally of low overlap.

Before delving into the component-wise analysis, we begin with a proposition that will be frequently used to calculate the average contribution of each term in the gradient. This proposition serves as the counterpart of \Cref{lem:g-1st-w-expectation} in the sparse case.
\begin{proposition}\label{prop:sparse-w-expectation-warmup} Suppose that the polarization level $\gamma = o(1)$ and the noise ($\xi_{m,l},\phi_{m,l}) \sim \unif\big(\SS^{k-1}(\phi_{m,l})\big)\otimes \unif(\cS_{k,m})$. Let $\rho =\dotp{\theta}{\theta^\star}$ where $\theta$ is the polarized direction in $w_{m,l}$. Assume that the event $\cE_{0,s-\ind\{s \text{ odd}\}}$ holds. Then we have that
\begin{align}
    \EE_{w_{m,l }}[\dotp{w_{m,l}}{\theta^\star}^{s}]&\simeq\begin{cases}
        \gamma\rho \cdot \big(\gamma|\rho| + k^{-1/2}|\theta_m^\star| +k^{-1 }\delta_{s-1})^{s-1} & \text{if $s$ odd};\\
         (\gamma|\rho|+ k^{-1/2}|\theta_m^\star| +k^{-1}\delta_s )^{s}& \text{if $s$ even},
        \end{cases}
\end{align}
where $\delta_s = (k^2/d)^{1/s}  = o(1)$ for any $s=O(1)$.
\end{proposition}

\begin{proof}[Proof of \Cref{prop:sparse-w-expectation-warmup}]
    See \Cref{proof:sparse-w-expectation-warmup}.
\end{proof}

Recall that we define the good event over the signal as $\cE_0= \cE_{0,\infty}\cap \cE_{0,\sharp} \cap \bigcap_{r\in S(s^\star)} \cE_{0,r}$, where $S(\sst) = \{s^\star-1,s^\star-\ind\{s^\star \text{ odd}\},2s^\star,4s^\star\}$. This definition facilitates our deferred analysis where we need to control multiple moments of different orders and we collect all the necessary good events in $\cE_0$ in the first place.

\subsection{Properties of the Gradient Step}
In this section, we preview some properties regarding the gradient defined in Eq.~\eqref{eq:sparse-gradient-0}. The following proposition deals with the first moment.
\begin{proposition}[First-order moment of the gradient]\label{prop:sparse-gradient-1st-moment} Suppose that $\theta^\star$ is fixed such that the nice event $\cE_0$ holds. Under \Cref{def:elem-setting-sparse}, we choose $\gamma \le \epsilon=o(1)$ such that $\Pr\big({\cE_m(\epsilon)}^c\big)\le O(k^{-s^\star })$ and
\begin{align}
    L = \Omega\Big(\log(d)\cdot \big(k\vee (\epsilon^{s^\star -1}\cdot k^{s^\star +1})\big)\Big).
\end{align}
Then there exists a $\{\xi_{m,l}\}_{l\in [L]}$-measurable event $\cE_{m,1}$ with $\Pr(\cE_{m,1})\ge 1-O(k^{-c_1})$ for some constant $c_1>0$, such that on $\cE_{m,1}\cap \cE_m(\epsilon)\cap \tilde\cE_m$, it holds for any $j\in [d]$ that
\begin{align}
    \dotp{\EE_{\PP_\tst}[\barg_{m}]}{e_j} & =  \EE_{\QQ}[ \zeta_\sst(y)\cdot \hatpsi_{\sst-1}(y)]\cdot \sqrt{\sst} \cdot \EE_{\tw_{m,l}}  [\dotp{\tw_{m,l}}{\theta^\star}^{\sst-1}]\theta_j^\star + R_{m,j},
\end{align}
where the expectation
\begin{align}
    \EE_{\tw_{m,l}}  [\dotp{\tw_{m,l}}{\theta^\star}^{\sst-1}]\simeq
    \begin{cases}
        \gamma\rho\cdot \big(\gamma|\rho| + k^{-1/2}|\theta_m^\star| + k^{-1}\delta_{\sst-2}\big)^{\sst-2} & \text{if $\sst$ is even};\\
         \big(\gamma|\rho| + k^{-1/2}|\theta_m^\star| + k^{-1}\delta_{\sst-1}\big)^{\sst-1}& \text{if $\sst$ is odd},
    \end{cases}
\end{align}
and $R_{m,j}$ is the remainder that can be bounded by
\begin{align}
    |R_{m,j}|&\lesssim \Big(\big(k^{-1}\vee (\gamma|\rho| + k^{-1/2} |\theta_m^\star|)\big)\cdot \big(\gamma|\rho| + k^{-1/2}|\theta_m^\star| + k^{-1}\delta_+\big)^{\sst-1} +k^{-\sst}\Big) \cdot |\theta_j^\star| \\
    &\qquad + (\gamma|\rho| +k^{-1/2}|\theta_m^\star| +k^{-1}\delta_+)^{\sst}\cdot (\gamma|\theta_j|+ k^{-1/2}\cdot(k/d)^{ \ind\{j\neq m\}/2} ) + k^{-(\sst+1)}.
\end{align}
\end{proposition}

\begin{proof}[Proof of \Cref{prop:sparse-gradient-1st-moment}]
    See \Cref{proof:sparse-gradient-1st-moment}.
\end{proof}

The statement of this proposition clarifies the leading term explicitly, which enables us to track the leading term in the strong alignment more precisely. To complete this section, we provide a proposition that characterizes the fluctuation of the gradient, serving as the counterpart of \Cref{prop:nonsparse-sample-concentration}.

\begin{proposition}[Fluctuation of mini-batch gradient]\label{prop:sparse-gradient-fluctuation}  Under the simplified setting \Cref{def:elem-setting-sparse} where $\psi$ follows \Cref{asp:oracle function}. Additionally, suppose that the sample size
\begin{align}
    n = \Omega\Big(\big((\gamma^2 + \epsilon^2\log k)^{s^\star-1} + L^{-1}\big)^{-1}\cdot \log(d)^{2C_p +2}\Big),
\end{align}
where $C_p$ is the order of the polynomial tail in \ref{asp:polynomial-like}. Then there exists a $\{(z_i,y_i)\}_{i\in [n]}$-measurable event $\cE_{m,2}$ with $\Pr(\cE_{m,2})\ge 1-O(d^{-(c+1)}/T)$, such that on $\cE_{m,2}\cap \cE_m(\epsilon)\cap \tilde\cE_m$, it holds that
\begin{align}
    \big|\dotp{\barg_m}{ e_j}-\dotp{\EE_{\PP_\tst}[\barg_m]}{ e_j}\big| \le \sqrt{\frac{\big((\gamma^2 + \epsilon^2\log k)^{s^\star-1} + L^{-1}\big)\cdot \log(d)}{n}},
\end{align}
for any $j\in [d]$.
\end{proposition}

\begin{proof}[Proof of \Cref{prop:sparse-gradient-fluctuation}]
    See \Cref{proof:sparse-gradient-fluctuation}.
\end{proof}

With these propositions, we have completed the preparation for the analysis of the gradient step and are ready to move on to the proof of the main theorem.

\subsection{Proof of the Main Theorem}

\begin{proof}[Proof of \Cref{thm:sparse}]
~\paragraph{Preparations} We first clarify the final good event that we will use throughout the proof. We first fix $\epsilon= k^{-1/2} \cdot \log k$, then it holds by \Cref{lem:sphere_coordinate_tail} that for each $m$ and $t$, we have
\begin{align}
    \Pr\Big({\cE_m^\t(\epsilon)}^c\Big) &\le  Ld\cdot O( \exp\{-ck\} +  k^{-\log k/4}).
\end{align}
Since $L$ and $d$ are at most polynomials in $k$, we see that for sufficiently large $k$, it holds that $ \Pr\Big({\cE_m^\t(\epsilon)}^c\Big) \le k^{- s^\star}$. Additionally, we see that $\gamma = k^{-1/2}$ is fixed and our parameter configuration
\begin{align}
    n = \Omega\Big( (k \log^3 k)^\sst \cdot \log d\Big) , \quad L  = \Omega(k^{(\sst+3)/2}\cdot \log(k)^{\sst-1 })
\end{align}
are clearly compatible with the conditions in \Cref{prop:sparse-gradient-1st-moment} and \Cref{prop:sparse-gradient-fluctuation}. At the $t$-th step, the mini-batch $\{(z_i^\t,y_i^\t)\}_{i\in [n]}$, $\theta = \theta_m^\t$ and the error-free gradient
\begin{align}
    \barg_m^\t = \frac{1}{nL}\sum_{l=1}^L\sum_{i=1}^n \big(\psi(y_i^\t, \dotp{w_{m,l}^\t}{z_i^\t})\cdot z_i^\t - \hatpsi_1 (y_i^\t)\cdot w_{m,l}^\t\big) \label{eq:sparse-gradient-bar-t}
\end{align}
together form an instance of \Cref{def:elem-setting-sparse}, for which we can find the event $\cE_{m,1}^\t$ and $\cE_{m,2}^\t$, both with probability at least $1-O(d^{-c-1}/T)$, such that on $\cE_{m,1}^\t\cap \cE_{m,2}^\t$ the results in \Cref{prop:sparse-gradient-1st-moment} and \Cref{prop:sparse-gradient-fluctuation} hold.

The gradient we use in the algorithm differs from Eq.~\eqref{eq:sparse-gradient-bar-t} by the error term $\err_{m,l,i}^\t\cdot z_i^\t$. To control this difference, we define
\begin{align}
    \cE_{m,3}^\t  &= \Big\{\sup_{i} \norm{z_i^\t}\le \sqrt{d} \Big\}.
\end{align}
Given our specification on $\sup_{m,l,i,t} \err_{m,l,i}^\t$, it holds on $\cE_{m,3}^\t$ that for any $v\in \SS^{d-1}$:
\begin{align}
   \big|\norm{g_m^\t}_2 -\norm{\barg_m^\t}_2 \big|\vee |\dotp {g_m^\t}{v} - \dotp{\barg_m^\t}{v}|  &\le \norm{g_m^\t - \barg_m^\t}_2 \\
   &\le \sup_{i} \norm{z_i^\t}\cdot \sup_{l} \norm{\err_{m,l,i}^\t}_2\\
   &\le d^{-9\sst}. \label{eq:sparse-gradient-difference-advnoise}
\end{align}
Our final event is fixed to be
\begin{align}
    \cE = \cE_0 \cap \bigcap_{m\in [d]} \bigcap_{t=1}^T \Big(\cE_m^\t(\epsilon)\cap \tilde\cE_m \cap \cE_{m,1}^\t \cap \cE_{m,2}^\t \cap \cE_{m,3}^\t\Big).
\end{align}
With union bound, we have that $\Pr(\cE)\ge 1-O(d^{-c})$ for some constant $c>0$.

To avoid confusion, we denote for each $m$ that $\cg_m^\t = \EE_{\PP_{\tst}}[\barg_m^\t]$.
Later we will encounter some data dependent index $\hatm$ and using $\cg_{\hatm}^\t$ avoids the ambiguity of the expectation. With our choice of $n$ and $L$, \Cref{prop:sparse-gradient-fluctuation} guarantees that for any $m,t,j$, it holds that
\begin{align}
    \big|\barg_{m,j}^\t - \cg_{m,j}^\t\big| &\lesssim k^{-(\sst -1/2)}\cdot \log^{-3/2}k.   \label{eq:sparse-gradient-fluctuation-bound}
\end{align}
In the sequel, we will drop the superscript $t$ whenever there is no ambiguity. We will frequently use $\tg_m^\t = P_{\mathsf{Top}_k (g_m^\t )}(g_m^\t)$.

\paragraph{Weak alignment}
The proof towards the weak alignment in the initial step comprises three parts. In the first place, we will show that the index of the gradient we choose $\hatm$ guarantees that $|\theta_{\hatm}^\star| \gtrsim k^{-1/2}$.
Thereby, the corresponding gradient exhibits good alignment towards the signal.
Based on this, we can show that the support we choose $\mathsf{Top}_{k}(g_{\hatm})$ is of considerable quality by successfully identifying $\phi^\aast = \{j: |\theta_j^\star|\ge  1/ \sqrt{2k}\}$.
Combining these elements, we can show that the gradient $\tg_{\hatm}$ is well-aligned with the signal $\theta^\star$.

We begin with analyzing the quality of $g_{\wh m }$ where $\wh m= \argmax_m \norm{\tg_{m}}_2 $.
With this objective in mind, we first work on deriving a signal-dependent upper bound for $\norm{\tg_{\hatm}}$.
Note that $\rho_m = \dotp{\theta_m}{\theta^\star}= \theta_m^\star$, and hence $|\rho_m|=|\theta_m^\star|$, where $\theta_m = e_m$ is the initial weight.
Applying \Cref{prop:sparse-gradient-1st-moment}, we have for any $\phi,|\phi| =k$ that
\begin{align}
    \sum_{j\in \phi} |\cg_{\wh m,j}|^2 &\lesssim  (k^{-1/2} |\theta_{\hatm}^\star|)^{2\ind \{s^\star \text{ even}\}} \cdot \big(k^{-1/2}|\theta_{\wh m}^\star| +k^{-1}\delta_+\big)^{2(s^\star-1-\ind\{s^\star \text{ even}\})} \cdot \sum_{j\in \phi}{\theta^\star_j}^2  \\
    &\qquad +  \big(k^{-2}\vee   (k^{-1/2}|\theta_{\hatm}^\star|)^2 \cdot ( k^{-1/2}|\theta_{\hatm}^\star| + k^{-1}\delta_+)^{2(\sst-1)} + k^{-2\sst} \big) \sum_{j\in \phi}  {\theta_j^\star}^2 \\
    &\qquad +   (k^{-1/2}|\theta_{\wh m}^\star|+k^{-1}\delta_+)^{2s^\star }
    \cdot  k^{-1} \sum_{j\in \phi}\Big(\theta_j^2 + (k/d)^{\ind\{j\neq m\}}\Big) +k^{-(2\sst+1)}
\end{align}
Note that we have $\norm{\theta^\star}_\infty\lesssim k^{-1/2}\cdot \log k $, it holds that
\begin{align}
    k^{-2} \vee(k^{-1} \cdot |\theta_\hatm^\star|^2) \cdot(k^{-1/2 }| \theta^\star_\hatm| + k^{-1}\delta_+)^{2(\sst-1)} & \lesssim k^{-2\sst} \cdot \log(k)^{2\sst}.
\end{align}
For any $\phi$ such that $|\phi|=k$, we have $ \sum_{j\in \phi} \theta_j^2 \le 1$ for any $\theta$ such that $\norm{\theta}_0 = k$. Therefore, we can further upper bound this quantity by
\begin{align}
    \sup_{|\phi|=k }\sum_{j\in \phi} |\cg_{\wh m,j}|^2  &\lesssim  (k^{-1/2} |\theta_{\hatm}^\star |+ k^{-1}\delta_+)^{2\sst-2} +k^{-2\sst} \cdot \log(k)^{2\sst} \\
    &\qquad + (k^{-1/2} |\theta_{\hatm}^\star |+ k^{-1}\delta_+)^{2\sst} \cdot  k^{-1} \\
    &\lesssim(k^{-1/2} |\theta_{\hatm}^\star |+ k^{-1}\delta_+)^{2\sst-2}  \\ &\qquad + k^{-2\sst}\cdot  \log^{s^\star/2} k +k^{-2\sst}\cdot \log(k)^{2\sst}.
    \label{eq:sparse-signal-g-topk-norm-upperbound-0}
\end{align}
Combining Eq.~\eqref{eq:sparse-signal-g-topk-norm-upperbound-0}, Eq.~\eqref{eq:sparse-gradient-fluctuation-bound} and Eq.~\eqref{eq:sparse-gradient-difference-advnoise}, we conclude that
\begin{align}
    \sup_{|\phi|=k}\sum_{j\in \phi }| g_{\widehat{m},j}|^2&\lesssim  k\cdot \sup_j |g_{\hatm,j} - \barg_{\hatm,j}|^2 +  k \cdot \sup_j |\barg_{\hatm,j} - \cg_{\hatm,j}|^2 + \sup_{|\phi|=k}\sum_{j\in \phi} |\cg_{\hatm,j}|^2\\
    &\lesssim   k^{-(\sst-1)}\cdot( |\theta^\star_{\hatm}|+ k^{-1/2}\delta_+)^{2 \sst-2}+k^{-2\sst }\cdot  \log(k)^{2\sst} .
    \label{eq:sparse-signal-g-topk-norm-upperbound-1}
\end{align}
By definition of $\tg_m$,  we can further conclude that
\begin{align}
\norm{\tg_\hatm}_2^2  &= \sup_{|\phi|=k }  \sum_{j\in \phi} |\dotp{g_{\wh m}}{e_j}|^2\\
&\lesssim k^{-(s^\star-1)} \cdot|\theta_{\wh m}^\star|^{2(s^\star-1)} + o\big(k^{-2(s^\star-1)}\big).
 \label{eq:sparse-signal-g-topk-norm-upperbound-2}
\end{align}
On the other hand, for $m\in \phi^\aast$, we have $(\gamma|\rho| +k^{-1/2}|\theta_m^\star| +k^{-1}\delta_+)^r \gtrsim k^{-r}$.  Then it holds by \Cref{prop:sparse-gradient-1st-moment} that
\begin{align}
    \sum_{j\in \phi^\star} |\cg_{m,j}|^2 &\gtrsim  k^{-2(s^\star-1)}  - \widetilde{O}(k^{-2s^\star}),
\end{align}
and by Eq.~\eqref{eq:sparse-gradient-fluctuation-bound} and~\eqref{eq:sparse-gradient-difference-advnoise}, we have that
\begin{align}
    \norm{\tg_m}_2^2 &\gtrsim k^{-2(s^\star-1)} - \widetilde{O}(k^{-2s^\star}). \label{eq:sparse-signal-g-topk-norm-lowerbound-1}
\end{align}
Now, combining Eq.~\eqref{eq:sparse-signal-g-topk-norm-lowerbound-1} with Eq.~\eqref{eq:sparse-signal-g-topk-norm-upperbound-2} by the definition of $\hatm$, we have that
\begin{align}
   k^{-(\sst-1)} |\theta_{\wh m}^\star|^{2(s^\star-1)} + o\big(k^{-2(s^\star-1)}\big)  \gtrsim k^{-2(s^\star-1)} - \widetilde{O}(k^{-2s^\star}) . \label{eq:sparse-signal-g-topk-norm-lowerbound-2}
\end{align}
We conclude the first step from the last inequality that there exists a global constant $c_{1}>0$ such that for sufficiently large $k$:
\begin{align}
 |\theta_{\wh m}^\star| \ge  \Big(c_1\cdot \big(1- o(1)\big) \cdot k^{-2(s^\star-1)}  \cdot k^{s^\star-1}\Big)^{1/(2(s^\star-1))} \ge c_1' k^{-1/2}.  \label{eq:sparse-signal-selected}
\end{align}

Now we move on to the support identification. We have shown that $|\theta^\star_\hatm| \gtrsim k^{-1/2}$. To establish $ \phi^\aast  \subset \hat\phi = \mathsf{Top}_k(g_\hatm)$, it is sufficient to demonstrate that
\begin{align}
    \sup_{j\notin \phi^\star} |g_{\hatm,j}|  \le \inf_{j\in \phi^\star} |g_{\hatm,j}|.\label{eq:relation}
\end{align}
In the following, we will bound each side separately. Consider $j\notin \phi^\star$, we have by \Cref{prop:sparse-gradient-1st-moment} that
\begin{align}
    |\cg_{\hatm,j}|  &\lesssim \big( k^{-1/2}|\theta_{\hatm}^\star| + k^{-1}\delta_+\big)^{\sst}\cdot\big( k^{-1/2} |\theta_j| + k^{-1/2}(k/d)^{1/2}\big)+ k^{-(\sst+1)}.
\end{align}
Combining this upper bound with Eq.~\eqref{eq:sparse-signal-selected} and Eq.~\eqref{eq:sparse-gradient-fluctuation-bound} gives us that
\begin{align}
    |g_{\hatm,j}| &\le  |g_{\hatm,j} - \barg_{\hatm,j}| + |\barg_{\hatm,j} - \cg_{\hatm,j}| + |\cg_{\hatm,j}| \\
    &\lesssim  (k^{-1/2} \cdot | \theta_{\hatm}^\star|+k^{-1}\delta_+)^{\sst} \cdot\big(k^{-1/2} |\theta_j|+ k^{-1/2}(k/d)^{1/2}\big) +k^{-(\sst+1)} \\ &\qquad +
    d^{-9\sst} + k^{-(\sst-1/2)}\cdot \log^{-3/2}k \\
    &\lesssim  k^{-(\sst-1/2)}\cdot \log^{-3/2}k + \widetilde{O}(k^{-(\sst+1)}).\label{eq:sparse-signal-g-coordinate-case-m-off-support}
\end{align}

 On the other hand, for $j\in \phi^\aast$, we have that
\begin{align}
    |\cg_{\hatm,j}| &\gtrsim |k^{-1/2}\theta_{\hatm}^\star + k^{-1}\delta_{+}|^{s^\star-1}\cdot |\theta_j^\star|\\
    &\gtrsim k^{-(\sst-1/2)}
\end{align}
Similarly, it holds that
\begin{align}
    |g_{\hatm,j}| &\gtrsim |\cg_{\hatm,j}|-|g_{\hatm,j} - \barg_{\hatm,j}| -|\barg_{\hatm,j} - \cg_{\hatm,j}|  \\
    &\gtrsim  k^{-(\sst-1/2)}- o(k^{-(\sst-1/2)}).  \label{eq:sparse-signal-g-coordinate-case-m-on-support}
\end{align}
Comparing Eq.~\eqref{eq:sparse-signal-g-coordinate-case-m-off-support} and Eq.~\eqref{eq:sparse-signal-g-coordinate-case-m-on-support}, we successfully validate Eq.~\eqref{eq:relation}, and consequently $\phi^\aast\subset \hat\phi$.

We complete our proof of weak alignment by analyzing the inner product between $\tg_{\hatm}$ and $\theta^\star$ and the  norm of $\tg_{\hatm}$ respectively.
Since $|\tst_\hatm|\gtrsim k^{-1/2}$, we have that
\begin{align}
    (\gamma|\rho_\hatm|+k^{-1/2}\cdot|\tst_\hatm|+k^{-1}\delta_+)^{s^\star-1} \wedge (\gamma|\rho_\hatm|+k^{-1/2}\cdot|\tst_\hatm|+k^{-1}\delta_+)^{s^\star-2} \cdot |\gamma \rho_\hatm| \gtrsim k^{-(s^\star-1)/2}\cdot |\theta^\star_\hatm|^{\sst-1}
\end{align}
For the inner product, we apply \Cref{prop:sparse-gradient-1st-moment} for $\hatm$ and each $j\in \phi^\aast$ and get that
\begin{align}
g_{\hatm,j}\cdot \theta_j^\star &\simeq   {{\theta^\star_j}^2}\cdot  k^{-(\sst-1)/2}\cdot |\theta^\star_{\hatm}|^{\sst-1 }\cdot \sign({\theta^\star_\hatm})  \\&\qquad-  \big(  |R_{\hatm,j}|+   |g_{\hatm,j}- \barg_{\hatm,j}| + |\barg_{\hatm,j}-\cg_{\hatm,j}| \big) \cdot|{\theta^\star_j}|.\label{eq:sparse-up-step1-0}
\end{align}
Since $ \phi^\aast \subset \hat\phi$, we can lower bound the summation of the leading term as
\begin{align}
\sum_{j\in \hat\phi}{\theta_j^\star}^2  \cdot k^{-(\sst-1)/2}\cdot |\theta^\star_\hatm|^{\sst-1} &\gtrsim  \sum_{j\in \phi^\aast}{\theta_j^\star}^2  \cdot k^{-(\sst-1)/2}\cdot |\theta^\star_\hatm|^{\sst-1} \\
&\gtrsim  k^{-(\sst-1)/2}\cdot |\theta^\star_\hatm|^{\sst-1}\cdot   k^{-1}\cdot |\phi^\aast| \\
&\ge  k^{-(\sst-1)/2}\cdot |\theta^\star_\hatm|^{\sst-1}, \label{eq:sparse-up-step1-1}
\end{align}
where the last line holds by the definition of $\cE_{0,\sharp}\subset \cE_0$.  For the term associated with $R_{\hatm,j}$, we have by the characterization in \Cref{prop:sparse-gradient-1st-moment} that
\begin{align}
    \sum_{j \in \hat\phi} |R_{\hatm,j}| \cdot |\theta_j^\star|&\lesssim (k^{-\sst} \cdot \log (k)^{\sst} +  k^{-\sst}) \cdot \sum_{j\in  \phi^\star}  |\theta_j^\star|^2  \\
    &\qquad +  k^{-\sst}\cdot \log(k)^{\sst} \cdot k^{-1/2}\cdot \Big(\sum_{j\in\hat\phi}|\theta_j| \cdot|\theta^\star_j|+ \sum_{j\in\hat\phi}|\theta^\star_j|\cdot(k/d)^{\ind\{j\neq m\}}\Big)  \\&\qquad + k^{-(\sst+1)} \cdot \sum_{j\in \hat\phi} |\theta_j^\star|.
\end{align}
Applying the Cauchy-Schwarz inequality, we have that
\begin{align}
    \sum_{j\in\hat\phi}|\theta_j^\star|\cdot\big( |\theta_j| +(k/d)^{\ind\{j\neq m\}}\big) &\le  (\sum_{j\in [d]} |\theta^\star_j|^2)^{1/2}\cdot \Big((\sum_{j\in [d]} |\theta_j|^2)^{1/2} +(1+k^3/d^2)^{1/2}\Big) \\ &=O(1). \label{eq:sparse-up-step1-2-ip}
\end{align}
And therefore
\begin{align}
    \sum_{j \in \hat\phi} |R_{\hatm,j}| \cdot |\theta_j^\star|&\lesssim  k^{-\sst} \cdot \log (k)^{\sst} +  k^{-\sst} + \widetilde{O}(k^{-(\sst+1/2)}). \label{eq:sparse-up-step1-2}
\end{align}
For the rest of the error terms, we have by Eq.~\eqref{eq:sparse-gradient-fluctuation-bound} and Eq.~\eqref{eq:sparse-gradient-difference-advnoise} that
\begin{align}
    \sum_{j\in \hat\phi}   \big(|g_{\hatm,j} -\barg_{\hatm,j}| +  |\barg_{\hatm,j}-\cg_{\hatm,j}|\big)\cdot |\theta_j^\star| &\lesssim   \sum_{j\in \phi^\star} |\theta_j^\star| \cdot k^{-(\sst-1/2)}\cdot \log(k)^{-3/2} \\
    &\le  k^{-(\sst-1)} \cdot \log(k)^{-3/2}. \label{eq:sparse-up-step1-3}
\end{align}
Combining Eq.~\eqref{eq:sparse-up-step1-0}, Eq.~\eqref{eq:sparse-up-step1-1}, Eq.~\eqref{eq:sparse-up-step1-2} and Eq.~\eqref{eq:sparse-up-step1-3}, we have that
\begin{align}
|\dotp{\tg_\hatm}{\theta^\star}|  \gtrsim  k^{-(\sst-1)/2}\cdot |\theta^\star_\hatm|^{\sst-1} - o(k^{-(\sst-1)}).
\end{align}
For the norm, we already have in Eq.~\eqref{eq:sparse-signal-g-topk-norm-upperbound-2} that
\begin{align}
    \norm{\tg_{\hatm}}_2 &\lesssim  k^{-(\sst-1)/2}\cdot |\theta_{\wh m}^\star|^{(\sst-1)} + o(k^{-(\sst-1)}).
\end{align}

Combining the last two inequalities, we conclude that
\begin{align}
    \frac{|\dotp{P_{\hat\phi}g_{\hatm}}{ \theta^\star}|}{\norm{P_{\hat\phi}g_{\hatm}}} &\gtrsim  \frac{|\theta_{\wh m}^\star|^{s^\star-1} - o(k^{-(\sst-1)/2})}{|\theta_{\wh m}^\star|^{s^\star-1} + o(k^{-(s^\star-1)/2})} = \Theta(1),
\end{align}
given that $|\theta^\star_{\wh m }| \ge c'_1 k^{-1/2}$ for some constant $c'_1>0$. This concludes the proof of weak alignment.

\paragraph{Strong Alignment}
Starting from the second step, we have by \Cref{alg:sparse-warm-up} that all the neurons share the same weight parameter.
Let $\theta$ be the weight parameter in any step after the first step and we suppose that $\rho = \dotp{\theta}{\theta^\star} = \Theta(1)$.
From \Cref{prop:sparse-gradient-1st-moment}, we see that now the choice of the gradient should not change the quality of the gradient significantly, as $\gamma|\rho|\gg k^{-1/2} |\theta_m^\star|$ for any $m\in [d]$.
Therefore, we start by analyzing the alignment increment using any $g_m$.
This alteration does not affect the alignment increment, but substantially simplifies the analysis.

We additionally define a support $\phi^{\dagger} = \{j\in [d]:|\theta_j^\star| \ge k^{-1}\}$.
In the following, we fix an arbitrary $m\in [d]$.
We begin with analyzing the magnitude of $g_{m,j}$ for $j\in \phi^\dagger\setminus \{m\}$.
Since $|\rho|=\Omega(1)$, it holds that
\begin{align}
    (\gamma|\rho|+k^{-1/2}\cdot |\theta_m^\star|+k^{-1}\delta_{s^\star -1})^{\sst-1}\simeq (\gamma|\rho|+k^{-1/2}\cdot |\theta_m^\star|+k^{-1}\delta_{s^\star -2})^{\sst-2}\cdot (\gamma|\rho|)\simeq k^{-(\sst-1)/2}.
\end{align}
With triangle inequality, \Cref{prop:sparse-gradient-1st-moment} indicates that, for $j\in \phi^\dagger$:
\begin{align}
|g_{m,j}| &\ge |\cg_{m,j}| - |\barg_{m,j} -\cg_{m,j}| - |\barg_{m,j}-  g_{m,j}|\\
&\ge  k^{-(\sst-1)/2}\cdot |\theta_j^\star|  -  k^{-\sst/2}\cdot |\theta_j^\star|\\ &\qquad- k^{-(\sst+1)/2}\cdot (|\theta_j|+ (k/d)^{\ind\{j\neq m\}/2 }) - \widetilde{O}( k^{-(\sst-1/2)}).
\end{align}
Similarly, we have for  $j\notin \phi^\star$ that
\begin{align}
|g_{m,j}| &\le |\cg_{m,j}| +|\barg_{m,j} -\cg_{m,j}| + |\barg_{m,j}-  g_{m,j}|\\
&\lesssim k^{-(\sst+1)/2} \cdot ( |\theta_j|+(k/d)^{\ind\{ j\neq m\}/2})  + \widetilde{O}(k^{-(\sst-1/2)}).
\end{align}
Comparing last two inequalities, we have that $|\theta_j^\star|>k^{-1}$ implies that $|g_{m,j}|\ge \max_{j\notin \phi^\star} |g_{m,j}|$ for sufficiently large $k$, and therefore
\begin{align}
    \min_{j\in \phi^\dagger} |g_{m,j}| > \max_{j\notin \phi^\star} |g_{m,j}|,
\end{align}
which means that $\phi^\dagger \subset \hat\phi_m =  \mathsf{Top}_k(g_m)$. Thereby, we have
\begin{align}
    \sum_{j\in \hat\phi_m} |\theta_j^\star|^2 \ge 1- \sum_{j\notin \phi^\dagger} |\theta_j^\star|^2\ge 1- k^{-1}.
\end{align}

We then move on to the alignment analysis. Returning to Eq.~\eqref{eq:sparse-signal-0}, we define
\begin{align}
    \beta_m(\rho, \{\xi_{m,l}\}_{l\in [L]}, \theta^\star,\theta) &= \frac{\EE_{\QQ}[\zeta_{\sst}(y)\cdot \hatpsi_{\sst-1}(y)]\cdot \sqrt{\sst}\cdot \EE_{\tw_{m,l}}[\dotp{w_{m,l}}{\theta^\star}^{\sst-1}]}{\sign(\rho)^{\ind \{s^\star \text{ even}\}}\cdot (\gamma|\rho|)^{s^\star-1}}; \\
    r_{m,j}(\rho, \{\xi_{m,l}\}_{l\in [L]}, \theta^\star,\theta) &=  \EE_{\PP_{\theta^\star}}[\dotp{\barg_m}{e_j}] - \theta_j^\star \cdot (\gamma |\rho|)^{s^\star-1} \cdot\sign(\rho)^{\ind \{s^\star \text{ even}\}}\cdot \beta_m.
    \label{eq:sparse-signal-g-coordinate-representation}
\end{align}
Then it follows from  \Cref{prop:sparse-gradient-1st-moment} that whenever $|\rho|= \Omega(1)$, we have that $\beta_m >0$ and
\begin{align}
\beta_m \vee \beta_m^{-1}&<B\\
 |r_{m,j}|& \le r_{m,j}^a + r_{m,j}^b, \label{eq:sparse-signal-g-coordinate-representation-0}
\end{align}
where
\begin{align}
    r_{m,j}^a &\le C_a k^{-(s^\star+1)/2}\cdot \big( |\theta_j| + (k/d)^{\ind\{j\neq m \}/2}\big);  \\  r_{m,j}^b &\le  C_b \cdot |\theta_j^\star|\cdot (\gamma|\rho |)^{ s^\star}
 \label{eq:sparse-signal-g-coordinate-representation-1}
\end{align}
with some global positive constant $B, C_a$ and $C_0$, whenever the designated parameters are compatible with the definition of our nice event.
With this representation, we can deduce by Eq.~\eqref{eq:sparse-gradient-fluctuation-bound} and Eq.~\eqref{eq:sparse-gradient-difference-advnoise} that
\begin{align}
 |\dotp{\tg_m}{\theta^\star}|&= \Big|\sum_{j \in \wh\phi_m} \dotp{\cg_m}{e_j}\cdot \theta_j^\star \Big| + \sum_{j\in \wh \phi_m}\big( | \dotp{g_m}{e_j} - \dotp{\barg_m}{e_j}| + |\dotp{\barg_m}{e_j} - \dotp{\cg_m}{e_j}|\big)\cdot |\theta_j^\star|\\
 \\
&\ge   \Big| \sum_{j\in \wh \phi_m  }\dotp{\cg_m}{e_j} \cdot \theta_j^\star \Big| -  \sum_{j\in \phi^\star} |\theta^\star_j|\cdot  O(k^{-(s^\star-1/2)}\cdot  \log^{-3/2}k )\\
&\ge  \beta_m   \cdot (\gamma |\rho|)^{s^\star-1}\cdot \sum_{j\in \wh \phi_m } |\theta_j^\star|^2  - \sum_{j\in \phi^\star }   |r_{m,j }|\cdot |\theta^\star_j |   -  O(k^{-(s^\star-1)}\cdot  \log^{-3/2}k )\\
&\ge  \big (\beta_m   \cdot (\gamma |\rho|)^{s^\star-1} - C_b\cdot(\gamma |\rho|)^{\sst}\big)\cdot \sum_{j\in \wh\phi_m} |\theta_j^\star|^2 - 2C_a\cdot k^{-(s^\star+1)/2} -O( k^{-(s^\star-1)}\cdot  \log^{-3/2}k ).\label{eq:sparse-signal-strong-alignment-innerproduct}
\end{align}
where the last inequality holds by Eq.~\eqref{eq:sparse-up-step1-2-ip}.
To complete the analysis, we need an upper bound for $\norm{\tg_m}_2$. Note that by the triangle inequality and Eq.~\eqref{eq:sparse-signal-g-coordinate-representation}, we have that
\begin{align}
\norm{ \tg_m  } &\le\norm{P_{\wh\phi_m}\cg_m}  + \norm{P_{\wh\phi_m}\barg_m - P_{\wh\phi_m}\cg_m } +\norm{P_{\wh\phi_m}\barg_m - P_{\wh\phi_m}g_m} \\
&\le \norm{ \beta_m (\gamma|\rho|)^{s^\star-1}\cdot  P_{\wh\phi_m}\theta^\star} + \norm{ P_{\wh\phi_m} r_m^a } + \norm{P_{\wh \phi_m}r_m^b}\\&\qquad + O(k^{-(s^\star-1)}\cdot \log^{-3/2}k),\label{eq:sparse-signal-strong-alignment-norm-1}
\end{align}
where $r_m^a = (r_{m,1},\dots,r_{m,d}^a )\in \RR^d$ and $r_m^b = (r_{m,1},\dots,r_{m,d}^b )\in \RR^d$. To proceed, note that by Eq.~\eqref{eq:sparse-signal-g-coordinate-representation-1} and Eq.~\eqref{eq:sparse-gradient-fluctuation-bound}, we have that
\begin{align}
    \norm{P_{\wh\phi} r_m^a } &\le  C_a k^{-(s^\star+1)/2}\cdot (\norm{\theta}_2 + \sqrt{1+k^2/d})\\ &\le 3C_a k^{-(s^\star+1)/2};\\
    \norm{P_{\wh\phi} r_m^b } &\le  C_b\cdot (\gamma|\rho|)^{s^\star}\cdot \Big({\sum_{j\in \wh \phi_m} |\theta^\star_j|^2}\Big)^{1/2}.
    \label{eq:sparse-signal-strong-alignment-norm-2}
\end{align}
Putting these upper bounds together, we have that
\begin{align}
    \norm{P_{\wh \phi}g_m}_2 &\le \big(\beta_m \cdot (\gamma |\rho|)^{s^\star-1} +C_b (\gamma |\rho|)^\sst\big)\cdot\Big({\sum_{j\in \wh \phi_m} |\theta^\star_j|^2}\Big)^{1/2}\\&\qquad + 3C_a k^{-(s^\star+1)/2} + O(k^{-(s^\star-1)}\cdot \log^{-3/2}k).
    \label{eq:sparse-signal-strong-alignment-norm-3}
\end{align}
Note that for any $a_1\wedge a_2 >b>0 $, it holds that
\begin{align}
    \frac{a_1-b}{a_2+b} = & \frac{(a_1-b)\cdot (a_2 - b)}{a_2^2-b^2}\ge (a_1/a_2 -b/a_2) \cdot(1-b/a_2).
\end{align}
Setting $\Delta = k^{-1}\vee \big(k^{-(s^\star-1)/2}\cdot \log^{-3/2}k\big)=o(1)$, we get by combining Eq.~\eqref{eq:sparse-signal-strong-alignment-innerproduct} and Eq.~\eqref{eq:sparse-signal-strong-alignment-norm-3} that
\begin{align}
    \frac{\dotp{\tg_m}{\theta^\star}}{\norm{\tg_m}} &\ge
\frac{(\beta_m -C_b \cdot\gamma|\rho|)\cdot \Big(\sum_{j\in \hat\phi_m} {\theta_j^\star}^2\Big)-3C_a \cdot k^{-1} - O(k^{-(\sst-1)/2}\cdot \log^{-3/2}k )}{ \big(\beta_m +  C_b\cdot \gamma|\rho|\big)\cdot \Big(\sum_{j\in \hat\phi_m} {\theta_j^\star}^2\Big)^{1/2} +3C_a\cdot  k^{-1}+O(k^{-(\sst-1)/2}\cdot \log^{-3/2} k)}
    \\
    &\ge  \big(1-  O (\Delta)\big)^{-1}\cdot \Big(\frac{1-C_b \beta_m^{-1}\gamma|\rho|}{1+C_b \beta_m^{-1}\gamma|\rho|} \cdot (\sum_{j\in\hat\phi_m}  {\theta_j^\star}^2 )^{1/2} - O(\Delta)\Big) \\
    &\ge 1-C\cdot k^{-1} - O(\Delta)
\end{align}
where the last line holds because $(1-Ck^{-1})^{r}\ge 1-rC\cdot k^{-1}$ for any $C,r>0$ and sufficiently large $k$. Note that  $k^{-1}= O(\Delta)$, we see that
\begin{align}
    \frac{\dotp{\tg_m}{\theta^\star}}{\norm{\tg_m}} &\ge 1- O(\Delta).
\end{align}
Since \Cref{alg:sparse-warm-up} updates all neurons by gradient sharing in Line~\ref{alg:sparse_update}, the same normalized vector is assigned to every $\theta_m^{(T)}$. Therefore the above alignment bound holds for all $m\in[M]$, concluding the proof of \Cref{thm:sparse}.

\end{proof}

\subsection{Proof of the Key Results}

\subsubsection{Proof of \texorpdfstring{\Cref{prop:sparse-gradient-1st-moment}}{the sparse first-moment proposition}}\label{proof:sparse-gradient-1st-moment}
\begin{proof}[Proof of \Cref{prop:sparse-gradient-1st-moment}]

First, by \Cref{lem:first-moment-decomp}, we have for each $j\in [d]$ that
\begin{align}
\dotp{\EE_{\PP_\tst}[\barg_m]}{e_j }  &=   \sum_{s\ge \sst}  \EE_{\QQ}[\zeta_s(y)\cdot \hatpsi_{s+1}(y)] \cdot \frac{\sqrt{s+1}}{L} \sum_{l=1}^L \dotp{w_{m,l}}{\tst}^{s} \cdot \dotp{w_{m,l}}{e_j}  \\
&\qquad +  \sum_{s\ge \sst}\EE_{\QQ}[\zeta_s(y)\cdot \hatpsi_{s-1}(y)]\cdot  \frac{\sqrt{s}}{L} \cdot \sum_{ l=1}^L \dotp{w_{m,l}}{\tst}^{s-1 } \cdot \dotp{\tst}{e_j}. \\
&= \EE_{\QQ}[\zeta_{\sst}(y)\cdot \hatpsi_{\sst-1}(y)] \cdot \frac{\sqrt{\sst}}{L} \sum_{l=1}^L \dotp{w_{m,l}}{\tst}^{\sst-1} \cdot \dotp{\tst}{e_j}\\ &\qquad+ R_1+R_2,
\label{eq:sparse-signal-0}
\end{align}
where the remainders $R_1,R_2$ are defined as
\begin{align}
R_1 &= \sum_{s\ge \sst}   \EE_{\QQ}[\zeta_{s+1}(y)\cdot \hatpsi_{s }(y)] \cdot \frac{\sqrt{s+1 }}{L} \sum_{l=1}^L \dotp{w_{m,l}}{\tst}^{s} \cdot \dotp{\tst }{e_j};\\
R_2 &=  \sum_{s\ge \sst} \EE_{\QQ}[\zeta_{s}(y)\cdot \hatpsi_{s+1}(y)]\cdot  \frac{\sqrt{s+1 }}{L} \cdot \sum_{ l=1}^L \dotp{w_{m,l}}{\tst}^{s} \cdot \dotp{w_{m,l }}{e_j}.
\end{align}
We also denote the leading signal term as
\begin{align}
    S = \EE_{\QQ}[\zeta_{\sst}(y)\cdot \hatpsi_{\sst-1}(y)] \cdot \frac{\sqrt{\sst}}{L} \sum_{l=1}^L \dotp{w_{m,l}}{\tst}^{\sst-1} \cdot \dotp{\tst}{e_j}.
\end{align}
By definition $R_1$ collects the higher order term that aligns with the signal and $R_2$ collects all the terms in the expected gradient that are parallel to $w_{m,l}$. In comparison to the non-sparse case, here we are analyzing the gradient coordinate-wisely. Therefore, $R_1$ and $R_2$ need to be controlled separately.

\paragraph{Analysis for the dominant term $S$ in Eq. ~\eqref{eq:sparse-signal-0}}
We first define
\begin{align}
    \tw_{m,l} = w_{m,l}\cdot \ind\{\sup_{j}|\dotp{\xi_{m,l}}{e_j}|\le \epsilon\}.
\end{align}
By definition of $\tilde w_{m,l}$, we have that $\tilde w_{m,l}, l\in[L]$ are independent of each other and $\tw_{m,l}=w_{m,l}$ on event $\cE_m(\epsilon)$. We can approximate the expectation of $\dotp{w_{m,l}}{\tst}$ as
\begin{align}
    \EE_{\tw_{m,l}} [\dotp{\tw_{m,l}}{\tst}^{s^\star-1}] &= \EE_{w_{m,l}}\big[\dotp{w_{m,l}}{\tst}^{s^\star-1} \cdot \ind\big\{\sup_{j }|\dotp{\xi_{m,l}}{e_j}|\le \epsilon\big\}\big] \\
    &\simeq \EE_{w_{m,l}}[\dotp{w_{m,l}}{\tst}^{s^\star-1}] \pm \Pr\big(\sup_{j}|\dotp{\xi_{m,l}}{e_j}|>\epsilon\big) \\
    &\simeq \EE_{w_{m,l}}[\dotp{w_{m,l}}{\tst}^{s^\star-1}] \pm  \Pr( {\cE_m(\epsilon)}^c).
\end{align}
Here the last line holds because  ${\cE_m(\epsilon)}^c = \cup_l \{\sup_j  |\dotp{\xi_{m,l}}{ e_j}| >\epsilon \}$.
For the first term, by \Cref{prop:sparse-w-expectation-warmup}, we have that on $\cE_{0,s^\star-1-\ind \{s^\star \text{ even}\}}$
\begin{align}
    \EE_{w_{m,l}}[\dotp{w_{m,l}}{\theta^\star}^{s^\star-1}] &\simeq \begin{cases}
            \gamma\rho \cdot \big( \gamma |\rho| +k^{-1/2}|\theta_m^\star| +k^{-1}\delta_{s^\star-2 }\big)^{s^\star-2} & \text{if $s^\star$ even};\\
              (\gamma |\rho| +k^{-1/2}|\theta_m^\star| + k^{-1}\delta_{s^\star-1})^{s^\star-1} & \text{if $s^\star$ odd}. \label{eq:sparse-signal-S-1}
    \end{cases}
\end{align}
For the second moment that is involved in the Bernstein's inequality, we have that on $\cE_{0,2s^\star-2}$
\begin{align}
    \EE[\dotp{\tilde w_{m,l}}{\theta^\star}^{2s^\star-2}] & =\EE[\dotp{ w_{m,l}}{\theta^\star}^{2s^\star-2}\ind\{\sup_{j}|\dotp {\xi_{m,l}}{e_j}| \le \epsilon\}] \\
    &\le \EE[\dotp{w_{m,l}}{\theta^\star}^{2s^\star-2}] \\
    &\simeq  (\gamma|\rho| + k^{-1/2}|\theta_m^\star| + k^{-1}\delta_{2\sst-2})^{2s^\star-2}
\end{align}
To proceed, we have by Bernstein's inequality (\Cref{lem:bernstein}) that there exists an event $\cE_{m,11}$ with $\Pr(\cE_{m,11})\ge 1-O(d^{-c_{b,11}})$. It holds on $\cE_{m,11}\cap \cE_m(\epsilon )$ that
\begin{align}
\frac{1}{L} \sum_{l}  \dotp{w_{m,l}}{\theta^\star }^{s^\star-1 } &= \frac{1}{L} \sum_{l}  \dotp{\tw_{m,l}}{\theta^\star }^{s^\star-1 }  \\
&\simeq \EE_{w_{m,l}} [\dotp{ w_{m,l}}{\theta^\star}^{s^\star-1} ] + E,
\end{align}
where the  error term $E$ can be bounded by
\begin{align}
    |E| &\le \big(\gamma|\rho|+k^{-1/2}|\theta_m^\star|+ k^{-1}\delta_{2\sst-2}\big)^{s^\star-1}\cdot \sqrt\frac{\log(d)}{L} + \frac{\epsilon^{s^\star-1}\log(d)}{L}+ \Pr\big({\cE}^c_{m}(\epsilon)\big).
\end{align}
Moreover, the assumption that
\begin{align}
    L &\gtrsim \log d  \cdot \Big( k^{2} \vee  \Big(\epsilon^{s^\star-1 } \cdot k^{s^\star}\Big)\Big) ; \qquad \Pr({\cE_m(\epsilon)}^c )\le    k^{-s^\star}; \label{eq:first-order-moment-cond-1}
\end{align}
allows us to simplify the upper bound for $E$, since
\begin{align}
 |E| &\lesssim    k^{-1}\cdot \big(\gamma|\rho|+k^{-1/2}|\theta_m^\star|+ k^{-1}\delta_{2\sst-2}\big)^{s^\star-1}   +  k^{-s^\star}. \label{eq:sparse-signal-S-E}
\end{align}

In conclusion, we have on $\cE_{m,11}\cap \cE_m(\epsilon)$ that
\begin{align}
S\simeq  (\EE_{w_{m,l}}[\dotp{w_{m,l}}{\theta^\star}^{s^\star-1}]+E)\cdot \dotp{\theta^\star}{e_j}.
\end{align}
We retain this form for further simplification.

\paragraph{Analysis for the first remainder $R_1$ in Eq.~\eqref{eq:sparse-signal-0}}

For any  $s,s'$, it holds by the  property of likelihood ratio decomposition that
\begin{align}
    \EE_\QQ[|\zeta_{s}(y)\cdot  \hatpsi_{s'}(y )|]&\le \EE_\QQ[ \zeta_s(y)^2]^{1/2}\cdot \EE_\QQ [\hatpsi_{s'}(y)^2]^{1/2} \\
    &\le  \sqrt{\sum_{s'\ge 0} \EE_\QQ[ \hatpsi_{s'}(y)^2]},\label{eq:sparse-R1-1}
\end{align}
and the last quantity is a constant that is independent of $s,s'$. To bound the summation for $s\ge \sst$, we use \Cref{lem:polarized-weight-nice-event-sparse}, which gives $|\dotp{w_{m,l}}{\tst}|\le  \gamma|\rho| +\epsilon<1/2$ on the nice event. Thus,
\begin{align}
    \sum_{s\ge \sst} \frac{\sqrt{s+1}}{L} \sum_{l=1}^L  |\dotp{w_{m,l}}{\tst}|^{s}
    &\lesssim \sum_{s\ge \sst} \sqrt{s+1}\cdot \Big(\frac{1}{2}\Big)^{s-\sst} \cdot \frac{1}{L}\sum_{l=1}^L |\dotp{w_{m,l}}{\tst}|^{\sst} \\
    &\lesssim  \frac{1}{L} \sum_{l=1}^L |\dotp{w_{m,l}}{\tst}|^{\sst}. \label{eq:sparse-R1-2}
\end{align}
Now it reduces to bounding the right-hand side of Eq.~\eqref{eq:sparse-R1-2}.
Note that on $\cE_m(\epsilon)$, $\tw_{m,l} = w_{m,l}$. We can first track the first and second moment of $\dotp{w_{m,l}}{\tst}$ as
\begin{align}
    \EE_{\tilde{w}_{m,l}}[|\dotp{\tilde{w}_{m,l}}{\theta^\star}|^{s^\star}] &\le  \EE_{w_{m,l}}[|\dotp{w_{m,l}}{\theta^\star}|^{s^\star}]\\&\le \EE_{w_{m,l}}[ \dotp{w_{m,l}}{\theta^\star}^{2s^\star}]^{1/2}.
\end{align}
To bound the last quantity, we see that given $\cE_{0, 2s^\star}$, \Cref{prop:sparse-w-expectation-warmup} reads
\begin{align}
    \EE_{w_{m,l}}[\dotp{w_{m,l}}{\theta^\star}^{2s^\star}] &\le (\gamma |\rho|+ k^{-1/2}|\theta_m^\star| + k^{-1}\delta_+)^{2s^\star}.
\end{align}
By Bernstein's inequality, there exists an event $\cE_{m,12}$ with $\Pr(\cE_{m,12})\ge 1-O(d^{-c_{b,12}})$ such that on $\cE_{m,12}\cap \cE_m(\epsilon)$, it holds that
\begin{align}
    \frac{1}{L} \sum_{l=1}^L |\dotp{w_{m,l}}{\theta^\star}|^{s^\star} &\lesssim \Big(1+\sqrt{\frac{\log d}{L}}\Big)\cdot (\gamma|\rho| + k^{-1/2}|\theta_m^\star| + k^{-1}\delta_+)^{s^\star} + \frac{\epsilon^{s^\star}\log(d)}{L},
\end{align}
Given that $L\gtrsim \log(d) \cdot \big(k\vee (\epsilon^{s^\star}\cdot k^{s^\star})\big)$, it further holds that
\begin{align}
    \frac{1}{L} \sum_{l=1}^L |\dotp{w_{m,l}}{\theta^\star}|^{s^\star} &\lesssim  (\gamma|\rho| + k^{-1/2}|\theta_m^\star| + k^{-1}\delta_+)^{s^\star}+ k^{-s^\star}.
\end{align}
In conclusion, it holds on $\cE_{m,12}\cap \cE_m(\epsilon)$ that
\begin{align}
    R_1 \lesssim \Big((\gamma|\rho| + k^{-1/2}|\theta_m^\star| + k^{-1}\delta_+)^{s^\star-1}  + k^{-s^\star}\Big)\cdot |\dotp{\theta^\star}{e_j}|.
\end{align}

\paragraph{Analysis for the second remainder $R_2$ in Eq.~\eqref{eq:sparse-signal-0}}

Similar to Eq.~\eqref{eq:sparse-R1-2}, we can first upper bound $R_2$ as
\begin{align}
    |R_2|&\lesssim \frac{1}{L} \sum_{l=1}^L |\dotp{w_{m,l}}{\theta^\star}|^{s^\star} \cdot  |\dotp{w_{m,l}}{e_j}| \\
    &= \frac{1}{L} \sum_{l=1}^L |\dotp{\tw_{m,l}}{\theta^\star}|^{s^\star} \cdot  |\dotp{\tw_{m,l}}{e_j}|,
\end{align}
where the last line holds by the definition of $\cE_m(\epsilon)$.
We decouple the product with the Cauchy-Schwarz inequality as follows:
\begin{align}
    \EE_{\tilde w_{m,l}}[|\dotp{\tilde w_{m,l}}{\theta^\star}|^{s^\star}\cdot |\dotp{\tilde{w}_{m,l}}{e_j }|] &\le \EE_{ w_{m,l}} [ |\dotp{ w_{m,l}}{{\theta^\star}}|^{s^\star} \cdot |\dotp{w_{m,l}}{e_j}|]
    \\& \le \EE_{w_{m,l}}[|\dotp{w_{m,l}}{\theta^\star}|^{2s^\star}]^{1/2} \cdot \EE_{w_{m,l}}[|\dotp{w_{m,l}}{e_j}|^2]^{1/2}
\end{align}
The first term in the upper bound can be tackled with \Cref{prop:sparse-w-expectation-warmup}. For the second term, we have that
\begin{align}
    \EE_{w_{m,l}}[\dotp{w_{m,l}}{e_j}^2] &\lesssim \EE[(\dotp{\xi_{m,l}}{e_j} + \gamma\cdot \dotp{\theta}{e_j})^2] \\
    &\lesssim \EE[\dotp{\xi_{m,l}}{e_j}^2] + \gamma^2\theta_j^2 \\
    &= \gamma^2\theta_j^2 + \EE[ \ind\{j\in \phi_{m,l}\}\cdot \xi_{m,l,j}^2] \\
    &\lesssim \gamma^2\theta_j^2 + k^{-1}\cdot (k/d)^{\ind\{j\neq m\}}.
\end{align}
Here, the first line holds because $\norm{\gamma \theta +\xi}_2 \ge 1/2$. The last line holds by applying \Cref{lem:moment_sphere} and that $\PP(j\in \phi_{m,l})\le  k/d$ for $j\neq m$.  Note that each term in the summation is bounded by $\epsilon \log (k)$ up to a constant on $\cE_m(\epsilon)$. By Bernstein's inequality (\Cref{lem:bernstein}), there exists an event $\cE_{m,13}$ with $\Pr(\cE_{m,13})\ge 1-O(d^{-c_{b,13}})$. It holds on $\cE_{m,13}\cap \cE_m(\epsilon)$ that
\begin{align}
    |R_2| {\lesssim}& \Bigg(1+ \sqrt{\frac{\log(d)}{L}}\Bigg)\cdot\Big(\gamma|\rho|+ k^{-1/2}|\theta_m^\star|+ k^{-1}\delta_+\Big)^{s^\star}  \cdot \big(\gamma |\theta_j| + k^{-1/2}(k/d)^{\ind\{j\neq m\}/2}\big) \\ &+ \frac{\epsilon^{s^\star+1} \log(d)}{L}.
\end{align}
Given that
\begin{align}
    L \gtrsim \log(d) \cdot \Big(k\vee (\epsilon^{s^\star+1}\cdot  k^{s^\star+1})\Big),\label{eq:first-order-moment-cond-5}
\end{align}
we conclude that it holds on $\cE_{m,13}\cap \cE_m(\epsilon)$ that
\begin{align}
    |R_2|{\lesssim}& \Big(\gamma|\rho|+ k^{-1/2}|\theta_m^\star|+ k^{-1}\delta_+\Big)^{s^\star}  \cdot \big(\gamma |\theta_j| + k^{-1/2}(k/d)^{\ind\{j\neq m\}/2}\big)+ k^{-(s^\star+1)}.
\end{align}

\paragraph{Summary of first-order moment}
We now merge previous results to summarize the results for the first-order moment.
Note that it is sufficient to set
\begin{align}
    L = \Omega\Big(\log(d)\cdot \big(k\vee \epsilon^{s^\star -1}( k \cdot \log k)^{s^\star +1}\big)\Big)
\end{align}
Define the final event as $\cE_{m,1}= \cE_{m,11}\cap \cE_{m,12}\cap \cE_{m,13}$, which is $\{w_{m,l}\}_{l\in [L]}$ measurable.
By previous analysis, it holds on this event that
\begin{align}
    S  &\simeq  \theta_j^\star\cdot (\gamma \rho)^{\ind \{s^\star \text{ even}\}} \cdot \big( \gamma |\rho| +k^{-1/2}|\theta_m^\star| +k^{-1}\delta_{s^\star-1-\ind \{s^\star \text{ even}\}}\big)^{s^\star-1-\ind \{s^\star \text{ even}\}} + \theta_j^\star \cdot E;\\
    R_1  &\lesssim \Big(\big(\gamma|\rho| + k^{-1/2}|\theta_m^\star| + k^{-1}\delta_+\big)^{s^\star} + k^{-s^\star} \Big)\cdot |\theta_j^\star|;\\
    R_2 &\lesssim \Big(\gamma|\rho| + k^{-1/2}|\theta_m^\star| + k^{-1}\delta_+\Big)^{s^\star  } \cdot (\gamma |\theta_j| + k^{-1/2}\cdot (k/d)^{\ind\{j\neq m \}/2}) +k^{-(\sst +1)}.
\end{align}
Following the error term $E$ in Eq.~\eqref{eq:sparse-signal-S-E}, we define $R= R_1 +R_2 +E$, which can be bounded by
\begin{align}
    |R|&\lesssim  \Big(k^{-1}\vee(\gamma |\rho|+k^{-1/2}|\theta_m^\star|)\cdot(\gamma|\rho|+k^{-1/2}|\theta_m^\star|+ k^{-1}\delta_+)^{s^\star-1}  + k^{-s^\star}\Big)\cdot |\theta_j^\star|  \\
    &\qquad + \Big(\gamma|\rho| + k^{-1/2}|\theta_m^\star| + k^{-1}\delta_+\Big)^{s^\star  } \cdot (\gamma |\theta_j| + k^{-1/2}\cdot (k/d)^{\ind\{j\neq m \}/2})  + k^{-(s^\star+1)}. \label{eq:sparse-signal-R}
\end{align}
And we summarize the first moment on $\cE_{m,1}\cap \cE_{m}(\epsilon)$ as
\begin{align}
    \EE_{\PP_{\theta^\star}}[\dotp{\barg_m}{e_j}] &\simeq  \theta_j^\star\cdot (\gamma\rho)^{\ind\{s^\star \text{ even}\}} \cdot \big( \gamma |\rho| +k^{-1/2}|\theta_m^\star| +k^{-1 }\delta_{s^\star-1-\ind \{s^\star \text{ even}\}}\big)^{s^\star-1-\ind\{s^\star \text{ even}\}}  + R,
\end{align}
where $R$ is upper bounded in Eq.~\eqref{eq:sparse-signal-R}.
\end{proof}

\subsubsection{Proof of \texorpdfstring{\Cref{prop:sparse-gradient-fluctuation}}{the sparse fluctuation proposition}}\label{proof:sparse-gradient-fluctuation}

\begin{proof}[Proof of \Cref{prop:sparse-gradient-fluctuation}]
Similar to the proof of \Cref{prop:nonsparse-sample-concentration}, the proof of this proposition comprises two parts. To begin with, we calculate the variance of each coordinate of $\barg_m$.
~\paragraph{Second moment calculation}
It suffices to consider the variance of the first sample. To this end, we define
\begin{align}
\barg_{m,1} = \frac{1}{L}\sum_{l=1}^L \big(\psi(y_1, \dotp{w_{m,l}}{z_1})\cdot z_1 - \hatpsi_1 (y_1)\cdot w_{m,l}\big).
\end{align}
For any $v\in \{e_1,e_2,\dots, e_d\}$, it holds by the definition of $\barg_{m, 1}$ that
\begin{align}
    \EE_{\PP_{\theta^\star}}[\langle \barg_{m,1}, v\rangle^2]
    &\lesssim \frac{1}{L^2} \sum_{l, l'=1}^L \EE_{\PP_{\theta^\star}}\left[
        \psi(y, \langle w_l, z\rangle)  \psi(y, \langle w_{l'}, z\rangle) \langle z, v\rangle^2
    \right]  + \frac{1}{L^2} \sum_{l, l'=1}^L \EE_{\PP_{\theta^\star}}\left[
        \hat\psi_1(y)^2   \langle w_{l}, v\rangle \langle w_{l'}, v\rangle
    \right]\\
    &= \frac{1}{L^2} \sum_{l\neq l'} \EE_{\QQ}\bigg[
        \psi(y, \langle w_l, z\rangle) \psi(y, \langle w_{l'}, z\rangle) \langle z, v\rangle^2 \cdot
        \bigg(
            1 + \sum_{s\ge s^\star} \zeta_s(y) h_s(\langle \theta^\star, z\rangle)
        \bigg)
    \bigg] \\
    &\qquad + \frac{1}{L^2} \sum_{l=1}^L \EE_{\QQ}\left[
        \psi(y, \langle w_l, z\rangle) \psi(y, \langle w_{l}, z\rangle) \langle z, v\rangle^2
    \right] + \frac{1}{L^2} \sum_{l\neq l'} \EE_\QQ[\hat\psi_1(y)^2] \langle w_l, v\rangle \langle w_{l'}, v\rangle \\
    &\qquad + \frac{1}{L^2} \sum_{l=1}^L \EE_\QQ[\hat\psi_1(y)^2] \langle w_l, v\rangle^2. \label{eq:sparse-2nd-moment-1}
\end{align}
In the same manner as the non-sparse case, we can derive an $O(1/L)$ upper bound for the second and the last summation, which run over all $l=l'$.
By \Cref{lem:polarized-weight-nice-event-sparse}, we already have that
\begin{align}
\sup_{l,j}|\dotp{w_{m,l}}{e_j }| &\lesssim \gamma +\epsilon\coloneqq \epsilon_{\textsf{chart}};\\  \sup_{l}|\dotp{w_{m,l}}{\theta^\star}| &\lesssim \gamma|\rho| + \epsilon \coloneqq \epsilon_{\textsf{signal}};\\
\sup_{l\neq l'} |\dotp{w_{m,l}}{w_{m,l'}}| &\lesssim \gamma^2 + \epsilon^2 \log(k) \coloneqq \epsilon_{\textsf{mutual}}.
\end{align}
Applying \Cref{lem:g-2nd-moment}, the desired expectation is well behaved if the ratio  $ {\epsilon_{\textsf{chart}}^2\vee \epsilon_{\textsf{signal}}^2 }/{\epsilon_{\textsf{mutual}}} $ is a constant-bounded term. To validate this fact, we note that
\begin{align}
    \frac{\epsilon_{\textsf{chart}}^2\vee \epsilon_{\textsf{signal}}^2 }{\epsilon_{\textsf{mutual}}} \lesssim \frac{\gamma^2 +\epsilon^2}{\gamma^2 + \epsilon^2 \log(k)}.
\end{align}
Since $\epsilon\ll 1\ll \log(k)^{1/2}$ , we conclude that ${\epsilon_{\textsf{chart}}^2\vee \epsilon_{\textsf{signal}}^2 }/{\epsilon_{\textsf{mutual}}}\lesssim 1$ for sufficiently large $k$. Therefore, we have by \Cref{lem:g-2nd-moment} that
\begin{align}
    \frac{1}{L^2}\sum_{l\neq l'} \EE_{\QQ}\Big[\psi(y_1, \dotp{w_{m,l}}{z_1})\cdot \psi(y_1, \dotp{w_{m,l'}}{z_1})\cdot \dotp{z_1}{v}^2\cdot\Big(1+ \sum_{s=s^\star}^\infty \zeta_s(y)h_s(\dotp{\theta^\star}{z})\Big)\Big]\lesssim \epsilon_{\textsf{mutual}}^{s^\star-1}.
\end{align}
On the other hand, we have for the third  term in Eq.~\eqref{eq:sparse-2nd-moment-1} that
\begin{align}
    \frac{1}{L^2} \sum_{l\neq l'} \EE_\QQ[\hat\psi_1(y)^2] \dotp{w_{m,l}}{v}\cdot \dotp{w_{m,l'}}{v} &\lesssim \sup_l |\dotp{w_{m,l'}}{v}|^2 \cdot \ind\{s^\star \le 2\}\\
    &\lesssim \epsilon_{\textsf{chart}}^2  \cdot \ind\{s^\star \le 2\} \\
    &\lesssim \epsilon_{\textsf{mutual}} \cdot \ind\{s^\star \le 2\},
\end{align}
where the first line holds by \Cref{lem:polarized-weight-nice-event-sparse}.

In summary, we have on the event $\cE_m(\epsilon)\cap \tilde\cE_m$ that
\begin{align}
   \sup_{v\in \{e_1,e_2 ,\dots,e_d\}} n\Var_{\PP_{\theta^\star}}[\dotp{g_{m}}{v}] &\lesssim  \epsilon_2^{s^\star-1} +\frac{1}{L}= \big(\gamma^2 + \epsilon^2 \log(k)\big)^{s^\star-1} +\frac{1}{L}.
\end{align}

\paragraph{Concentration}

We now turn to validate the condition of \Cref{lem:poly_tail_bound}. For any $v\in  \{e_1,e_2,\dots,e_d\}$, we set $G(z,y, w) = |\psi(y, \dotp{w}{z})\cdot \dotp{z}{v}| + |\hatpsi_1(y)\cdot \dotp{w}{v}|$ with the domain measure defined as $\diff\PP_{\theta^\star}(z_1, y_1)\times \diff\mu(w)$, where
$d\mu(w) =L^{-1} \sum_{ l} \delta_{w_{m,l}}$, the integral Minkowski's inequality implies that
\begin{align}
    \EE_{\PP_{\theta^\star}}[ |\dotp{g_{m,1}}{v}|^r]^{1/r} &= \Big(\int \diff \PP_{\theta^\star}(y,z) \Big(\int \diff \mu(w) |G(z,y,w)|\Big)^r\Big)^{1/r} \\
    &\le  \int \diff\mu(w) \Big(\int \diff\PP_{\theta^\star}(y,z) |G(z,y,w)|^r\Big)^{1/r} \\
    &= \frac{1}{L}\sum_{l=1}^L \EE_{\PP_{\theta^\star}}[ |\psi(y_i, \dotp{w_{m,l}}{z_i})\cdot \dotp{z_i}{v }|^r]^{1/r} + \frac{1}{L} \sum_{l} |\dotp{w_{m,l}}{v}|.\label{eq:sparse-sample-poly-tail}
\end{align}
To proceed, we leverage Cauchy-Schwarz inequality to decouple the average of the product in the first term, which reads
\begin{align}
    \frac{1}{L}\sum_{l=1}^L \EE_{\PP_{\theta^\star}}[ |\psi(y_i, \dotp{w_{m,l}}{z_i})\cdot \dotp{z_i}{v }|^r]^{1/r} &\le \frac{1}{L}\sum_{l=1}^L \EE_{\PP_{\theta^\star}}[ |\psi(y_i, \dotp{w_{m,l}}{z_i})|^{2r}]^{1/2r}\cdot \EE_{\PP_{\theta^\star}}[|\dotp{z_i}{v}|^{2r}]^{1/2r}.
\end{align}
Similar to the proof of \Cref{prop:nonsparse-sample-concentration}, we have that
\begin{align}
    \EE_{\PP_{\theta^\star}} [\psi(y, \dotp{w_{m,l}}{z})^{2r}] &\le \EE_{\QQ}[ U_{\dotp{\tst}{w_{m,l}}} \Big(\frac{\PP(x,y)}{\QQ(x,y)}\Big)^2 ]^{1/2} \cdot  \EE_\QQ[ \psi(y,x)^{4r}]^{1/2} \lesssim r^{C_p 4r},
\end{align}
where the first inequality exactly repeats Eq.~\eqref{eq:gaussian-noise-op} and the second inequality holds by \ref{asp:polynomial-like}. On the other hand, we have that $\EE_{\PP_\tst}[\dotp{z_i}{v} ^{2r}]^{1/2r} \le  r^{1/2}$. Since the second term in Eq.~\eqref{eq:sparse-sample-poly-tail} is bounded by $O(1)$, we conclude that
\begin{align}
    \EE_{\PP_{\theta^\star}}[ |\dotp{g_{m,1}}{v}|^r]^{1/r}  &\lesssim  r^{C_p +1/2}.
\end{align}
Thus, \Cref{lem:poly_tail_bound} implies that there exists a $\{(z_i,y_i)\}_{i\in [n]}$-measurable event $\cE_{m,2}$ with probability at least $1-O(d^{-c-1}/T)$, on which for any $v\in \{e_1,e_2,\dots,e_d\}$, it holds that
\begin{align}
\big|\dotp{g_m}{v} - \EE_{\PP_{\theta^\star}}[\dotp{g_{m}}{v }]\big| &{\lesssim} \sqrt{\frac{ \EE_{\PP_{\theta^\star}} [\dotp{g_{m,1}}{v}^2]\cdot\log (d^{c+1}T)}{n}} + \frac{\log (d^{c+1}T)\cdot \log(d^{c+1} T n)^{C_p +1/2}}{n}\\
&\lesssim \sqrt{\frac{ \big(\big(\gamma^2 +\epsilon^2 \log(k)\big)^{s^\star-1} +L^{-1}\big)\cdot\log (d)}{n}} + \frac{\log(d)^{C_p +3/2}}{n},
\end{align}
given that $T,n$ are at most of polynomial rate in $d$. Since we assume that
\begin{align}
    n = \Omega  \Big(\big({(\gamma^2 + \epsilon^2 \log(k))^{s^\star-1 } + L^{-1}} \big)^{-1} \cdot \log(d)^{2C_p +2 } \Big),
\end{align}
the above inequality can be further simplified as
\begin{align}
    \big|\dotp{g_m}{v} - \EE_{\PP_{\theta^\star}}[\dotp{g_{m}}{v }]\big| &{\lesssim} \sqrt{\frac{\big( \big(\gamma^2 +\epsilon^2  \log(k)\big)^{s^\star-1} + L^{-1}\big)\cdot\log (d)}{n}}.\label{eq:sparse-sample-concentration-coordinate}
\end{align}
Additionally,  $\cE_{m,2}$ is the desired event. This concludes the proof of \Cref{prop:sparse-gradient-fluctuation}.
\end{proof}

\subsection{Proofs for Technical Results in the Sparse Case}

\begin{proof}[Proof of \Cref{lem:good-signal}]\label{proof:good-signal}
        With a slight abuse of notation, we assume that $\theta^\star \sim \unif(\SS^{k-1})$. We first consider the event $\cE_{0,\infty } = \{\norm{\theta^\star}_{\infty}\le C\cdot k^{-1/2}\log(k)^{1/2}\}$. From the proof of \Cref{lem:sphere_coordinate_tail}, we see that
        \begin{align}
            \PP(\norm{\theta^\star}_\infty \ge t) \le 2k \cdot\PP(\theta^\star_1 \ge t) \le 2k \exp(-k /16 )+ 2k \exp (-t^2k/4).
        \end{align}
        Take $t = C\cdot k^{-1/2}\log(k)^{1/2}$, we have that the failure probability is upper bounded by $2k \exp(-k/16) + 2k^{1-C^2/4}$.

        For the $r$-norm, we leverage the property that $\theta^\star\overset{d.}{=}{Z/ \norm{Z}_2}$ where $Z\sim \cN(0,I_k)$. Now,
        $\norm{Z}_2^2 = \sum_{i\le k} Z_i^2$, where $Z_i^2 - \EE[Z_i^2]\ge -1$. Applying one-sided Bernstein's inequality with failure probability $k^{-c_0}$, we have that
        \begin{align}
            \norm{Z}_2^2 \le  k + \sqrt{2c_0 k \log(k)}+ c_0 \log(k)/3.
        \end{align}

        On the other hand, note that we have for $c_1>2$ that
        \begin{align}
             \PP(\max_{i\le k} |Z_i| > \sqrt{2c_1\log k}) \le k \cdot \PP(|Z_1| > \sqrt{2c_1\log k}) \le k  \exp \{-c\log k \} =k^{1-c_1},
        \end{align}
        To apply Bernstein's inequality, we note that
        \begin{align}
            \EE[ |Z_i|^r \cdot \ind\{|Z_1|\le \sqrt{2c_1 \log k }\}] &\le \EE[|Z_i|^{2r}]^{1/2};\\
            \EE[ |Z_i|^{2r} \cdot \ind\{|Z_1|\le \sqrt{2c_1 \log k }\}] &\le \EE[|Z_i|^{2r}],
        \end{align}
        where $\EE[Z_i^{2r}] =  (2r-1)!!$. Therefore, it holds by truncated Bernstein's inequality that
        \begin{align}
            \PP\Big(\norm{Z}_r^r > (k + \sqrt{2c_2k \log(k) })\cdot \EE[|Z_1|^{2r}]^{1/2}+\big(\sqrt{2C \log(k)}\big)^r\cdot   c_2 \log(k)/3\Big) \le k^{1-c_1} + k^{-c_2 }
        \end{align}
        Combining the two bounds, we conclude that with probability $1-O(k^{1 -c_0 \vee c_1 \vee c_2})$, it holds that
        \begin{align}
            \norm{\theta^\star}_r^r \overset{d.}{= }  \frac{\norm{Z}_r^r}{\norm{Z}_2^r} \lesssim \frac{ k+ \sqrt{k\log k} + (\log k)^{1+r/2}}{(k +\sqrt{k\log k } +\log k)^{r/2 } }  \lesssim k^{1-r/2},
        \end{align}
        with probability at least $1-O(k^{-c})$ for some constant $c>0$.

        We now move on to consider the event $\cE_{0,\sharp} = \big\{ \sum_{i\le k } \ind\{|\theta_i^\star| \ge 1/\sqrt{2k}\}\ge k/4\big\}$.
        First, it follows from Hoeffding's inequality that
        \begin{align}
            \PP\Big( \sum_{i\le k} \ind\{Z_i^2\ge   3/4 \} \ge k/4\Big)  &\ge 1-  2\exp(-2(p-1/4)^2 k),
\end{align}
        where $p = \PP(Z_1^2 \ge  3/4)>1/4$. Denote the above event as $\cA_1$.
        On the other hand, we have by the Bernstein's inequality that
        \begin{align}
        \PP\Big(\Big|k^{-1}\sum_{i\le k} Z_i^2- 1\Big| \le 1/2 \Big)\ge 1 -  2\exp\{ -k/32\}.
        \end{align}
        Denote the above event as $\cA_2$. Then on the event $\cA_1\cap \cA_2$, we have that
        \begin{align}
            \sum_{i\le k}\ind\Big\{\frac{|Z_i|}{\norm{Z}}>  \frac{1}{\sqrt{2k}}\Big\} &= \sum_{i\le k } \ind\Big\{ Z_i^2 >  \frac{1}{2k} \sum_{i\le k}Z_i^2\Big\} \\
            &\overset{\cA_2}{\ge } \sum_{i\le k} \ind\{Z_i^2 > \frac{3}{4} \}  \\
            &\overset{\cA_1}{\ge } k/4.
         \end{align}
        In conclusion we have that $\PP(|\{i: |\theta_i^\star| >1/\sqrt{2k}\}| > k/4 ) \ge \PP(\cA_1 \cap \cA_2)\ge 1- \exp\{-c_{3}k\}$ for some constant $c_3>0$.

\end{proof}

\begin{proof}[Proof of \Cref{lem:polarized-weight-nice-event-sparse}] \label{proof:polarized-weight-nice-event-sparse}
Clearly, it holds that
\begin{align}
     \norm{ \gamma\theta + \xi_{m,l}}_2 \ge \norm{\xi_{m,l}}_2 - \gamma \cdot \norm{\xi_{m,l}}_2 \ge 1/2.
\end{align}
By substituting this lower bound for the denominator, we have for any $j,l$ that
    \begin{align}
         |\dotp{w_{m,l}}{e_j}|  &\le  2( \gamma  |\theta_j| +  |\dotp{\xi_{m,l}}{e_j}|)  \\
         &\le 2(\gamma|\theta_j| +  \epsilon) .
    \end{align}
The last line holds by the definition of $\cE_{m}(\epsilon)$.
On the other hand, we have for any $l$ that
\begin{align}
 |\dotp{w_{m,l}}{\theta^\star}|&\le 2\Big(   \gamma |\rho| + \Big|\sum_{j\in [d]} \xi_{m,l,j}\cdot \theta_j^\star\cdot \ind \{ j\in \phi^\star \cap \phi_{m,l}\}\Big|\Big)\\
&\le 2\gamma|\rho | + 2\big( \sum_j \xi_{m,l,j}^2\cdot \ind\{j\in \phi^\star\cap \phi_{m,l}\}\big)^{1/2}\cdot\big( \sum_j {\theta^\star_j}^2\cdot \ind\{j\in \phi^\star\cap \phi_{m,l}\}\big)^{1/2}  \\
&\le  2\gamma|\rho| +  2 \sup_{j} |\xi_{m,l,j}| \cdot \norm{\theta^\star}_\infty \cdot |\phi^\star\cap \phi_{m,l}|
\end{align}
To proceed, note that on the event $\tilde\cE_{m} \cap \cE_{m}(\epsilon)$, it holds that $| \phi^\star \cap \phi_{m,l}|\le \log k$ and that $\sup_j |\xi_{m,l,j}|\le \epsilon$. Since we assume that $\norm{\theta^\star}_{\infty}\le 1/ \log k$, it holds that
\begin{align}
    |\dotp{w_{m,l}}{\theta^\star}| \le 2(\gamma|\rho| + \epsilon).
\end{align}

Now we turn to consider the correlation between $w_{m,l}$ and $w_{m,l'}$.
\begin{align}
    |\dotp{w_{m,l}}{w_{m,l'}}| &\le 2 \Big( \gamma^2 + \sum_{j} |\xi_{m,l,j}|\cdot|{\xi_{m,l',j}}|\cdot \ind\{j\in \phi_{m,l}\cap \phi_{m,l'}\} \\
    &\qquad + \gamma \sum_{j}|\theta_j| \cdot |\xi_{m,l,j}| \cdot \ind\{j \in \phi_{m,l} \cap \mathsf{supp}(\theta) \} \\
    &\qquad+ \gamma \sum_{j}|\theta_j| \cdot |\xi_{m,l',j}| \cdot \ind\{j \in \phi_{m,l'} \cap \mathsf{supp}(\theta) \}\Big).
\end{align}
For the second term, we have with the definition of $\cE_m(\epsilon)$ that
\begin{align}
    \sum_j |\xi_{m,l,j}|\cdot|{\xi_{m,l',j}}|\cdot \ind\{j\in \phi_{m,l}\cap \phi_{m,l'}\} &\le \max_{j,l}|\xi_{m,l,j}|^2\cdot |\phi_{m,l}\cap  \phi_{m,l'}|\\& \le \epsilon^2 \log k.
\end{align}
For the last two terms, applying the Cauchy-Schwarz inequality gives, for $\ell\in\{l,l'\}$, that
\begin{align}
\gamma\sum_{j} |\theta_j|\cdot |\xi_{m,\ell,j}| \cdot\ind\{j\in \phi_{m,\ell}\cap  \mathsf{supp}(\theta)\} &\le  \gamma \norm{\theta}_2 \cdot \epsilon  \sqrt{\log k} \le \gamma^2 + \epsilon^2\log k.
\end{align}
Putting them together, we have that
\begin{align}
    |\dotp{w_{m,l}}{w_{m,l'}}| \le 4(\gamma^2 + \epsilon^2\log k).
\end{align}
This concludes the proof of \Cref{lem:polarized-weight-nice-event-sparse}.
\end{proof}

\begin{proof}[Proof of \Cref{prop:sparse-w-expectation-warmup}]\label{proof:sparse-w-expectation-warmup}
For conciseness, we momentarily drop the subscript $m,l$ in the following analysis. Conditioning on fixed $\phi$, we have that
\begin{align}
\EE_w[\dotp{w}{\theta^\star}^{s}] &= \EE_w \big[ \norm{ \gamma\theta + \xi }_2^{-s}\cdot \big(\gamma \dotp{\theta}{\theta^\star}  + \dotp{\xi}{\theta^\star}\big)^s \big] \\
&= \EE_\phi \big[\EE_w \big[ \norm{ \gamma\theta + \xi }_2^{-s}\cdot \big(\gamma \rho + \dotp{\xi}{P_\phi \theta^\star}\big)^s\biggiven \phi \big]\big]. \label{eq:sparse-w-expectation-0}
\end{align}
Given the polarization level $\gamma =o(1)$, we see that $\norm{\gamma \theta + \xi}_2^{s+1} \simeq 1 \pm o(1)$, and it suffices to evaluate $\EE_w\big[\big(\gamma\rho +\dotp{\xi}{P_\phi \theta^\star}\big)^s\biggiven \phi\big]$.
Without loss of generality, we assume that $1\in \phi$ and we can translate $P_\phi \theta^\star$ into the first coordinate by the isotropy of $\xi$ over $\SS^{k-1}(\phi)$. To this end, we can characterize the first term as follows:
\begin{align}
\EE \big[\big(\gamma\rho +\dotp{\xi}{P_\phi \theta^\star}\big)^s\biggiven \phi\big] &= \EE \big[\big(\gamma \rho  + \dotp{\xi}{\norm{P_\phi \theta^\star}_2 \cdot e_1 }\big)^s\biggiven \phi\big] \\
&= \sum_{r=0}^{s} \binom{2\lfloor s/2\rfloor}{r} (\gamma\rho)^{s-r} \cdot \norm{P_\phi \theta^\star}_2^r \cdot\EE\big[\xi_1^r \biggiven \phi\big] \cdot\ind\{r\text{ even}\}\\
&\overset{\textrm{(i)}}{\simeq}{}   \sum_{r=0}^{\lfloor s/2\rfloor} \binom{2\lfloor s/2\rfloor}{2r} (\gamma\rho)^{2\lfloor s/2\rfloor-2r} \cdot \norm{P_\phi \theta^\star}_2^{2r} \cdot k^{- r}\cdot (\gamma \rho)^{\ind\{s \text{ odd}\}}\\
&= (\gamma \rho)^{\ind\{s \text{ odd}\}} \cdot  \big((\gamma \rho + k^{-1/2}\norm{P_\phi \theta^\star}_2)^{2\lfloor s/2\rfloor}+(\gamma \rho - k^{-1/2}\norm{P_\phi \theta^\star}_2)^{2\lfloor s/2\rfloor}\big) /2 \\
&\simeq (\gamma \rho)^{\ind\{s\text{ odd}\}} \cdot (\gamma |\rho| + k^{-1/2}\norm{P_\phi \theta^\star}_2)^{2\lfloor s/2\rfloor}.\label{eq:sparse-w-expectation-1}
\end{align}
Here, (i) holds by applying \Cref{lem:moment_sphere} and $\simeq$ denotes equality up to an $s$-dependent multiplicative constant.

Putting together Eq.~\eqref{eq:sparse-w-expectation-0} and~\eqref{eq:sparse-w-expectation-1},
 we conclude that
\begin{align}
\EE_w[\dotp{w}{\theta^\star}^{s}\giv \phi] &\simeq (\gamma\rho)^{\ind\{s \text{ odd}\}}(\gamma|\rho| + k^{-1/2}\norm{P_{\phi}  \theta^\star}_2)^{s-\ind\{s \text{ odd}\}}. \label{eq:sparse-w-expectation-4}
\end{align}
In the sequel, we consider averaging over $\phi$. From Eq.~\eqref{eq:sparse-w-expectation-4}, we see that it suffices to consider $\EE_\phi[(\gamma|\rho| +k^{-1/2}\norm{P_{\phi}  \theta^\star}_2)^r]$ for some $r\ge 2$. We alter the notation to facilitate some deferred calculation. Consider $\bbm\subset[d]$ with constant size $|\bbm|=O(1)$ that does not scale with $k$ or $d$. Now define $\phi_{\bbm} \sim\unif\{\cS_{k,\bbm}\}$, where $\cS_{k,\bbm} = \{S\subset[d]:|S| = k, \bbm \subset S\}$. It is easily seen that this definition covers previous definition of $\cS_{k,m}$ by setting $\bbm = \{m\}$. We characterize the magnitude of $\EE_{\phi_{\bbm}} [\norm{P_{\phi_\bbm}\theta^\star}_2^r]$ from both sides as follows. For the lower bound, we have that
\begin{align}
\EE_{\phi_\bbm}[\norm{P_{\phi_\bbm}\theta^\star}_2^r] &= \EE_{\phi_\bbm}\Big[\big(\|\theta_\bbm^\star\|_2^2 +\sum_{j\notin \bbm} |\theta_j^\star|^2 \ind \{j\in \phi_\bbm\}\big)^{r/2}\Big]  \\
&\ge \EE_{\phi_{\bbm}}\Big[\|\theta_\bbm^\star\|_r^r  +\sum_{j\notin \bbm} |\theta_j^\star|^r \ind \{j\in \phi_\bbm\}\Big]\\
&\overset{\rm(i)}{ \simeq} (1-k/d)\cdot \norm{\theta^\star_{\bbm}}_r^r + \frac{k}{d}\cdot \norm{\theta^\star}_r^r \\
&\overset{\rm(ii)}{\gtrsim} \norm{\theta^\star_{\bbm}}_r^r +\frac{k}{d}\cdot  k ^{1 - r/2}\cdot \norm{\theta^\star}_2^{r/2} \\
&=  \norm{\theta^\star_{\bbm}}_r^r +\frac{k^2}{d}\cdot  k ^{- r/2}.
\end{align}
Here (i) holds by the  fact that $\EE[\ind\{j\in \phi_{\bbm}\}] \simeq  k/d$ for $j\notin \bbm$, and (ii) is a consequence of Jensen's inequality.
For the upper bound, we have that
\begin{align}
\EE_{\phi_\bbm}[\norm{P_{\phi_\bbm}\theta^\star}_2^r] & = \EE_{\phi_\bbm} \Big[ \Big( \|\theta^\star_\bbm\|_2^2 + \sum_{j\notin \bbm} |\theta_j^\star|^2\cdot \ind \{j \in \phi_\bbm\}\Big)^{r/2 }\Big] \\
&\lesssim  \EE_{\phi_\bbm} \Big[ \|\theta^\star_\bbm\|_r^r + \Big( \underbrace{ \sum_{j\notin \bbm} |\theta_j ^\star|^2 \cdot \ind\{j\in \phi_\bbm\}}_{ |(\phi_\bbm\cap \phi^\star) \setminus \bbm| \text{ nonzero summands}}\Big)^{r/2} \Big]  \\
&\overset{\text{Jensen}}\lesssim \|\theta^\star_\bbm\|_r^r + \EE_{\phi_\bbm} \Big[|(\phi_\bbm\cap \phi^\star) \setminus \bbm|^{r/2-1}\cdot \Big( \sum_{j\notin \bbm} |\theta_j ^\star|^r \cdot \ind\{j\in \phi_\bbm\}\Big) \Big]. \label{eq:sparse-w-expectation-5}
\end{align}
Next, we apply Cauchy-Schwarz inequality as follows:
\begin{align}
\eqref{eq:sparse-w-expectation-5} &= \|\theta^\star_\bbm\|_r^r + \EE_{\phi_{\bbm}} \Big[  \sum_{j\notin \bbm} |\theta_j ^\star|^r \cdot \ind\{j\in \phi_\bbm\}^2 \cdot |(\phi_\bbm\cap \phi^\star) \setminus \bbm|^{r/2-1} \Big] \\
&\le  \|\theta^\star_\bbm\|_r^r + \EE_{\phi_{\bbm}}\Big[ \sum_{j\notin \bbm} |\theta_j^\star|^{2r} \cdot \ind\{j\in \phi_\bbm\}\Big]^{1/2} \cdot \EE_{\phi_{\bbm}}\Big[ \sum_{j\notin \bbm} \ind\{j\in \phi_\bbm\}\cdot |(\phi_\bbm\cap \phi^\star) \setminus \bbm|^{r-2} \Big]^{1/2} \\
& = \|\theta^\star_\bbm\|_r^r + \Big( \frac{k}{d}\cdot \sum_{j\notin \bbm} |\theta_j^\star|^{2r}\Big)^{1/2}\cdot \EE_{\phi_{\bbm}}\big[  |(\phi_\bbm\cap \phi^\star) \setminus \bbm|^{r-1}\big]^{1/2} \\
& \overset{ (i)}{\lesssim} \|\theta^\star_\bbm\|_r^r + \Big(\frac{k}{d}\cdot k^{1-r}\Big)^{1/2} \cdot \Big(\frac{k^2}{d}\Big)^{(r-1)/2 } \\
&= \|\theta^\star_\bbm\|_r^r + \frac{k^2}{d}\cdot k^{-r/2},
\end{align}
where (i) holds by $\cE_{0,2r}$ and \Cref{lem:hypergeometric-moment}.
In conclusion, we have that $\EE_{\phi_\bbm}[\norm{P_\phi\theta^\star}_2^r] \simeq \|\theta_\bbm^\star\|_r^r + k^{-r/2}\cdot \delta$ given that $k = o (\sqrt{d})$.  Combining this result with Eq.~\eqref{eq:sparse-w-expectation-4}, we obtain that
\begin{align}
\EE_w[\dotp{w}{\theta^\star}^{s}] &\simeq (\gamma\rho)^{\ind\{s \text{ odd}\}} (\gamma|\rho| + k^{-1/2} |\theta_m^\star| + k^{-1}\delta_{s-\ind\{s \text{ odd}\}})^{s-\ind\{s \text{ odd}\}}
\end{align}
where $\delta = k^2/d = o(1)$ and $\delta_r = \delta^{1/r}$.
\end{proof}

\section{Statistical Query Lower Bound for Sparse Signal Recovery}

In this section, we provide a $k^{s^\star}$ sample complexity lower bound for the single index model with $k$-sparse signal when querying a $\VSTAT$ oracle. The statistical query (SQ) framework was developed in \citet{feldman2017statistical} and for completeness, we present essential definitions and results here.
\begin{definition}[$\VSTAT$ Oracle]
    Let $D^\star$ be the input distribution over domain $\cX$. For a sample size parameter $n>0$, $\VSTAT(D^\star, n)$ oracle is the oracle that for any query function $h: \cX\rightarrow [0, 1]$, returns a value $v\in[p-\tau, p+\tau]$, where $p = \EE_{x\sim D^\star}[h(x)]$ and $\tau = \max\{ n^{-1}, \sqrt{p(1 - p) / n} \}$.
\end{definition}
To define a key concept \emph{statistical query dimension}, we first introduce the following notation.
\begin{definition}[Relative Pairwise Correlation]
    Given two distributions $D_1, D_2 \in \Delta(\cX)$ and a reference distribution $D \in \Delta(\cX)$,
    \begin{align}
        \chi_{D}(D_1, D_2) = \EE_{x\sim D} \left[ \frac{D_1(x)}{D(x)} \cdot \frac{D_2(x)}{D(x)} \right] -1.
    \end{align}
\end{definition}
\begin{definition}[Statistical Dimension]
    For $\bar{\gamma}>0$, $\eta\in(0, 1)$, domain $\cX$, a set of distributions $\cD$ over $\cX$,
    the \textbf{statistical dimension} $\SDA(\cD,\bar{\gamma}, \eta)$ of $\cD$ with average correlation $\bar{\gamma}$ and solution set bound $\eta$ is defined as the largest value $m'$ such that there exists a \emph{reference distribution} $D \in \Delta(\cX)$ and a finite set of distributions $\cD_{D}\subseteq \cD$ which can depend on the reference $D$ with the following property: for any solution $D^\star\in \cD$,
    \begin{itemize}
        \item[(\romannumeral1)] $|\cD_{D} \setminus \{D^\star\}| \geq (1-\eta) |\cD_D|$;
        \item[(\romannumeral2)] for any subset $\cD_D' \subseteq \cD_D \setminus \{D^\star\}$ such that $|\cD_D'| \ge |\cD_D \setminus \{ D^\star\}| / m'$,
        \begin{align}
            \frac{1}{|\cD_D'|^2} \sum_{D_i, D_j \in\cD_D'} \chi_{D}(D_i, D_j) \le \bar{\gamma}.
        \end{align}
    \end{itemize}
  \end{definition}
  The above definition of the statistical dimension is a special case of the original Definition 3.1 in \citet{feldman2017statistical} where we consider a search problem of \emph{exact recovery} of the ground truth $D^\star$.
\begin{definition}[$(\gamma, \beta)$-correlated Distributions]
    We say that a set of $m$ distributions $\cD = \{D_1,\ldots,D_m\}$ over $\cX$ is $(\gamma,\beta)$-correlated relative to a reference distribution $D \in \Delta(\cX)$ if:
    $$ \chi_D(D_i, D_j) \leq
    \begin{cases}
    \beta  &\mbox{ for } i=j\in[m]\\
    \gamma &\mbox{ for } i\not = j \in[m].
    \end{cases}
    $$
\end{definition}
The following lemma borrowed from Lemma 3.10 of \citet{feldman2017statistical} provides a lower bound on the statistical dimension in terms of the $(\gamma, \beta)$-correlation property of the set of candidate distributions.
\begin{lemma} \label{lem:correlated-sq dim}
    Given a set of candidate distributions $\cD$ that are $(\gamma, \beta)$-correlated with respect to a reference distribution $D$, then for any $\gamma'>0$ and $\eta > |\cD|^{-1}$,
    \begin{align}
        \SDA(\cD, \gamma + \gamma', \eta) \ge \frac{(|\cD|-1) \gamma'} {\beta -\gamma}.
    \end{align}
\end{lemma}
The main result in the SQ framework is the following statement that relates the number of queries required to the statistical dimension, which is borrowed from Theorem 3.2 of \citet{feldman2017statistical}.
\begin{lemma}\label{lem:sq-lower-bound}
    Let $\cX$ be a domain and $\cD$ be a set of candidate distributions over $\cX$. For any $\bar\gamma > 0$ and $\eta \in (0, 1)$, any randomized SQ algorithm that solves the problem of finding the input distribution $D^\star \in \cD$ with probability at least $\alpha > \eta$ requires at least $(\alpha - \eta)/(1-\eta) \cdot \SDA(\cD, \bar\gamma, \eta)$ calls to the $\VSTAT(D^\star, (3\bar\gamma)^{-1})$ oracle.
\end{lemma}
Our strategy for proving the lower bound is to first construct a set of candidate distributions $\cD$ that are $(\omega(k^{-1}), \beta)$-correlated with respect to reference distribution $\QQ$ with $\beta = D_{\chi^2} (\PP_{\theta^\star} \,\|\, \QQ)$ and $|\cD|$ exponentially large.
Then by \Cref{lem:correlated-sq dim} and \Cref{lem:sq-lower-bound}, we can derive the desired hardness result.
It remains to construct the set of candidate distributions $\cD$ that are $(\omega(k^{-1}), \beta)$-correlated with respect to $\QQ$.
To this end, we introduce the following result on the packing number of $k$-sparse vectors.
\begin{lemma}[Packing Number for $k$-Sparse Vectors]
    \label{lem:packing number}
    Define $\rho(u, v) = |\langle u, v\rangle|$. Let packing number $\cM_\rho(d, k, t)$ be the maximal cardinality of the set of $k$-sparse vectors in $\SSS^{d-1}$ such that $\rho(u, v) < t$ for any $u\neq v$ in the set. We have for any $t \in (1/k, 1)$ that
    \begin{align}
        \cM_\rho(d, k, t) \ge \frac 1 2 \cdot \exp\left( \frac{\min\left\{(d-k) t^2 \:,\: {3k t} \right\}}{8}\right). \label{eq:packing-number}
    \end{align}
\end{lemma}
With all these ingredients in place, we are ready to prove the main theorem.
\begin{proof}[Proof of \Cref{thm:sq-lower-bound-sparse}]
Let us pick parameter $\kappa_d \in ((\log d)^2, k/4)$ that scales with $d$ and set
\begin{align}
    t \ge
    \max\left\{ \sqrt\frac{\kappa_d}{d-k}, \frac{\kappa_d}{3k}\right\} \in (1/k, 1/2).
    \label{eq:gamma-choice}
\end{align}
Note that $t \in (1/k, 1/2)$ is able to hold by our choice of $\kappa_d$ and condition that $\omega((\log d)^2)\le k \le d/2$.
In this vein, we can pick $\cD$ to be the \emph{maximal} set of distributions $\PP_\theta$ for some $k$-sparse vectors $\theta \in \SSS^{d-1}$ satisfying $\rho(\theta, \theta') < t$ for any $\theta \neq \theta'$ in the set.
It follows from plugging \eqref{eq:gamma-choice} into \Cref{lem:packing number} that $|\cD| \ge \exp(\kappa_d/8)/2$, which is super polynomially large in $d$ for our choice of $\kappa_d$.

Next, we configure the remaining parameters in \Cref{lem:correlated-sq dim} and \Cref{lem:sq-lower-bound}.
We choose the reference distribution to be $\QQ$, in which the covariate $z$ is independent of the output $y$.
For $\beta$, we note that
\begin{align}
    \chi_{\QQ}(\PP_{\theta}, \PP_{\theta}) = D_{\chi^2} (\PP_{\theta} \,\|\, \QQ) = O(1),
\end{align}
which is a constant independent of $\theta$ due to the rotational invariance of the likelihood ratio with respect to $\theta$.
Thus, we denote this quantity by $B$ and set $\beta = D_{\chi^2} (\PP_{\theta^\star} \,\|\, \QQ) = B$.
For $\gamma$, we note that for any two $\PP_\theta, \PP_{\theta'}$ in $\cD$ for $\theta \neq \theta'$,
\begin{align}
    |\chi_\QQ(\PP_\theta, \PP_{\theta'}) |
    &= \left|\EE_{x\sim \QQ} \left[ \frac{\PP_\theta(x)}{\QQ(x)} \cdot \frac{\PP_{\theta'}(x)}{\QQ(x)} \right] - 1\right| \\
    &= \left|\EE_{x\sim \QQ} \Biggl[ \biggl( 1 + \sum_{s\ge s^\star} \zeta_s(y) h_s(\langle \theta, z\rangle) \biggr) \cdot \biggl( 1 + \sum_{s'\ge s^\star} \zeta_{s'}(y) h_{s'}(\langle \theta', z\rangle) \biggr) \Biggr] - 1\right| \\
    &= \sum_{s\ge s^\star} \EE_{\QQ}[\zeta_s(y)^2] \cdot |\langle \theta, \theta'\rangle|^{s} \le \sum_{s\ge s^\star} \EE_{\QQ}[\zeta_s(y)^2] \cdot t^{s} \le \EE_{\QQ}[\zeta_{s^\star}(y)^2] \cdot t^{s^\star} + \frac{t^{s^\star + 1}}{1 - t}.
\end{align}
Here, the third equality follows from the fact that only when $s=s'$, the cross term $\EE_{\QQ}[h_s(\langle \theta, z\rangle) h_s(\langle \theta', z\rangle)]$ is non-zero.
In particular, by the property of the Gaussian noise operator introduced in \eqref{eq:hermite-noise-operator}, we have that
\(
    \EE_{\QQ}[h_s(\langle \theta, z\rangle) h_s(\langle \theta', z\rangle)] = \langle \theta, \theta'\rangle^s < t^s.
\)
For the last inequality above, we simply use the fact that $\EE_{\QQ}[\zeta_s(y)^2] \le 1$ for any $s$ \citep{damian2024computational} and $t < 1$.
Now, we conclude that
\begin{align}
    |\chi_\QQ(\PP_\theta, \PP_{\theta'})| \le \left(\EE_{\QQ}[\zeta_{s^\star}(y)^2] + \frac{t}{1-t}\right) \cdot t^{s^\star} \le \left(\EE_{\QQ}[\zeta_{s^\star}(y)^2] + 1\right) \cdot t^{s^\star}.
\end{align}
We thus set $\gamma' = \gamma = \left(\EE_{\QQ}[\zeta_{s^\star}(y)^2] + 1\right) \cdot t^{s^\star} = \Theta(t^{s^\star})$.
Finally, we set $\eta = 1/3$ and $\alpha = 2/3$.
Then all the conditions in both \Cref{lem:correlated-sq dim} and \Cref{lem:sq-lower-bound} are satisfied and we have
\begin{align}
    \SDA(\cD, 2\gamma, 1/3) \ge \frac{(|\cD| - 1) \gamma}{\beta - \gamma} \ge \frac{|\cD|\gamma}{2\beta} \ge \frac{\gamma \exp(\kappa_d/8)}{4\beta}.
\end{align}
Lastly, recall that we have $|\langle \theta, \theta'\rangle| \le t$ for any $\theta \neq \theta'$ in $\cD$, which means that in order to achieve alignment at least $2t$ with the true signal $\theta^\star$, we need to \emph{exactly} identify the distribution $\PP_{\theta^\star}$.
Consequently, by \Cref{lem:sq-lower-bound}, we have that any randomized SQ algorithm that solves the problem of achieving alignment $2t$ with probability at least $2/3$ requires at least
${\gamma \exp(\kappa_d/8)}/(8 B) $ calls to the $\VSTAT(\PP_{\theta^\star}, (6\gamma)^{-1}) $ oracle.

\paragraph{Simplification of the lower bound}
To simplify the lower bound, let us take $\kappa_d = (\log d)^c/2$ for some constant $c>2$. Thus, the alignment $2t$ is upper bounded by
\begin{align}
    2 t \le \begin{cases}
        \tilde \omega(k^{-1}) & \mbox{ if } k < \sqrt d \\
        \tilde \omega(d^{-1/2}) & \mbox{ if } k \ge \sqrt d
    \end{cases},
\end{align}
where $\tilde \omega(\cdot)$ hides some poly-logarithmic factors.
The number of queries is still super polynomially large in $d$.
Following from \eqref{eq:gamma-choice}, we can safely set
\begin{align}
    t = \begin{cases}
        (\log d)^c/k & \mbox{ if } (\log d)^2 < k < \sqrt{d (\log d)^c } \\
        \sqrt{(\log d)^c/d} & \mbox{ if } \sqrt{d(\log d)^c } \le k \le d/2
    \end{cases},
\end{align}
Hence, the sample size parameter
\begin{align}
    (6\gamma)^{-1} = \frac{t^{-s^\star}}{6} \simeq \begin{cases}
        \frac{k^{s^\star}}{(\log d)^{cs^\star}} & \mbox{ if } (\log d)^2 < k < \sqrt{d (\log d)^c } \\
        \frac{d^{s^\star/2}}{6 (\log d)^{c s^\star/2}} & \mbox{ if } \sqrt{d(\log d)^c } \le k \le d/2
    \end{cases}.
\end{align}

Hence, we have established the desired lower bound on the sample complexity.
\end{proof}

\begin{proof}[Proof of \Cref{lem:packing number}]
    We use the probability method to prove the existence of a set of $k$-sparse vectors in $\SSS^{d-1}$ with the desired property.
    We i.i.d. sample $m$ vectors $\omega^{(1)}, \ldots, \omega^{(m)}$ from the following distribution:
    \begin{align}
        \omega: \quad \phi \sim \unif(\cS_{k}), \quad \omega_j = \begin{cases}
            \frac{1}{\sqrt k}, & \mbox{ w.p. }\frac 1 2 \mbox{ if } j\in \phi \\
            - \frac{1}{\sqrt k}, & \mbox{ w.p. }\frac 1 2 \mbox{ if } j\in \phi \\
            0, & j \notin \phi.
        \end{cases}, \quad j\in [d].
    \end{align}
    where we recall that $\cS_k$ is the set of all size-$k$ subsets in $[d]$.
    Since each $\omega^{(i)}$ is i.i.d. sampled,
    we can equivalently view $\langle \omega^{(i)}, \omega^{(j)}\rangle$ for $i\neq j$ as a random variable sampled from the following distribution:
    \begin{align}
        \langle \omega^{(i)}, \omega^{(j)}\rangle \overset{d}{=} \frac{R_X}{k}, \where R_X = r_1+ \ldots, r_X, \quad X\sim \hypergeom(d, k, k),
        \label{eq:hypergeom equiv}
    \end{align}
    where $r_1, r_2, \ldots$ are i.i.d. Rademacher random variables.
    Let us consider random variable $W$ distributed as
    \begin{align}
        W \overset{d}{=} \frac{R_Y}{k}, \where
        R_Y = r_1 + \ldots + r_Y,
        \quad Y\sim \binomial\Bigl(k, \frac{k}{d-k}\Bigr).
        \label{eq:binomial approx}
    \end{align}
    We will invoke the following fact on the tail probability regarding the above two random variables.
    \begin{proposition}
        \label{prop:tail-prob dominate}
        For $R_X$ and $R_Y$ defined in \eqref{eq:hypergeom equiv} and \eqref{eq:binomial approx}, respectively, we have that $\PP(R_X\ge t) \le 2\PP(R_Y \ge t)$ for any $t>1$.
    \end{proposition}
    The proof of the proposition is deferred to the end of the proof.
    Thus, it suffices to study the tail probability of $W$.
    Note that $W \overset{d}{=} \sum_{j=1}^k w_j$ where $w_j$ are i.i.d. sampled from
    \begin{align}
        w_j = \begin{cases}
            \frac{1}{k}, & \mbox{ w.p. }\frac{k}{2(d-k)} \\
            - \frac{1}{k}, & \mbox{ w.p. }\frac{k}{2(d-k)} \\
            0, & \mbox{ w.p. }1-\frac{k}{d-k}
        \end{cases}, \quad j\in [k].
    \end{align}
    where $\EE[w_j] = 0$ and $\EE[w_j^2] = (k(d-k))^{-1}$.
    Hence, we can apply the Bernstein inequality to obtain that for any $t > 1/k$,
    \begin{align}
        \PP(\langle \omega^{(i)}, \omega^{(j)}\rangle \ge t) &\le 2 \PP(W \ge t) \le 2 \exp\left(
            - \frac{k (t/k)^2 / 2}{ (k(d-k))^{-1} + t/(3 k^2)}
        \right) \\
        &= 2\exp\left(
            -\frac{k^2 t^2}{2k^2/(d-k) + 2kt/3}
        \right) \le 2 \exp\left( - \min\left\{ \frac{(d-k)t^2}{4}\:,\: \frac{3kt}{4} \right\}\right).
    \end{align}
    Suppose we randomly sample $m$ i.i.d. $\omega^{(i)}$ from the same distribution.
    Then the probability that all such pair $|\langle \omega^{(i)}, \omega^{(j)}\rangle| < t$ for $t > 1/k$ is lower bounded by
    \begin{align}
        \PP\bigl(|\langle \omega^{(i)}, \omega^{(j)}\rangle| < t, \forall i\neq j\bigr) &\ge 1 - m^2 \cdot  2\PP(\langle \omega^{(i)}, \omega^{(j)}\rangle \ge t) \\
        &\ge 1 - 4 m^2 \cdot \exp\left( - \min\left\{ \frac{(d-k)t^2}{4}\:,\: \frac{3kt}{4} \right\}\right).
    \end{align}
    Ensuring that the probability is nonzero will give us a valid construction of the set $\cD$.
    Therefore, there must exist a $\cD$ satisfying $|\langle \omega^{(i)}, \omega^{(j)}\rangle| < t$ for any $i\neq j$ and with size
    \begin{align}
        |\cD| \ge \frac 1 2 \cdot \exp\left( \frac{\min\left\{(d-k)t^2 \:,\: {3kt} \right\}}{8}\right).
    \end{align}
    Hence, we complete the proof.
\end{proof}
Next, we aim to present the proof of \Cref{prop:tail-prob dominate}. To proceed, let us introduce the definition of stochastic dominance.
\begin{definition}[Stochastic Dominance]
    For any real-valued random variable $X$ and $Y$, we say that $X$ is \emph{stochastically dominated} by $Y$, denoted by $X \preceq Y$, if $\PP(X \ge t) \leq \PP (Y \ge t)$ for every $t$.
\end{definition}
The following result is from Theorem A in Chapter 2 of \citet{szekli2012stochastic}.
\begin{proposition}
    \label{prop:coupling st}
For random variables $X$ and $Y$, $X \preceq Y$ if and only if there exists a coupling $(\hat X, \hat Y)$ with $\hat X \overset{d}{=} X$ and $\hat Y \overset{d}{=} Y$ such that $\hat X \le \hat Y$ almost surely.
\end{proposition}
\begin{proposition}[Theorem 1.1, \citet{klenke2010stochastic}]
    \label{prop:hypergeom stle binomial}
For any integers $d$ and $k$, let $X\sim \hypergeom(d, k, k)$ and $Y\sim \binomial(k, k/(d-k))$.
Then $X\preceq Y$.
\end{proposition}
Another way to think of the problem is that $\hypergeom(d, k, k)$ corresponds to the number of times a black ball is drawn when sampling for $k$ times from an urn with $d-k$ white balls and $k$ black balls without replacement, while $\binomial(k, k/(d-k))$ corresponds to sampling in the same urn but with replacement.
We claim the following fact on the tail probability of sums of Rademacher random variables.
\begin{proposition}[Sum of Rademacher Random Variables]
    \label{prop:convolution-rad}
    Let $r_1, r_2, \ldots$ be i.i.d. Rademacher random variables. Let $R_l = r_1 + \ldots + r_l$ for $l = 1, 2, \ldots$. Let $p_l(\cdot )$ be the probability mass function of $R_l$. Then the following holds for any $l = 1, 2, \ldots$:
    \begin{enumerate}
        \item $p_l$ is symmetric and supported on the set of odd integers if $l$ is odd, and supported on the set of even integers if $l$ is even.
        \item For $i \in \supp(p_l)$ and $i \ge 0$, $p_{l}(i)$ is a non-increasing function of $i$.
        \item $\PP(R_l \ge t) \le \PP(R_{l+2} \ge t)$ for any $t > 1$.
        \item $\PP(R_l \ge t) \le 2\PP(R_{l+1} \ge t)$ for any $t > 1$.
        \item $\PP(R_{l} \ge t) \le 2 \PP(R_{l + l'} \ge t)$ for any $l \ge 1$ and $l' \ge 1$.
    \end{enumerate}
\end{proposition}
\begin{proof}[Proof of \Cref{prop:convolution-rad}]
    The first claim is immediate from the symmetry of the Rademacher random variables and the fact that the sum of an odd number of Rademacher random variables is odd, while the sum of an even number of Rademacher random variables is even.
    For the second claim, we note that
    \begin{align}
        p_l(i) = 2^{-l} \cdot \binom{l}{(i + l)/2}, \quad i \in \supp(p_l),
    \end{align}
    which is a non-increasing function for $i \ge 0$.
    For the third claim, we let $t^\star = 2\ceil{t/2}$ if $l$ is even and $t^\star = 2\ceil{(t-1)/2} + 1$ if $l$ is odd.
    In other words, $t^\star = \min\{ \tau \in \supp(p_l): \tau \ge t\}$.
    Then we have that
    \begin{align}
        \PP(R_{l+2}\ge t)
        &= \PP(R_{l} \ge t^\star + 2) + \PP(R_{l} = t^\star) \cdot \PP(r_{l+1}+r_{l+2} \ge 0) \\
        &\qqquad + \PP(R_{l} = t^\star - 2) \cdot \PP(r_{l+1}+r_{l+2} = 2)\\
        & = \PP(R_{l} \ge t^\star) + \left(\PP(R_{l} = t^\star - 2) - \PP(R_{l} = t^\star)\right) \cdot \PP(r_{l+1}+r_{l+2} = 2) \\
        &\ge \PP(R_{l} \ge t^\star) = \PP(R_l \ge t).
    \end{align}
    where in the first equality we use the fact that $r_{l+1} + r_{l+2}$ is supported on $\{-2, 0, 2\}$ and in the second equality we use the symmetric property of the distribution of $r_{l+1} + r_{l+2}$.
    The last inequality follows from the monotonicity of the probability mass function of $R_l$ for $t^\star - 2 \ge 0$ when $t > 1$.
    For the fourth claim, we similarly have that
    \begin{align}
        \PP(R_{l+1}\ge t)
        &\ge \PP(R_{l} \ge t^\star) - \PP(R_{l} = t^\star ) \cdot \PP(r_{l+1} = -1) \ge \frac 1 2 \PP(R_l \ge t^\star) = \frac 1 2 \PP(R_l \ge t).
    \end{align}
    The last claim follows from a combination of the third and fourth claims where
    \begin{align}
        \PP(R_{l+l'} \ge t) \ge \frac{1}{2} \PP(R_{l + 2\floor{l'/2}} \ge t) \ge \frac{1}{2} \PP(R_{l + 2\floor{l'/2} - 2} \ge t) \ge \ldots \ge \frac{1}{2} \PP(R_{l} \ge t).
    \end{align}
    Hence, the proof is complete.
\end{proof}
Next, we proceed to the proof of \Cref{prop:tail-prob dominate}.
\begin{proof}[Proof of \Cref{prop:tail-prob dominate}]
    By \Cref{prop:hypergeom stle binomial} and \Cref{prop:coupling st}, there exists a coupling $\hat X, \hat Y$ with $\law(\hat X) = \law(X)$ and $\law(\hat Y) = \law(Y)$ such that $\hat X \le \hat Y$ almost surely where $X\sim \hypergeom(d, k, k)$ and $Y\sim \binomial(k, k/(d-k))$.

    Consider i.i.d. Rademacher random variables $r_1, r_2, \ldots, r_{k}$.
    Let $R_l = r_1 + \ldots + r_l$ for $l = 1, 2, \ldots$.
    Since $R_{\hat X} = r_1+\ldots+r_{\hat X} \given \hat X \overset{d}{=} 2\binomial(\hat X, 1/2) - \hat X$ and $R_{\hat Y} = r_1+\ldots+r_{\hat Y} \given \hat Y \overset{d}{=} 2\binomial(\hat Y, 1/2) - \hat Y$ for the coupling $(\hat X, \hat Y)$ with $\hat X \le \hat Y$, we consider the conditional random variable
    \begin{align}
        r_{\hat X +i} \given (R_{\hat X + i-1}, \hat X, \hat Y) = r_{\hat X + i} = \begin{cases}
            1, & \mbox{ w.p. } 1/2\\
            -1, & \mbox{ w.p. } 1/2
        \end{cases}, \quad i = 1, 2, \ldots, \hat Y - \hat X
    \end{align}
    The equality holds by the i.i.d. property of these Rademacher random variables.
    From the distributional perspective, the distribution of $R_{\hat Y}$ is obtained by conducting convolution with the Rademacher distribution for $\hat Y - \hat X$ times on the distribution of $R_{\hat X}$. Invoking \Cref{prop:convolution-rad}, we directly conclude that $\PP(R_{\hat Y} \ge t \given \hat X, \hat Y) \ge \PP(R_{\hat X} \ge t \given \hat X, \hat Y)/2 $ for any $t > 1$ and $\hat Y \ge \hat X$.
    As $\hat Y \ge \hat X$ holds almost surely, by the law of total probability, we arrive at the conclusion that $\PP(R_{\hat Y} \ge t) \ge \PP(R_{\hat X} \ge t)/2$ for any $t > 1$.
\end{proof}

\section{Supporting Lemmas on Moment Calculations}\label{app:lem_moment_calculation}
\begin{lemma}[First moment]\label{lem:g-1st-moment}
Under \Cref{asp:oracle function}, for any $s\geq 0$, it holds for any $y\in\RR$ and $w,\theta \in \SSS^{d-1}$ that
\begin{align}
    \EE_{z\sim \cN(0,I_d)}\big[
        \psi(y, \langle w, z\rangle ) z \cdot h_s(\langle \theta, z\rangle)
    \big]  = 
        \sqrt{s+1} \cdot \hat \psi_{s+1}(y)
    \cdot \langle w,\theta\rangle^s w + 
    \sqrt{s}\cdot \hat \psi_{s-1}(y) \cdot \langle w,\theta\rangle^{s-1} \theta,
\end{align}
in the $L^2$ sense over the marginal distribution of $y$ under $\QQ$.
Here and below, the second term is interpreted as zero when $s=0$.
\end{lemma}
\begin{proof}[Proof of \Cref{lem:g-1st-moment}]
For convenience, we denote $\rho:=\langle w,\theta\rangle$.
We first consider $|\rho|<1$; the cases $|\rho|=1$ follow by continuity.
Decomposing $z$ into the direction of $\theta$ and the orthogonal complement of $\theta$, i.e., $z=\langle \theta,z\rangle\theta + P_{\theta}^\perp z$ where $P_{\theta}^{\perp} = I - {\theta}{\theta}^\top$, we have
\begin{align}
    \EE_{z\sim\cN_d}[\psi(y,\langle w, z\rangle) z \cdot h_s(\langle\theta, z\rangle)]
    & = \EE_{z\sim\cN_d} \big[
        \psi(y,\langle w,z\rangle) \cdot \langle\theta,z\rangle \cdot h_s(\langle\theta,z\rangle)
    \big] \cdot {\theta} \\
    &\qquad + \EE_{z\sim \cN_d} \big[
        \psi(y,\langle w,z\rangle) \cdot h_s(\langle\theta,z\rangle) \cdot P_{\theta}^{\perp} z
    \big] \label{eq:g-1st-moment-1}.
\end{align}
Further note that $P_{\theta}^{\perp} z = \langle P_{\theta}^\perp z,P_{\theta}^\perp w\rangle P_{\theta}^\perp w/\|P_{\theta}^\perp w\|_2^2 + P_w^\perp P_{\theta}^\perp z$ where $P_{\theta}^\perp w = w-\rho\theta$ and $\|P_{\theta}^\perp w\|_2^2 = 1-\rho^2$.
Then by symmetry of the Gaussian distribution, and plugging in $P_{\theta}^\perp z = z - \langle\theta,z\rangle \theta$, the second term on the right-hand side of \eqref{eq:g-1st-moment-1} can be written as 
\begin{align}
    \EE_{z\sim \cN_d} \big[
        \psi(y,\langle w,z\rangle) \cdot h_s(\langle\theta,z\rangle) \cdot P_{\theta}^{\perp} z
    \big] &= \EE_{z\sim\cN_d} \big[
        \psi(y,\langle w,z\rangle) \cdot h_s(\langle\theta,z\rangle) \cdot \langle P_{\theta}^\perp z, w-\rho\theta\rangle
    \big] \cdot \frac{w-\rho\theta}{1-\rho^2} \\
    &= \EE_{z\sim\cN_d} \big[
        \psi(y,\langle w,z\rangle) \cdot \langle w,z\rangle \cdot h_s(\langle\theta,z\rangle)
    \big] \cdot \frac{w-\rho\theta}{1-\rho^2} \\
    &\qquad + \EE_{z\sim\cN_d} \big[
        \psi(y,\langle w,z\rangle) \cdot \langle \theta,z\rangle \cdot h_s(\langle\theta,z\rangle) 
    \big] \cdot \frac{-\rho w+\rho^2\theta}{1-\rho^2}
\end{align}
Plugging the above identity back into \eqref{eq:g-1st-moment-1}, we obtain
\begin{align}
    &\EE_{z\sim\cN_d}[\psi(y,\langle w, z\rangle) z \cdot h_s(\langle\theta, z\rangle)] \label{eq:g-1st-moment-2}\\
    &\quad = \underbrace{\EE_{z\sim\cN_d} \big[
        \psi(y,\langle w,z\rangle) \cdot \langle\theta,z\rangle \cdot h_s(\langle\theta,z\rangle)
    \big]}_{\dr (\RNum{1})} \cdot \frac{{\theta} - \rho w}{1- \rho^2} 
    + \underbrace{\EE_{z\sim \cN_d} \big[
        \psi(y,\langle w,z\rangle) \cdot \langle w,z\rangle \cdot h_s(\langle\theta,z\rangle)\big]}_{\dr (\RNum{2})} \cdot \frac{w - \rho {\theta}}{1 - \rho^2}.
\end{align}
Next, we analyze terms (\RNum{1}) and (\RNum{2}) in \eqref{eq:g-1st-moment-2}.

For convenience, we define the following Gaussian noise operator over the second argument of $\psi$: 
\begin{align}
    \noiseop{\rho} \psi(y, x) = \EE_{x'\sim \cN}\big[
        \psi\big(y, \rho x + \sqrt{1-\rho^2} x'\big)
    \big].
\end{align}
For term (\RNum{1}), writing $x=\langle\theta,z\rangle$ and $x'=\langle P_{\theta}^\perp w, z\rangle / \|P_{\theta}^\perp w\|_2$, it follows from $z\sim\cN_d$ that $x,x'\iidfrom\cN$.
Also, $\langle w,z\rangle = \rho x + \sqrt{1-\rho^2} x'$.
Then we have
\begin{align}
    \text{(\RNum{1})} 
    &= \EE_{x\sim\cN} \left[
        \noiseop{\rho} \psi(y,x) \cdot x \cdot h_s(x) \right]
    = \EE_{x\sim\cN} \left[
        \sqrt{s+1} \cdot \noiseop{\rho} \psi(y,x) \cdot h_{s+1}(x) + \sqrt{s} \cdot \noiseop{\rho} \psi(y,x) \cdot h_{s-1}(x) 
    \right] \\
    &\LQeq \sqrt{s+1} \cdot \hat\psi_{s+1}(y) \cdot \rho^{s+1} + \sqrt{s} \cdot \hat\psi_{s-1}(y) \cdot \rho^{s-1}.
    \label{eq:g-1st-moment-term1}
\end{align}
where the second equality follows from the recurrence relation of the Hermite polynomials in \eqref{eq:hermite-recurrence}, 
and the third equality follows from the property of the Gaussian noise operator in \eqref{eq:hermite-noise-operator}. 
Similarly for term (\RNum{2}), we apply \eqref{eq:hermite-noise-operator} to obtain
\begin{align}
    \text{(\RNum{2})}
    &= \EE_{x\sim \cN} \left[
        \noiseop{\rho}\left( \psi(y, x) x\right) \cdot h_s(x)
    \right]
    \overset{L^2(\QQ)}{=} \rho^s \cdot \EE_{x\sim \cN}\left[\psi(y, x) \cdot x \cdot h_s(x) \right] \\
    &\LQeq \rho^s \cdot \left(\sqrt{s+1} \cdot \hat\psi_{s+1}(y) + \sqrt{s} \cdot \hat\psi_{s-1}(y) \right), 
    \label{eq:g-1st-moment-term2}
\end{align}
where the last equality again follows from the recurrence relation of the Hermite polynomials in \eqref{eq:hermite-recurrence}.
Finally, plugging \eqref{eq:g-1st-moment-term1} and \eqref{eq:g-1st-moment-term2} into \eqref{eq:g-1st-moment-1}, 
\begin{align}
    \EE_{z\sim\cN_d}[\psi(y,\langle w, z\rangle) z \cdot h_s(\langle\theta, z\rangle)]
    & \overset{L^2(\QQ)}{=}  \left(\sqrt{s+1} \cdot \hat\psi_{s+1}(y) \cdot \rho^{s+1} + \sqrt{s} \cdot \hat\psi_{s-1}(y) \cdot \rho^{s-1}\right) \cdot \frac{{\theta} - \rho w}{1- \rho^2} \\
    &\qquad + \rho^s \cdot \left(\sqrt{s+1} \cdot \hat\psi_{s+1}(y) + \sqrt{s} \cdot \hat\psi_{s-1}(y) \right)  \cdot \frac{w - \rho {\theta}}{1 - \rho^2} \\
    & \overset{~~~~~~}{=}
        \sqrt{s+1} \cdot \hat \psi_{s+1}(y)
    \cdot \rho^{s} w + 
    \sqrt{s}\cdot \hat \psi_{s-1}(y) \cdot \rho^{s-1} {\theta}, 
\end{align}
which completes the proof.
\end{proof}

An implication of the previous lemma is that 
\begin{align}
    \EE_\QQ[h_{s^\star}(\langle \theta^\star, z\rangle) \cdot \sigma'(\langle z, \theta\rangle) \cdot \langle z, \theta^\star\rangle] = s^\star \cdot \hat\sigma^{(s^\star)} \cdot \langle \theta^\star, \theta\rangle^{s^\star -1} + \sqrt{(s^\star+1)(s^\star+2)} \cdot \hat\sigma^{(s^\star+2)} \cdot \langle \theta^\star, \theta\rangle^{s^\star+1},
\end{align}
where we take $\hat\sigma^{(s)}$ as the $s$-th normalized Hermite coefficient of $\sigma$. 
Here, we take $\psi(y, x)$ as $\sigma'(x)$ and thus $\hat \psi_s(y) = \sqrt{s+1} \cdot \hat\sigma^{(s+1)}$.

\begin{lemma}[Decomposition of first-order moment]\label{lem:first-moment-decomp}
Under \Cref{asp:oracle function}, define
\begin{align}
    g= \frac{1}{nL} \sum_{i=1}^n \sum_{l=1}^L \big(\psi(y_i, \langle w_l, z_i\rangle) \cdot z_i - \widehat\psi_1(y_i) \cdot w_l\big),
\end{align}
where $(z_1,y_1),\ldots,(z_n,y_n)\iidfrom \PP_{\theta^\star}$ and $w_1,\ldots,w_L\in\SSS^{d-1}$ are fixed vectors.
Then it holds that 
\begin{align}
     \EE_{\PP_{\theta^\star}}[g] &=
     \sum_{s\ge s^\star} \EE_{\QQ}[\zeta_s(y)\cdot \hat\psi_{s-1}(y)] \cdot \frac{\sqrt{s}}{L} \sum_{l=1}^L  \langle  w_l, \theta^\star\rangle^{s-1} \cdot \theta^\star \\
     &\qquad +\sum_{s\ge s^\star} \EE_{\QQ}[\zeta_s(y)\cdot \hat\psi_{s+1}(y)] \cdot \frac{\sqrt{s+1}}{L} \sum_{l=1}^L  \langle  w_l, \theta^\star\rangle^{s}\cdot w_l  .
\end{align}
\end{lemma}
\begin{proof}[Proof of \Cref{lem:first-moment-decomp}]  \label{proof:first-moment-decomp} 
        Applying a change of measure from $\PP_{\theta^\star}$ to $\QQ$ and invoking \eqref{eq:likelihood-ratio-decomp}, we get
\begin{align}
    \EE_{\PP_{\theta^\star}}[g] &= \frac 1 L \sum_{l=1}^L \EE_{\PP_{\theta^\star}}\left[\psi(y, \langle  w_l, z\rangle)\cdot  z - \hat\psi_1(y)\cdot   w_l
    \right] \nonumber\\
    &=  \frac 1 L \sum_{l=1}^L \EE_{\QQ}\bigg[
        \psi(y, \langle  w_l, z\rangle ) z \cdot
        \bigg(1 +  \sum_{s\ge s^\star} \zeta_s(y) h_s(\langle \theta^\star, z\rangle)\bigg) - \hat\psi_1(y) w_l \bigg]\\
    &= \frac{1}{L}\sum_{l=1}^L \sum_{s\geq s^\star} \EE_\QQ[\zeta_s(y)\cdot\psi(y,\langle w_l,z\rangle) \cdot z \cdot h_s(\langle\theta^\star,z\rangle)]
    + \frac{1}{L}\sum_{l=1}^L \EE_\QQ\left[\psi(y,\langle w_l,z\rangle)z - \hat\psi_1(y) w_l\right].
\end{align}
Applying \Cref{lem:g-1st-moment} to each term in the above expression completes the proof.
\end{proof}

\begin{lemma}[Second moment on nice event]\label{lem:g-2nd-moment}
Suppose $\psi:\RR\times\RR\rightarrow \RR$ satisfies \Cref{asp:oracle function} with generative exponent $s^\star\geq1$.
Suppose $\EE_{\QQ}[\zeta_s(y)^2] \le C$ for some universal $C=O(1)$ and for all $s\ge s^\star$.
For $\epsilon,\epsilon_1\in(0,1/2)$ satisfying $4e\epsilon^2<1/2$, consider any $w, w', \theta^\star, v\in\SSS^{d-1}$ such that
\begin{gather}
    \max\{|\langle \theta^\star, w\rangle|, |\langle \theta^\star, w'\rangle|\} \le \epsilon, \quad |\langle w, w'\rangle| \le \epsilon_1
\end{gather}
Also denote $\epsilon_0:=\max\{|\langle v, w\rangle|, |\langle v, w'\rangle|\}\in[0,1]$.
Then it holds that
\begin{align}
    \EE_{\PP_{\theta^\star}} \left[
            \psi(y, \langle w, z\rangle) \psi(y, \langle w', z\rangle) \langle v, z\rangle^2
    \right] &\lesssim \ind(s^\star=1) + \epsilon_1^{s^\star -2} \epsilon_0^2 + \epsilon_1^{s^\star - 1} + \epsilon_0^{2} \cdot \epsilon^{2s^\star - 2} + \ind(s^\star \ge 4) \cdot 
    \epsilon_0^{2} \cdot \epsilon^{2 s^\star -3} \\
    &\qquad + \epsilon^{2s^\star - 1}
    + |\langle\theta^\star,v\rangle| \big( \epsilon_1^{s^\star - 2} \cdot \epsilon_0 \cdot \epsilon 
    + \epsilon_0 \cdot \epsilon^{2s^\star - 2} \big),
\end{align}
where $\lesssim$ hides constants that only depend on $s^\star$, $\EE_\QQ[\psi(y, x)^4]$ and $C$.
Moreover, if we further suppose that either $v=\theta^\star$ or $v\perp\theta^\star$, then for $s^\star\geq2$,
\begin{align}
&\EE_{\PP_{\theta^\star}} \left[
    \psi(y, \langle w, z\rangle) \psi(y, \langle w', z\rangle) \langle v, z\rangle^2
\right] \\
& \qquad \lesssim \epsilon_1^{s^\star - 1} \cdot \bigg( 1 +  \frac{\epsilon_0^2}{\epsilon_1} + \epsilon\left(\frac{\epsilon^2}{\epsilon_1}\right)^{s^\star - 1} + \ind(v\perp \theta^\star)\cdot \left(\frac{\epsilon^2}{\epsilon_1}\right)^{s^\star - 2} \cdot \frac{\epsilon_0^2}{\epsilon_1} \cdot (\epsilon^2 + \epsilon \cdot \ind(s^\star \ge 4))\bigg).
\end{align}
\end{lemma}
\begin{proof}
For convenience, we denote the target quantity as $F:= \EE_{\PP_{\theta^\star}}[\psi(y,\langle w,z\rangle) \psi(y,\langle w',z\rangle) \langle v,z\rangle^2]$.
For $s^\star =1$, by \ref{asp:polynomial-like}, we have 
\begin{align}
    |F|
    \leq \EE_{\PP_{\theta^\star}} [\psi(y, \langle w, z\rangle)^4]^{1/4} \cdot \EE_{\PP_{\theta^\star}} [\psi(y, \langle w', z\rangle)^4]^{1/4} \cdot \EE_{\PP_{\theta^\star}}[\langle v, z\rangle^4]^{1/2} = O(1).
\end{align}
For the general case $s^\star\geq 2$, we apply the likelihood ratio decomposition in \eqref{eq:likelihood-ratio-decomp} to get
\begin{align}
    F 
    &= \EE_{\QQ} \left[
        \psi(y, \langle w, z\rangle) \psi(y, \langle w', z\rangle) \langle v, z\rangle^2 
        \right]
    + \sum_{s=s^\star}^\infty \EE_{\QQ} \left[ \zeta_s(y)\psi(y, \langle w, z\rangle) \psi(y, \langle w', z\rangle) \langle v, z\rangle^2 h_s(\langle \theta^\star, z\rangle) \right].
\end{align}
Moreover, we decompose each of the terms using the following identity
\begin{align}
    h_s(\langle\theta^\star, z \rangle) \langle v, z\rangle^2
    &= \sqrt{(s+2)(s+1)} \cdot \bh_{s+2}(z)[(\theta^\star)^{\otimes s} \otimes v^{\otimes 2}] + \bh_s(z) [(\theta^\star)^{\otimes s}]  \label{eq:g-2nd-moment-1}\\
    &\qquad + 2s \cdot \langle\theta^\star,v\rangle \cdot \bh_s(z)[(\theta^\star)^{\otimes s-1} \otimes v] 
    + \sqrt{s(s-1)} \cdot \langle \theta^\star, v\rangle^2 \cdot \bh_{s-2}(z)[(\theta^\star)^{\otimes s-2}],
\end{align}
which is derived from \eqref{eq:hermite-tensor-directional} and \Cref{prop:2nd-moment-tensor}, and we get 
\begin{align}
    F &= \underbrace{\EE_{\QQ} \left[\psi(y, \langle w, z\rangle) \psi(y, \langle w', z\rangle)(\sqrt{2} \cdot \bh_{2}(z)[v^{\otimes 2}] + 1)\right]}_{F_1}\\
    &\qquad+ \underbrace{\sum_{s=s^\star}^\infty \sqrt{(s+2)(s+1)} \cdot \EE_{\QQ} \left[
        \zeta_s(y) \psi(y, \langle w, z\rangle) \psi(y, \langle w', z\rangle) \bh_{s+2}(z)[(\theta^\star)^{\otimes s} \otimes v^{\otimes 2}]
    \right]}_{F_2}\\
    &\qquad + \underbrace{\sum_{s=s^\star}^\infty \EE_{\QQ} \left[
        \zeta_s(y) \psi(y, \langle w, z\rangle) \psi(y, \langle w', z\rangle) \bh_s(z)[(\theta^\star)^{\otimes s}]
    \right]}_{F_3}\\
    &\qquad + \underbrace{\sum_{s=s^\star}^\infty 2s \cdot \langle \theta^\star, v\rangle \cdot \EE_{\QQ} \left[
        \zeta_s(y) \psi(y, \langle w, z\rangle) \psi(y, \langle w', z\rangle) \bh_s(z)[(\theta^\star)^{\otimes (s-1)} \otimes v]
    \right]}_{F_4}\\
    &\qquad + \underbrace{\sum_{s=s^\star}^\infty \sqrt{s(s-1)} \cdot \langle \theta^\star, v\rangle^2 \cdot \EE_{\QQ} \left[
        \zeta_s(y) \psi(y, \langle w, z\rangle) \psi(y, \langle w', z\rangle) \bh_{s-2}(z)[(\theta^\star)^{\otimes (s-2)}]
    \right]}_{F_5}. \label{eq:F_decomp}
\end{align}
We bound each term $F_1, F_2, F_3, F_4, F_5$ separately.

First, for the term $F_1$ in \eqref{eq:F_decomp}, 
we apply \Cref{lem:psi-2nd-moment} with test tensor $T_2 = v^{\otimes 2}$ and $c_0 = 2$ to get
\begin{align}
    \left|\EE_\QQ\left[\psi(y, \langle w, z\rangle) \psi(y, \langle w', z\rangle) \bh_{2}(z)[v^{\otimes 2}]\right]\right|
    &\lesssim \epsilon_1^{s^\star-2} \epsilon_0^2
\end{align}
where $\lesssim$ hides constants that only depend on $s^\star$ and $\EE_\QQ[\psi(y, x)^4]$.
Similarly, applying \Cref{lem:psi-2nd-moment} for $\bh_0(z)[1] = 1$, we have $|\EE_\QQ[\psi(y,\langle w,z\rangle)\psi(y,\langle w',z\rangle)]|\lesssim \epsilon_1^{s^\star - 1}$.
Therefore, 
\begin{align}\label{eq:bound_F1}
    |F_1|
    = \left|\EE_{\QQ} \left[
        \psi(y, \langle w, z\rangle) \psi(y, \langle w', z\rangle) \langle v, z\rangle^2 \right] \right|
    &\lesssim \epsilon_1^{s^\star - 2} \epsilon_0^2 + \epsilon_1^{s^\star - 1}.
\end{align}
Then for the term $F_2$ in \eqref{eq:F_decomp} corresponding to $\sqrt{(s+2)(s+1)} \cdot \bh_{s+2}(z)[(\theta^\star)^{\otimes s} \otimes v^{\otimes 2}]$,
we apply \Cref{prop:psi-2nd-moment-series} with test tensor $T_{s} = v^{\otimes 2} \otimes (\theta^\star)^{\otimes (s-2)}, c_0 = 2$ and $s_0= s^\star + 2$ to get that for $s^\star\geq 2$,
\begin{align}
    |F_2| &= \bigg|\sum_{s=s^\star}^\infty \sqrt{(s+2)(s+1)} \cdot \EE_{\QQ} \left[
        \zeta_s(y) \psi(y, \langle w, z\rangle) \psi(y, \langle w', z\rangle) \bh_{s+2}(z)[(\theta^\star)^{\otimes s} \otimes v^{\otimes 2}]\right]\bigg| \\
    &\lesssim  \epsilon_0^{2} \cdot \epsilon^{2s^\star - 2} + \ind(s^\star \ge 4) \cdot 
        \big( \epsilon_1^{s^\star -2} \cdot \epsilon_0^{2} \cdot \epsilon + \epsilon_0^{2} \cdot \epsilon^{2 s^\star -3} \big).
\end{align}
Next for the term $F_3$ in \eqref{eq:F_decomp} corresponding to $\bh_s(z) [(\theta^\star)^{\otimes s}]$, we apply \Cref{prop:psi-2nd-moment-series} with test tensor $T_s = (\theta^\star)^{\otimes s}, c_0 = 0$ and $s_0 = s^\star$ to get that for $s^\star\geq 2$,
\begin{align}
    |F_3| &= \bigg|\sum_{s=s^\star}^\infty \EE_{\QQ} \left[
        \zeta_s(y) \psi(y, \langle w, z\rangle) \psi(y, \langle w', z\rangle) \bh_s(z)[(\theta^\star)^{\otimes s}]\right]\bigg|\\
    &\lesssim \epsilon^{2s^\star} + \epsilon_1^{s^\star - 1} \cdot \epsilon + \epsilon^{2s^\star - 1} 
    \lesssim \epsilon_1^{s^\star - 1} \cdot \epsilon + \epsilon^{2s^\star - 1}. 
\end{align}
Next for the term $F_4$ in \eqref{eq:F_decomp} corresponding to $2s \cdot \bh_s(z)[(\theta^\star)^{\otimes s-1} \otimes v]$, we apply \Cref{prop:psi-2nd-moment-series} with test tensor $T_s = v \otimes (\theta^\star)^{\otimes (s-1)}, c_0 = 1$ and $s_0 = s^\star$ to get
\begin{align}
    |F_4| &= \bigg| \sum_{s=s^\star}^\infty 2s \langle\theta^\star,v\rangle \cdot \EE_{\QQ} \left[
        \zeta_s(y) \psi(y, \langle w, z\rangle) \psi(y, \langle w', z\rangle) \bh_s(z)[(\theta^\star)^{\otimes (s-1)} \otimes v]\right]\bigg|\\
    &\lesssim |\langle\theta^\star,v\rangle| \big(\epsilon_0\cdot \epsilon^{2s^\star-1} + \epsilon_1^{s^\star - 2} \cdot \epsilon_0 \cdot \epsilon + \epsilon_0 \cdot \epsilon^{2s^\star - 2} \big)
    \lesssim |\langle\theta^\star,v\rangle| \big(\epsilon_1^{s^\star - 2} \cdot \epsilon_0 \cdot \epsilon + \epsilon_0 \cdot \epsilon^{2s^\star - 2}\big).
\end{align}
For the last term $F_5$ in \eqref{eq:F_decomp} corresponding to $\sqrt{s(s-1)} \cdot \bh_{s-2}(z)[(\theta^\star)^{\otimes s-2}]$, we reindex by $r=s-2$ and apply \Cref{prop:psi-2nd-moment-series} with test tensor $T_r = (\theta^\star)^{\otimes r}, c_0 = 0$ and $s_0 = s^\star - 2$ to get
\begin{align}
    |F_5| &= \bigg|  \sum_{s=s^\star}^\infty \sqrt{s(s-1)}\cdot \langle \theta^\star, v\rangle^2 \cdot \EE_{\QQ} \left[
        \zeta_s(y) \psi(y, \langle w, z\rangle) \psi(y, \langle w', z\rangle) \bh_{s-2}(z)[(\theta^\star)^{\otimes s-2}]\right]\bigg| \\
    &\lesssim \langle\theta^\star,v\rangle^2 \big(\ind(s^\star = 2) \cdot \epsilon_1  + \epsilon^{2s^\star} + \epsilon_1^{s^\star - 1} \cdot \epsilon + \epsilon^{2s^\star -1} \big)
    \lesssim \langle\theta^\star,v\rangle^2 \big(\ind(s^\star = 2) \cdot \epsilon_1 + \epsilon_1^{s^\star - 1} \cdot \epsilon + \epsilon^{2s^\star -1}\big). 
\end{align}

Combining the above bounds and rearranging the terms, we obtain
\begin{align}
    F
    &\lesssim \big(\epsilon_1^{s^\star -2} \epsilon_0^2 + \epsilon_1^{s^\star - 1} \big) 
    + \big(\epsilon_0^{2} \cdot \epsilon^{2s^\star - 2} + \ind(s^\star \ge 4) \cdot ( \epsilon_1^{s^\star -2} \cdot \epsilon_0^{2} \cdot \epsilon + \epsilon_0^{2} \cdot \epsilon^{2 s^\star -3}) \big) 
    + \big(\epsilon_1^{s^\star - 1} \cdot \epsilon + \epsilon^{2s^\star - 1} \big)\\
    &\quad + |\langle\theta^\star,v\rangle| \big( \epsilon_1^{s^\star - 2} \cdot \epsilon_0 \cdot \epsilon + \epsilon_0 \cdot \epsilon^{2s^\star - 2} 
    + \ind(s^\star = 2) \cdot \epsilon_1 + \epsilon_1^{s^\star - 1} \cdot \epsilon + \epsilon^{2s^\star -1} \big)\\
    &\lesssim \epsilon_1^{s^\star -2} \epsilon_0^2 + \epsilon_1^{s^\star - 1} + \epsilon_0^{2} \cdot \epsilon^{2s^\star - 2} + \ind(s^\star \ge 4) \cdot 
     \epsilon_0^{2} \cdot \epsilon^{2 s^\star -3} + \epsilon^{2s^\star - 1}
    + |\langle\theta^\star,v\rangle| \big( \epsilon_1^{s^\star - 2} \cdot \epsilon_0 \cdot \epsilon 
    + \epsilon_0 \cdot \epsilon^{2s^\star - 2} \big).
\end{align}
If $v = \theta^\star$, then we additionally have $\epsilon_0 = \epsilon$, which simplifies the above bound to
\begin{align}
    F \given_{s^\star \ge 2, v=\theta^\star}\lesssim \epsilon_1^{s^\star -2} \epsilon^2 + \epsilon_1^{s^\star - 1} + \epsilon^{2s^\star - 1}
    = \epsilon_1^{s^\star - 1} \cdot \bigg( 1 +  \frac{\epsilon_0^2}{\epsilon_1} + \left(\frac{\epsilon^2}{\epsilon_1}\right)^{s^\star - 1} \cdot \epsilon\bigg).
\end{align}
For $v\perp \theta^\star$, we have 
\begin{align}
    F \given_{s^\star \ge 2, v\perp \theta^\star} 
    &\lesssim \epsilon_1^{s^\star -2} \epsilon_0^2 + \epsilon_1^{s^\star - 1} + \epsilon_0^{2} \cdot \epsilon^{2s^\star - 2} + \ind(s^\star \ge 4) \cdot 
     \epsilon_0^{2} \cdot \epsilon^{2 s^\star -3} + \epsilon^{2s^\star - 1} \\
    & \lesssim \epsilon_1^{s^\star - 1} \bigg(1 + \frac{\epsilon_0^2}{\epsilon_1} 
    + \left(\frac{\epsilon^2}{\epsilon_1}\right)^{s^\star - 1} \cdot \epsilon 
    + \left(\frac{\epsilon^2}{\epsilon_1}\right)^{s^\star - 2} \cdot \frac{\epsilon_0^2}{\epsilon_1} \cdot \big(\epsilon^2 + \epsilon \cdot \ind(s^\star \ge 4)\big)\bigg). 
\end{align}
Hence, we complete the proof. 
\end{proof}
\begin{lemma}\label{lem:g-1st-w-expectation}
For polarization level $\gamma \in (0,1/2)$ and $e_1=(1,0,\ldots,0)^\top$, consider the polarized random vector 
\[
    w = \frac{\gamma e_1 + \xi}{\norm{\gamma e_1 + \xi}_2}, \where \xi \sim \unif(\SSS^{d-1}), 
\] 
For $\rho\in[0,1]$, let $\theta^\star = (\rho, \sqrt{1-\rho^2}, 0, \ldots, 0)\in\SSS^{d-1}$ be a fixed direction. 
Then, for $s\geq 1$,
\begin{align}
    \EE[\langle \theta^\star, w\rangle^{s}] \simeq \begin{cases}
        \ds 
        \left(|\rho|(\gamma + d^{-1/2}) + \sqrt{1-\rho^2}d^{-1/2}\right)^{s} & \text{if $s$ is even}\\
        \ds \rho \gamma \left(|\rho|(\gamma + d^{-1/2}) + \sqrt{1-\rho^2}d^{-1/2}\right)^{s-1} & \text{if $s$ is odd}
        \end{cases}
\end{align}
where $\simeq$ hides constants that only depend on $s$.
\end{lemma}

\begin{proof}[Proof of \Cref{lem:g-1st-w-expectation}]
Expanding $\langle\theta^\star,w\rangle^s$, we have
\begin{align}
    \EE[\langle \theta^\star, w\rangle^{s}] &= \sum_{j=0}^s \binom{s}{j} \rho^j (1-\rho^2)^{(s-j)/2} \cdot \EE[w_1^j w_2^{s-j}]\notag\\
    &= \sum_{j=0}^s \binom{s}{j} \rho^j (1-\rho^2)^{(s-j)/2} \cdot \ind(\text{$s-j$ even}) \cdot \EE[w_1^j w_2^{s-j}]\label{eq:expand_theta_w}
\end{align}
where the second equality follows from \Cref{lem:polar-moment}.
Next, we consider the case of even $s$ and odd $s$ separately.

For even $s\geq 1$, $s-j$ is even if and only if $j$ is even, and thus we can rewrite \eqref{eq:expand_theta_w} as
\begin{align}
    \EE[\langle \theta^\star, w\rangle^{s}] 
    &= \sum_{j=0}^{s/2} \binom{s}{2j} \rho^{2j} (1-\rho^2)^{(s-2j)/2} \cdot \EE[w_1^{2j} w_2^{s-2j}].\notag
\end{align}
Now applying \Cref{lem:polar-moment} again, we have 
\begin{align}
    \EE[w_1^{2j} w_2^{s-2j}] = C(2j e_1+(s-2j)e_2,\gamma) \cdot \bigg(\gamma + \frac{1}{\sqrt{d}}\bigg)^{2j}  \cdot \bigg(\frac{1}{\sqrt{d}}\bigg)^{s-2j}
    \simeq \bigg(\gamma + \frac{1}{\sqrt{d}}\bigg)^{2j}  \cdot \bigg(\frac{1}{\sqrt{d}}\bigg)^{s-2j}.\notag
\end{align}
where $\simeq$ hides constants that only depend on $s$.
Therefore, we further have 
\begin{align}
    \EE[\langle\theta^\star,w\rangle^s] &\simeq \sum_{j=0}^{s/2} \binom{s}{2j} \rho^{2j} (1-\rho^2)^{(s-2j)/2} \cdot \bigg(\gamma + \frac{1}{\sqrt{d}}\bigg)^{2j}  \cdot \bigg(\frac{1}{\sqrt{d}}\bigg)^{s-2j}\notag\\
    &= \sum_{j=0}^{s/2} \binom{s}{2j} \bigg(|\rho|\gamma + \frac{|\rho|}{\sqrt{d}}\bigg)^{2j} \cdot \bigg(\frac{\sqrt{1-\rho^2}}{\sqrt{d}}\bigg)^{s-2j}\notag\\
    &\simeq \Big(|\rho|(\gamma+d^{-1/2}) + \sqrt{1-\rho^2}d^{-1/2}\Big)^s\notag
\end{align}
where the last equality is due to \Cref{lem:binomial_ratio}.

The case for odd $s\geq 1$ is similar.
Here $s-j$ is even if and only if $j$ is odd, and then it follows from \Cref{lem:polar-moment} and \Cref{lem:binomial_ratio} that
\begin{center}
\makebox[\textwidth][c]{$\displaystyle\begin{aligned}
    \EE[\langle\theta^\star,w\rangle^s] &= \sum_{j=0}^{(s-1)/2} \binom{s}{2j+1} \rho^{2j+1} (1-\rho^2)^{(s-2j-1)/2} \cdot \EE[w_1^{2j+1} w_2^{s-2j-1}]\\
    &\simeq \rho\sum_{j=0}^{(s-1)/2} \binom{s}{2j+1} \rho^{2j} (1-\rho^2)^{(s-2j-1)/2} \cdot \gamma \bigg(\gamma+\frac{1}{\sqrt{d}}\bigg)^{2j} \cdot \bigg(\frac{1}{\sqrt{d}}\bigg)^{s-2j-1}\\
    &\simeq \rho\gamma \Big(|\rho|(\gamma+d^{-1/2}) + \sqrt{1-\rho^2}d^{-1/2}\Big)^{s-1}
\end{aligned}$}
\end{center}
where $\simeq$ hides constants that only depend on $s$.
This completes the proof.
\end{proof}

\section{Technical Results}\label{app:technical lemma}

\subsection{Technical Results for Hermite Tensor}
\begin{proposition}\label{prop:2nd-moment-tensor}
Let $s\in\NN_0$. 
Then for any $z\in\RR^d$ and $i,j\in[d]$,
\begin{align}\label{eq:2nd-moment-tensor}
    z_i z_j \bh_s(z) 
    &=\Sym\Bigl(\sqrt{(s+2)(s+1)} \cdot \bh_{s+2}(z)[e_i\otimes e_j] + \delta_{ij} \bh_s(z) + s\cdot \bh_s(z)[e_j] \otimes e_i \\
    &\hspace{3cm} + s \cdot \bh_s(z)[e_i] \otimes e_j + \sqrt{s(s-1)} \cdot \bh_{s-2}(z)\otimes e_i \otimes e_j \Bigr), 
\end{align}
where $e_i$ and $e_j$ are the $i$-th and $j$-th canonical basis vectors in $\RR^d$ respectively.
Here by convention, $\bh_{-1}(z)$ and $\bh_{-2}(z)$ are always zero tensors.
\end{proposition}
\begin{proof}[Proof of \Cref{prop:2nd-moment-tensor}]
Note that each element of $h_s(\langle z,\theta^\star\rangle) z z^\top$ is a polynomial of the entries of $z$ of degree at most $s+2$.
Therefore, it suffices to verify \eqref{eq:2nd-moment-tensor} by considering the product of $h_s(\langle z,\theta^\star\rangle) z_iz_j$ with any monomial of degree at most $s+2$, under the expectation of $z\sim\cN_d$.
Consider any such test function $F:\RR^d\to\RR$ being a monomial of degree at most $s+2$, and we have
\begin{align}
    \EE_{z\sim\cN_d}[F(z) h_s(\langle z, \theta^\star\rangle) z_i z_j] = \EE_{z\sim\cN_d}[F(z) z_i z_j \cdot \bh_s(z)[(\theta^\star)^{\otimes s}]]
\end{align}
where we apply \eqref{eq:hermite-tensor-directional} to rewrite the Hermite polynomial in terms of the Hermite tensor.
Since $F$ is a monomial, it is infinitely differentiable, and moreover, all derivatives of $F$ are square-integrable with respect to the standard normal distribution.
Then by Stein's lemma for Hermite tensor \eqref{eq:hermite-tensor-Stein's lemma},
\begin{align}
    &\sqrt{s!} \cdot \EE_{z\sim\cN_d} [F(z) z_i z_j \bh_s(z)] = \EE_{z\sim\cN_d} [\nabla^s (F(z) z_i z_j)] \\
    &\quad =  \Sym\left(\EE_{z\sim\cN_d} [\nabla^s F(z) z_i z_j] + 
    s \cdot \EE_{z\sim\cN_d} [\nabla^{s-1}F(z) \otimes \nabla (z_i z_j) ] 
    + \frac{s(s-1)}{2} \cdot \EE_{z\sim\cN_d} [\nabla^{s-2}F(z)\otimes \nabla^2 (z_i z_j)]\right)\label{eq:2nd-moment-tensor-0}.
\end{align}
Here we recall the symmetrization operator $\Sym$ defined in \eqref{eq:def_symmetrization}.
Further apply Stein's lemma to $\EE_{z\sim\cN_d} [\nabla^s F(z) z_i z_j]$, and we get
\begin{align}
    \EE_{z\sim\cN_d} [\nabla^s F(z) z_i z_j] 
    &= \Sym\left(\EE_{z\sim\cN_d} [\nabla^s F(z) \delta_{ij}] + \EE_{z\sim\cN_d} [\nabla^{s+1} F(z)[e_j] z_i]\right) \\
    &= \Sym\left(\EE_{z\sim\cN_d} [\nabla^s F(z) \delta_{ij}] + \EE_{z\sim\cN_d} [\nabla^{s+2} F(z)[e_i\otimes e_j]]\right).\label{eq:2nd-moment-tensor-01}
\end{align}
Similarly, we apply Stein's lemma to $\EE_{z\sim\cN_d} [\nabla^{s-1}F(z)\otimes \nabla (z_i z_j)]$ to obtain
\begin{align}
    \Sym\left(\EE_{z\sim\cN_d} [\nabla^{s-1}F(z) \nabla (z_i z_j) ]\right)
    &= \Sym\left(\EE_{z\sim\cN_d} [z_j \nabla^{s-1}F(z) \otimes e_i + z_i \nabla^{s-1}F(z) \otimes e_j]\right) \\
    &= \Sym\left(\EE_{z\sim\cN_d} [\nabla^{s}F(z)[e_j] \otimes e_i] + \EE_{z\sim\cN_d} [\nabla^{s}F(z)[e_i] \otimes e_j]\right). \label{eq:2nd-moment-tensor-02}
\end{align}
Finally, for the third term in \eqref{eq:2nd-moment-tensor-0}, it follows from $\nabla^2(z_iz_j)=e_i\otimes e_j + e_j\otimes e_i$ that 
\begin{align}
    \EE_{z\sim\cN_d} [\nabla^{s-2}F(z)\otimes \nabla^2 (z_i z_j)] = \EE_{z\sim\cN_d} [\nabla^{s-2} F(z) \otimes (e_i\otimes e_j + e_j\otimes e_i)].\label{eq:2nd-moment-tensor-03}
\end{align}
Now, plugging \eqref{eq:2nd-moment-tensor-01}, \eqref{eq:2nd-moment-tensor-02}, and \eqref{eq:2nd-moment-tensor-03} back into \eqref{eq:2nd-moment-tensor-0}, we obtain the following decomposition:
\begin{align}
    \sqrt{s!} \cdot \EE_{z\sim\cN_d} [F(z) z_i z_j \bh_s(z)] &= \Sym\left(\EE_{z\sim\cN_d} [\nabla^s F(z) \delta_{ij}] + \EE_{z\sim\cN_d} [\nabla^{s+2} F(z)[e_i\otimes e_j]]\right)\\
    &\qquad  + \Sym\left(s\cdot\EE_{z\sim\cN_d} [\nabla^{s}F(z)[e_j] \otimes e_i]+ s\cdot\EE_{z\sim\cN_d} [\nabla^{s}F(z)[e_i] \otimes e_j]\right)\\
    &\qquad + \Sym\left(s(s-1)\cdot\EE_{z\sim\cN_d} [\nabla^{s-2}F(z) \otimes e_i \otimes e_j] \right). 
    \label{eq:2nd-moment-tensor-1}
\end{align}
Now, applying Stein's lemma \eqref{eq:hermite-tensor-Stein's lemma} to relate the derivatives of $F$ back to the Hermite tensor, i.e., $\EE_{z\sim\cN_d} [\nabla^s F(z)] = \sqrt{s!} \cdot \EE_{z\sim\cN_d} [F(z) \bh_s(z)]$, it follows from \eqref{eq:2nd-moment-tensor-1} that
\begin{align}
    \EE_{z\sim\cN_d} [F(z) z_i z_j \bh_s(z)]  
    &=\EE_{z\sim\cN_d} \Bigl[ F(z) \cdot \Sym\Bigl(\delta_{ij} \bh_s(z) + \sqrt{(s+2)(s+1)} \cdot \bh_{s+2}(z)[e_i\otimes e_j] \\
    &\qquad + s\cdot \bh_s(z)[e_j] \otimes e_i + s \cdot \bh_s(z)[e_i] \otimes e_j + \sqrt{s(s-1)} \cdot \bh_{s-2}(z)\otimes e_i \otimes e_j \Bigr) \Bigr]. 
\end{align}
Since the monomial $F$ is arbitrary, we conclude that 
\begin{align}
    z_i z_j \bh_s(z) 
    &=\Sym\Bigl(\sqrt{(s+2)(s+1)} \cdot \bh_{s+2}(z)[e_i\otimes e_j] + \delta_{ij} \bh_s(z) + s\cdot \bh_s(z)[e_j] \otimes e_i \\
    &\hspace{3cm} + s \cdot \bh_s(z)[e_i] \otimes e_j + \sqrt{s(s-1)} \cdot \bh_{s-2}(z)\otimes e_i \otimes e_j \Bigr). 
\end{align}
This completes the proof.
\end{proof}

\begin{proposition}\label{prop:2nd-moment-tensor-product}
Let $w, w'\in\SSS^{d-1}$ and $s\in\NN_0$. 
For any nonnegative integers $i,j$, if $|i - j| \le s \le i + j$ and $s \equiv i - j \mod 2$, then
\begin{align}
    &\EE_{z\sim\cN_d}\left[h_i(\langle w, z\rangle) h_j(\langle w', z\rangle) \cdot \bh_s(z)\right] \\
    &\qquad= \binom{s}{(i - j + s)/2} \cdot \sqrt\frac{i! j!}{s!} \cdot \frac{\langle w, w'\rangle^{(i + j - s)/2}}{((i + j - s)/2)! } \cdot \Sym\bigl( w^{\otimes (i - j + s)/2} \otimes {w'}^{\otimes (j - i +s)/2} \bigr).
\end{align}
Otherwise $\EE_{z\sim\cN_d}\left[h_i(\langle w, z\rangle) h_j(\langle w', z\rangle) \cdot \bh_s(z)\right]$ is the zero tensor.
\end{proposition}
\begin{proof}[Proof of \Cref{prop:2nd-moment-tensor-product}]
By Stein's lemma for Hermite tensors \eqref{eq:hermite-tensor-Stein's lemma}, we have
\begin{align}
    \EE_{z\sim\cN_d}\left[h_i(\langle w, z\rangle) h_j(\langle w', z\rangle) \cdot \bh_s(z)\right] 
    &= \frac{1}{\sqrt{s!}} \cdot \EE_{z\sim\cN_d}\left[ \nabla_z^s \big(h_i(\langle w, z\rangle) h_j(\langle w', z\rangle)\big) \right] \\
    &= \frac{1}{\sqrt{s!}} \cdot \sum_{\tau =0}^s \binom{s}{\tau} \EE_{z\sim\cN_d}\left[ \Sym\left(\nabla_z^\tau h_i(\langle w, z\rangle) \otimes \nabla_z^{s-\tau} h_j(\langle w', z\rangle) \right) \right],
\end{align}
where we recall the $\Sym$ operator defined in \eqref{eq:def_symmetrization}.
Further applying the recurrence relation of Hermite polynomials in \eqref{eq:hermite-recurrence} yields
\begin{align}
    &\EE_{z\sim\cN_d}\left[h_i(\langle w, z\rangle) h_j(\langle w', z\rangle) \cdot \bh_s(z)\right]\\
    &\qquad= \frac{1}{\sqrt{s!}} \cdot \sum_{\tau =0}^s \binom{s}{\tau} \sqrt{\frac{i! j!}{(i-\tau)! (j - s + \tau)!}} \cdot \EE_{z\sim\cN_d}\left[ h_{i -\tau}(\langle w, z\rangle) h_{j-s+\tau}(\langle w', z\rangle)\right] \cdot \Sym\bigl( w^{\otimes \tau} \otimes {w'}^{\otimes s-\tau} \bigr) \\
    &\qquad = \sum_{\tau=0}^s \ind(j = i + s - 2\tau, i\ge \tau) \cdot \frac{\binom{s}{\tau}}{(i-\tau)! } \sqrt\frac{i! j!}{s!} \langle w, w'\rangle^{i-\tau}\Sym\bigl( w^{\otimes \tau} \otimes {w'}^{\otimes s-\tau} \bigr)
\end{align}
where the second equality follows from \eqref{eq:hermite-noise-operator}.
Therefore, the above summation can be non-zero only if $|i-j|\leq s\leq i+j$ and $s \equiv i - j \mod 2$, in which case the only non-zero term is when $\tau = (i - j + s)/2$. 
This gives the desired result.
\end{proof}

\begin{lemma}\label{lem:psi-2nd-moment}
Let $\psi:\RR\rightarrow\RR$ such that $\psi^2 \in L^2(\cN(0,1))$.
For some integer $s^\star\ge 1$, suppose that $\psi$ is high-pass in the sense that the coefficient for its Hermite expansion $\hat\psi_i = 0$ for any $i< s^\star-1$.
For any $\epsilon_1\in(0,1/2)$, fix any $w, w'\in\SSS^{d-1}$ such that $|\langle w,w'\rangle|\leq \epsilon_1$, and consider a series of test tensors $\{T_s = v_1 \otimes v_2 \otimes \ldots \otimes v_s \in (\RR^d)^{\otimes s}\}_{s=0}^\infty$ such that $\sup_{i > c_0} \{|\langle w, v_i\rangle|\lor |\langle w', v_i\rangle|\} \leq \epsilon$
for some $\epsilon \in (0, 1/2)$ and integer $c_0\geq 0$.  
Define $\epsilon_0\defeq \max_{1\le i\le c_0} \{ |\langle w, v_i\rangle| \lor |\langle w', v_i\rangle| \}$.
Then for any $s\in\NN_0$, it holds that
\begin{align}
    \left|\EE_{z\sim \cN_d}\bigl[\psi(\langle w, z\rangle ) \psi(\langle w', z\rangle)\cdot \bh_s(z)[T_s] \bigr]\right|  
    \le 4 \norm{\psi}_2^2 \left(4e s^\star\right)^{s/2} 
    \sqrt{s+s^\star}  
    \cdot \epsilon_1^{(s^\star - 1 - \floor{s/2}) \lor 0} \cdot  \epsilon_0^{s \land c_0} \cdot \epsilon^{(s - c_0) \lor 0}, 
\end{align}
where $\norm{\psi}_2^2 := \EE_{x\sim\cN}[\psi^2(x)]$.
\end{lemma}
\begin{proof}[Proof of \Cref{lem:psi-2nd-moment}]
Since $\psi^2\in L^2(\cN(0,1))$, we know that $\psi(\langle w,z\rangle)\psi(\langle w',z\rangle)\in L^2(\cN_d)$ as a function of $z\sim\cN_d$ for any fixed $w,w'\in\SSS^{d-1}$.
Therefore, we can apply the Hermite expansion of $\psi$ to obtain
\begin{align}
    \EE_{z\sim\cN_d} \left[\psi(\langle w, z\rangle) \psi(\langle w', z\rangle) \cdot \bh_s(z)\right]
    = \sum_{i=0}^\infty \sum_{j=0}^\infty \EE_{z\sim\cN_d} \left[
         \hat\psi_i \hat\psi_j h_i(\langle w, z\rangle) h_j(\langle w', z\rangle) \cdot \bh_s(z)\right].
\end{align}
Further applying \Cref{prop:2nd-moment-tensor-product} and restricting the summation over $j=0,1,\ldots$ to $j=i+s-2\tau$ for $\tau=0,1,\ldots,s$, we have
\begin{align}
    &\EE_{z\sim\cN_d} \left[\psi(\langle w, z\rangle) \psi(\langle w', z\rangle) \cdot \bh_s(z)\right] \label{eq:psi-2nd-moment-0}\\
    &\qquad = \sum_{i=0}^\infty \sum_{\tau=0}^s \ind(i\ge \tau)\cdot \binom{s}{\tau} \sqrt\frac{1}{s! } \frac{\sqrt{i! (i + s - 2\tau)!}}{(i-\tau)!}\hat\psi_{i} \hat\psi_{i+s - 2\tau} \langle w, w'\rangle^{i-\tau} \cdot \Sym\bigl(w^{\otimes \tau} \otimes {w'}^{\otimes (s-\tau)}\bigr), 
\end{align}
Note that the double sums are interchangeable only if the series converges for each $\tau$, which can be verified using our assumptions:
For the test tensor $T_s = v_1 \otimes v_2 \otimes \ldots \otimes v_s\in(\RR^d)^{\otimes s}$, it holds that $|w^{\otimes\tau}\otimes w'^{\otimes(s-\tau)}[T_s]|=|\prod_{i=1}^\tau \langle w, v_i\rangle \prod_{j=1}^{s-\tau}\langle w', v_{\tau+j}\rangle|\leq (\epsilon\vee\epsilon_0)^s$, and similarly we have $|\Sym(w^{\otimes\tau}\otimes w'^{\otimes(s-\tau)})[T_s]|\leq (\epsilon\vee\epsilon_0)^s$.
Then, since $|\langle w,w'\rangle|\leq \epsilon_1$ where $\epsilon_1\in(0,1/2)$,
\begin{align}
    &\bigg|\sum_{i=\tau}^\infty \frac{\sqrt{i! (i + s - 2\tau)!}}{(i-\tau)!}\hat\psi_{i} \hat\psi_{i+s - 2\tau} \langle w, w'\rangle^{i-\tau} \cdot \Sym\bigl(w^{\otimes \tau} \otimes {w'}^{\otimes (s-\tau)}\bigr)[T_s]\bigg| \\
    &\qquad \le \sum_{i=\tau}^\infty (i+s)^{s/2}\cdot \left|\hat\psi_{i} \hat\psi_{i+s - 2\tau} \right|\cdot \epsilon_1^{i-\tau} \cdot (\epsilon\vee\epsilon_0)^s
     < \infty,
\end{align}
where we note that the exponential decay of $\epsilon_1^{i-\tau}$ will dominate the polynomial growth of $(i+s)^{s/2}$, and that $\hat\psi_i \hat\psi_{i+s-2\tau} $ is uniformly bounded because $\sum_{i=0}^\infty \hat\psi_i^2 < \infty$.
Hence, we can exchange the order of summation in \eqref{eq:psi-2nd-moment-0} to obtain
\begin{align}
    & \EE_{z\sim\cN_d} \left[
        \psi(\langle w, z\rangle) \psi(\langle w', z\rangle) \cdot \bh_s(z)[T_s] 
    \right] \label{eq:psi-2nd-moment-1}\\
    &\quad = \sum_{\tau=0}^s \binom{s}{\tau} \frac{1}{\sqrt{s!}} \cdot \Sym\bigl(w^{\otimes \tau} \otimes {w'}^{\otimes (s-\tau)}\bigr)[T_s] \cdot \langle w,w'\rangle^{i_\tau-\tau} 
    \underbrace{\sum_{i=i_\tau}^\infty \frac{\sqrt{i! (i + s - 2\tau)!}}{(i-\tau)!} \hat\psi_{i} \hat\psi_{i+s - 2\tau} \langle w, w'\rangle^{i-i_\tau}}_{A_\tau}. 
\end{align}
where by the high-pass assumption that $\hat\psi_i=0$ for all $i<s^\star-1$, the summation over index $i$ starts from $i_\tau := \max\{s^\star-1,\ s^\star-1 + 2\tau -s,\ \tau \}$.
It remains to analyze each $A_\tau$ for $\tau=0,1,\ldots,s$.
To this end, isolating the dominant term with $i=i_\tau$ and bounding the remaining terms using $|\hat\psi_i\hat\psi_{i+s-2\tau}|\leq \|\psi\|_2^2=\EE_{x\sim\cN(0,1)}[\psi^2(x)]$, we have
\begin{align}
    A_\tau
    &= \frac{\sqrt{i_\tau! (i_\tau + s - 2\tau)!}}{(i_\tau - \tau)!}\hat\psi_{i_\tau} \hat\psi_{i_\tau + s - 2\tau} 
    \pm \norm{\psi}_2^2 
    \sum_{i=i_\tau+1}^\infty |\langle w,w'\rangle|^{i-i_\tau} \frac{\sqrt{i!(i+s-2\tau)!}}{(i-\tau)!}\label{eq:A_tau_1}.
\end{align}
we can further simplify the second term as follows.
First note that by the definition of $i_\tau$,
\begin{align}
    \max\{i_\tau, i_\tau + s - 2\tau\} &= (s^\star-1) \lor (s^\star-1 + 2\tau -s) \lor \tau \lor (s^\star - 1 + s - 2\tau) \lor (s^\star -1) \lor (s - \tau) \\
    & = (s^\star-1 + 2\tau -s) \lor (s^\star - 1 + s - 2\tau) \lor \tau \lor (s - \tau) \\
    & = (s^\star-1 + |2\tau -s|) \lor (|\tau - s/2| + s/2)\\
    &\le s^\star +s - 1. 
\end{align}
This implies that for any $i>i_\tau$, we have $\max\{i,i+s-2\tau\}=i-i_\tau+\max\{i_\tau,i_\tau+s-2\tau\}\leq s^\star+s-1+i-i_\tau$.
Therefore, for any $i \ge i_\tau + 1$, writing $j=i-i_\tau-1$, we have
\begin{align}
    \frac{\sqrt{i! (i + s - 2\tau)!}}{(i-\tau)!} 
    &= \sqrt{i (i-1) \cdots (i-\tau + 1) } \cdot \sqrt{(i + s - 2\tau) (i + s - 2\tau - 1) \cdots (i-\tau + 1)} \\
    &\le (j + s^\star+s) (j + s^\star + s-1) \cdots (j + s^\star+\ceil{(s+1)/2}). \label{eq:power-sum-bound}
\end{align}
Applying this bound to the second term in \eqref{eq:A_tau_1}, we obtain
\begin{align}
    A_\tau
    &= \frac{\sqrt{i_\tau! (i_\tau + s - 2\tau)!}}{(i_\tau - \tau)!}\hat\psi_{i_\tau} \hat\psi_{i_\tau + s - 2\tau} 
    \pm \norm{\psi}_2^2 \sum_{j=0}^\infty |\langle w,w'\rangle|^{j+1} \prod_{k=\lceil (s+1)/2\rceil}^{s} (j+s^\star+k).\label{eq:A_tau_2}
\end{align}
Moreover, since $|\langle w, w'\rangle|\leq\epsilon_1\in(0,1/2)$ and clearly $s^\star + s \ge 2\floor{(s+1)/2} - 1$, we can apply \Cref{prop:power sum bound} to obtain
\begin{align}
    \sum_{j=0}^\infty |\langle w,w'\rangle|^j \prod_{k=\lceil (s+1)/2\rceil}^{s} (j+s^\star+k)
    &\le \frac{2}{1 - |\langle w, w'\rangle|} \prod_{k=\lceil (s+1)/2\rceil}^{s} (s^\star + k)
    \le 4\prod_{k=\lceil (s+1)/2\rceil}^{s} (s^\star + k).\label{eq:A_tau_3}
\end{align}
Then combining \eqref{eq:A_tau_2} and \eqref{eq:A_tau_3}, we get the following decomposition for $A_\tau$ in \eqref{eq:psi-2nd-moment-1}:
\begin{align}
    A_\tau = \frac{\sqrt{i_\tau! (i_\tau + s - 2\tau)!}}{(i_\tau - \tau)!}\hat\psi_{i_\tau} \hat\psi_{i_\tau + s - 2\tau}
    \pm 4\norm{\psi}_2^2 |\langle w, w'\rangle|
    \prod_{k=\lceil (s+1)/2\rceil}^{s} (s^\star + k).\label{eq:A_tau_4}
\end{align}
Further apply \eqref{eq:power-sum-bound} (which is also true for $i=i_\tau$) to the first term above, and we have
\begin{align}
    |A_\tau| \leq 4\|\psi\|_2^2 \prod_{k=\lceil (s+1)/2\rceil}^{s} (s^\star + k).\label{eq:A_tau_bound}
\end{align}
Using this bound, we can proceed to control $\EE_{z\sim\cN_d}[\psi(\langle w,z\rangle)\psi(\langle w',z\rangle)\cdot \bh_s(z)[T_s]]$ according to \eqref{eq:psi-2nd-moment-1}, and we consider two cases of $s=0$ and $s\geq 1$ separately.
\paragraph{Case $s\ge 1$}
Expanding $\Sym(w^{\otimes \tau}\otimes w'^{\otimes(s-\tau)})[T_s]$ in \eqref{eq:psi-2nd-moment-1} by the definition of $T_s$, we have
\begin{align}
    &\left|\EE_{z\sim\cN_d} \left[
        \psi(\langle w, z\rangle) \psi(\langle w', z\rangle) \cdot \bh_s(z)[T_s] 
    \right]\right| \\
    & \quad \le \sum_{\tau=0}^s \binom{s}{\tau} \frac{1}{\sqrt{s!}} |A_\tau| \cdot |\langle w, w'\rangle|^{i_\tau -\tau}\cdot  \left|\Sym\bigl(w^{\otimes \tau} \otimes {w'}^{\otimes (s-\tau)}\bigr)[T_s]\right| \\
    & \quad \le \sum_{\tau=0}^s \binom{s}{\tau} \frac{1}{\sqrt{s!}} |A_\tau| \cdot |\langle w, w'\rangle|^{i_\tau -\tau} \cdot
    \bigg(\frac{1}{s!} \sum_{\pi\in\varPi_s} \prod_{i=1}^{\tau} |\langle w, v_{\pi(i)}\rangle| \prod_{j=\tau+1}^s |\langle w', v_{\pi(j)}\rangle|\bigg), 
\end{align}
where $\varPi_s$ denotes the set of all permutations of $[s]$.
By the assumption that $\sup_{i > c_0} \{|\langle w, v_i\rangle|\lor |\langle w', v_i\rangle|\} \leq \epsilon$ and the definition $\epsilon_0=\max_{1\leq i\leq c_0}\{|\langle w, v_i\rangle|\lor |\langle w', v_i\rangle|\}$, it follows that
\begin{align}
    \left|\EE_{z\sim\cN_d} \left[
        \psi(\langle w, z\rangle) \psi(\langle w', z\rangle) \cdot \bh_s(z)[T_s] 
    \right]\right| 
    &\le \sum_{\tau=0}^s \binom{s}{\tau} \frac{1}{\sqrt{s!}} \max_{0 \le \tau\le s} |A_\tau| \cdot |\langle w, w'\rangle|^{i_\tau -\tau}\cdot  \epsilon_0^{s \land c_0} \cdot \epsilon^{(s - c_0) \lor 0} \\
    &\le \frac{2^s}{\sqrt{s!}} \max_{0 \le \tau\le s} |A_\tau| \cdot |\langle w, w'\rangle|^{(s^\star - 1 - \floor{s/2}) \lor 0} \cdot  \epsilon_0^{s \land c_0} \cdot \epsilon^{(s - c_0) \lor 0}. \label{eq:psi-2nd-moment-2}
\end{align}
where the last inequality follows from the following fact as $i_\tau=\max\{s^\star-1,\ s^\star-1 + 2\tau -s,\ \tau\}$:
\begin{align}
    i_\tau -\tau &= (s^\star - 1 -\tau) \lor (s^\star-1+\tau - s) \lor 0 = (s^\star - 1 - s/2 + |\tau - s/2|) \lor 0 \\
    &\ge (s^\star - 1 - s/2 + \ind(s~\text{odd})/2) \lor 0 = (s^\star - 1 - \floor{s/2}) \lor 0.
\end{align} 
Next, invoking the bound \eqref{eq:A_tau_bound} for $|A_\tau|$, it follows that
\begin{align}
    &\left|\EE_{z\sim\cN_d} \left[
        \psi(\langle w, z\rangle) \psi(\langle w', z\rangle) \cdot \bh_s(z)[T_s] 
    \right]\right|\\
    &\qquad\le  \frac{2^s}{\sqrt{s!}} 4\norm{\psi}_2^2  \prod_{k=\lceil (s+1)/2\rceil}^s (s^\star+k)\cdot \epsilon_1^{(s^\star - 1 - \floor{s/2}) \lor 0} \cdot  \epsilon_0^{s \land c_0} \cdot \epsilon^{(s - c_0) \lor 0}\\
    &\qquad\le \bigg(\frac{4e (s+s^\star -1)}{s}\bigg)^{s/2} \sqrt{s+s^\star} \cdot 4 \norm{\psi}_2^2 \cdot \epsilon_1^{(s^\star - 1 - \floor{s/2})
     \lor 0} \cdot  \epsilon_0^{s \land c_0} \cdot \epsilon^{(s - c_0) \lor 0} \\
    &\qquad\le \left(4e s^\star\right)^{s/2}\sqrt{s+s^\star} \cdot 4 \norm{\psi}_2^2 \cdot \epsilon_1^{(s^\star - 1 - \floor{s/2})
    \lor 0} \cdot  \epsilon_0^{s \land c_0} \cdot \epsilon^{(s - c_0) \lor 0}.\label{eq:psi-2nd-moment-3}
\end{align}
Here, the second inequality follows from the fact that $s!\geq (s/e)^s$ for all $s\geq 1$, and the last inequality holds because $s \ge 1$.

\paragraph{Case $s = 0$}
For the case $s = 0$, \eqref{eq:psi-2nd-moment-1} simplifies to
\begin{align}
    & \left|\EE_{z\sim\cN_d} \left[
        \psi(\langle w, z\rangle) \psi(\langle w', z\rangle)
    \right]\right| 
    = \bigg|\sum_{i=s^\star-1}^\infty \hat\psi_i^2 \langle w, w'\rangle^{i}\bigg| \le 4 \norm{\psi}_2^2 \epsilon_1^{s^\star - 1}, 
\end{align}
which is also encompassed by the bound in \eqref{eq:psi-2nd-moment-3} with $s=0$.
Hence, the proof is completed.
\end{proof}

\begin{proposition}\label{prop:psi-2nd-moment-series}
Let $\psi:\RR\times\RR\to\RR$ satisfy \ref{asp:quadratic_integrability} and \ref{asp:high_pass} with $s^\star$ being the generative exponent. 
Consider any fixed $w,w'\in\SSS^{d-1}$ and a series of test tensors $\{T_s\}_{s=0}^\infty$ that satisfy the conditions in \Cref{lem:psi-2nd-moment} with constants $c_0,\epsilon,\epsilon_0$ and $\epsilon_1$ specified therein, and further assume that $c_0\in\{0,1,2\}$ and $4e\epsilon^2<1/2$.
Let $\zeta_s:\RR\to\RR$ for $s\geq 0$ be a series of functions such that there exists an absolute constant $c>0$, $\EE_\QQ[\zeta_s(y)^2]\leq c$ for all $s\geq 0$.
Then for any $s_0\geq0$, the following holds:
If $s^\star = 1$, then there exists a constant $C>0$ depending on $s_0,c_0,c,\EE_\QQ[\psi(y,x)^4]$ such that 
\begin{align}
    \sum_{s\ge s_0} (s+2) \cdot \left|\EE_\QQ\left[\zeta_s(y) \psi(y, \langle w, z\rangle)\psi(y, \langle w', z\rangle) \bh_{s}(z)[T_s] \right]\right| \leq C(\epsilon_0^{s_0} + \epsilon_0^{c_0} \epsilon^{(s_0 - c_0)\lor 0}).
\end{align}
In addition, if $s^\star \ge 2$, then it holds that
\begin{align}
    &\frac{1}{C}\sum_{s\ge s_0} (s+2) \cdot \left|\EE_\QQ\left[\zeta_s(y) \psi(y, \langle w, z\rangle)\psi(y, \langle w', z\rangle) \bh_{s}(z)[T_s] \right]\right| \\
    &\qquad \leq \ind(s_0\le c_0) \cdot \left(\epsilon_1^{s^\star -1 - \floor{s_0/2}} \cdot \epsilon_0^{s_0} + \epsilon_1^{s^\star -1 - \floor{c_0/2}} \cdot \epsilon_0^{c_0} \right)  + \epsilon_0^{c_0} \cdot \epsilon^{(2s^\star)\lor s_0 - c_0}\\
    &\qquad\qquad + \ind(s_0 \le 2(s^\star -1)) \cdot 
    \left( \epsilon_1^{s^\star -1 - \floor{(c_0+1)/2}} \cdot \epsilon_0^{c_0} \cdot \epsilon + \epsilon_0^{c_0} \cdot \epsilon^{2(s^\star -1) +1 - c_0} \right).
\end{align}
\end{proposition}
\begin{proof}[Proof of \Cref{prop:psi-2nd-moment-series}]
Let $F$ denote the target quantity. 
Invoking \Cref{lem:psi-2nd-moment} for every $s\geq s_0$,
\begin{align}
    F &\le \sum_{s\ge s_0} (s+2) \cdot \EE_{y\sim \QQ}\left[|\zeta_s(y)|\cdot \left|\EE_{z\sim\cN_d}\left[\psi(y, \langle w, z\rangle)\psi(y, \langle w', z\rangle) \bh_{s}(z)[T_s] \right]\right|\right] \\
    & \le \sum_{s\ge s_0} (s+2) \cdot \EE_{y\sim \QQ}\left[
        4 |\zeta_s(y)| \cdot \EE_{x\sim\cN}[\psi(y, x)^2]\right] \left(4e s^\star\right)^{s/2} \sqrt{s+s^\star} \cdot \epsilon_1^{(s^\star - 1 - \floor{s/2})\lor 0} \cdot \epsilon_0^{s \land c_0} \cdot \epsilon^{(s - c_0) \lor 0}
    \\
    & \le \sum_{s\ge s_0} (s+2) \cdot 4 \sqrt{\EE_{\QQ}[
       \zeta_s(y)^2] \cdot \EE_\QQ[\psi(y, x)^4]} \left(4e s^\star\right)^{s/2} \sqrt{s+s^\star} \cdot \epsilon_1^{(s^\star - 1 - \floor{s/2})
        \lor 0} \cdot \epsilon_0^{s \land c_0} \cdot \epsilon^{(s - c_0) \lor 0}, 
\end{align}
where the last inequality follows from Cauchy-Schwarz inequality.
Note that by our assumptions, $\EE_{\QQ}[\zeta_s(y)^2] \cdot \EE_\QQ[\psi(y, x)^4] \leq c$ uniformly over $s\geq0$ for some absolute constant $c>0$.
Therefore,
\begin{align}
    F \leq C\sum_{s\ge s_0} (s+2)\sqrt{s+s^\star} \left(4e s^\star\right)^{s/2} \epsilon_1^{(s^\star - 1 - \floor{s/2})
        \lor 0} \epsilon_0^{s \land c_0} \epsilon^{(s - c_0) \lor 0},\label{eq:F_bound_0}
\end{align}
for a constant $C>0$. 
Below the meaning of the constant $C$ may change from instance to instance, but it will always depend only on $s_0, c_0, c, \EE_\QQ[\psi(y,x)^4]$.

When $s^\star=1$, \eqref{eq:F_bound_0} simplifies to
\begin{align}
    F
    &\leq \sum_{s\ge s_0} (s+2)\sqrt{s+1} \left(4e\right)^{s/2} \epsilon_0^{s \land c_0} \epsilon^{(s - c_0) \lor 0} \\
    &\leq C \cdot \ind(c_0>s_0) \cdot \sup_{s_0 \le s \le c_0 -1} \epsilon_0^{s} + C \sum_{s\ge c_0 \lor s_0} (s+2)\sqrt{s+1} \left(4e\right)^{s/2} \epsilon_0^{c_0} \epsilon^{s - c_0} \\
    &\leq C\cdot\ind(c_0 > s_0) \cdot \epsilon_0^{s_0} + C \epsilon_0^{c_0} \epsilon^{(s_0 - c_0) \lor 0} \sum_{s\ge c_0 \lor s_0} (s+2)\sqrt{s+1} \left(4e\epsilon^2\right)^{(s - c_0 - (s_0-c_0)\vee 0)/2} \\
    &\leq C(\epsilon_0^{s_0} + \epsilon_0^{c_0} \epsilon^{(s_0 - c_0) \lor 0})
\end{align}
where the last inequality holds by the assumption that $4e\epsilon^2<1/2$.

For the case $s^\star \ge 2$, we note that $c_0\le 2 \le 2(s^\star - 1)$, and thus
\begin{align}
    F 
    &\leq C \sum_{s\ge s_0} (s+2)\sqrt{s+s^\star} \left(4e s^\star\right)^{s/2} \epsilon_1^{(s^\star - 1 - \floor{s/2})
    \lor 0} \epsilon_0^{s \land c_0} \epsilon^{(s - c_0) \lor 0} \\
    &\leq C \cdot \ind(s_0\le c_0) \cdot \max_{s_0\le s\le c_0} \epsilon_1^{s^\star -1 - \floor{s/2}} \cdot \epsilon_0^s
    + C \cdot \ind(s_0 \le 2(s^\star -1)) \cdot \max_{c_0+1\le s\le 2(s^\star -1)+1} \epsilon_1^{s^\star -1 - \floor{s/2}} \cdot \epsilon_0^{c_0} \cdot \epsilon^{s - c_0} \\
    &\qqquad + C \sum_{s\ge (2(s^\star -1) + 2)\lor s_0 } (s+2)\sqrt{s+s^\star} \left(4e s^\star\right)^{s/2} \epsilon_0^{c_0} \epsilon^{s - c_0} \\
    &\leq C \cdot \ind(s_0\le c_0) \cdot \left(\epsilon_1^{s^\star -1 - \floor{s_0/2}} \cdot \epsilon_0^{s_0} + \epsilon_1^{s^\star -1 - \floor{c_0/2}} \cdot \epsilon_0^{c_0} \right)  
    + C \epsilon_0^{c_0} \cdot \epsilon^{(2(s^\star -1) + 2)\lor s_0 - c_0}\\
    &\qqquad + C\cdot\ind(s_0 \le 2(s^\star -1)) \cdot 
    \left( \epsilon_1^{s^\star -1 - \floor{(c_0+1)/2}} \cdot \epsilon_0^{c_0} \cdot \epsilon + \epsilon_0^{c_0} \cdot \epsilon^{2(s^\star -1) +1 - c_0} \right)
\end{align}
where the last inequality follows from the same argument as in the previous case.
This completes the proof.
\end{proof}

\subsection{Technical Results for Uniform Distribution on the Sphere}
\begin{lemma}[Moment of polynomial on a sphere, adapted from \citet{folland2001integrate}]\label{lem:moment_sphere}
Let $\xi=(\xi_1, \xi_2, \ldots, \xi_d)^\top\sim\unif(\SSS^{d-1})$ and $s_1, s_2, \ldots, s_d\in\NN_0$. 
Denote $s = \sum_{i=1}^d s_i$.
Then we have
\begin{align}
    \EE_{\xi}\bigg[
        \prod_{i=1}^d \xi_i^{s_i}
    \bigg] = \begin{cases}
        0 & \text{if some $s_i$ is odd}, \\
        {\ds \frac{\Gamma(d/2)}{\Gamma((s+d)/2)} \prod_{i=1}^d \frac{\Gamma((s_i +1)/2)}{\Gamma(1/2)}} & \text{if all $s_i$ are even},
    \end{cases}
\end{align}
where $\Gamma$ is the gamma function.
\end{lemma}
To evaluate the formula given by the above lemma, we provide the following bound.
\begin{lemma}\label{fact:factorial_bound}
Let $s_1, s_2, \ldots, s_d\in\NN_0$ be even integers and denote $s = \sum_{i=1}^d s_i$. 
Then 
\begin{align}
    \left(\frac{1}{d+s}\right)^{s/2}\le \frac{\Gamma(d/2)}{\Gamma((s+d)/2)} \prod_{i=1}^d \frac{\Gamma((s_i +1)/2)}{\Gamma(1/2)} \le \left(\frac{s}{d}\right)^{s/2}, 
\end{align}
where we adopt the convention $(0)^0 = 1$ for the case of $s=0$. 
\end{lemma}
\begin{proof}[Proof of \Cref{fact:factorial_bound}]
The gamma function satisfies that $\Gamma(a+1)=a\Gamma(a)$ for any $a>0$.
Applying this fact, we have
\begin{align}
    \frac{\Gamma(d/2)}{\Gamma((s+d)/2) }  \prod_{i=1}^d \frac{\Gamma((s_i +1)/2)}{\Gamma(1/2)}
    = \frac{\prod_{i=1}^d (s_i -1)!! }{(s + d -2)(s + d -4) \cdots d}.
\end{align}
The lower bound immediately follows from the fact that $\prod_{i=1}^d(s_i-1)!!\geq 1$ and that all $s/2$ terms in the denominator are at most $s+d$.
For the upper bound, it suffices to lower bound each of the $s/2$ terms in the denominator by $d$ and upper bound each of the $s/2$ terms in the numerator by $s$.
\end{proof}

Note that the second moment $\EE[\xi_i^2]\asymp d^{-1}$, thus \Cref{fact:factorial_bound} can be viewed as some kind of \emph{reverse Holder's inequality} where we use lower moment to control higher moment.
Notably, the reverse inequality gives a \emph{dimension-free} bound for the moments of the polynomial on the sphere.
The following proposition formalizes this intuition.

\begin{proposition}\label{prop:sphere-moment-bound}
Let $\xi = (\xi_1, \xi_2, \ldots, \xi_d)^\top \sim \unif(\SSS^{d-1})$. 
For any $\varepsilon>0$, let $f: \RR^{d} \to \RR$ be a function such that
$ \EE_\xi[(f(\xi) - 1)^2] \le \varepsilon^2$.
Let $s_1,s_2,\ldots,s_d\in\NN_0$ be even integers and denote $s=\sum_{i=1}^d s_i$.
Then it holds that
\begin{align}
    \bigg(1 - \bigg(\frac{2s(d+s)}{d}\bigg)^{s/2}\varepsilon\bigg) \cdot \bigg(\frac{1}{d+s}\bigg)^{s/2} 
    \le \EE_\xi \bigg[
        f(\xi) \prod_{i=1}^d \xi_i^{s_i} 
    \bigg] \le (1 + 2^{s/2} \varepsilon ) \cdot \left(\frac{s}{d}\right)^{s/2}. 
\end{align}
\end{proposition}
\begin{proof}[Proof of \Cref{prop:sphere-moment-bound}]
    By Cauchy-Schwarz inequality, we have
    \begin{align}
        \bigg| \EE\bigg[ (f(\xi) - 1) \prod_{i=1}^d \xi_i^{s_i} \bigg] \bigg| 
        & \le \sqrt{\EE[(f(\xi) - 1)^2]} \cdot \sqrt{\EE\bigg[\prod_{i=1}^d \xi_i^{2s_i}\bigg]} 
        \le \varepsilon \cdot \left(\frac{2 s}{d}\right)^{s/2},
    \end{align}
    where we use \Cref{lem:moment_sphere} and the upper bound in \Cref{fact:factorial_bound} for the second inequality.
    Then the proof is completed by incorporating the lower and upper bounds for $\EE[\prod_{i=1}^d \xi_i^{s_i}]$ from \Cref{fact:factorial_bound}.
\end{proof}

\begin{proposition}
    \label{prop:G(s)}
Let $\xi=(\xi_1,\ldots,\xi_d)^\top\sim\unif(\SSS^{d-1})$.
Let $\bs=(s_1, s_2, \ldots, s_d)^\top$ where $s_1,\ldots,s_d$ are non-negative integers, and denote $s=\sum_{i=1}^d s_i$.
Then for any fixed $\gamma>0$, it holds that
\begin{align}
    \EE \bigg[ \left(\xi_1 + \gamma\right)^{s_1} \prod_{i=2}^d \xi_i^{s_i} \bigg] 
    & = \begin{cases}
        0, & \text{if $s_i$ is odd for some $2\leq i\leq d$}, \\
        C(\bm{s}, \gamma) \cdot \gamma^{\ind(s_1~\text{odd})}
        \big(\gamma + \frac{1}{\sqrt d}\big)^{2\floor{s_1/2}} 
        \big(\frac{1}{\sqrt d}\big)^{s-s_1}, & \text{otherwise},
    \end{cases}
\end{align}
where $C(\bs,\gamma)$ satisfies that 
\begin{align}
    \frac{1}{5}\bigg(\frac{\sqrt{d}}{\sqrt{d+s}}\bigg)^{s-\ind(s_1\text{ odd})} \le C(\bm{s}, \gamma) \le s^{s/2}.
\end{align}
\end{proposition}
\begin{proof}[Proof of \Cref{prop:G(s)}]
For convenience, we denote the target quantity by $G(\bs)$.
Expanding $(\xi_1+\gamma)^{s_1}$, we have
\begin{align}
    G(\bm{s}) := \EE \bigg[ \left(\xi_1 + \gamma\right)^{s_1} \prod_{i=2}^d \xi_i^{s_i} \bigg]
    = \EE \bigg[
        \sum_{j=0}^{s_1} \binom{s_1}{j} \xi_1^j \gamma^{s_1 - j} \prod_{i=2}^d \xi_i^{s_i}
    \bigg]. 
\end{align}
Note that if there exists any odd $s_i$ for $i\ge 2$, then $G(\bm{s})= 0$ due to \Cref{lem:moment_sphere}.
It remains to consider the case where all $s_i$ for $i\geq 2$ are even, and by the same symmetry only the terms with even powers of $\xi_1$ contribute.
Applying \Cref{lem:moment_sphere}, we get
\begin{align}
    G(\bm{s}) &= \sum_{j=0}^{\floor{s_1/2}} \binom{s_1}{2j} \gamma^{s_1 - 2j} \cdot \EE\bigg[\xi_1^{2j} \prod_{i=2}^d \xi_i^{s_i} \bigg] \\
    &= \sum_{j=0}^{\floor{s_1/2}} \binom{s_1}{2j} \gamma^{s_1 - 2j} \cdot \frac{ \Gamma(d/2)}{\Gamma((s - s_1 + 2j +d)/2)} \cdot \frac{\Gamma(j + 1/2)}{\Gamma(1/2)} \cdot \prod_{i=2}^d \frac{\Gamma((s_i +1)/2)}{\Gamma(1/2)} \\
    &= \sum_{j=0}^{\floor{s_1/2}} \binom{s_1}{2j} \gamma^{s_1 - 2j} \cdot \underbrace{\frac{(2j -1)!! \cdot \prod_{i=2}^d (s_i -1)!! }{(s-s_1 + 2j + d -2) \cdot (s-s_1 + 2j + d -4) \cdots d}}_{\dr (\RNum{1})}, \label{eq:Gs_0} 
\end{align}
where the third equality follows from the property of the gamma function that $\Gamma(a+1)=a\Gamma(a)$ for any $a>0$.
Below we derive the lower and upper bounds of $G(\bm{s})$ separately.
\paragraph{Lower Bound}
For the lower bound of $G(\bs)$, we lower bound the numerator in (\RNum{1}) by 1 and upper bound each of the $(s-s_1+2j)/2$ terms in the denominator by $d+s$.
This yields
\begin{align}
    G(\bs) &\geq \sum_{j=0}^{\floor{s_1/2}} \binom{s_1}{2j} \gamma^{s_1 - 2j} \cdot \bigg(\frac{1}{\sqrt{d+s}}\bigg)^{2j + s-s_1}
    \geq \bigg(\frac{\sqrt{d}}{\sqrt{d+s}}\bigg)^{s-\ind(s_1\text{ odd})} \sum_{j=0}^{\floor{s_1/2}} \binom{s_1}{2j} \gamma^{s_1 - 2j} \cdot \bigg(\frac{1}{\sqrt{d}}\bigg)^{2j + s-s_1}\\
    &\geq \bigg(\frac{\sqrt{d}}{\sqrt{d+s}}\bigg)^{s-\ind(s_1\text{ odd})} \sum_{j=0}^{\floor{s_1/2}} \binom{2\floor{s_1/2}}{2j} \gamma^{2\floor{s_1/2} - 2j} \bigg(\frac{1}{\sqrt{d}}\bigg)^{2j} \cdot \gamma^{\ind(s_1~\text{odd})} \cdot \bigg(\frac{1}{\sqrt{d}}\bigg)^{s-s_1}\\
    &= \bigg(\frac{\sqrt{d}}{\sqrt{d+s}}\bigg)^{s-\ind(s_1\text{ odd})} \sum_{j=0,\ j\text{ even}}^{2\floor{s_1/2}} \binom{2\floor{s_1/2}}{j} \gamma^{2\floor{s_1/2} - j} \bigg(\frac{1}{\sqrt{d}}\bigg)^{j} \cdot \gamma^{\ind(s_1~\text{odd})} \cdot \bigg(\frac{1}{\sqrt{d}}\bigg)^{s-s_1}\label{eq:Gs_lb_0}
\end{align}
where the third equality is true because $s_1\geq 2\floor{s_1/2}$.
Next we further relate the sum over even $j$ to the sum over all $j=0,1,\ldots,2\floor{s_1/2}$.
To this end, we will apply \Cref{lem:binomial_ratio} by defining 
\begin{align}
    A_j= \binom{2\floor{s_1/2}}{j} \cdot \gamma^{2 \floor{s_1/2} - j}  \left(\frac{1}{\sqrt d} \right)^{j}.
\end{align}
Then invoking \Cref{lem:binomial_ratio} with $a=\gamma$ and $b=1/\sqrt{d}$, we get 
\begin{align}
    A_j \le 2(A_{j-1} + A_{j+1}), \quad j = 1, 3, \ldots, 2\floor{s_1/2} - 1.
\end{align}
Therefore, the summation of $A_j$ over all odd $j=1,3,\ldots,2\floor{s_1/2}-1$ can be upper bounded by 4 times the summation of $A_j$ over all even $j=0,2,\ldots,2\floor{s_1/2}$.
Applying this to \eqref{eq:Gs_lb_0} yields
\begin{align}
    G(\bm{s}) 
    &\ge \frac{1}{5} \bigg(\frac{\sqrt{d}}{\sqrt{d+s}}\bigg)^{s-\ind(s_1\text{ odd})} \sum_{j=0}^{2\floor{s_1/2}} \binom{2\floor{s_1/2}}{j} \gamma^{2 \floor{s_1/2} - j} \left(\frac{1}{\sqrt d} \right)^{j} \cdot \gamma^{\ind(s_1~\text{odd})} \cdot \bigg(\frac{1}{\sqrt{d}}\bigg)^{s-s_1}\\
    &= \frac{1}{5} \bigg(\frac{\sqrt{d}}{\sqrt{d+s}}\bigg)^{s-\ind(s_1\text{ odd})} \cdot \left( \gamma + \frac{1}{\sqrt d}\right)^{2\floor{s_1/2}} \cdot \gamma^{\ind(s_1~\text{odd})} \cdot \left(\frac{1}{\sqrt d}\right)^{s-s_1}.\label{eq:Gs_lb}
\end{align}
This gives the lower bound of $G(\bm{s})$.
\paragraph{Upper Bound}
For the upper bound of $G(\bs)$, we upper bound the numerator in (\RNum{1}) by $s^{s/2}$ and lower bound each of the $s/2$ terms in the denominator by $d$.
This yields
\begin{align}
    G(\bm{s}) 
    &\le \sum_{j=0}^{\floor{s_1/2}} \binom{s_1}{2j} \gamma^{s_1 - 2j} \cdot \bigg(\frac{1}{\sqrt{d}}\bigg)^{2j+s-s_1} \cdot s^{s/2} \\
    &\le \sum_{j=0}^{2\floor{s_1/2}} \binom{2\floor{s_1/2}}{j} \gamma^{2 \floor{s_1/2} - j} \left(\frac{1}{\sqrt d} \right)^{j} \cdot \bigg(\frac{1}{\sqrt{d}}\bigg)^{s-s_1} \cdot \gamma^{\ind(s_1~\text{odd})} \cdot s^{s/2} \\
    &=  s^{s/2}  \cdot \left( \gamma + \frac{1}{\sqrt d}\right)^{2\floor{s_1/2}} \cdot \gamma^{\ind(s_1~\text{odd})} \cdot \left(\frac{1}{\sqrt d}\right)^{s-s_1}, \label{eq:Gs_ub}
\end{align}
where the second inequality holds since $\gamma>0$.
Finally, combining the lower bound in \eqref{eq:Gs_lb} and the upper bound in \eqref{eq:Gs_ub}, we conclude that 
\begin{align}
    G(\bm{s}) = C(\bm{s}, \gamma) \cdot \left( \gamma + \frac{1}{\sqrt d}\right)^{2\floor{s_1/2}} \cdot \gamma^{\ind(s_1~\text{odd})} \cdot \left(\frac{1}{\sqrt d}\right)^{s-s_1}
\end{align}
for $(\sqrt{d}/\sqrt{d+s})^{s-\ind(s_1\text{ odd})}/5 \le C(\bm{s}, \gamma) \le s^{s/2}$, when $s_2,\ldots,s_d$ are even, completing the proof.
\end{proof}

\subsection{Technical Results on Polarized Random Vectors}
\begin{lemma}[Moments of weakly polarized random vector]
    \label{lem:polar-moment}
Let $\xi = (\xi_1, \xi_2, \ldots, \xi_d)^\top \sim \unif(\SSS^{d-1})$ and $e_1=(1,0,\ldots,0)$ be the first canonical basis vector in $\RR^d$.
For a fixed integer $s\geq 1$, consider $\bs=(s_1,s_2,\ldots,s_d)^\top$ where $s_1,s_2,\ldots,s_d$ are non-negative integers such that $s=\sum_{i=1}^d s_i$.
For any fixed polarization level $\gamma\in(0,1)$, define the polarized vector $w$ as
\begin{align}
    w = \frac{\xi + \gamma e_1}{\norm{\xi + \gamma e_1}_2}, 
\end{align}
where $e_1=(1,0,\ldots,0)^\top$ is the first canonical basis vector in $\RR^d$.
Then it holds that
\begin{align}
    \EE\bigg[
        \prod_{i=1}^d w_i^{s_i}
    \bigg] = 
    \begin{cases}
        0, & \text{if some $s_i$ is odd for $i=2, 3, \ldots, d$}, \\
        C(\bm{s}, \gamma) \cdot \left( \gamma + \frac{1}{\sqrt d}\right)^{s_1} \cdot \left(\frac{1}{\sqrt d}\right)^{s-s_1}, & \text{if $s_1$ is even and $s_2, \ldots, s_d$ are even}, \\
        C(\bm{s}, \gamma) \cdot \gamma^{\ind(s_1\text{ odd})} 
        \big( \gamma + \frac{1}{\sqrt d}\big)^{s_1-1} \big(\frac{1}{\sqrt d}\big)^{s-s_1}, & \text{otherwise},
    \end{cases}
\end{align}
where $C(\bm{s}, \gamma)$ satisfies that 
\begin{align}
    \frac{s^{(s-1)/2}}{5(1+\gamma)^s} \lesssim C(\bm{s},\gamma) \leq \frac{s^{s/2}}{(1-\gamma)^s}.
\end{align}
\end{lemma}
\begin{proof}[Proof of \Cref{lem:polar-moment}]
Let $r := \norm{\xi + \gamma e_1}_2$, and denote the target quantity by $G(w,\bs)$ for convenience.
Then we can rewrite $G(w,\bs)$ as
\begin{align}
    G(w,\bs) = \EE \bigg[
        \prod_{i=1}^d w_i^{s_i}
    \bigg] 
    & = \EE \bigg[
        \frac{\left(\xi_1 + \gamma\right)^{s_1}}{r^s} \prod_{i=2}^d \xi_i^{s_i}
    \bigg].\label{eq:def_Gws}
\end{align}
By symmetry, we have $G(w,\bs)=0$ if some $s_i$ is odd for $i=2, 3, \ldots, d$.
It remains to consider the case where $s_2, s_3, \ldots, s_d$ are all even, and we still study the case for $s_1$ being even and odd separately.

\paragraph{Case 1: $s_1$ is even}
We first consider the simpler case where $s_1$ is even.
In this case, $(\xi_1+\gamma)^{s_1},\xi_2^{s_2},\ldots,\xi_d^{s_d}$ are all non-negative.
Also, by triangle inequality we have $r\geq 1-\gamma$.
Therefore,
\begin{align}
    G(w,\bs) &\leq \frac{1}{(1-\gamma)^s} \cdot \EE\bigg[(\xi_1+\gamma)^{s_1} \prod_{i=2}^d \xi_i^{s_i}\bigg] = \frac{C(\bs,\gamma)}{(1-\gamma)^s} \bigg(\gamma + \frac{1}{\sqrt{d}}\bigg)^{s_1} \bigg(\frac{1}{\sqrt{d}}\bigg)^{s-s_1}
\end{align}
where $C(\bs,\gamma)$ is given by \Cref{prop:G(s)}.
Similarly, since $r\leq 1+\gamma$, we also have 
\begin{align}
    G(w,\bs) &\geq \frac{C(\bs,\gamma)}{(1+\gamma)^s}\bigg(\gamma+\frac{1}{\sqrt{d}}\bigg)^{s_1} \bigg(\frac{1}{\sqrt{d}}\bigg)^{s-s_1}.
\end{align}
Combining the above two inequalities gives the desired result for the case where $s_1$ is even.

\paragraph{Case 2: $s_1$ is odd}
We now consider the more complicated case where $s_1$ is odd.
We first rewrite $G(w,\bs)$ as follows
\begin{align}
    G(w,\bs)
    & = \EE \bigg[
        \frac{\left(\xi_1 + \gamma\right)^{s_1}}{r^s} \prod_{i=2}^d \xi_i^{s_i}
    \bigg] \label{eq:polar-moment-1}\\
    & = \underbrace{\sum_{j=0}^{(s_1 -1)/2} \binom{s_1}{2j} \gamma^{s_1 - 2j} \cdot \EE \bigg[ \frac{\xi_1^{2j}}{r^s} \prod_{i=2}^d \xi_i^{s_i}\bigg]}_{\dr (\RNum{1})~even~terms} + 
    \underbrace{\sum_{j=0}^{(s_1 -1)/2} \binom{s_1}{2j+1} \gamma^{s_1 - 2j - 1} \cdot \EE \left[ \frac{\xi_1^{2j+1}}{r^s} \prod_{i=2}^d \xi_i^{s_i} \right]}_{\dr (\RNum{2})~odd~terms}. 
\end{align}
Let us first look at the odd terms of \eqref{eq:polar-moment-1}. 
Denote $r_+ := \norm{\xi + \gamma e_1}_2$ and $r_- := \norm{-\xi + \gamma e_1}_2$.
By symmetry, each of the odd terms in \eqref{eq:polar-moment-1} can be written as
\begin{align}
    \EE \bigg[ \frac{\xi_1^{2j+1}}{r^s} \prod_{i=2}^d \xi_i^{s_i} \bigg] 
    & = \frac 1 2 \cdot \EE\bigg[
        \left(\frac{1}{r_+^s} - \frac{1}{r_-^s}\right)\xi_1^{2j+1}\prod_{i=2}^d \xi_i^{s_i}
    \bigg] 
    = \frac 1 2 \cdot \EE\bigg[
        \frac{(r_-^2 - r_+^2) \sum_{l=0}^{s-1} r_-^l r_+^{s-1-l}}{r_+^s \cdot r_-^s \cdot (r_+ + r_-)} \cdot \xi_1^{2j+1} \prod_{i=2}^d \xi_i^{s_i}
    \bigg] \\
    & = - \gamma \cdot \EE\bigg[
        \frac{2 \sum_{l=0}^{s-1} r_-^l r_+^{s-1-l}}{r_+^s \cdot r_-^s \cdot (r_+ + r_-)} \cdot \xi_1^{2j+2} \prod_{i=2}^d \xi_i^{s_i}
    \bigg]\label{eq:polar-moment-odd-1}
\end{align}
where the last equality holds by noting that $r_-^2 - r_+^2 = -4\gamma \xi_1$.
Since both $r_+,r_-\geq 1-\gamma$, we have
\begin{align}
    \sup_{\xi\in\SSS^{d-1}}\frac{\frac{1}{s}\sum_{l=0}^{s-1}2r_-^l r_+^{s-1-l}}{r_+^s r_-^s(r_+ + r_-)}
    \leq \sup_{\xi\in\SSS^{d-1}} \max_{l=0,1,\ldots,s-1} \frac{1}{r_+^{l+1} r_-^{s-l} (r_+ + r_-)/2}
    \leq \frac{1}{(1-\gamma)^{s+2}}.
\end{align}
Applying the above bound to \eqref{eq:polar-moment-odd-1}, and noting that every quantity inside the expectation on the right-hand side is non-negative, we obtain that for every $j=0,1,\ldots,(s_1-1)/2$,
\begin{align}
    -\EE \bigg[ \frac{\xi_1^{2j+1}}{r^s} \prod_{i=2}^d \xi_i^{s_i} \bigg]
    &\leq \frac{s\gamma}{(1-\gamma)^{s+2}} \cdot \EE\bigg[\xi_1^{2j+2} \prod_{i=2}^d \xi_i^{s_i}\bigg]\\
    &\leq \frac{s\gamma}{(1-\gamma)^{s+2}} \cdot \bigg(\frac{s-s_1+2j+2}{d}\bigg)^{(s-s_1+2j+2)/2}\\
    &\leq \frac{s\gamma}{(1-\gamma)^{s+2}} \cdot \bigg(\frac{s+1}{d}\bigg)^{(s-s_1+2j+2)/2}
\end{align}
Collecting the above upper bound for all the odd terms in \eqref{eq:polar-moment-1}, we have
\begin{align}
    - \text{(\RNum{2})} 
    &\le \frac{s\gamma}{(1-\gamma)^{s+2}} \sum_{j=0}^{(s_1 -1)/2} \binom{s_1}{2j+1} \cdot \gamma^{s_1 - 2j-1} \cdot \left(\frac{s+1}{d}\right)^{(s-s_1 + 2j +2)/2} \\
    &\leq \frac{s_1s\gamma}{(1-\gamma)^{s+2}} \left(\frac{s+1}{d}\right)^{(s-s_1 +2)/2} \sum_{j=0}^{(s_1-1)/2} \binom{s_1-1}{2j} \cdot \gamma^{s_1-1 -2j} \cdot \bigg(\frac{s+1}{d}\bigg)^{j}\\
    &\leq \frac{s_1s\gamma}{(1-\gamma)^{s+2}} \left(\frac{s+1}{d}\right)^{(s-s_1 +2)/2} \sum_{j=0}^{s_1-1} \binom{s_1-1}{j} \cdot \gamma^{s_1-1-j} \cdot \bigg(\frac{s+1}{d}\bigg)^{j/2}
\end{align}
where the second inequality holds because $\binom{s_1}{2j+1}\leq \binom{s_1-1}{2j} \cdot s_1$ for all $j=0,1,\ldots,(s_1-1)/2$.
Thus, we get the following lower bound for $\text{(\RNum{2})}$:
\begin{align}
    \text{(\RNum{2})} \geq - \frac{s_1s\gamma}{(1-\gamma)^{s+2}} 
    \bigg(\gamma + \frac{\sqrt{s+1}}{\sqrt{d}}\bigg)^{s_1-1} 
    \bigg(\frac{\sqrt{s+1}}{\sqrt{d}}\bigg)^{s-s_1+2}. \label{eq:odd_terms_lower_bound}
\end{align}

Next, we study the even terms (\RNum{1}) in \eqref{eq:polar-moment-1}.
Similar to the analysis for the case of even $s_1$, applying \Cref{fact:factorial_bound} yields the following lower and upper bounds for every $j=0,1,\ldots,(s_1-1)/2$:
\begin{align}
    \EE\bigg[\frac{\xi_1^{2j}}{r^s}\prod_{i=2}^d\xi_i^{s_i}\bigg] 
    &\leq \frac{1}{(1-\gamma)^s} \cdot \EE\bigg[\xi_1^{2j}\prod_{i=2}^d \xi_i^{s_i}\bigg]
    \leq \frac{1}{(1-\gamma)^s}\bigg(\frac{s-1}{d}\bigg)^{(s-s_1+2j)/2},\\
    \EE\bigg[\frac{\xi_1^{2j}}{r^s}\prod_{i=2}^d\xi_i^{s_i}\bigg]
    &\geq \frac{1}{(1+\gamma)^s} \cdot \EE\bigg[\xi_1^{2j}\prod_{i=2}^d\xi_i^{s_i}\bigg]
    \geq \frac{1}{(1+\gamma)^s} \bigg(\frac{1}{d+s-1}\bigg)^{(s-s_1+2j)/2}
\end{align}
Collecting the upper bound for all the even terms in \eqref{eq:polar-moment-1}, we have
\begin{align}
    \text{(\RNum{1})} 
    &\le \frac{1}{(1-\gamma)^s} \sum_{j=0}^{(s_1-1)/2} \binom{s_1-1}{2j} \cdot \gamma^{s_1 - 2j} \cdot \left(\frac{s-1}{d}\right)^{(s-s_1 + 2j)/2} \\
    &\leq \frac{s_1\gamma}{(1-\gamma)^s} \left(\frac{s-1}{d}\right)^{(s-s_1)/2} 
    \sum_{j=0}^{s_1-1} \binom{s_1-1}{j} \cdot \gamma^{s_1-1 -j} \cdot \left(\frac{s-1}{d}\right)^{j/2}\\
    & = \frac{s_1\gamma}{(1-\gamma)^s} \bigg( \gamma + \frac{\sqrt{s-1}}{\sqrt{d}}\bigg)^{s_1-1} \left(\frac{\sqrt{s-1}}{\sqrt{d}}\right)^{s-s_1}, \label{eq:even_terms_upper_bound}
\end{align}
where the second inequality is because $\binom{s_1}{2j}\leq \binom{s_1-1}{j}\cdot s_1$ for all $j=0,1,\ldots,(s_1-1)/2$.
Similarly, collecting the lower bound for all $j=0,1,\ldots,(s_1-1)/2$, we have
\begin{align}
    \text{(\RNum{1})} &\geq \frac{1}{(1+\gamma)^s} \sum_{j=0}^{(s_1-1)/2} \binom{s_1-1}{2j} \cdot \gamma^{s_1 - 2j} \cdot \left(\frac{1}{d+s-1}\right)^{(s-s_1 + 2j)/2} \\
    &\geq \frac{\gamma}{5(1+\gamma)^s} \left(\frac{1}{d+s-1}\right)^{(s-s_1)/2} \sum_{j=0}^{s_1-1} \binom{s_1-1}{j} \cdot \gamma^{s_1-1-j} \cdot \left(\frac{1}{d+s-1}\right)^{j/2}\\
    &= \frac{\gamma}{5(1+\gamma)^s} \bigg( \gamma + \frac{1}{\sqrt{d+s-1}}\bigg)^{s_1-1} \left(\frac{1}{\sqrt{d+s-1}}\right)^{s-s_1}, \label{eq:even_terms_lower_bound}
\end{align}
where the second inequality follows from the same argument as the analysis for \eqref{eq:Gs_lb_0}.

Since $\text{\RNum{2}}\leq 0$ by \eqref{eq:polar-moment-odd-1}, it follows from the upper bound for (\RNum{1}) in \eqref{eq:even_terms_upper_bound} that
\begin{align}
    G(w,\bs) &\leq \frac{s_1\gamma}{(1-\gamma)^s} \bigg( \gamma + \frac{\sqrt{s-1}}{\sqrt{d}}\bigg)^{s_1-1} \left(\frac{\sqrt{s-1}}{\sqrt{d}}\right)^{s-s_1}\\
    &= \frac{s_1}{(1-\gamma)^s} \bigg(\frac{\gamma\sqrt{d}+\sqrt{s-1}}{\gamma\sqrt{d}+1}\bigg)^{s_1-1} (s-1)^{(s-s_1)/2} \cdot \gamma\bigg(\gamma + \frac{1}{\sqrt{d}}\bigg)^{s_1-1} \bigg(\frac{1}{\sqrt{d}}\bigg)^{s-s_1}\\
    &\leq \frac{s^{s/2}}{(1-\gamma)^s}\cdot \gamma \bigg(\gamma+\frac{1}{\sqrt{d}}\bigg)^{s_1-1} \bigg(\frac{1}{\sqrt{d}}\bigg)^{s-s_1}.\label{eq:Gws_upper_bound}
\end{align}
Combining the lower bound for (\RNum{1}) in \eqref{eq:even_terms_lower_bound} and the lower bound for (\RNum{2}) in \eqref{eq:odd_terms_lower_bound}, we have
\begin{align}
    G(w,\bs) &\geq \frac{\gamma}{5(1+\gamma)^s} \bigg( \gamma + \frac{1}{\sqrt{d+s-1}}\bigg)^{s_1-1} \left(\frac{1}{\sqrt{d+s-1}}\right)^{s-s_1}\\
    &\qquad - \frac{s_1s\gamma}{(1-\gamma)^{s+2}} 
    \cdot \bigg(\gamma + \frac{\sqrt{s+1}}{\sqrt{d}}\bigg)^{s_1-1} 
    \cdot \bigg(\frac{\sqrt{s+1}}{\sqrt{d}}\bigg)^{s-s_1+2}\\
    &= \frac{1}{5(1+\gamma)^s} \bigg(\frac{(\gamma\sqrt{d+s-1}+1)\sqrt{d}}{\sqrt{d+s-1}(\gamma\sqrt{d}+1)}\bigg)^{s_1-1} \bigg(\frac{\sqrt{d}}{\sqrt{d+s-1}}\bigg)^{s-s_1} \cdot \gamma\bigg(\gamma + \frac{1}{\sqrt{d}}\bigg)^{s_1-1} \bigg(\frac{1}{\sqrt{d}}\bigg)^{s-s_1}\\
    &\qquad - \frac{s_1s}{(1-\gamma)^{s+2}}\bigg(\frac{\gamma\sqrt{d}+\sqrt{s-1}}{\gamma\sqrt{d}+1}\bigg)^{s_1-1} (s+1)^{(s-s_1)/2}\frac{s+1}{d} \cdot \gamma\bigg(\gamma + \frac{1}{\sqrt{d}}\bigg)^{s_1-1} \bigg(\frac{1}{\sqrt{d}}\bigg)^{s-s_1}\\
    &\geq \bigg[\frac{s^{(s-1)/2}}{5(1+\gamma)^s} - \frac{s^2(s+1)^{(s+1)/2}}{d(1-\gamma)^{s+2}}\bigg]\cdot \gamma\bigg(\gamma + \frac{1}{\sqrt{d}}\bigg)^{s_1-1} \bigg(\frac{1}{\sqrt{d}}\bigg)^{s-s_1}.\label{eq:Gws_lower_bound}
\end{align}
Combining \eqref{eq:Gws_upper_bound} and \eqref{eq:Gws_lower_bound} gives the desired result.
\end{proof}

\subsection{Auxiliary Lemmas}\label{app:auxiliary lemma}

\begin{lemma}[Bernstein's inequality]
    \label{lem:bernstein}
    Let $X_1, \ldots, X_n$ be independent random variables with $|X_i - \EE[X_i]|\leq C$ for all $i\in[n]$.
    Then for any $t>0$, it holds that
    \begin{align}
        \PP\bigg(\bigg|\frac{1}{n}\sum_{i=1}^n X_i - \frac{1}{n}\sum_{i=1}^n \EE X_i\bigg| \geq t\bigg) \leq 2\exp\bigg(-\frac{n t^2/2}{n^{-1} \cdot \sum_{i=1}^n \Var[X_i] + Ct/3}\bigg),
    \end{align}
    or equivalently, for any $\delta\in(0,1)$,
    \begin{align}
        \PP\bigg(
            \bigg|\frac{1}{n}\sum_{i=1}^n X_i - \EE X_i\bigg| \leq  \sqrt{\frac{2  \cdot n^{-1} \sum_{i=1}^n \Var[X_i] \cdot \log \delta^{-1}}{n}} + \frac{C \log \delta^{-1}}{3n}
        \bigg) \ge 1- \delta.
    \end{align}
\end{lemma}
For the vector case, by a union bound over all the coordinates, we have the following corollary.
\begin{corollary}[Vector version of Bernstein's inequality]
    \label{cor:vector_bernstein}
    Let $X_1, \ldots, X_n$ be independent random vectors in $\RR^d$ with $\|X_i - \EE[X_i]\|_\infty \leq C$ for all $i\in[n]$.
    Then for any $\delta\in(0,1)$, it holds with probability at least $1-\delta$ that
    \begin{align}
        \bigg\|\frac{1}{n}\sum_{i=1}^n X_i \bigg\|_2 \lesssim  \bigg\|\EE\bigg[\frac{1}{n}\sum_{i=1}^n X_i\bigg]\bigg\|_2 + \sqrt{\frac{n^{-1} \sum_{i=1}^n \trace(\Cov[X_i]) \cdot \log (d\delta^{-1})}{n}} + \frac{\sqrt d C \log (d\delta^{-1})}{n}.
    \end{align}
\end{corollary}

\begin{lemma}[Lemma I.3. in \citet{damian2024computational}]\label{lem:poly_tail_bound}
    Let $X_1,\ldots,X_n \in \RR^d$ be independent mean-zero random vectors such that for all $p \ge 2$, $\EE[\norm{X_i}^p]^{1/p} \le C p^{k/2}$ for some constants $k,C\geq 0$ and vector norm $\norm{\cdot}$.
    Define $\sigma^2 := n^{-1} \sum_{i=1}^n \EE[\norm{X_i}^2]$ and $Y := n^{-1} \cdot \sum_{i=1}^n X_i$.
    Then with probability at least $1-2\delta$,
    \begin{align}
        \norm{Y} \lesssim \sigma \cdot \sqrt{\frac{\log(1/\delta)}{n}} + \frac{C \log(1/\delta) \log(n/\delta)^{k/2}}{n}.
    \end{align}
    where $\lesssim$ only hides constant that depends on $k$.
\end{lemma}

\begin{lemma}[Ratio bound of binomial expansion]
    \label{lem:binomial_ratio}
    For real numbers $a,b> 0$ and integer $s \ge 2$, define
    \begin{align}
        A_j := \binom{s}{j} \cdot a^{s-j} \cdot b^j, \quad \text{for }j= 0, 1, \ldots, s.
    \end{align}
    Then for all $j=1,2,\ldots,s-1$, it holds that $A_j \le 2(A_{j-1} + A_{j+1})$.
\end{lemma}
\begin{proof}[Proof of \Cref{lem:binomial_ratio}]
By the definition of $A_j$,
    \begin{align}
        \min\left\{ \frac{A_j}{A_{j-1}}, \frac{A_j}{A_{j+1}} \right\}
        & = \min\left\{ \frac{b}{a} \cdot \frac{s-j+1}{j}, \frac{a}{b} \cdot \frac{j+1}{s-j} \right\}
        \le \sqrt{\frac{b}{a} \cdot \frac{s-j+1}{j} \cdot \frac{a}{b} \cdot \frac{j+1}{s-j}} \\
        &\le \sqrt{\left(1 + \frac{1}{s-j}\right) \left(1 + \frac{1}{j}\right)}
        \le 2.
    \end{align}
    Hence, the proof is complete by the nonnegativity of $A_j$.
\end{proof}

\begin{proposition}\label{prop:power sum bound}
    For $\epsilon \in [0, 1/2)$ and $s, r\in\NN_0$ with $s\ge 2r - 1$, it holds that
    \begin{align}
        \sum_{j=0}^\infty (j + s) (j + s -1 ) \cdots (j + s - r + 1) \cdot  \epsilon^j \le 2 s\cdot (s - 1) \cdots (s - r + 1) \cdot \frac{1}{1-\epsilon}.
    \end{align}
\end{proposition}
\begin{proof}[Proof of \Cref{prop:power sum bound}]
Denote $F(x)=\sum_{j=0}^\infty (j+s)(j+s-1)\cdots(j+s-r+1) x^j$ for $x\in(0,1)$.
The desired quantity on the left-hand side of the inequality is simply $F(\epsilon)$.
It can be verified using the expansion of $1/(1-x) = \sum_{j=0}^\infty x^j$ for $x\in(0,1)$ that
\begin{align}
    F(x) = \frac{\rd^r}{\rd x^r} \Big(\frac{x^{s}}{1 - x} \Big) \cdot x^{-(s-r)}.
\end{align}
Expanding this expression, we have
\begin{align}
    F(x) &= \sum_{\tau = 0}^r \binom{r}{\tau} \frac{\rd^\tau}{\rd x^\tau} \left( x^{s} \right) \cdot \frac{\rd^{r-\tau}}{\rd x^{r-\tau}} \left( \frac{1}{1-x} \right) \cdot x^{-(s-r)} \\
    &= \sum_{\tau = 0}^r \binom{r}{\tau} s (s-1)\cdots (s-\tau + 1) \cdot x^{s-\tau} \cdot \frac{(-1)^{r-\tau}(r-\tau)!}{(1-x)^{r-\tau + 1}} \cdot x^{-(s-r)}\\
    &= \sum_{\tau=0}^r \left(r \cdot (r - 1) \cdots (\tau + 1) \right) \cdot \left( s \cdot (s - 1) \cdots (s - \tau + 1) \right) \cdot (-1)^{r-\tau} \cdot \frac{x^{r-\tau}}{(1-x)^{r-\tau + 1}}.
\end{align}
Write the above summation as $F(x)=\sum_{\tau=0}^r F_\tau(x)$, where $F_\tau(x)$ is the $\tau$-th term in the summation.
Then for each $\tau=0,\ldots,r-1$, when $x\in(0,1)$, we have
\begin{align}
    \frac{F_{\tau}(x)}{F_{\tau+1}(x)} = - \frac{s-\tau}{\tau + 1} \cdot \frac{1 - x}{x} < -1.
\end{align}
Note that $F_r(x)$ is positive, and for each $k=1,\ldots,\lfloor r/2\rfloor$, $F_{r-2k+1}(x) < - F_{r-2k}(x) < 0$.
Since $F(x)$ is positive, it is then upper bounded by $F_\tau(x)+|F_0(x)|$, i.e.,
\begin{align}
    F(x) &\le s \cdot (s -1) \cdots (s - r +1) \cdot \frac{1}{1-x} + r! \cdot \frac{x^r}{(1-x)^{r+1}} \le 2 s\cdot (s - 1) \cdots (s - r + 1) \cdot \frac{1}{1-x}.
\end{align}
The proof is complete by setting $x = \epsilon$.
\end{proof}

\begin{lemma}[Gaussian-like tail bound for spherical coordinate]\label{lem:sphere_coordinate_tail}
Suppose $\xi\sim\unif(\SS^{d-1})$.
Then the first coordinate of $\xi$, denoted by $\xi_1$, satisfies that for any $t\geq 0$,
\begin{align}
    \Pr(\xi_1 \ge \sqrt{C d^{-1} \log d}) \le  \exp(-d/16) + d^{-C/4},
\end{align}
where $C>0$ is a constant.
\end{lemma}
\begin{proof}[Proof of \Cref{lem:sphere_coordinate_tail}]
Consider $z\sim\cN(0,I_d)$, and it holds that $\xi_1 \overset{d}{=} z_1/\|z\|_2$, where $z_1$ is the first coordinate of $z$.
Note that $\norm{z}_2^2 \sim \chi^2_d$, and by a standard tail bound for the $\chi^2_d$ distribution, we have
\begin{align}
    \Pr(\norm{z}_2^2 \le d - 2 \sqrt{d x}) \le \exp(-x), \quad\text{for any }x\geq 0.
\end{align}
By taking $x = d/16$, we get $\Pr(\norm{z}_2^2 \le d/2) \le \exp(-d/16)$.
Thus, applying a union bound,
\begin{align}
    \Pr(\xi_1 \ge t) &= \Pr\Big(\frac{z_1}{\norm{z}_2} \ge t\Big)
    \le \Pr\left(\norm{z}_2^2 \le d/2\right) + \Pr(z_1 \ge t \sqrt{d/2})\\
    &\le \exp(-d/16) + \exp(-t^2 d/4).
\end{align}
The proof is complete by taking $t = \sqrt{C d^{-1} \log d}$.
\end{proof}

\begin{lemma}[Hypergeometric moment and tail bound]\label{lem:hypergeometric-moment}
   Consider random variable $X\sim \mathrm{Hypergeometric}(d,k,k)$ with probability mass
   \begin{align}
        \Pr(X=x) = \frac{\binom{k}{x}\binom{d-k}{k-x}}{\binom{d}{k }}, \quad \text{for }x=1,2,\ldots k.
   \end{align}
   Suppose $k=o(\sqrt{d})$, then for any constant $s>0$, it holds that $\EE[X^s] \simeq k^2/d$.
   In addition, the following tail bound holds:
\begin{align}
    \Pr(X\ge \log k ) \lesssim  (k^2/d)^{\log k}.
\end{align}
\end{lemma}

\begin{proof}[Proof of \Cref{lem:hypergeometric-moment}]
We first notice that for $x\ge k^2/d$, since $k=o(\sqrt{d})$,
\begin{align}
    \frac{\Pr(X=x+1)}{\Pr(X=x)} &= \frac{(k-x)^2}{(x+1)(d-2k-x+1)}
    =\Big(1-\frac{k}{d}\Big)^2  \Big/ \Big(\frac{d}{k^2}+o\Big(\frac{d}{k^2}\Big)\Big)
    \lesssim \frac{k^2}{d}. \label{eq:hypergeo_ratio}
\end{align}
This immediately implies the tail bound:
\begin{align}
    \Pr(X\ge\log k) &\le  \sum_{j=\lceil \log k\rceil}^k  \Pr(X=j) \lesssim  \Big(\frac{k^2}{d}\Big)^{\log k}.
\end{align}
Next, for the moment $\EE[X^s]$, we first study the magnitude of $\PP(X=0)$, which, by Stirling's approximation, is given by
\begin{align}
    \Pr(X=0) &=  \frac{((d-k)!)^2 }{d!\cdot (d-2k)!}
    \simeq \frac {(d-k)^{2(d-k) + 1} \cdot e^{-2(d-k)}}{ d^{d+1/2} \cdot (d-2k)^{(d-2k) +1/2}\cdot  e^{-2(d-k)}} \\
    &=  \frac{(1-2k/d+k^2/d^2)^{d-k +1/2}}{ (1-2k/d)^{d-2k+1/2  }} \\
    &= \bigg(1+\frac{1}{(1-2k/d)\cdot d^2/k^2}\bigg)^{d-k+1/2} \bigg(1-\frac{2k}{d}\bigg)^k.
\end{align}
Since $k=o(\sqrt{d})$, we have $(1-2k/d)^k\simeq 1$.
Further applying the fact that $(1+1/m)^m = \Theta(1)$,
\begin{align}
    \Pr(X=0) &\simeq \bigg(1+ \frac{1}{(1-2k/d)\cdot d^2/k^2}\bigg)^{(1-2k/d)\cdot d^2/k^2 \cdot k^2/d}
    = \Theta(1).
\end{align}
As a consequence, using the first equality in \eqref{eq:hypergeo_ratio}, we can lower bound the expectation of $X^s$ by
\begin{align}
    \EE[X^s] &\ge  \Pr(X=1)
    = \Pr(X=0)\cdot \frac{k^2}{ d-2k+1}
    \gtrsim \frac{k^2}{d}.
\end{align}
For the upper bound, we again use the first equality in \eqref{eq:hypergeo_ratio} to get
\begin{align}
    \Pr(X=x+1) \leq \frac{k^2}{(x+1)(d-2k-x+1)}\cdot \Pr(X=x)
    \lesssim \frac{\Pr(X=x)}{x+1} \cdot \frac{k^2}{d}.
\end{align}
Recursive application of this inequality yields that $\PP(X=x) \lesssim \PP(X=0) \cdot (k^2/d)^x/x!$, and thus
\begin{align}
    \EE[X^s] &= \sum_{x=1}^k x^s\cdot \Pr(X=x)
    \le \Pr(X=0) \cdot \sum_{x=1}^k \frac{x^s}{x!} \Big(\frac{k^2}{d}\Big)^{x}
    \lesssim \frac{k^2}{d}.
\end{align}
Therefore, we conclude that $\EE[X^s] \simeq k^2/d$.
This completes the proof.
\end{proof}

\end{document}